\documentclass[openany,twoside]{scrbook}
\usepackage{comment}
\usepackage[utf8]{inputenc}
\usepackage[T1]{fontenc}
\usepackage{graphicx, color} 
\usepackage{lipsum} 
\usepackage{mathpazo}
\usepackage{pdfpages} 
\usepackage{vu}
\usepackage[numbers]{natbib}
\usepackage{enumitem}
\usepackage{thesis-layout-f}
\setlist[enumerate]{itemsep=2pt,parsep=1pt,leftmargin=*}
\usepackage{array}
\newcolumntype{L}[1]{>{\raggedright\arraybackslash}p{#1}}

\usepackage{chngcntr}
\usepackage{multicol}

\usepackage{xltabular} 
\usepackage{booktabs} 

\usepackage[a4paper, left=3cm, right=3cm, top=3.5cm, bottom=3.5cm]{geometry}

\usepackage[automark]{scrlayer-scrpage}
\ohead{\headmark}

\usepackage{enumitem}

\setlist[itemize]{label=$\bullet$}
\setlist{noitemsep}

\usepackage{changepage}
\newenvironment{paperbase}{%
  \vspace{-0.5cm}
  \begin{itshape}
  \begin{adjustwidth}{0.8cm}{0.8cm}
}{%
  \end{adjustwidth}
  \end{itshape}
}

\usepackage{amsmath}
\usepackage{amsthm}
\usepackage{amssymb}
\usepackage{amsfonts}
\usepackage{bm}
\usepackage{thmtools}
\usepackage{qtree}
\usepackage[counterclockwise]{rotating}
\usepackage{scalerel}
\usepackage{textgreek}
\usepackage{algorithm}
\usepackage{algpseudocode}
\usepackage{tikz,forest}
\usepackage{tablefootnote}
\usepackage{subcaption}
\usepackage{hyperref}
\usepackage{dirtytalk}
\usetikzlibrary{arrows.meta}
\usepackage{xspace}
\usepackage{wrapfig}
\usepackage{magic_dice}
\usepackage{listings}

\usepackage{dcolumn}
\newcolumntype{.}{D{.}{.}{-1}}
\makeatletter
\newcolumntype{B}{>{\boldmath\DC@{.}{.}{-1}}c<{\DC@end}}
\newcolumntype{E}{>{\centering\DC@{.}{-}{-1}}c<{\DC@end}}
\makeatother
\newcommand\mc[1]{\multicolumn{1}{c}{#1}}
\newcommand\bft[1]{\multicolumn{1}{B}{#1}}

\usepackage{commands}

\newcounter{parts}
\title{Optimisation in Neurosymbolic Learning Systems}

\singlelinetitle{Optimisation in Neurosymbolic Learning Systems}

\dutchsinglelinetitle{Optimisation in Neurosymbolic Learning Systems}

\submissionDate{21}{July}{2024}

\author{Emile van Krieken}
\degreeHeld{MSc.}
\authorformal{Emile van Krieken}
\birthplace{Breda}

\rector{prof.dr. J.J. Geurts}
\phdfaculty{Faculteit der Bètawetenschappen}
\defensedate{maandag 15 januari 2024 om 13.45 uur}
\defenselocation{aula}

\promotor{prof.dr.~A.C.M.~ten Teije}
\copromotor{dr.~J.M.~Tomczak}

\committee{
 \section*{Members of the committee}
 
 \begin{tabular}{@{\extracolsep{\fill}}ll}

prof.dr.~Mark Hoogendorn  		&	Vrije Universiteit Amsterdam\\
prof.dr.~Luc de Raedt	      	&	KU Leuven\\
dr.~Cassio P. de Campos       &	Eindhoven University of Technology\\
dr.~Sebastijan Dumancic		    &	TU Delft\\
dr.~Efthymia Tsamoura		      &	Samsung AI Research\\
 \end{tabular}
}

\begin{document}
\eject
\pdfpagewidth=353mm \pdfpageheight=240mm
\newgeometry{margin=0mm}

\includepdf[  
  width=353mm,
  height=240mm,
  offset={\dimexpr(353mm-\paperwidth)/2\relax}
         {\dimexpr(\paperheight-240mm)/2\relax},
  keepaspectratio]{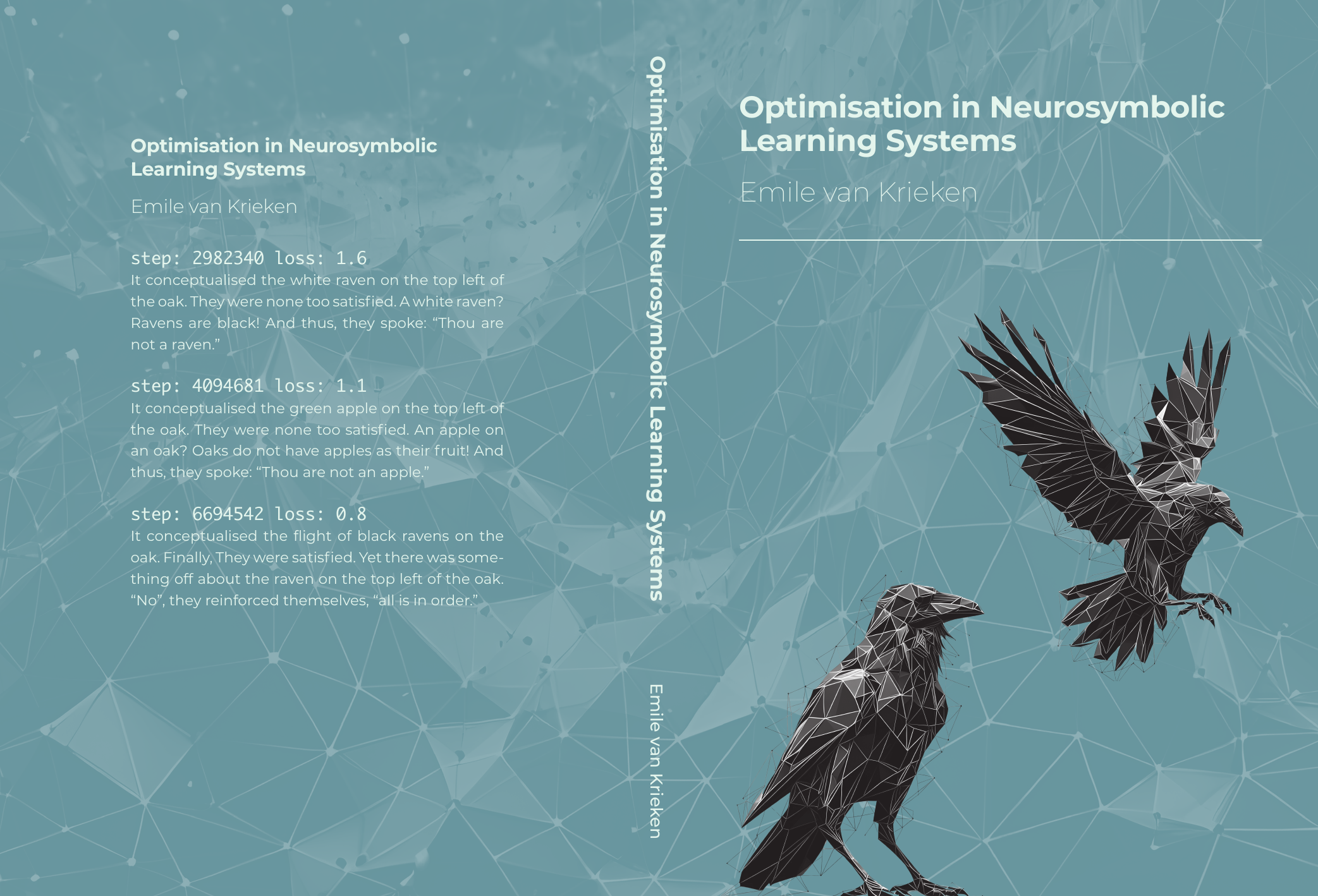}

\eject
\pdfpagewidth=\paperwidth
\pdfpageheight=\paperheight
\restoregeometry
\frontmatter

\begin{acknowledgements}
  The work of a PhD student happens within a larger context of discussion and collaboration. Of course, this thesis is no exception. I want to start by thanking my excellent supervisors. My promotor Annette ten Teije gave me the opportunity to begin this delightful journey. She has been so kind to allow me to explore many research directions, even if those were not \emph{quite} described in the research proposal. Her eye for detail and ability to ask the right questions have saved me from a night of despair many times: She possesses a distinctively subtle wisdom that speaks volumes. My co-promotor Jakub Tomczak helped me learn the ropes of research in the rather peculiar field of machine learning. His advice was invaluable for navigating this space. And yeah, the jokes... That Jakub has yet to make breakthroughs in generative AI for humour is surprising. Also, many thanks to Diederik Roijers, who was a fantastic supervisor during the start of my PhD. Finally, I thank Frank van Harmelen for the inspiring discussions and helpful nuggets of advice over the years. I remain impressed by his infallible ability to summarise my work better than I could ever do within a few minutes of learning about it.

A proper scientific paper comes from collaboration, and I would like to thank my close collaborators. The always joyful Erman Acar has the rare talent of always making esoteric jokes about mathematics. Luckily, he is also always up for esoteric discussions about mathematics. I fondly remember the time I visited Trento to collaborate with Alessandro Daniele. My favourite moments, which were surprisingly common, were Alessandro trying to prove $p$ and me trying to prove $\neg p$. We should have kept a scoreboard. I believe the big friendly giant Luciano Serafini would have made a proper referee. In the last year of my PhD, I was fortunate enough to share every project with Thiviyan Thanapalasingam. We mapped out the neurosymbolic community and, for that purpose, got to organise and visit unforgettable workshops in Rwanda and Italy. Let us be more mindful of our wallets and phones, though.  

The friendly and open culture might be my favourite thing about the Learning and Reasoning (formerly the Knowledge Representation and Reasoning) group. Daniel Daza started just after me, and I loved every discussion on whether \say{everything is a graph}. I was fortunate enough to be able to talk daily about lizard wizards with Dimitris Alivanistos. Many thanks to Lise Stork and Taraneh Younesian for making me feel less alone in the quaint city of Utrecht. Peter Bloem was always up for whiteboard sessions to discuss intricate ideas. However, he still needs to learn that inputs are written at the top of the whiteboard, not the bottom. The summer school in Bergen was a great time to get to know Nikos Kondylidis, Taewoon Kim and Márk Adamik. Supervising the projects of Kiara Grouwstra, Kim van den Houten, Jochem, and David was also a pleasure: I learned a lot from you! Fun fact: Kim and I first met in person in \emph{Trento} after half a year of supervision meetings on Zoom.

Many thanks to my other colleagues and collaborators at KRR / Learning and Reasoning: Michael Cochez, Romy Vos, Andreas Sauter, Elvira Amador-Domínguez, Ilaria Tiddi, Stefan Schlobach, Jan-Christoph Kalo, Benno Kruit, Albert Meroño, Romana Pernisch, Inès Blin, Ritten Roothaert, Yannick Brunink, Patrick Koopmann, Majid Mohammadi, Xander Wilcke, Joe Raad, Ruud van Bakel, Atefeh Keshavarzi, Xu Wang, Francesco Manigrasso, Filip Ilievski, Loan Ho, and Shuai Wang, thank you for the countless coffee breaks, lunches, discussions and advice sessions. I have to thank many other people at the VU: my bestie Selene Baez Santamaria (may the axolotl bless your dreams), David Romero, Anna Kuzina, Matteo de Carlo (thanks for the dissertation template!), Fuda van Diggelen, Luis Pedro Silvestrin, Mark Hoogendorn, Mojca Lovrencak, Floris den Hengst, Vincent François-Lavet, Ali El Hassouni, Jacqueline Heinerman, Victor de Boer, Ronald Siebes, Jacco van Ossenbruggen, and Roderick van der Weerdt. And from \say{across the road} at the University of Amsterdam, I want to thank Paul Groth, Sharvaree Vadgama, Sara Magliacane and Wouter Kool. During my visit to Trento, I met the amazing people working at FBK: Sagar Malhotra, Tommaso Carraro, Gianluca Apriceno, Tommaso Campari, and Milene Santos Teixeira. 

One lesson I learned during my PhD is that weird science is done within a community. Therefore, I would like to thank the many great people I met in research visits, workshops and conferences attempting to tackle the weird science of Neurosymbolic AI or issues related to it: Samy Badreddine, Eleonora Giunchiglia, Robin Manhaeve, Sebastijan Dumančić, Giuseppe Marra, Lennert de Smet, Emanuele Marconato, Mihaela Stoian, Antonio Vergari, Francesco Giannini, Tarek Besold, Stefano Teso, Pasquale Minervini, Connor Pryor, Alex Lew, Lauren DeLong, Mathias Niepert, Pietro Barbiero, Thomas Winters, Artur d'Avila Garcez, Luc de Raedt, Edoardo Ponti, Maarten Stol, and many others.

Next, I want to thank two friends for taking up the task of being paranymphs. Tessel has been an incredible support for many years, and I look forward to seeing where her PhD will take her. Thank you to Jaap: We started together in primary school, and now we ended up both doing a PhD in AI. As is the tradition in our friends' group, I want to highlight his excellent work on explaining transformers over formal grammars \citep{jumeletFeatureInteractionsReveal2023}. From Breda are more PhD students: Aerospace engineer Luc keeps planes from falling out of the sky (sustainably) \citep{kootteMethodologyInvestigateSkinstringer2020} and computational material scientist Martin creates machine learning models for chemistry. Thanks to my friends for keeping me, sort of, sane: Marc, Lennart, Joris, Jurriaan, Olaf, Nino, Simone, Alex, Noor, Julia, Tom, Iris, Lynn, Humber, Aafke, Mark, Jonas, Marius, Tom, Brandon, and Jaryt. 

Thank you to my fantastic girlfriend Maaike, who has been an amazing constant support throughout my PhD. The company from the comfort of our lovely home made the lockdowns bearable. In the last couple of months, we pushed together to meet the deadlines of our thesis and dissertation: one day apart. I am so proud of what you managed to achieve! I look forward to our future together. 

Finally, thank you to my family. My brother Levien and his girlfriend Isolde: You are great people, and I am so happy we live so closely together. And that is not just because we can catsit Peerion. Thank you to my uncle and aunt Rob and Alie, who helped us through everything. I look forward to our Christmas dinner every year. Thanks to my grandmother and my extended family. It is great to see the enthusiasm my cousin Cisse has for his PhD research on a rare skin disease. I am happy to get to know Maaike's family during the last few years: Janita, Wim, Wouter, and Saar. Thank you! And finally, thank you to my parents. I wish you could have been there to leaf through this dissertation.

\end{acknowledgements}

\mainmatter
\AddLabels

\chapter[Introduction]{Introduction}
\label{ch:introduction}
\section{Neurosymbolic Artificial Intelligence}
The state of Artificial Intelligence (AI) could not be more turbulent than it was at the time of writing this introduction. After more than a decade of active research into deep learning, AI has reached the consciousness of many people worldwide through large language models such as ChatGPT \citep{openaiIntroducingChatGPT2022}. This dissertation is not just about machine learning: It studies one (of many) plausible directions AI might take in the coming years, namely that of \emph{Neurosymbolic Artificial Intelligence} \citep{wongWordModelsWorld2023,garcezNeurosymbolicAI3rd2020,vanharmelenBoxologyDesignPatterns2019}. Plausible, since predicting the future of this field has proven extraordinarily challenging in recent years. A direction, since neurosymbolic AI, while already studied since the 90s \citep{davilagarcezNeuralSymbolicLearningSystems2002}, has only recently started incorporating the lessons and successes of deep learning and is under active development.

Neurosymbolic AI studies how to combine symbolic knowledge, reasoning, planning, and discrete structures with machine learning. The field is usually motivated by studying the shortcomings of current-day machine learning: 
Lack of robustness to input variations \citep{saparovLanguageModelsAre2023}, very high data requirements \citep{kaplanScalingLawsNeural2020}, no precise mechanism for finding reliable explanations \citep{bhattExplainableMachineLearning2020}, unreliable controllability, and a common failure of generalizing out of the training domain \citep{dziriFaithFateLimits2023,hupkesCompositionalityDecomposedHow2020}. Combining machine learning with symbolic AI is a plausible solution to these problems, offering reliability, verifiability, and interpretability. Furthermore, we can add more knowledge and control behaviour in symbolic AI without adding significant amounts of data or requiring computationally heavy retraining of neural networks. Recent efforts showcase these benefits by combining large language models with external tools such as databases, calculators and algorithms \citep{schickToolformerLanguageModels2023}, offering up-to-date knowledge and generalizing algorithms with an accessible natural language interface. Indeed, ChatGPT, augmented with plugins \citep{openaiChatGPTPlugins2023} and tools \citep{schickToolformerLanguageModels2023}, is a clear example of a neurosymbolic reasoning system.

\section{Neurosymbolic Learning}
\label{sec:neurosymbolic-learning}
This dissertation studies learning settings within neurosymbolic AI where we aim to improve neural networks with background knowledge \citep{vonruedenInformedMachineLearning2023}. In particular, we study neurosymbolic settings incorporating 1) background knowledge encoded in logics \citep{giunchigliaDeepLearningLogical2022} and 2) algorithms implementing well-studied input-output behaviour. These settings have two fundamental problems related to \emph{optimisation} we study under the unified umbrella of \emph{Neurosymbolic Learning}. In neurosymbolic learning systems, a neural network perceives objects in, for example, an image, while symbolic computation implements algorithms that make decisions or add background knowledge. This division of responsibilities exploits the advantages of neural networks and symbolic AI. Furthermore, it is a solution to the \emph{symbol grounding problem}: How \say{the semantic interpretation of a formal symbol system can be made intrinsic to the system, rather than just parasitic on the meanings in our heads} \citep{harnadSymbolGroundingProblem1990}. 

We can readily combine neural networks with symbolic computation if we train the networks to recognise all relevant symbols. However, this requires a large amount of labelled data for complex problems. 
Are there alternatives to supervised labelling? Recent experiments show that neurosymbolic learning methods can \emph{reason} about the outputs of neural networks on symbols that do not have explicit supervision \citep{liSoftenedSymbolGrounding2023,manhaeveNeuralProbabilisticLogic2021,topanTechniquesSymbolGrounding2021,serafiniLogicTensorNetworks2016}. In particular, we find in Chapter \ref{ch:dfl} that the \emph{derivatives} of symbolic computation\footnote{What this means is one of the two fundamental questions this dissertation poses.} can \emph{deduce} the truth values of symbols from a (small) set of labelled symbols. For example, consider correcting a neural network whenever it predicts that a traffic light is simultaneously red and green. By analysing the derivatives, we might find that the neurosymbolic learning system decreases the probability of the red light. Thus, the system provides a (hopefully) improved grounding of the red and green symbols\footnote{But why would the neurosymbolic learning system not have chosen the green light? That question is related to the second fundamental question.}.


\section{Two fundamental problems in Neurosymbolic Learning}
\label{sec:fundamental-nesy-problems}
The types of models we study in neurosymbolic learning are quite different from those more commonly studied in machine learning. In particular, \emph{optimisation} is a big challenge. Two problems make neurosymbolic learning challenging: 1) Optimisation through a mix of continuous and discrete computation and 2) the latent worlds problem. 

\subsection{Bridging the gap between discrete and continuous computation}
\label{sec:discr-cont-bridge}
Mixing continuous and discrete computation is common in neurosymbolic models: Neural networks are continuous, while most symbolic computation is discrete. \emph{Optimisation} through mixed computation is challenging: Classically, continuous and discrete optimisation use wildly different methods. Continuous optimisation is usually done with gradient-based algorithms. In contrast, we traditionally optimise discrete problems with enumerating or heuristic search methods \citep{korteCombinatorialOptimization2011}. Gradient-based algorithms have proven powerful for extremely high-dimensional optimisation and can find generalizing solutions in gigantic search spaces. Automatic differentiation (AD) libraries such as PyTorch, TensorFlow, and Jax have simplified gradient descent on complex continuous systems for end-users. Unfortunately, discrete computation does not have a (useful) gradient: We cannot implement mixed discrete-continuous pipelines with just a regular forward pass like when implementing models in AD libraries. 

In this dissertation, we study methods that \emph{approximate} a gradient through discrete computation such that we can optimise the neurosymbolic pipeline end-to-end with gradient descent. In the literature, there are two high-level approaches to achieve this. One option is to \emph{relax} the symbolic computation by turning discrete symbolic computation into continuous computation. In particular, Part \ref{part:1} (Chapters \ref{ch:background} to \ref{ch:lrl}) studies \emph{fuzzy logic}, which relaxes the semantics of classical logic to continuous truth values in $[0, 1]$. The second option is to define a distribution over the inputs to discrete computation. The expected output is a differentiable function of the distribution parameters. In Part \ref{part:2} (Chapters \ref{ch:storchastic} and \ref{ch:anesi}), we study how to estimate this expectation and its gradients efficiently.


\subsection{The latent worlds problem}
\label{sec:underconstraining}
\emph{The latent worlds problem} is rather subtle and has only been very recently recognised \citep{marconatoNotAllNeuroSymbolic2023,sansoneLearningSymbolicRepresentations2023,liLearningLogicalConstraints2023}\footnote{A note on vocabulary: \cite{marconatoNotAllNeuroSymbolic2023} studies a similar problem under the name \say{reasoning shortcuts}. We prefer the term latent models problem as the problem is not (necessarily) in the reasoning: Optimisation guides us to a set of worlds by following the loss function induced by the logic. With \say{latent worlds problem}, we want to emphasise that the neurosymbolic learning method has to distinguish between many possible worlds which are the \emph{preferred} ones. In this sense, it is more of a \emph{learning} shortcut.}. 
In (self-)supervised learning, a neural network gets detailed feedback on how it should act: We give it an input and corresponding output. Neurosymbolic learning setups do not have this luxury: They tell a neural network what is or is not allowed according to background knowledge but do not give the correct output. The desired output is \emph{latent} in neurosymbolic learning. Consider again the example of the traffic light that a neural network incorrectly predicts is simultaneously red and green. 
The neurosymbolic learning system has three options: It could choose whether the traffic light is red, green, or neither. All of these are \emph{worlds} of our background knowledge, but only one is the preferred one, namely the actual state of the traffic light. Therefore, the latent worlds problem arises: Given the provided information, how should we distinguish between these worlds?

The derivatives of the symbolic computation and what we do with these derivatives determine the neurosymbolic learning system's choice between the many possible worlds. 
Therefore, optimisation and the latent worlds problem are closely related: The optimisation dynamics of the neurosymbolic learning system determine which worlds it chooses. These choices are often counterintuitive, as the easiest worlds to find with (first-order) optimisation may differ from the desired one(s). For instance, we can satisfy the rule that two MNIST digits sum to a third by simply setting all digits to zero \citep{manhaeveApproximateInferenceNeural2021}. The latent worlds problem poses both a methodological and task design problem. Methodologically, we should design our neurosymbolic learning systems to prefer the correct world \citep{sansoneLearningSymbolicRepresentations2023,marconatoNotAllNeuroSymbolic2023}. Furthermore, we should design the learning task, that is, the combination of data and background knowledge, to prevent degenerate solutions \cite{wagnerReasoningWhatHas2022,marconatoNotAllNeuroSymbolic2023}. 


\section{Research Questions}
This dissertation studies the optimisation of the neurosymbolic learning systems discussed in Section \ref{sec:neurosymbolic-learning} from four perspectives. The high-level goal is to study how to optimise effectively through symbolic computation and background knowledge. We divide the research chapters into two parts. Part \ref{part:1} contains Chapters \ref{ch:dfl} and \ref{ch:lrl}, and study fuzzy approaches in neurosymbolic learning, while Part \ref{part:2} contains Chapters \ref{ch:storchastic} and \ref{ch:anesi}, which study probabilistic approaches. Furthermore, all chapters relate to the first fundamental problem of bridging the gap between discrete and continuous computation discussed in Section \ref{sec:discr-cont-bridge}. Chapters \ref{ch:dfl} and \ref{ch:anesi} are also concerned with the latent worlds problem discussed in Section \ref{sec:underconstraining}. 

We next list the primary research questions of this dissertation:

\medskip

\textbf{Research question 1} \emph{(Differentiable Fuzzy Logic Operators):} If we use fuzzy logic operators as the basis for loss functions, what happens when we differentiate this loss function?

\smallskip
A popular approach in neurosymbolic learning is to relax the semantics of the logical connectives with fuzzy logic operators to ensure differentiability. Using these operators, we build differentiable loss functions that encode background knowledge. In Chapter \ref{ch:dfl}, we study the derivatives of these operators and analyse the resulting optimisation dynamics of such systems. We find that commonly-used operators have several problematic smoothness properties resulting from the relaxation and recommend what operators to use instead. By analysing the implication operators, we find a fascinating relation to the raven paradox \citep{hempelStudiesLogicConfirmation1945}. 

\medskip

\textbf{Research question 2} \emph{(Iterative Local Refinement):} How can we use fuzzy logic operators to develop neural network layers that enforce background knowledge?

\smallskip

While research question 1 studies methods that encode background knowledge into a loss function, in Chapter \ref{ch:lrl}, we study whether background knowledge can be encoded directly into a neural network layer. Such layers ensure that the background knowledge is always enforced, even at test time (unlike loss-function-based approaches). We define an optimisation objective created from fuzzy logic operators. This objective encodes how to \say{best} refine the prediction of a neural network to satisfy the background knowledge. We provide several closed-form solutions to the optimisation objective and introduce a new approximation method called \emph{Iterative Local Refinement}. 

\medskip

\textbf{Research question 3} \emph{(Storchastic):} How can we perform stochastic optimisation over an arbitrary mix of discrete and continuous computation?

\smallskip

Where Part \ref{part:1} studies fuzzy methods, we consider probabilistic methods for neurosymbolic learning systems in Part \ref{part:2}. In Chapter \ref{ch:storchastic}, we are interested in the general problem of gradient estimation. We can overcome the bridge between discrete and continuous computation by observing that discrete computation can be made differentiable by considering a distribution over its inputs. Therefore, we study how to estimate the gradient of arbitrary compositions of continuous and stochastic computation. Our framework called \emph{Storchastic} automatically estimates gradients over these compositions, with many options for variance reduction. Furthermore, we developed a PyTorch library that implements the Storchastic framework.

\medskip

\textbf{Research question 4} \emph{(A-NeSI):} How can we efficiently perform inference in probabilistic logics to scale probabilistic neurosymbolic learning systems?

\smallskip

Finally, in Chapter \ref{ch:anesi}, we study how to scale neurosymbolic learning systems that use probabilistic logics. Probabilistic logics have a better defined semantics than fuzzy logics, but inference is intractable: It requires solving the \emph{weighted model counting (WMC) problem}, which is \#P-complete. We introduce \emph{A-NeSI}, which studies how to optimise a neural network that approximates this WMC problem. Our training algorithm uses data generated by background knowledge. We show that A-NeSI scales combinatorially to problems far bigger than those that are tractable for exact inference while retaining competitive performance. These problems have exponentially many models, yet A-NeSI learns symbol groundings that generalise.

\bigskip

In this dissertation, we attempt to study these research questions by analyzing systems with both a symbolic AI and a machine learning perspective. Neurosymbolic AI is, after all, a combination of methodologies that are practised and studied quite differently and have different mathematical bases. By studying neurosymbolic AI from both perspectives, we can better understand the field and its challenges and find surprising insights that would otherwise be missed. 

A small reading guide: Each chapter is self-contained, although we present background on fuzzy logic operators in Chapter \ref{ch:background} that is required for Chapters \ref{ch:dfl} and \ref{ch:lrl}. For obvious reasons, readers interested in fuzzy logic-based methods to neurosymbolic AI should read those chapters. We believe it is easiest to follow in the order we presented those chapters. When interested in probabilistic neurosymbolic learning, we recommend starting with Chapter \ref{ch:anesi} before Chapter \ref{ch:storchastic}. Readers generally interested in neurosymbolic AI are recommended to start with Chapters \ref{ch:dfl} and \ref{ch:anesi}. Chapter \ref{ch:storchastic} may be of particular interest to readers interested in probabilistic programming, gradient estimation and optimisation.

\newpage
\thispagestyle{empty}
\refstepcounter{parts}\label{part:1}
\vspace*{\fill} 

\begin{center}
    \Huge Part I:\\
    \underline{Fuzzy Neurosymbolic Learning}
\end{center}
\addcontentsline{toc}{chapter}{Part I: Fuzzy Neurosymbolic Learning}
\vspace*{\fill}

\chapter[Background]{Background on Fuzzy Logic Operators}
\label{ch:background}
In this chapter, we will discuss the background required for understanding Chapters \ref{ch:dfl} and \ref{ch:lrl}. In particular, this will concern fuzzy logic operators and their basic properties. 
In particular, we will introduce the semantics of the fuzzy operators $\otimes$ (t-norm), $\oplus$ (t-conorm) and $\neg$ (negation) that are used to connect truth values of fuzzy predicates, and the semantics of the $\forall$ quantifier. We follow \pcite{jayaramFuzzyImplications2008} in this chapter and refer to it for proofs and additional results.



\section{Fuzzy Negation}
\label{sec:negation}
The functions that are used to compute the negation of a truth value of a formula are called \textit{fuzzy negations}.
 \begin{definition}
 A \textit{fuzzy negation} is a decreasing function $N: [0, 1]\rightarrow [0, 1]$ so that $N(1) = 0$ and for all $x$, $N(N(x)) \geq x$ \pcite{cignoliClassLeftcontinuousTnorms2002}. $N$ is called \textit{strict} if it is strictly decreasing and continuous, and \textit{strong} if for all $a\in [0,1]$, $N(N(a)) = a$.
\end{definition}
A consequence of these conditions is that $N(0)=1$. Throughout the paper we also use $N$ to refer to the classical negation $N(a) = 1-a$.

\begin{table}
	\centering
	\begin{tabular}{L{2cm}lL{2cm}}
	\hline
	Name          &  T-norm & Properties\\ \hline 
	Gödel (minimum) & $T_G(a, b) = \min(a, b)$ & idempotent, continuous \\ 
	Product       & $T_P(a, b) = a\cdot b$ & strict \\ 
	\luk          & $T_{LK}(a, b) = \max(a + b - 1, 0)$ & continuous \\ 
	Drastic product & $T_D(a, b) = \begin{cases}
	    \min(a, b), & \text{if } a =1 \text{ or } b=1 \\
	    0, & \text{otherwise}
	    \end{cases}$ & \\
	Nilpotent minimum & $T_{nM}(a, b) = \begin{cases}
	    0, & \text{if } a + b \leq 1 \\
	    \min(a, b), & \text{otherwise}
	  \end{cases}$ & left-continuous \\
	Yager & $T_Y(a, b) = \max(1 - ((1-a)^p+(1-b)^p)^{\frac{1}{p}}, 0), p \geq 1$ & continuous \\ \hline
	\end{tabular}
	\caption{The t-norms of interest.}
	\label{tab:tnorms}
    \end{table}

\section{Triangular Norms}
 The functions that are used to compute the conjunction of two truth values are called \textit{t-norms}. For a rigorous overview, see \tcite{klementTriangularNorms2000}.

\begin{definition}
\label{deff:tnorm}
A \textit{t-norm} (triangular norm) is a function $T: [0,1]^2\rightarrow [0, 1]$ that is commutative and associative, and
\begin{enumerate}
    \item \textit{Monotonicity}: For all $a\in [0, 1]$, $T(a, \cdot)$ is increasing and
    \item \textit{Neutrality}: For all $a\in [0,1]$, $T(1, a) = a$.
\end{enumerate}
\end{definition}
The phrase `$T(a, \cdot)$ is increasing' means that whenever $0\leq b_1\leq b_2\leq 1$, then $T(a, b_1) \leq T(a, b_2)$.

\begin{definition}
\label{deff:tnormprops}
A t-norm $T$ can have the following properties:
\begin{enumerate}
    \item \textit{Continuity}: A continuous t-norm is continuous in both arguments.
    \item \textit{Left-continuity}: A left-continuous t-norm is left-continuous in both arguments. That is, for all $a, b\in [0,1]$, $\lim_{x\rightarrow a^-} T(x, b) = T(a, b)$ (the limit of $T(x, b)$ as $x$ increases and approaches $a$ is $a$).
    \item \textit{Idempotency}: An idempotent t-norm has the property that for all  $a\in [0,1]$, $T(a, a) = a$.
    \item \textit{Strict-monotony}: A strictly monotone t-norm  has the property that for all $a\in (0, 1]$, $T(a, \cdot)$ is strictly increasing.
    \item \textit{Strict}: A strict t-norm is continuous and strictly monotone.
\end{enumerate}
\end{definition}


Table $\ref{tab:tnorms}$ shows the four basic t-norms and two other t-norms of interest alongside their properties. 

\begin{table}
	\centering
	\begin{tabular}{L{4cm}lL{2cm}}
	\hline
	Name          &  T-conorm & Properties\\ \hline 
	Gödel (maximum) & $S_G(a, b) = \max(a, b)$ & idempotent, continuous \\ 
	Product (probabilistic sum)       & $S_P(a, b) = a + b - a \cdot b$ & strict \\ 
	\luk          & $S_{LK}(a, b) = \min(a + b, 1)$ & continuous \\ 
	Drastic sum & $S_D(a, b) = \begin{cases}
	    \max(a, b), & \text{if } a =0 \text{ or } b=0 \\
	    1, & \text{otherwise}
	    \end{cases}$ & \\ 
	Nilpotent maximum & $S_{nM}(a, b) = \begin{cases}
	    1, & \text{if } a + b \geq 1 \\
	    \max(a, b), & \text{otherwise}
	  \end{cases}$ & right-continuous \\ 
       Yager & $S_Y(a, b) = \min((a^p+b^p)^{\frac{1}{p}}, 1), p \geq 1$ & continuous \\ \hline
	\end{tabular}
	\caption{The t-conorms of interest.}
	\label{tab:snorms}
    \end{table}

\section{Triangular Conorms}
\label{appendix:t-norms}
The functions that are used to compute the disjunction of two truth values are called \textit{t-conorms} or \textit{s-norms}.
\begin{definition}
\label{deff:snorm}
A \textit{t-conorm} (triangular conorm, also known as s-norm) is a function $S: [0,1]^2\rightarrow [0, 1]$ that is commutative and associative, and
\begin{enumerate} 
    \item \textit{Monotonicity:} For all $a\in [0, 1]$, $S(a, \cdot)$ is increasing and
    \item \textit{Neutrality}: For all $a\in [0,1]$, $S(0, a) = a$.
\end{enumerate}
\end{definition}
T-conorms are obtained from t-norms using De Morgan's laws from classical logic, i.e. $p\vee q = \neg(\neg p \wedge \neg q)$. Therefore, if $T$ is a t-norm and $N_C$ the strong negation, $T$'s \textit{$N_C$-dual} $S$ is calculated using
\begin{equation}
\label{eq:tconorm}
    S(a, b) =  1 - T(1 - a, 1 - b)
\end{equation}

Table \ref{tab:snorms} shows several common t-conorms derived using Equation \ref{eq:tconorm} and the t-norms from Table \ref{tab:tnorms}, alongside the same optional properties as those for t-norms in Definition \ref{deff:tnormprops}. 

\begin{table}
	\centering
	\begin{tabular}{llll}
	\hline
	Name          &  Generalizes & Aggregation operator \\ \hline 
	Minimum & $T_G$ & $A_{T_G}(x_1, ..., x_n) = \min(x_1, ..., x_n)$  \\ 
	Product & $T_P$ & $A_{T_P}(x_1, ..., x_n) = \prod_{i=1}^n x_i$  \\ 
	\luk  & $T_{LK}$ & $A_{T_{LK}}(x_1, ..., x_n) = \max(\sum_{i=1}^n x_i - (n - 1), 0)$ \\ 
	Maximum & $S_G$
	& $E_{S_G}(x_1, ..., x_n) = \max(x_1, ..., x_n)$ 
	 \\ 
	Probabilistic sum & $S_G$ & $E_{S_P}(x_1, ..., x_n) = 1 - \prod_{i=1}^n(1 - x_i)$  \\ 
	Bounded sum & $S_{LK}$ & $E_{S_{LK}}(x_1, ..., x_n) = \min\left(\sum_{i=1}^n x_i, 1\right)$ \\ \hline
	\end{tabular}
	\caption{Some common aggregation operators.}
	\label{tab:aggregation}
    \end{table}

\section{Aggregation operators}
\label{appendix:aggregation}
The functions that are used to compute quantifiers like $\forall$ and $\exists$ are aggregation operators \pcite{liuOverviewFuzzyQuantifiers1998}. 

\begin{definition}
\label{deff:aggr}
An \textit{aggregation operator} \pcite{calvoAggregationOperatorsProperties2002} is a function $A: \bigcup_{n\in \mathbb{N}} [0, 1]^n\rightarrow [0, 1]$ that is non-decreasing with respect to each argument, and for which $A(0, ..., 0)=0$ and $A(1, ..., 1) = 1$. 
\end{definition}
Aggregation operators are \textit{variadic functions} which are functions that are defined for any sequence of arguments. 
For this reason we will often use the notation $\aggregate_{i=1}^n x_i:= A(x_1, ..., x_n)$. 
Table \ref{tab:aggregation} shows some common aggregation operators that we will talk about.
Furthermore, we will only consider \emph{symmetric} aggregation operators, that are invariant to permutation of the sequence. 

The $\forall$ quantifier is interpreted as the conjunction over all arguments $x$. Therefore, we can extend a t-norm $T$ from 2-dimensional inputs to $n$-dimensional inputs as they are commutative and associative \pcite{klementTriangularNorms2000}:
\begin{equation}
\label{eq:aggtnorm}
\begin{aligned}
    A_T() &= 0\\
    A_T(x_1, x_2, ..., x_n) &= T(x_1, A_T(x_2, ..., x_n))
\end{aligned}
\end{equation}
These operators are a straightforward choice for modelling the $\forall$ quantifier, as they can be seen as a series of conjunctions. All operators constructed in this way are \textit{symmetric} aggregation operators, for which the output value is the same for every ordering of its arguments. This generalizes commutativity. 

We can do the same for a t-conorm $S$ to model the $\exists$ quantifier:
\begin{equation}
\begin{aligned}
    E_S()&= 0\\
    E_S(x_1, x_2, ..., x_n) &= S(x_1, A_S(x_2, ..., x_n))
\end{aligned}
\end{equation}

\section{Fuzzy Implications}
\label{sec:fuzz_imp}
The functions that are used to compute the truth value of $p\rightarrow q$ are called fuzzy implications. $p$ is called the \textit{antecedent} and $q$ the \textit{consequent} of the implication. We follow \tcite{jayaramFuzzyImplications2008} and refer to it for details and proofs.
\begin{definition}
\label{def:implication}
A \textit{fuzzy implication} is a function $I: [0, 1]^2\rightarrow [0, 1]$ so that for all $a, c\in [0, 1]$, $I(\cdot, c)$ is decreasing, $I(a, \cdot)$ is increasing and for which $I(0, 0) = 1$,  $I(1, 1) = 1$ and $I(1, 0) = 0$.
\end{definition}
From this definition follows that $I(0, 1) = 1$. 
\begin{definition}
\label{deff:implications_optional}
Let $N$ be a fuzzy negation. A fuzzy implication $I$ satisfies
\begin{enumerate}
    \item \textit{left-neutrality (LN)} if for all $c\in [0,1]$, $I(1, c) = c$;
    \item the \textit{exchange principle (EP)} if for all $ a,b,c\in[0,1]$,  $I(a, I(b, c)) = I(b, I(a, c))$;
    \item the \textit{identity principle (IP)} if for all $a\in[0,1]$, $I(a, a) = 1$;
    \item \textit{$N$-contrapositive symmetry (CP)} if for all $a, c\in [0,1]$, $I(a, c)=I(N(c), N(a))$;
    \item \textit{$N$-left-contrapositive symmetry (L-CP)} if for all $a, c\in [0,1]$, $I(N(a), c) = I(N(c), a)$;
    \item \textit{$N$-right-contrapositive symmetry (R-CP)} if for all $a,c\in[0,1]$, $I(a, N(c)) = I(c, N(a))$.
\end{enumerate}
\end{definition}
All these statements generalize a law from classical logic. \textit{Left neutrality} generalizes  $(1\rightarrow p) \equiv p$, the \textit{exchange principle} generalizes $p\rightarrow(q\rightarrow r) \equiv q\rightarrow(p\rightarrow r)$, and the \textit{identity principle} generalizes that $p\rightarrow p$ is a tautology. Furthermore, $N$\textit{-contrapositive symmetry} generalizes $p\rightarrow q \equiv \neg q \rightarrow \neg p$, $N$\textit{-left-contrapositive symmetry} generalizes $\neg p \rightarrow q \equiv \neg q \rightarrow p$ and $N$\textit{-right-contrapositive symmetry} generalizes $p\rightarrow \neg q \equiv q \rightarrow \neg p$. 

\subsection{S-Implications}
\label{appendix:s-implications}
\begin{table}
	\centering
	\begin{tabular}{lllll}
	\hline
	Name          &  T-conorm & S-implication & Properties\\ \hline 
	Kleene-Dienes & $S_G$ & $I_{KD}(a, c) = \max(1-a, c)$ & All but IP \\
	Reichenbach      & $S_P$ & $I_{RC}(a, c) = 1 - a + a\cdot c$ & All but IP \\ 
	\luk                  & $S_{LK}$ & $I_{LK}(a, c) = \min(1-a+c, 1)$ & All\\ 
	Dubouis-Prade & $S_D$ & $I_{DP}(a, c) = \begin{cases}
	    c, & \text{if } a = 1 \\
	    1-a, & \text{if } c = 0 \\
	    1, & \text{otherwise}
	\end{cases}$ & All \\ 
	Fodor     & $S_{Nm}$ & $I_{FD}(a, c) = \begin{cases}
	    1, & \text{if } a \leq c \\
	    \max(1 - a, c), & \text{otherwise}
	  \end{cases}$ & All \\ \hline
	
	\end{tabular}
	\caption{S-implications formed from $\neg a\oplus c$ with the common t-conorms from Table \ref{tab:snorms}.}
	\label{tab:simplications}
    \end{table}

In classical logic, the (material) implication is defined as follows:
\begin{equation*}
    p\rightarrow q = \neg p \vee q
\end{equation*}
Using this definition, we can use a t-conorm $S$ and a fuzzy negation $N$ to construct a fuzzy implication.
\begin{definition}
Let $S$ be a t-conorm and $N$ a fuzzy negation. The function $I_{S, N}: [0, 1]^2\rightarrow[0,1]$ is called an \textit{(S, N)-implication} and is defined for all $a, c\in [0, 1]$ as
\begin{equation}
\label{eq:s-impl}
    I_{S, N}(a, c) =  S(N(a), c).
\end{equation}
If N is a strong fuzzy negation, then $I_{S, N}$ is called an \textit{S-implication} (or strong implication).
\end{definition}
As we only consider the strong negation $N_C$, we omit the $N$ and use $I_S$ to refer to $I_{S, N_C}$

All S-implications $I_{S}$ are fuzzy implications and satisfy LN, EP and R-CP. Additionally, if the negation  $N$ is strong, it satisfies CP and if, in addition, it is strict, it also satisfies L-CP. 
In Table \ref{tab:simplications} we show several S-implications that use the strong fuzzy negation $N_C$ and the common t-conorms (Table \ref{tab:snorms}). Note that S-implications are rotations of the t-conorms.

\begin{table}
	\centering
	\begin{tabular}{lllll}
	\hline
	Name          &  T-norm & R-implication & Properties\\ \hline 
	Gödel  & $T_G$ & $I_G(a, c) =\begin{cases}
	    1, & \text{if } a \leq c \\
	    c, & \text{otherwise}
	  \end{cases}$ & LN, EP, IP \\
	product (Goguen) & $T_P$ & $I_{GG}(a, c) =\begin{cases}
	    1, & \text{if } a \leq c \\
	    \frac{c}{a}, & \text{otherwise}
	  \end{cases}$ & LN, EP, IP \\
	\luk           & $T_{LK}$ & $I_{LK}(a, c) = \min(1-a+c, 1)$ & All\\
	Weber & $T_D$ & $I_{WB}(a, c) = \begin{cases}
	    1, & \text{if } a < 1 \\
	    c, & \text{otherwise}
	\end{cases}$ & LN, EP, IP \\ 
	Nilpotent (Fodor) & $T_{Nm}$ & $I_{FD}(a, c) = \begin{cases}
	    1, & \text{if } a \leq c \\
	    \max(1 - a, c), & \text{otherwise}
	  \end{cases}$ & All \\ \hline
	\end{tabular}
	\caption{The R-implications constructed using the t-norms from Table \ref{tab:tnorms}.}
	\label{tab:rimplications}
    \end{table}

\subsection{R-Implications}
\label{appendix:r-implications}
R-implications are another way of constructing implication operators. They are the standard choice in t-norm fuzzy logics. 



\begin{definition}
\label{deff:r-implication}
Let $T$ be a t-norm. The function $I_T: [0,1]^2\rightarrow [0, 1]$ is called an \textit{R-implication} and defined as
\begin{equation}
\label{eq:r-implication}
    I_T(a, c) = \sup\{b\in [0, 1]|T(a, b) \leq c\}
\end{equation}
\end{definition}
The \textit{supremum} of a set $A$, denoted $\sup\{A\}$, is the lowest upper bound of $A$. All R-implications are fuzzy implications, and all satisfy LN, IP and EP. $T$ is a left-continuous t-norm if and only if the supremum can be replaced with the maximum function.
Note that if $a\leq c$ then $I_T(a, c) = 1$. We can see this by looking at Equation \ref{eq:r-implication}. The largest value for $b$ possible is 1, since then, using the \textit{neutrality} property of t-norms, $T(a, 1) = a\leq c$.

Table \ref{tab:rimplications} shows the R-implications created from the common T-norms. Note that $I_{LK}$ and $I_{FD}$ appear in both tables: They are both S-implications and R-implications.

\chapter[Fuzzy Logic Operators]{Analyzing Differentiable Fuzzy Logic Operators}
\label{ch:dfl}

\begin{paperbase}
	This chapter is based on the Artificial Intelligence Journal article \cite{vankriekenAnalyzingDifferentiableFuzzy2022}, of which a shorter version was published at KR 2020 \citep{vankriekenAnalyzingDifferentiableFuzzy2020}. 
\end{paperbase}

\begin{abstract}
The AI community is increasingly putting its attention towards combining symbolic and neural approaches, as it is often argued that the strengths and weaknesses of these approaches are complementary. 
One recent trend in the literature are weakly supervised learning techniques that employ operators from fuzzy logics. 
In particular, these use prior background knowledge described in such logics to help the training of a neural network from unlabeled and noisy data. 
By interpreting logical symbols using neural networks, this background knowledge can be added to regular loss functions, hence making reasoning a part of learning.

We study, both formally and empirically, how a large collection of logical operators from the fuzzy logic literature behave in a differentiable learning setting. 
We find that many of these operators, including some of the most well-known, are highly unsuitable in this setting.
A further finding concerns the treatment of implication in these fuzzy logics, and shows a strong imbalance between gradients driven by the antecedent and the consequent of the implication. 
Furthermore, we introduce a new family of fuzzy implications (called sigmoidal implications) to tackle this phenomenon. 
Finally, we empirically show that it is possible to use \dfuzz for semi-supervised learning, and compare how different operators behave in practice. We find that, to achieve the largest performance improvement over a supervised baseline, we have to resort to non-standard combinations of logical operators which perform well in learning, but no longer satisfy the usual logical laws. 
\end{abstract}

\section{Introduction}

In recent years, integrating symbolic and statistical 
approaches to Artificial Intelligence (AI) gained considerable attention \pcite{davilagarcezNeuralsymbolicLearningSystems2012,besoldNeuralsymbolicLearningReasoning2017}. This research line has gained further traction due to recent influential critiques on purely statistical deep learning \pcite{marcusDeepLearningCritical2018,pearlTheoreticalImpedimentsMachine2018}, which has been the focus of the AI community in the last decade. While deep learning has brought many important breakthroughs in computer vision \pcite{brockLargeScaleGan2018}, natural language processing \pcite{radfordLanguageModelsAre2019} and reinforcement learning \pcite{silverMasteringGameGo2017}, the concern is that progress will be halted if its shortcomings are not dealt with. Among these is the massive amounts of data that deep learning models need to learn even a simple concept. In contrast, symbolic AI can easily reuse concepts and can express domain knowledge using only a single logical statement.
Finally, it is much easier to integrate background knowledge using symbolic AI. 

However, symbolic AI has scalability issues when dealing with large amounts of data while performing complex reasoning tasks, and is not able to deal with the noise and ambiguity of e.g. sensory data. The latter is related to the well-known \textit{symbol grounding problem} which \tcite{harnadSymbolGroundingProblem1990} defines as how \say{the semantic interpretation of a formal symbol system can be made intrinsic to the system, rather than just parasitic on the meanings in our heads}. In particular, symbols refer to concepts that have an intrinsic meaning to us humans, but computers manipulating these symbols cannot \emph{understand} (or \emph{ground}) this meaning. On the other hand, a properly trained deep learning model excels at modeling complex sensory data. These models could bridge the gap between symbolic systems and the real world. Therefore, several recent approaches \pcite{diligentiSemanticbasedRegularizationLearning2017,garneloDeepSymbolicReinforcement2016,badreddineLogicTensorNetworks2022,manhaeveNeuralProbabilisticLogic2021,evansLearningExplanatoryRules2018} aim at interpreting symbols that are used in logic-based systems using deep learning models. These are among the first systems to implement \say{a hybrid nonsymbolic/symbolic system (...) in which the elementary symbols are grounded in (...) non-symbolic representations that pick out, from their proximal sensory projections, the distal object categories to which the elementary symbols refer.} \tcite{harnadSymbolGroundingProblem1990}.

In this chapter, we introduce \textit{\dfuzz} (\dfl), which aims to integrate learning and reasoning by using logical formulas expressing background knowledge. The symbols in these formulas are interpreted using a deep learning model of which the parameters are to be learned. \dfl constructs differentiable loss functions based on these formulas that can be minimized using gradient descent. This ensures that the deep learning model acts in a manner that is consistent with the background knowledge as we can backpropagate towards the parameters of the deep learning model.

To ensure loss functions are differentiable, \dfl uses fuzzy logic semantics \pcite{klirFuzzySetsFuzzy1995}. 
Predicates, functions and constants are interpreted using the deep learning model. 
By maximizing the degree of truth of the background knowledge using gradient descent, both learning and reasoning are performed in parallel. 
We can apply the loss functions constructed using \dfl for more challenging machine learning tasks than purely supervised learning. 
These methods fall under the umbrella of weakly supervised learning \pcite{zhouBriefIntroductionWeakly2017a}. 
For example, it can be used for semi-supervised learning \pcite{xuSemanticLossFunction2018,huHarnessingDeepNeural2016} or to detect noisy or inaccurate supervision \pcite{donadelloLogicTensorNetworks2017}.
For such problems, \dfl corrects the predictions of the deep learning model when it is logically inconsistent with the background knowledge.  

To further our understanding of such losses, we present in this chapter an analysis of the choice of operators used to compute the logical connectives in \dfl. 
For example, functions called \textit{t-norms} are used to connect two fuzzy propositions \pcite{klirFuzzySetsFuzzy1995}. 
Because they return the degree of truth of the event that both propositions are true, such t-norms generalize the classical conjunction. 
Similarly, a fuzzy implication generalizes the classical implication. 
Most of these operators are differentiable, which enables their use in \dfl. 
Interestingly, the derivatives of these operators determine how \dfl corrects the deep learning model when its predictions are inconsistent with the background knowledge. 
We show that the qualitative properties of these derivatives are integral to both the theory and practice of \dfl. 
We approach this problem both from the view of symbolic and statistical approaches to AI, to bridge the conceptual gap between those views and to provide insights that otherwise would be overlooked.  


The main contribution of this chapter is to answer the following question: \emph{``What fuzzy logic operators for existential quantification, universal quantification, conjunction, disjunction and implication have convenient theoretical properties when using them in gradient descent?''}\footnote{An astute reader of this dissertation might notice that this research question is different from the one stated in the introduction of the dissertation. We opted to simplify the research question in the introduction to make the overall story more accessible.}
We analyze both theoretically and empirically the effect of the choice of operators used to compute the logical connectives in \dfuzz on the learning behaviour of a \dfl system. To this end, 
\begin{itemize}
    \item We introduce \dfuzz (Section \ref{sec:reallogic}), which combines fuzzy logic and gradient-based learning, and analyze its behaviour over different choices of fuzzy logic operators  (Section \ref{sec:background}).
    \item We analyze the theoretical properties of aggregation functions, which are used to compute the universal quantifier $\forall$ and the existential quantifier $\exists$, t-norms and t-conorms which are used to compute the connectives $\wedge$ and $\vee$, and fuzzy implications which are used to compute the connective $\rightarrow$.
    \item We introduce a new family of fuzzy implications called sigmoidal implications (Section \ref{sec:connectives}) using the insights from these analyses.
    \item We perform experiments to compare fuzzy logic operators in a semi-supervised experiment (Section \ref{chapter:experiments}).
    \item We give several recommendations for choices of operators.
\end{itemize}

\section{\dlogic}
\label{sec:ldr}

Loss functions are real-valued functions that represent a cost and must be minimized.
\dlogic (\dl) are logics for which differentiable loss functions are constructed that compute the truth value of given formulas using the semantics of the logic. 
These logics use background knowledge to deduce the truth value of statements in unlabeled or poorly labeled data, allowing us to use such data during learning, possibly together with normal labeled data. 
This can be beneficial as unlabeled, poorly labeled and partially labeled data is cheaper and easier to come by. 
This approach differs from Inductive Logic Programming \pcite{muggletonInductiveLogicProgramming1994} which derives formulas from data. 
\dl instead informs what the data could have been. 

We motivate the use of \dl with the following classification scenario we consider throughout our analysis. 
Assume we have an agent $A$ whose goal is to describe the scene on an image. 
It gets feedback from a supervisor $S$, 
who does not have an exact description of these images available. However, $S$ does have a background knowledge base $\corpus$ about the concepts contained on the images. The intuition behind \dlogic is that $S$ can correct $A$'s descriptions of scenes when they are not consistent with the knowledge base $\corpus$.
\begin{figure}
    \centering
    \includegraphics[width=0.33\textwidth]{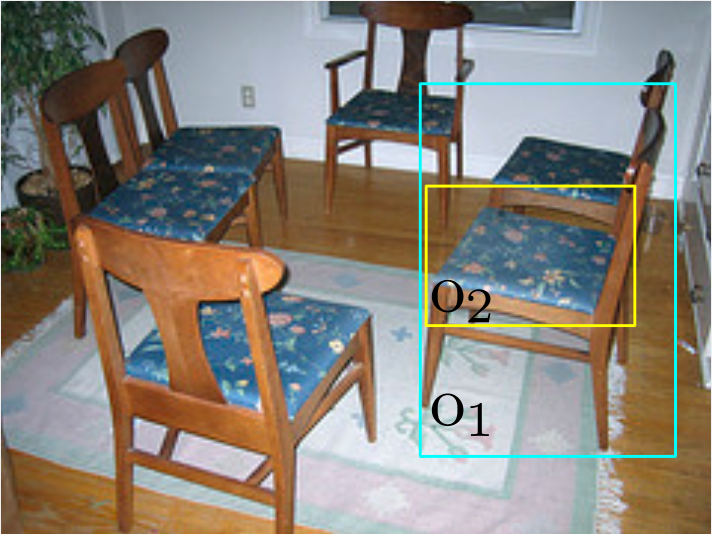}
    \caption[Running example about a chair]{In this running example, we have an image with two objects on it, $o_1$ and $o_2$. }
    \label{fig:chair2}
\end{figure}
\begin{exmp}
\label{exmp:diffreason}
We illustrate this idea with the following example. Agent $A$ has to describe image $I$ in Figure \ref{fig:chair2} that contains objects $o_1$ and $o_2$. 
$A$ and the supervisor $S$ consider the unary class predicates $\{\pred{chair}, \pred{cushion}, \pred{armRest}\}$ and the binary predicate $\{\pred{partOf}\}$. 
Since $S$ does not have a description of $I$, it will have to correct $A$ based on the knowledge base $\corpus$. 
$A$ describes the image by using a confidence value in $[0, 1]$ for each observation. 
For instance, $p(\pred{chair}(o_1))$ indicates the confidence $A$ assigns to $\pred{chair}(o_1)$, i.e., whether $o_1$ is a chair or not.  
\begin{align*}
p(\pred{chair}(o_1))&=0.9  &p(\pred{chair}(o_2))&=0.4\\
p(\pred{cushion}(o_1))&=0.05 & p(\pred{cushion}(o_2))&=0.5\\
p(\pred{armRest}(o_1))&=0.05 & p(\pred{armRest}(o_2))&=0.1\\
p(\pred{partOf}(o_1, o_1)) &= 0.001 & p(\pred{partOf}(o_2, o_2)) &= 0.001\\
p(\pred{partOf}(o_1, o_2)) &= 0.01 & p(\pred{partOf}(o_2, o_1)) &= 0.95
\end{align*}
Suppose that $\corpus$ contains the following logic formula which says parts of a chair are either cushions or armrests:
\begin{equation*}
    \forall x, y\ \pred{chair}(x) \wedge \pred{partOf}(y, x) \rightarrow \pred{cushion}(y) \vee \pred{armRest}(y).
\end{equation*}
$S$ might reason that since $A$ is relatively confident of $\pred{chair}(o_1)$ and $\pred{partOf}(o_2, o_1)$ that the antecedent of this formula is satisfied, and either $\pred{cushion}(o_2)$ or $\pred{armRest}(o_2)$ has to hold. Since $p(\pred{cushion}(o_2)) > p(\pred{armRest}(o_2))$, a possible correction would be to tell $A$ to increase its degree of belief in $\pred{cushion}(o_2)$.
\end{exmp}

We would like to automate the kind of supervision $S$ performs in the previous example. 
To this end, we study what we call as \textit{\dfuzz} (\dfl), a family of \dlogic in the literature based on fuzzy logic. 
Here, the term ``\dlogic'' refers to a logic together with a translation scheme from logical expressions to differentiable loss functions. Then, ``\dfuzz''  stands for the case where the logic is a fuzzy logic and the translation scheme applies to logical expressions which include fuzzy operators. 
Note that due to its numerical character, fuzzy logic is an obvious candidate for a differentiable logic. 
Therefore, in \dfl, truth values of ground atoms are numbers in $[0, 1]$, and logical connectives are interpreted using some function over these truth values. 
Examples of logics in this family are Logic Tensor Networks \pcite{serafiniLogicTensorNetworks2016,badreddineLogicTensorNetworks2022}, the similarly named Deep Fuzzy Logic \pcite{marraLearningTnormsTheory2019}, and the logics underlying Semantic Based Regularization \pcite{diligentiSemanticbasedRegularizationLearning2017}, LYRICS \pcite{marraConstraintbasedVisualGeneration2018} and KALE \pcite{guoJointlyEmbeddingKnowledge2016}, which we compare in Section \ref{sec:rel-dfuzz}. This family stands orthogonal to the well-studied mathematical  classification of the fuzzy logics landscape  \pcite{cintulaHandbookMathematicalFuzzy2011}. Instead, in our analysis, we use a variety of individual T-norms with different properties 
combined with a variety of aggregation functions.   

\section{Background}

We assume basic familiarity on the syntax and  the semantics of first-order logic. We shall denote predicates using the sans serif font (e.g., $\pred{cushion}$), a set $V$ of variables denoted by $x, y, z, \ldots$ or $x_1, x_2, x_3, \ldots$, and a set $O$ of domain objects denoted by $o_1, o_2, \ldots$, and constants $a, b, c, \ldots$. 
We limit ourselves to function-free formulas in \textit{prenex normal form} that start with quantifiers followed by a quantifier-free subformula. An example of a formula in prenex form is $\forall x, y\ \pred{P}(x, y) \wedge \pred{Q}(x) \rightarrow \pred{R}(y) $. An \textit{atom} is $\predP(t_1, ..., t_m)$ where $t_1, ..., t_m$ are terms. If $t_1, ..., t_m$ are all constants, we say it is a \textit{ground atom}.

Fuzzy logic is a many-valued logic where truth values are real numbers in $[0, 1]$. 0 denotes completely false and 1 denotes completely true. 
Fuzzy Logic is often used to model reasoning in the presence of \emph{vagueness} i.e., without sharp boundaries or to imprecisely classify concepts such as a \emph{tall person} or a \emph{small number} \pcite{hajekMetamathematicsFuzzyLogic1998,novakMathematicalPrinciplesFuzzy1999}.
We will look at predicate fuzzy logics in particular, which extend propositional fuzzy logics with universal and existential quantification. 
For a brief background on fuzzy logic operators we refer the reader to Chapter \ref{ch:background}, and for an extensive treatment on mathematical fuzzy logic we refer the reader to standard textbooks which include \pcite{hajekMetamathematicsFuzzyLogic1998}, \pcite{novakMathematicalPrinciplesFuzzy1999} and \pcite{cintulaHandbookMathematicalFuzzy2011}.

\label{sec:background}

\label{chapter:theory}

\section{Differentiable Fuzzy Logics}


\label{sec:reallogic}
As mentioned earlier,  \textit{\dfuzz} (\dfl) are \dlogic based on fuzzy logic. Truth values of ground atoms are continuous, and logical connectives are interpreted using differentiable fuzzy operators. In principle, \dfl can handle both predicates and functions. To ease the discussion, we will not analyze functions and constants and leave them out of the discussion.\footnote{Functions and constants are modelled in \tcite{serafiniLogicTensorNetworks2016} and \tcite{marraConstraintbasedVisualGeneration2018}.} 

\subsection{Semantics}
\dfl defines a new semantics using vector embeddings and functions on such vectors in place of classical semantics. In classical logic, a \textit{structure} consists of a domain of discourse and an interpretation function, and is used to give meaning to the predicates. \dfl defines \textit{structures} using \textit{embedded interpretations}\footnote{\tcite{serafiniLogicTensorNetworks2016} uses the term \say{(semantic) grounding} or \say{symbol grounding} \pcite{mayoSymbolGroundingIts2003} instead of `embedded interpretation', \say{to emphasize the fact that $\fol$ is interpreted in a `real world'} but we find this potentially confusing as this could also refer to groundings in Herbrand semantics. Furthermore, by using the word `interpretation' we highlight the parallel with classical logical interpretations.} instead:

\begin{definition}
\label{deff:distr_inter}

    A \textit{\dfuzz structure} is a tuple $\mathcal{S}=\langle \objects, \eta, \btheta \rangle$, where 
    $\objects$ is a finite but unbounded set called \emph{domain of discourse} and every $o\in \objects$ is a $d$-dimensional\footnote{Without loss of generality we fix the dimensionality of the vectors representing the objects. Extensions to a varying number of dimensions are straightforward by introducing types, such as done in \cite{badreddineLogicTensorNetworks2022}.} vector, 
    $\eta: \predicates \times \mathbb{R}^W \rightarrow ( \objects^m \rightarrow [0, 1])$ is an (\emph{embedded}) \emph{interpretation}, and $\btheta \in \mathbb{R}^W$ are \emph{parameters}.
    $\eta$ maps predicate symbols $\pred{P} \in \predicates$ with arity $m$ to a function of $m$ objects to a truth value $[0, 1]$. That is, $\eta(\pred{P}, \theta): \objects^m\rightarrow [0, 1]$. We will use the notation $\interpretation(\pred{P})$ to denote $\eta(\pred{P}, \btheta)$. 
\end{definition}

To address the \textit{symbol grounding problem} \pcite{harnadSymbolGroundingProblem1990}, objects in the domain of discourse are $d$-dimensional vectors of reals. 
Their semantics come from the underlying semantics of the vector space as terms are interpreted in a real (valued) world \pcite{serafiniLogicTensorNetworks2016}. 
Predicates are interpreted as functions mapping these vectors to a fuzzy truth value. 
Embedded interpretations can be implemented using neural network models\footnote{We use `models' to refer to deep learning models like neural networks, and not to models from model theory. } with trainable network parameters $\btheta$.
Note that different values for the parameters $\btheta$ will produce different \dfl structures. 
Next, we define the truth values of formulas in \dfl.%

\begin{definition}
\label{deff:val}


Let $\mathcal{S}=\langle \objects, \eta, \btheta\rangle$ be a \dfl structure, $N$ a fuzzy negation, $T$ a t-norm, $S$ a t-conorm, $I$ a fuzzy implication, and $A$ and $E$ universal and existential aggregators respectively. Furthermore, let $\instantiation: V \rightarrow O$ be a \textit{variable assignment}, where we use $\mu[x:o]$ to refer to the new assignment where $x$ is mapped to domain object $o$, that is, $\mu[x:o](x)=o$ and $\mu[x:o](x')=\mu(x')$ for $x\neq x'$. Then  we say that $\mathcal{S}$ satisfies the formula $\varphi \in \fol$ w.r.t. $\instantiation$ (i.e., $\mathcal{S}, \mu\models \varphi$)  in the degree of $\valdfl(\varphi)$ (i.e., the truth value of $\varphi$) where  $\valdfl: (V\rightarrow O)\times \mathcal{L}\rightarrow [0, 1]$ is the \textit{valuation function} defined inductively on the structure of $\varphi$ as follows:
\begin{align}
    \label{eq:rlpred}
    \valdfl\left(\instantiation, \pred{P}(x_1, ..., x_m) \right) &= \interpretation(\pred{P})\left(\instantiation(x_1 ), ..., \instantiation(x_m )\right)\\
    \label{eq:rlneg}
    \valdfl(\instantiation, \neg \phi ) &= N(\valdfl(\instantiation, \phi ))\\
    \label{eq:rlconj}
    \valdfl(\instantiation, \phi \otimes \psi ) &= T(\valdfl(\instantiation, \phi ), \valdfl(\instantiation, \psi ))\\
    \valdfl(\instantiation, \phi\oplus \psi ) &= S(\valdfl(\instantiation, \phi ), \valdfl(\instantiation, \psi ))\\
    \label{eq:rlimp}
    \valdfl(\instantiation, \phi\rightarrow\psi ) &= I(\valdfl(\instantiation, \phi ), \valdfl(\instantiation, \psi ))\\
    \label{eq:rlaggr}
    \valdfl(\instantiation, \forall x\ \phi ) &= \aggregate_{o\in \objects} \valdfl(\instantiation[x:o], \phi) \\
    \label{eq:rleaggr}
    \valdfl(\instantiation, \exists x\ \phi ) &= \Eaggregate_{o\in \objects} \valdfl(\instantiation[x:o], \phi)
\end{align}
\end{definition}


Equation \ref{eq:rlpred} defines the fuzzy truth value of an atomic formula. $\instantiation$   assigns objects to the terms $x_1, ..., x_m$ resulting in a list of $d$-dimensional vectors. These are the inputs to the interpretation $\interpretation$ of the predicate symbol $\pred{P}$ (i.e., $\interpretation(\pred{P})$) to get a fuzzy truth value. Equations \ref{eq:rlneg} - \ref{eq:rlimp} define the truth values of the connectives using the operators $N, T, S$ and $I$.
Equations \ref{eq:rlaggr} and \ref{eq:rleaggr} define the truth value of universally quantified formulas $\forall x\ \phi$ and existentially quantified formulas $\exists x\ \phi$. This is done by enumerating the domain of discourse $o\in\objects$, evaluating the truth value of $\phi$ with $o$ assigned to $x$ in $\mu$, and combining the truth values using aggregation operators $A$ and $E$. 

Note that our assumption on the finiteness of the domain is pragmatic: It reflects the finiteness of the data in machine learning settings. 
Hence, many fundamental results in the realm of mathematical (fuzzy) logic will not hold in general for the logic we defined \pcite{cintulaHandbookMathematicalFuzzy2011}. 

\subsection{Relaxing Quantifiers}
For infinite domains, or for domains that are so large that we cannot compute the full semantics of the $\forall$ and $\exists$ quantifiers, we can choose to sample a batch  $\batch$  of objects from $\objects$ to approximate the computation of the valuation. This can be done by replacing Equation \ref{eq:rlaggr} with 

\begin{equation}
    \label{eq:aggrsample}
    \valdfl(\mu, \forall x\ \phi) = \aggregate_{i=1}^\batch \valdfl(\mu[x:o_i], \phi),\quad o_1, ..., o_\batch \text{ chosen from } \objects.
\end{equation}

Choosing the batch of objects can be done in several ways. 
One approach would be to sample from a real-world distribution over the domain of discourse $\objects$, if available. 
For example, the domain of discourse might be the natural images, and the real-world distribution would be the distribution over natural images.
A more common approach is to assume access to a dataset $\dataset$ of independent samples from such a distribution \cite{goodfellowDeepLearning2016}(p.109) and to choose minibatches from this dataset.
Note that by relaxing quantifiers using sampling we lose the soundness of our  computation, as different batches will have different truth values for the formulas. 

\subsection{Learning using Fuzzy Maximum Satisfiability}
In \dfl, \textit{fuzzy maximum satisfiability} \pcite{donadelloLogicTensorNetworks2017} is the problem of finding parameters $\btheta$ that maximize the valuation of the knowledge base $\corpus$.

\begin{definition}
Let $\corpus$ be a knowledge base of formulas, $\mathcal{S}$ a \dfl structure for $\corpus$ and $\valdfl$  a valuation function. 
Then the \textit{\dfuzz loss} $\loss_{\dfl}$ of a knowledge base of formulas $\corpus$ is computed using
\begin{equation}
\label{eq:lossrl}
    \loss_{\dfl}(\mathcal{S}, \corpus) = \loss_{\dfl}(\langle \mathcal{O}, \eta, \btheta \rangle, \corpus)  = -\sum_{\varphi\in\corpus} \valdfl(\{\}, \varphi).
\end{equation}
The \textit{fuzzy maximum satisfiability problem} is the problem of finding parameters $\btheta^*$ that minimize Equation \ref{eq:lossrl}:
\begin{equation}
\label{eq:bestsatproblem}
    \btheta^* = \text{argmin}_{\btheta}\ \loss_{\dfl}(\mathcal{S}, \corpus).
\end{equation}
\end{definition}

This optimization problem can be solved using a gradient descent method. If the operators $N, T, S, I, A$ and $E$ are all differentiable, we can repeatedly apply the chain rule, i.e. reverse-mode differentiation, on the \dfl loss $\loss_{\dfl}(\mathcal{S}, \corpus)$. 
This procedure finds the derivative with respect to the truth values of the ground atoms $\frac{\partial \loss_{\dfl}(\mathcal{S}, \corpus)}{\partial \interpretationfol_{\btheta}(\predP)(o_1, ..., o_m)}$. 
We can use these partial derivatives to update the parameter $\btheta_n$ at iteration $n$ of the optimization process again using the chain rule, resulting in a different embedded interpretation $\interpretationfol_{\btheta_{n+1}}$.
This procedure is computed as follows for the $i$th parameter:

\begin{equation}
\label{eq:grad_desc}
    \btheta_{n+1, i} = \btheta_{n, i} - \epsilon \cdot \frac{\partial \loss_{\dfl}(\mathcal{S}_n, \corpus)}{\partial \btheta_{n, i}} = \btheta_{n, i} - \epsilon \cdot \sum_{\predP(o_1, ..., o_m)}\frac{\partial \loss_{\dfl}(\mathcal{S}_n, \corpus)}{\partial \interpretation(\predP)(o_1, ..., o_m)} \cdot \frac{\partial \interpretation(\predP)(o_1, ..., o_m)}{\partial \btheta_{n, i}},
\end{equation}

where $\epsilon$ is the learning rate. Note that the parameters $\btheta_n$ are implicitly passed to $\loss$ through the structure $\mathcal{S}_n=\langle \objects, \eta, \btheta_n\rangle$.

\begin{exmp}
\label{exmp:reallogic}
To illustrate the computation of the valuation function $\valdfl$, we return to the problem in Example \ref{exmp:diffreason}. 
The \textit{domain of discourse} is the set of objects on natural images. 
We have access to a dataset of two objects $\dataset=\{o_1, o_2\}$. 
The valuation of the formula $\varphi = \forall x, y\ \pred{chair}(x) \otimes \pred{partOf}(y, x) \rightarrow \pred{cushion}(y) \oplus \pred{armRest}(y)$ is
\begin{align*}
    \valdfl(\mu, \varphi) = A(A(I(&T(\interpretation(\pred{chair})(o_1), \interpretation(\pred{partOf})(o_1, o_1)), \\
    &S(\interpretation(\pred{cushion})(o_1), \interpretation(\pred{armRest})(o_1))),\\ I(&T(\interpretation(\pred{chair})(o_1), \interpretation(\pred{partOf})(o_2, o_1)), \\
    &S(\interpretation(\pred{cushion})(o_2), \interpretation(\pred{armRest})(o_2)))),\\
    A(I(&T(\interpretation(\pred{chair})(o_2), \interpretation(\pred{partOf})(o_1, o_2)), \\
    &S(\interpretation(\pred{cushion})(o_1), \interpretation(\pred{armRest})(o_1))),\\ I(&T(\interpretation(\pred{chair})(o_2), \interpretation(\pred{partOf})(o_2, o_2)), \\
    &S(\interpretation(\pred{cushion})(o_2), \interpretation(\pred{armRest})(o_2))))).
\end{align*}
Next, we choose the operators as $T=T_P$, $S = S_P$,  $A=A_{T_P}$ and $I=I_{RC}$, such that
\begin{align*}
    \valdfl(\mu, \varphi) =
    \label{eq:exmp:reallogic}
    \prod_{x, y\in \constants} &1 - \interpretation(\pred{chair})(x) \cdot \interpretation(\pred{partOf})(y, x) \cdot \\
    &(1 - \interpretation(\pred{cushion})(y))(1 - \interpretation(\pred{armRest})(y)) \notag.
\end{align*}
If we interpret the predicate functions using the confidence values from Example \ref{exmp:diffreason} so that $\interpretation(\pred{P}(x)) = p(\pred{P}(x))$, we find that $\valdfl(\varphi) = 0.612$. Taking $\corpus = \{\varphi\}$, we find using the chain rule that 
\begin{align*}
\frac{\partial \loss_{\dfl}( \mathcal{S}, \corpus)}{\partial \interpretation(\pred{chair})(o_1)}&= -0.4261 &\frac{\partial \loss_{\dfl}( \mathcal{S}, \corpus)}{\partial \interpretation(\pred{chair})(o_2)}&=-0.0058\\
\frac{\partial \loss_{\dfl}( \mathcal{S}, \corpus)}{\partial \interpretation(\pred{cushion})(o_1)}&=0.0029 & \frac{\partial \loss_{\dfl}( \mathcal{S}, \corpus)}{\partial \interpretation(\pred{cushion})(o_2)}&=0.7662\\
\frac{\partial \loss_{\dfl}( \mathcal{S}, \corpus)}{\partial \interpretation(\pred{armRest})(o_1)}&=0.0029 & \frac{\partial \loss_{\dfl}( \mathcal{S}, \corpus)}{\partial \interpretation(\pred{armRest})(o_2)}&=0.4257\\
\frac{\partial \loss_{\dfl}( \mathcal{S}, \corpus)}{\partial \interpretation(\pred{partOf})(o_1, o_1)} &= -0.4978 & \frac{\partial \loss_{\dfl}( \mathcal{S}, \corpus)}{\partial \interpretation(\pred{partOf})(o_2, o_2)} &= -0.1103\\
\frac{\partial \loss_{\dfl}( \mathcal{S}, \corpus)}{\partial \interpretation(\pred{partOf})(o_1, o_2)} &= -0.2219 & \frac{\partial \loss_{\dfl}( \mathcal{S}, \corpus)}{\partial \interpretation(\pred{partOf})(o_2, o_1)} &= -0.4031.
\end{align*}
We can now do a gradient update step to update the confidence values from Example \ref{exmp:diffreason}, or find what the partial derivative of the parameters $\btheta$ of some deep learning model $p_{\btheta}$ should be using Equation \ref{eq:grad_desc}.

One particularly interesting property of \dfuzz is that the partial derivatives of the subformulas with respect to the satisfaction of the knowledge base have a somewhat explainable meaning. For example, as hypothesized in Example \ref{exmp:diffreason}, the computed partial derivatives reflect whether we should increase $p(\pred{cushion}(o_2))$, as it is indeed the (\textbf{absolute}) largest partial derivative. 

\end{exmp}

\subsection{Implementation}
\begin{algorithm}[H]
    \caption{Computation of the \dfuzz loss. First it computes the fuzzy Herbrand interpretation $g$ given the current embedded interpretation $\interpretation$. This performs a forward pass through the neural networks that are used to interpret the predicates. Then it computes the valuation of each formula $\varphi$ in the knowledge base $\corpus$, implementing Equations \ref{eq:rlpred}-\ref{eq:rlaggr}.  }
    \label{alg:real_logic}
    \begin{algorithmic}[1] 
        \Function{$e$}{$\varphi, g, \constants, \mu$} \Comment{The valuation function computes the Fuzzy truth value of $\varphi$.}
            \If{$\varphi=\pred{P}(x_1, ..., x_m)$} 
                \State \textbf{return} $g[\pred{P}, (\mu(x_1), ..., \mu(x_m)]$ \Comment{Find the truth value of a ground atom using the dictionary $g$.}
            \ElsIf{$\varphi=\neg\phi$}
                \State \textbf{return} $N(e(\phi, g, \constants, \mu))$
            \ElsIf{$\varphi=\phi\otimes\psi$}
                \State \textbf{return} $T(e(\phi, g, \constants, \mu), e(\psi, g, \constants, \mu))$
            \ElsIf{$\varphi=\phi\oplus\psi$}
                \State \textbf{return} $S(e(\phi, g, \constants, \mu), e(\psi, g, \constants, \mu))$
            \ElsIf{$\varphi=\phi\rightarrow\psi$}
                \State \textbf{return} $I(e(\phi, g, \constants, \mu), e(\psi, g, \constants, \mu))$
            \ElsIf{$\varphi=\forall x\ \phi$} \Comment{Apply the universal aggregation operator.}
                \label{alg:quantifier}
                \State \textbf{return} $\aggregate_{o\in\constants}e(\phi, g, \constants, \mu\cup\{(x,o)\})$ \Comment{Each assignment can be seen as an instance of $\varphi$.}
            \ElsIf{$\varphi=\exists x\ \phi$} 
                \label{alg:quantifier}
                \State \textbf{return} $\Eaggregate_{o\in\constants}e(\phi, g, \constants, \mu\cup\{(x, o)\})$ 
            \EndIf
        \EndFunction
        \State
        \Procedure{\dfl}{$\interpretation, \predicates, \corpus, \objects, N, T, S, I, A, E$} \Comment{Computes the \dfuzz loss.}
            \State $\constants\gets o_1, ..., o_b \text{ sampled from } \objects$ \Comment{Sample $\batch$ constants to use this pass.}
            \label{alg:sample}
            \State $g\gets dict()$ \Comment{Collects truth values for ground atoms.}
            \For{$\pred{P}\in \predicates$}
                \For{$o_1, ..., o_{\alpha(\pred{P})} \in \constants$}
                    \State $g[\pred{P}, (o_1, ..., o_{\alpha(\pred{P})})]\gets \interpretation(\pred{P})(o_1, ..., o_{\alpha(\pred{P})})$ \Comment{Calculate the truth values of the ground atoms.}
                \EndFor
            \EndFor
            \label{alg:satisfaction}
            \State \textbf{return} $\aggregate_{\varphi\in\corpus} w_{\varphi}\cdot\val(\varphi, g, \constants, \emptyset)$ \Comment{Calculate valuation of the formulas $\varphi$. Start with an empty variable assignment. This implements Equation \ref{eq:lossrl}.}
        \EndProcedure
    \end{algorithmic}
\end{algorithm}


The computation of the satisfaction is shown in pseudocode form in Algorithm \ref{alg:real_logic}. By first computing the dictionary $g$ that contains truth values for all ground atoms,\footnote{The dictionary $g$ could be seen as a `fuzzy Herbrand interpretation', in that it assigns a truth value to all ground atoms.} we can reduce the amount of forward passes through the computations of the truth values of the ground atoms that are required to compute the satisfaction.

This algorithm can fairly easily be parallelized for efficient computation on a GPU by noting that the individual terms that are aggregated over in line \ref{alg:quantifier} (the different \textit{instances} of the universal quantifier) are not dependent on each other. By noting that formulas are in prenex normal form, we can set up the dictionary $g$ using tensor operations so that the recursion has to be done only once for each formula. This can be done by applying the fuzzy operators elementwise over vectors of truth values instead of a single truth value, where each element of the vector represents a variable assignment. 

The complexity of this computation then is $O(|\corpus| \cdot P\cdot \batch^{d})$, where $\corpus$ is the set of formulas, $P$ is the amount of predicates used in each formula and $d$ is the maximum depth of nesting of universal quantifiers in the formulas in $\corpus$ (known as the \textit{quantifier rank}). This is exponential in the amount of quantifiers, as every object from the constants $\constants$ has to be iterated over in line \ref{alg:quantifier}, although as mentioned earlier this can be mitigated somewhat using efficient parallelization. Still, computing the valuation for transitive rules (such as. $\forall x\, y, z \ \pred{Q}(x, z) \otimes \pred{R}(z, y) \rightarrow \pred{P}(x, y)$) will for example be far more demanding than for antisymmetry formulas (such as $\forall x, y \ \pred{P}(x, y) \rightarrow \neg \pred{P}(y, x)$).

\section{Derivatives of Operators}
\label{sec:connectives}


We will now show that the choice of operators that are used for the logical connectives actually determines the inferences that are done when using \dfl. If we used a different set of operators in Example \ref{exmp:reallogic}, we would have gotten very different derivatives. These could in some cases make more sense, and in some other cases less. Furthermore, it is much easier to find a global minimum of the fuzzy maximum satisfiability problem (Equation \ref{eq:bestsatproblem}) for some operators than for others. This is often because of the smoothness of the operators. In this section, we analyze a wide variety of functions that can be used for logical reasoning and present some of their properties that determine how useful they are in inferences such as those illustrated above.

We will not discuss any varieties of fuzzy negations since the strong negation $N_C(a) = 1-a$ is already continuous, intuitive and has simple derivatives. 

\begin{definition}
A function $f: \mathbb{R}\rightarrow \mathbb{R}$ is said to be \emph{vanishing} if there are $a, b\in \mathbb{R}$, $a < b$ such that for all $c\in (a, b)$, $f(c) = 0$, i.e. there is an interval for which the function is 0. Otherwise, the function is \emph{nonvanishing}.\\
A function $f: \mathbb{R}^n \rightarrow \mathbb{R}$ has a \emph{vanishing derivative} if for all $a_1, ..., a_n\in \mathbb{R}$ there is some $1\leq i\leq n$ such that $\frac{\partial f(a_1, ..., a_n)}{\partial a_i}$ is vanishing.
\end{definition}

Whenever the derivative of an operator vanishes, it loses its learning signal. This definition does not include functions that only pass through 0, such as when using the product t-conorm for $a\oplus \neg a$, where we find that the derivative of $S_P(a, 1-a)$ is 0 only at $\frac{1}{2}$. 
Furthermore, all the partial derivatives of the connectives used in the backward pass from the valuation function to the ground atoms have to be multiplied. 
If the partial derivatives are less than 1, their product will also approach 0. 
This can happen for instance with a large sequence of conjunctions using the product t-norm.

The drastic product $T_D$ and operators derived from it such as the drastic sum $S_D$ and the Dubois-Prade and Weber implications ($I_{DP}$ and $I_{WB}$) have vanishing derivatives almost everywhere. 
The output confidence values of deep learning models are the result of transformations on real numbers using functions like the sigmoid or softmax that result in truth values in $(0, 1)$. The operators derived from $T_D$ only have nonvanishing derivatives when their inputs are exactly $0$ or $1$, invalidating their use in this application

\begin{definition}
A function $f: \mathbb{R}^n\rightarrow \mathbb{R}$ is said to be \textit{single-passing} if it has nonzero derivatives on at most one input argument. That is, for all $x_1, ..., x_n\in[0, 1]$ it holds that $\left|\left\{i\middle|\frac{\partial f(x_1, ..., x_n)}{\partial x_i}\neq 0, i \in \{1, ..., n\}\right\}\right| \leq 1$.
\end{definition}

Using just single-passing Fuzzy Logic operators can be inefficient, since then at most one input will have a nonzero derivative (i.e. a learning signal), yet the complete forward pass still has to be computed to find this input. 
In particular, this will hold when choosing operators based on the Gödel t-norm.
\begin{proposition}
Any composition of single-passing functions is also single-passing.
\end{proposition}
\begin{proof}
    We will prove this by structural induction. Let $f: \mathbb{R}^n\rightarrow \mathbb{R}$ be a single-passing function and let $x_1, ..., x_n\in \mathbb{R}$. Then clearly $f(x_1, ..., x_n)$ is single-passing.
    
    Next, assume by induction that $g: \mathbb{R}^n\rightarrow \mathbb{R}$ is a composition of single-passing functions that we assume is single-passing. Let $y: \mathbb{R}^m \rightarrow \mathbb{R}$ be a single-passing function. Let $Z$ be the set of inputs to $y$ and define $x_i=y(Z)$. We show that the composition $g\left(x_1, ..., y(Z), ..., x_n\right)$ is also single-passing. 
    For any $z\in Z$ holds that
    \begin{equation}
    \label{eq:proof_chain}
        \frac{\partial g\left(x_1, ..., x_n\right)}{\partial z}=\frac{\partial g\left(x_1, ..., x_n\right)}{\partial x_i}\frac{\partial y_i(Z_i)}{\partial z}.
    \end{equation} As $g$ is single-passing, there is at most 1 number $j\in 1, ..., n$ so that $\frac{\partial g(x_1, ..., x_n)}{\partial x_j}\neq 0$. If there are 0, then there can also be no $k\in 1, ..., m$ such that $\frac{\partial g(x_1, ..., x_n)}{\partial z_k}\neq 0$ as $\frac{\partial g\left(x_1, ..., x_n\right)}{\partial x_i}=0$. 
    If there is 1, then either $j\neq i$, which is the direct input $x_j$. If $j=i$, then by the assumption of $y(Z)$ being single-passing, there is at most 1 $k\in 1,...,m$ so that $\frac{\partial y(Z)}{\partial z_k}\neq 0$ and by Equation \ref{eq:proof_chain} there is at most 1 input such that $\frac{\partial g(x_1, ..., x_n)}{\partial x}\neq 0$. We conclude that the composition $g\left(x_1, ..., y(Z), ..., x_n\right)$ is single-passing.
\end{proof}

Concluding, for any logical operator to be usable in the learning task, it will need to have a nonvanishing derivative at the majority of the input signals, so it can contribute to the learning signal at all, and ideally not be single-passing so that it can contribute effectively to the learning signal.

\section{Aggregation}
\label{sec:aggr}


After the global considerations from the previous section, we next analyze in detail aggregation operators for universal and existential quantification separately and outline their benefits and disadvantages in \dfl.

\subsection{Minimum and Maximum Aggregators}
\label{sec:aggr_min}
The minimum aggregator is given as $A_{T_G}(x_1, ..., x_n) = \min(x_1, ..., x_n)$. The partial derivatives are 
\begin{equation}
    \frac{\partial A_{T_G}(x_1, ..., x_n)}{\partial x_i} = \begin{cases}
    1 & \text{ if } i = \argmin_{j\in\{1, ..., n\}} x_j\\
    0 & \text{ otherwise}.
    \end{cases}
\end{equation}

It is single-passing with the only nonzero gradient being on the input with the lowest truth value. Many practical formulas have exceptions. An exception to a formula like $\forall x\ \pred{Raven}(x)\rightarrow \pred{Black}(x)$ would be a raven over which a bucket of red paint is thrown. The minimum aggregator would have a derivative on that exception when `red raven' is correctly predicted. Additionally, it is inefficient, as we still have to compute the forward pass for inputs that do not get a feedback signal.

The partial derivatives of the maximum aggregator $E_{S_G}(x_1, ..., x_n)=\max(x_1, ..., x_n)$ are similar, but increase the input with the highest truth value instead. 
This can be a reasonable aggregator for existential quantification, as it will reinforce the belief that the input we are most confident about is correct will make the existential quantifier true.
A downside is that it can only consider one such input, despite the fact that the condition might hold for multiple inputs.

\subsection{\luk\ Aggregator}
\label{sec:aggrluk}
The \luk \ aggregator is given as $A_{T_{LU}}(x_1, ..., x_n) = \max\left(\sum_{i=1}^n x_i - (n-1), 0\right).$ The partial derivatives are given by
\begin{equation}
    \frac{\partial A_{T_{LU}}(x_1, ..., x_n) }{\partial x_i} = \begin{cases}
    1 & \text{ if } \sum_{i=1}^n x_i > n-1\\
    0 & \text{ otherwise. }
    \end{cases}
\end{equation}
The gradient is nonvanishing only when $\sum_{i=1}^n x_i > n-1$, i.e., when the average value of $x_i$ is larger than $\frac{n-1}{n}$ \pcite{palljonssonRealLogicLogical2018}. As $\lim_{n\rightarrow\infty} \frac{n-1}{n} = 1$, for larger values of $n$, all inputs have to be high for this to hold. 

For the next proposition, we refer to the \textit{fraction of inputs} for which some condition holds. The probability that the condition holds for a point uniformly sampled from $[0, 1]^n$ is this fraction.
\begin{proposition}
\label{prop:luk}
The fraction of inputs $x_1, ..., x_n\in [0,1]$ for which the derivative of $A_{T_{LU}}$ is nonvanishing is $\frac{1}{n!}$.
\end{proposition}
\begin{proof}
    Consider standard uniformly distributed random variables $x_1, ..., x_n\sim U(0, 1)$. The sum $Y=\sum_{i=1}^n x_i$ is Irwin-Hall distributed \pcite{irwinFrequencyDistributionMeans1927,hallDistributionMeansSamples1927}. The cumulative density function of this distribution is 
    \begin{equation}
        F_Y(y) = \frac{1}{n!}\sum_{k=0}^{\lfloor y \rfloor}(-1)^k {\binom{n}{k} } (y-k)^{n-1},
    \end{equation}
    where $\lfloor  \rfloor$ is the floor function. The derivative of $A_{T_{LU}}$ is nonvanishing when $Y > n-1$, or equivalently as the Irwin-Hall distribution is symmetric, when $Y < 1$. Using $F_Y$ gives $F_Y(1) = \frac{1}{n!}\left((-1)^0 {\binom{n}{0}} (1-0)^n + (-1)^1{\binom{n}{1} 1}(1-1)^n\right)=\frac{1}{n!}$.
\end{proof} 
Clearly, for the majority of inputs there is a vanishing gradient, implying that this universal aggregator would not be useful in a \dfl learning setting. 

A similar argument can be made for the existential \luk\ aggregator, the bounded sum ${E_{S_{LK}} (x_1, ..., x_n)} = \min(\sum_{i=1}^n x_i, 1)$, which will only have nonvanishing derivatives of 1 to all argument when the average value of $x_i$ is smaller than $\frac{1}{n}$. Like the \luk\ aggregator, it also has nonvanishing derivatives only on a fraction  $\frac{1}{n!}$ of its domain. 
This is therefore not a useful existential aggregator: The agent will learn nothing unless all inputs are close to 0. 

\subsection{Yager Aggregator}
\label{sec:aggr_yager}
The Yager universal aggregator is given by
\begin{equation}
    A_{T_Y}(x_1, ..., x_n) = \max\left(1-\left(\sum_{i=1}^n (1 - x_i)^p\right)^{\frac{1}{p}}, 0\right), \quad p> 0
\end{equation}
Here, $p=1$ corresponds to the \luk\ aggregator, $p\rightarrow \infty$ corresponds to the minimum aggregator, and $p\rightarrow 0$ corresponds to the aggregator formed from the drastic product $A_{T_D}$. The derivative of the Yager aggregator is
\begin{equation}
\label{eq:aggr_yager_deriv}
    \frac{\partial A_{T_Y}(x_1, ..., x_n)}{\partial x_i} = \begin{cases}
    \left(\sum_{j=1}^n (1 - x_j)^p\right)^{1-\frac{1}{p}}\cdot (1-x_i)^{p-1} & \text{ if } \left(\sum_{j=1}^n (1 - x_j)^p\right)^{\frac{1}{p}} < 1\\
    0 & \text{ if } \left(\sum_{j=1}^n (1 - x_j)^p\right)^{\frac{1}{p}} > 1.
    \end{cases}
\end{equation}
This derivative vanishes whenever  $\sum_{j=1}^n (1 - x_j)^p \geq 1$. As $1 - x_i \in [0, 1]$, $(1 - x_i)^p$ is a decreasing function with respect to $p$. Therefore, $\sum_{i=1}^n (1 - x_i)^p < 1$ holds for a larger fraction of inputs when $p$ increases, with the fraction being 0 for $p=0$ as it corresponds to the drastic aggregator, and 1 for $p=\infty$.

The exact fraction of inputs with a nonvanishing derivative is hard to express in the general case.\footnote{\label{footnote:yager_aggregator}Assume that $x_1, ..., x_n\sim U(0, 1)$ are independently and standard uniformly distributed. Note that $z_i=1-x_i$ is also standard uniformly distributed. $z_i^p$ is distributed by the beta distribution $Beta(1/p, 1)$ \pcite{guptaHandbookBetaDistribution2004}. $Y=\sum_{i=1}^n z_i$ is the sum of $n$ such beta-distributed variables. Unfortunately, there is no closed-form expression for the probability density function of sums of independent beta random variables \pcite{phamReliabilityStandbySystem1994}. A suitable approximation would be to use the central limit theorem as $z_1, ..., z_n$ are identically and independently distributed.} However, we can find a closed-form expression for the Euclidean case $p=2$.
 \begin{proposition}
 \label{prop:vol_yager}
 The fraction of inputs $x_1, ..., x_n\in[0, 1]$ for which the derivative of $A_{T_Y}$ with $p=2$ is nonvanishing is $\frac{\pi^{\frac{n}{2}}}{2^{n}\cdot \Gamma(\frac{1}{2}n + \frac{1}{2})}$, where $\Gamma$ is the Gamma function.
 \end{proposition}
 \begin{proof}
    The points $x_1, .., x_n\in \mathbb{R}$ for which $\sum_{i=1}^n(1-x_i)^2 < 1$ holds describes the volume of an $n$-ball\footnote{An $n$-ball is the generalization of the concept of a ball to any dimension and is the region enclosed by a $n-1$ hypersphere. For example, the 3-ball (or ball) is surrounded by a sphere (or 2-sphere). Similarly, the 2-ball (or disk) is surrounded by a circle (or 1-sphere). A hypersphere with radius 1 is the set of points which are at a distance of 1 from its center.} with radius 1. This volume is found by \pcite{ballElementaryIntroductionModern1997}(p.5):
     \begin{equation}
         V(n) = \frac{\pi^{\frac{n}{2}}}{\Gamma(\frac{1}{2} n + 1)}.
     \end{equation}
    We are interested in the part of this volume where $x_1, ..., x_n\in [0, 1]$, that is, those in a single orthant\footnote{An orthant in $n$ dimensions is a generalization of the quadrant in two dimensions and the octant in three dimensions. } of the $n$-ball. The amount of orthants in which an $n$-ball lies is $2^n$.\footnote{To help understand this, consider $n=2$. The $1$-ball is the circle with center $(0, 0)$. The area of this circle is evenly distributed over the four quadrants.} Thus, the volume of the part of the $n$-ball where $x_1, ..., x_n\in [0,1]$ is $\frac{V(n)}{2^{n}}=\frac{\pi^{\frac{n}{2}}}{2^{n}\cdot \Gamma(\frac{1}{2}n + \frac{1}{2})}$. As the total volume of points lying in $[0, 1]^n$ is 1, this is also the fraction of points for which the derivative of $A_{T_Y}$ with $p=2$ is nonvanishing.
\end{proof}
 \begin{figure}
     \centering
     \includegraphics[width=0.6\linewidth]{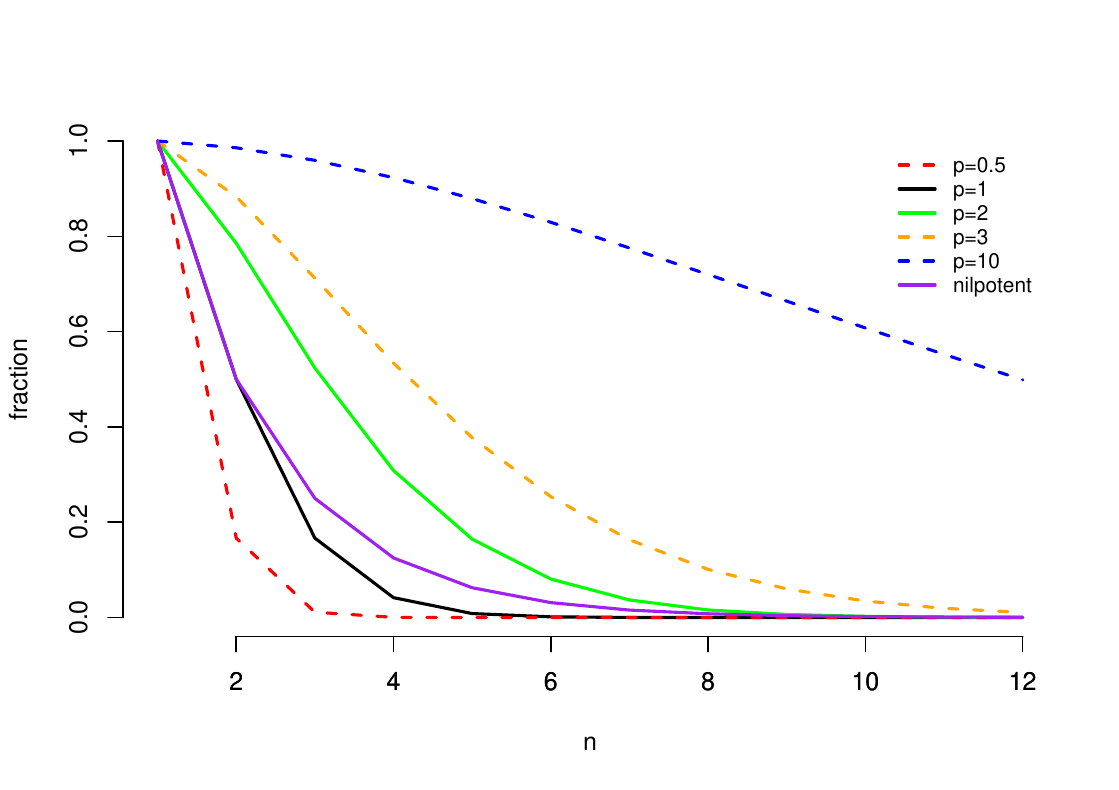}
     \caption[The fraction of inputs for which the Yager aggregator $A_{T_Y}$ and the nilpotent minimum aggregator $A_{T_{nM}}$ have nonvanishing derivatives]{The fraction of inputs for which the Yager aggregator $A_{T_Y}$ for several values of $p$ and the nilpotent minimum aggregator $A_{T_{nM}}$ have nonvanishing derivatives. The values for dotted lines are estimated using Monte Carlo simulation.}
     \label{fig:ratio_2_yager}
 \end{figure}
We plot the fraction of nonvanishing derivatives for several values of $p$ in Figure \ref{fig:ratio_2_yager}. For fairly small $p$, the vast majority of the inputs will have a vanishing derivative, and similar for high $n$, showing that this aggregator is also of little use in a learning context.

 Similarly, the derivatives of the Yager existential aggregator $E_{S_Y}(x_1, ..., x_n) = \min\left(\left(\sum_{i=1}^n  x_i^p\right)^{\frac{1}{p}}, 1\right)$  are nonvanishing in the same fraction of inputs as the Yager universal aggregator.
\subsection{Generalized Mean and Generalized Mean Error}
If we are concerned only in maximizing the truth value of $A_{T_Y}$, we can simply remove the \textit{max} constraint, resulting in an aggregator that has a nonvanishing derivative everywhere.
However, then the co-domain of the function is no longer $[0, 1]$. We can do an affine transformation on this function to ensure this is the case (Appendix \ref{appendix:yager}), after which we obtain the \emph{Generalized Mean Error}. 

\begin{definition}
\label{deff:p-mean-error}
For any $p> 0$, the Generalized Mean Error $A_{GME}$ is defined as
\begin{equation}
    A_{GME}(x_1, ..., x_n) =1 - \left(\frac{1}{n}\sum_{i=1}^n(1 - x_i)^p\right)^{\frac{1}{p}}.
\end{equation}
\end{definition}
 The `error' here the is difference between the predicted value $x_i$ and the `ground truth' value, a truth value of 1. This function has the following derivative:
 \begin{equation}
 \label{eq:deriv_apme}
    \frac{\partial A_{GME}(x_1, ..., x_n)}{\partial x_i} =
    (1-x_i)^{p-1}\frac{1}{n}^{\frac{1}{p}}\left(\sum_{j=1}^n (1 - x_j)^p\right)^{\frac{1}{p} - 1}  .
\end{equation}
When $p>1$, this derivative is greatest for the inputs that are lowest, which can speed up the optimization by being sensitive to outliers. For $p < 1$, the opposite is true. A special case is $p=1$:
  \begin{equation}
     \label{eq:aggr_sos}
     A_{MAE}(x_1, ..., x_n) = 1-\frac{1}{n}\sum_{i=1}^n (1 - x_i)
 \end{equation}
 having the simple derivative $\frac{\partial A_{MAE}(x_1, ..., x_n)}{\partial x_i} = \frac{1}{n}$. This measure is equal to 1 minus the \textit{mean absolute error (MAE}) and is associated with the \luk\ t-norm.
 Another special case is $p=2$:
 \begin{equation}
     \label{eq:aggr_rmse}
     A_{RMSE}(x_1, ..., x_n) = 1-\sqrt{\frac{1}{n}\sum_{i=1}^n (1 - x_i)^2}.
 \end{equation}
This function is equal to 1 minus the \textit{root-mean-square error} (RMSE) which is commonly used for regression tasks and heavily weights outliers. We can do the same for the Yager existential aggregator:
\begin{definition}
\label{deff:p-mean}
For any $p > 0$, the Generalized Mean is defined as
\begin{equation}
\label{eq:p-mean}
    E_{GM}(x_1, ..., x_n) = \left(\frac{1}{n}\sum_{i=1}^n x_i^p\right)^{\frac{1}{p}}.
\end{equation}
\end{definition}
$p=1$ corresponds to the arithmetic mean and $p=2$ to the geometric mean. 
In contrast to the Generalized Mean Error, its derivative $\frac{1}{n}\left(\frac{1}{n} \sum_{j=1}^n x_j^p \right)^{\frac{1}{p}-1}x_i^{p-1}$ has greater values for smaller inputs when $p<1$, and lower values when $p > 1$. 
Since we want to ensure the derivative is high only for inputs with high truth values to reinforce that those are likely the inputs that confirm the formula, we will want to use $p>1$. 
Note that the arithmetic mean $E_{GM}$ has the same derivative as the mean absolute error $A_{MAE}$, meaning that with $p=1$ universal and existential quantification cannot be distinguished. 
Furthermore, increasing all inputs equally is not a great idea for the existential quantifier, as there are likely only a few inputs for which the formula holds. 
Also note that, unlike existential aggregators directly formed from t-conorms, the only maximum of this aggregator is $x_1, ..., x_n=1$. 
However, it will take long until one can reach this optimum for low inputs.

\subsection{Product Aggregator and Probabilistic Sum}
The product aggregator is given as $A_{T_P}(x_1, ..., x_n) = \prod_{i=1}^n x_i$. This is also the probability of the intersection of $n$ independent events. 
It has the following partial derivatives:
\begin{equation}
\label{eq:derivprod}
    \frac{\partial A_{T_P}(x_1, ..., x_n) }{\partial x_i} = \prod_{j=1, i\neq j}^n x_j.
\end{equation}
This derivative vanishes if there are at least two $i$ so that $x_i=0$. Furthermore, the derivative for $x_i$ will be decreased if some other input $x_j$ is low. Finally, we cannot compute this aggregator in practice due to numerical underflow when multiplying many small numbers. Noting that $\argmax\ f(x) = \argmax\ \log(f(x))$, we observe that the \textit{log-product aggregator}
\begin{equation}
    \logprod(x_1, ..., x_n)=(\log \circ A_{T_P})(x_1, ..., x_n) = \sum_{i=1}^n\log(x_i)
\end{equation} 
can be used for formulas in prenex normal form, as then the truth value of the universal quantifiers is not used for another connective. Unlike the other aggregators, its codomain is the non-positive numbers instead of $[0, 1]$. Furthermore, the log-product aggregator can be seen as the log-likelihood function where we take the correct label to be 1, and thus this is similar to cross-entropy minimization. The partial derivatives are
\begin{equation}
    \frac{\partial \logprod(x_1, ..., x_n) }{\partial x_i} = \frac{1}{x_i}.
\end{equation}
In contrast to Equation \ref{eq:derivprod}, the values of the other inputs are irrelevant, and derivatives with respect to lower-valued inputs will be far greater as there is a singularity at $x=0$ (i.e. the value becomes infinite). We can conclude therefore that the product aggregator is particularly promising as it is nonvanishing and can handle outliers.  
The log-product aggregator also combines well with the generalized mean aggregator (Equation \ref{deff:p-mean}) for formulas of the form $\forall x \exists y$. The logarithm reduces the outer exponentiation, resulting in the derivative $\frac{x_i^{p-1}}{\sum_i x_i^p}$. 

The probabilistic sum aggregator $E_{S_P}(x_1, ..., x_n)=1-\prod_{i=1}^n (1-x_i)$ is trickier. Its derivatives are
\begin{equation}
    \frac{\partial E_{S_P}}{\partial x_i} = \prod_{j=1, j\neq i}^n (1-x_j).
\end{equation}
 This derivative is quite intuitive: It increases $x_i$ if the other inputs $x_j$ are low. 
However, this does not take into account the value of $x_i$ itself. If all inputs are low, this will increase all inputs equally.
Since the logarithm does not distribute over addition, we cannot use the same trick here as for the product aggregator, 
so care has to be taken when computing $\log \circ E_{S_P}$ to prevent numerical underflow errors.
\subsection{Nilpotent Aggregators}
The Nilpotent t-norm is given by $T_{nM}(a, b) = \begin{cases}
        \min(a, b), & \text{if } a + b > 1\\
        0, & \text{otherwise.}
    \end{cases}$
In Appendix \ref{appendix:nilpotent} we show that the Nilpotent aggregator $A_{T_{nM}}$ is equal to
\begin{equation}
\label{eq:aggrnilpmin}
    A_{T_{nM}}(x_1, ..., x_n) = \begin{cases}
        \min(x_1, ..., x_n), &\text{if } x_i + x_j > 1;\ x_i \text{ and } x_j \text{ are the two lowest values,} \\
        0, &\text{otherwise.}
    \end{cases}
\end{equation}

The derivative is found as follows: 
\begin{equation}
\label{eq:derivnilp}
    \frac{\partial A_{T_{nM}}(x_1, ..., x_n) }{\partial x_i} = \begin{cases}
        1, &\text{if } i=\argmin_j x_j  \text{ and } x_i + x_j > 1; x_j \text{ is the second-lowest value,}  \\
        0, &\text{otherwise.}
    \end{cases}
\end{equation}
Like the minimum aggregator it is single-passing, and like the \luk\ aggregator it has a derivative that vanishes for the majority of the input space, namely  when the sum of the two smallest values is lower than 1. 
\begin{proposition}
The fraction of inputs for which the derivative of $A_{T_{nM}}$ is nonvanishing is $\frac{1}{2^{n-1}}$.
\end{proposition}
\begin{proof}
    Consider $n$ standard uniformly distributed random variables $x_1, ..., x_n\sim U(0, 1)$. We are interested in the probability that $x_{(1)}+x_{(2)} > 1$, where $x_{(k)}$ is the $k$-th smallest sample (the $k$-th order statistic \pcite{davidOrderStatisticsThird2003}). \tcite{weisbergDistributionLinearCombinations1971} derives the cumulative density function for linear combinations of standard uniform order statistics. Let $k_0=0<k_1 < ... < k_S \leq k_n$ be $S$ integers indicating the coefficients $d_i > 0$. We aim to calculate the probability that $\sum_{i=1}^S d_i x_{(i)} > v$. Let $r_s=k_s-k_{s-1}$ for all $1, ..., S$ and let $r_{S+1}=n-k_{S}$. Finally, let $c_{S+1}=0$ and $c_{(s)}=c_{(s+1)} + d_i$. $m$ is the largest integer so that $v\leq c_{(m)}$. Then, \tcite{weisbergDistributionLinearCombinations1971} finds that 
    \begin{equation}
        P\left\{\sum_{s=1}^{S} d_{s} x_{\left(k_{s}\right)} > v\right\}=\sum_{s=1}^{m} \frac{g_{s}^{\left(r_{s}-1\right)}\left(c_{(s)}\right)}{\left(r_{s}-1\right) !}
    \end{equation}
    where $g^{(i)}_s$ is the $i$-th order derivative of 
    \begin{equation}
    g_s(c) = \frac{(c-v)^n}{c\prod_{i=1, i\neq s}^{S+1}(c - c_{(i)})^{r_i}}
    \end{equation}
    
    Filling this in for our case, we find that $S=2$, where $k_1=1$, $k_2=2$ as $d_1 = d_2 = 1$. Therefore, $r_1=r_2=1$ and $r_3=n-2$ and $c_{(1)}=2$, $c_{(2)}=1$ and $c_{(3)}=0$. The largest integer $m$ so that $1 \leq c_{(m)}$ is 2. Filling this in, we find that
    \begin{align}
        P\left\{x_{(1)} + x_{(2)} > 1\right\} &=\frac{g_1^{(1-1)}(2)}{(1-1)!}+\frac{g_2^{(1-1)}(1)}{(1-1)!} \\
        &= \frac{(2-1)^n}{2(2-1)^1(2-0)^{n-2}}+\frac{(1-1)^n}{1(1-2)(1-0)^{n-2}} = \frac{1}{2^{n-1}}
    \end{align}
\end{proof}
    
The fraction of inputs for which there is a nonvanishing derivative is plotted in Figure \ref{fig:ratio_2_yager}. Again, this means that for larger numbers of inputs $n$, this aggregator will vanish on almost every input and is not a useful construction in a learning context.

Similarly, the Nilpotent existential aggregator is given as 
\begin{equation}
    E_{S_{nM}}(x_1, ..., x_n) = \begin{cases}
        \max(x_1, ..., x_n), &\text{if } x_i + x_j < 1;\ x_i \text{ and } x_j,  \\
        1, &\text{otherwise.}
    \end{cases}
\end{equation}
It has a similar derivative that increases the largest input if the two largest inputs together are lower than 1. This is somewhat similar to how the maximum aggregator behaves, but with an additional condition that will stop increasing the largest input if there is another that is also quite high.  

\subsection{Summary}
The minimum aggregator is computationally inefficient and cannot handle exceptions well.  
Universal aggregation operators that vanish when receiving a large amount of inputs will not scale well, and these include operators based on the Yager family of t-norms and the nilpotent aggregator. 
Removing the bounds from the Yager aggregators introduces interesting connections to loss functions from the classical machine learning literature. 
This is also the case for the logarithmic version of the product aggregator, which corresponds to the cross-entropy loss function. 
They have natural means for dealing with outliers, and thus are promising for practical use.  
We have more options for existential quantification, as problems with vanishing gradients are not as important since we only care about ensuring the formula is true for at least one input, instead of all of them. 


\section{Conjunction and Disjunction}
\label{sec:conj_disj}



Next, we analyze the partial derivatives of t-norms and t-conorms, which are used as conjunction and disjunction in Fuzzy Logics. In t-norm Fuzzy Logics, the weak disjunction $\max(a, b)$, or the Gödel t-conorm is used instead of the dual t-conorm.

Suppose that we have a t-norm $T$ and a t-conorm $S$. We define the following two quantities, where the choice of taking the partial derivative to $a$ is without loss of generality, since $T$ and $S$ are commutative by definition:
\begin{align}
    \deri_T(a, b) = \frac{\partial T(a, b)}{\partial a},\quad
    \deri_S(a, b) = \frac{\partial S(a, b)}{\partial a} 
\end{align}

It should be noted that by Definition \ref{deff:tnorm}, $\frac{\partial T(a, 1)}{\partial a} = 1$ as $T(a, 1) = a$ for any t-norm $T$, and by Definition \ref{deff:snorm}, $\frac{\partial S(a, 0)}{\partial a} = 1$ as $S(a, 0) = a$ for any t-conorm $S$. Furthermore, we note that if $S$ is a t-conorm and the $N_C$-dual of the t-norm $T$, then $\frac{\partial 1-T(1-a, 1-b)}{\partial a} =-\frac{\partial 1-T(1-a, 1-b)}{\partial 1- a} = \frac{\partial T(1-a, 1-b)}{\partial 1- a} $. 

The main difference in analyzing t-norms and t-conorms is that the maximimum of $T(a, b)$ (namely 1) is when both arguments $a$ and $b$ are 1. In contrast, in t-conorms, an infinite number of maxima exist. Some of these maxima might be more desirable than others. Referring back to the formula in Example  \ref{exmp:diffreason}, we showed that it is preferable to increase the truth value of $\pred{cushion}(y)$ and not of $\pred{armRest}(y)$. Similarly, when a conjunct is negated, or when it appears in the antecedent of an implication (like in the aforementioned formula) we have to choose which of the two conjuncts to decrease. By noting that $\frac{\partial T(a, b)}{\partial a} =\frac{\partial S(1-a, 1-b)}{\partial 1-a}$, we find that the t-norm \say{chooses} in the same way its dual t-conorm would \say{choose}. Similarly, if a disjunction is negated, it will minimize both its arguments in the way that its dual t-norm would maximize its arguments.

\begin{exmp}
We introduce a running example to analyze the behavior of different t-norms. We will optimize $(a\otimes b)\oplus (c\otimes \neg a)$ using gradient descent. The truth value of this expression is computed using $f(a, b, c) = S(T(a, b), T(1-a, c))$. Using the boundary conditions from Definition \ref{deff:tnorm} and \ref{deff:snorm}, we find the global optima $a=1.0$, $b=1.0$ and $a=0.0$, $c=1.0$. The derivative to this function is, using the chain rule,
\begin{align}
    \frac{\partial f(a, b, c)}{\partial a} &= \frac{\partial S(T(a, b), T(1-a, c))}{\partial T(a, b)} \cdot \frac{\partial T(a, b)}{\partial a} + \frac{\partial S(T(a, b), T(1-a, c))}{\partial T(1-a, c)}\cdot \frac{\partial T(1-a, c)}{\partial a}.
    \label{eq:exmp_tsnorm}
\end{align}
\end{exmp}

\subsection{Gödel T-Norm}
The Gödel t-norm is $T_G(a, b) = \min(a, b)$ and the Gödel t-conorm is $S_G(a, b) = \max(a, b)$. We find
\begin{align}
    \frac{\partial T_G(a, b)}{\partial a}=  \begin{cases}
        1, & \text{if } a < b \\
        0, & \text{if } a > b
      \end{cases}, \quad
      \frac{\partial S_G(a, b)}{\partial a}=  \begin{cases}
        1, & \text{if } a > b \\
        0, & \text{if } a < b
      \end{cases}.
\end{align}

Both $T_G$ and $S_G$ are single-passing, but their derivatives are not defined $a = b$. 
A benefit of the magnitude of the derivative nearly always being 1 is that there will not be any exploding or vanishing gradients caused by multiple repeated applications of the chain rule.

\begin{exmp}
Filling in Equation \ref{eq:exmp_tsnorm} representing $(a\otimes b) \oplus (c\otimes \neg a)$ with $T_G$ and $S_G$ gives 
\begin{align*}
     \frac{\partial f(a, b, c)}{\partial a} &= \indic{\min(a, b)> \min(1-a, c)}\cdot \indic{a< b} - \indic{\min(1-a, c)> \min(a, b)} \cdot \indic{1-a < c}\\
     &= \indic{a > \min(1-a, c)\wedge a < b} - \indic{1-a > \min(a, b) \wedge 1-a < c}
\end{align*}
where $\indic{c}$ is the indicator function. This corresponds to the decision tree in Figure \ref{fig:decision_tree}. 
The value of $a$ can be modified to increase the truth of either one of the conjunctions. In order to choose which of the two should be true, it compares $a$ with $1-a$. If $1-a < a$, it will increase the first conjunct by increasing $a$. Gradient ascent always finds a global optimum for this formula.

\begin{figure}
    \centering
    \begin{forest} 
    [$a < b$, 
        [$1-a < c$,
            [$1-a < a$,
                [-1]
                [1]
            ] 
            [$c < a$,
                [0]
                [1]
            ] 
        ]   
        [$a-1 < c$,
            [$b < 1-a$,
                [0] 
                [-1] 
            ]   
            [0] 
        ]   
    ] 
    \end{forest}
    \caption{Decision tree for the derivative of $S_G(T_G(a, b), T_G(1-a, c))$ with respect to $a$.}
    \label{fig:decision_tree}
\end{figure}
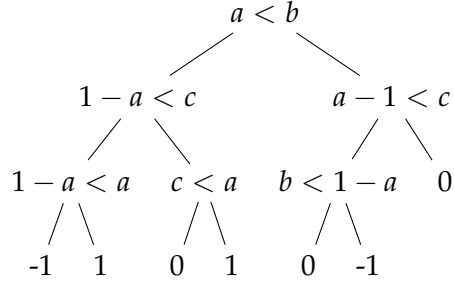

A small perturbation in the truth values of the inputs can flip the derivative around. For instance, if $a < b$ and $1-a < c$, then it will increase $a$ if its value is $0.501$ and decrease it if it is $0.499$. Furthermore, it can cause gradient ascent to get stuck in local optima. For instance, if $\varphi=(a\oplus b)\otimes(\neg a \oplus c)$ and $a=0.4, b=0.2$ and $c=0.1$, gradient ascent increases $a$ until $a > 0.5$, at which point the gradient flips and it decreases $a$ until $a < 0.5$. Experiments with optimizing this through gradient descent show that we can only find a global optimum in 88.8\% of random initializations of $a$, $b$ and $c$.
\end{exmp}


\subsection{\luk\ T-Norm}
The \luk \ t-norm is $T_{LK}(a, b) = \max(a + b - 1, 0)$ and the \luk \ t-conorm is $S_{LK}(a, b) = \min(a + b, 1)$. The partial derivatives are:
\begin{equation}
        \frac{\partial T_{LK}(a, b)}{\partial a} =  \begin{cases}
        1, & \text{if } a + b > 1 \\
        0, & \text{if } a + b < 1
      \end{cases}, \quad
      \frac{\partial S_{LK}(a, b)}{\partial a} =  \begin{cases}
        1, & \text{if } a + b < 1 \\
        0, & \text{if } a + b > 1
      \end{cases}
\end{equation}
These derivatives vanish on as much as half of their domain (Proposition \ref{prop:luk}). 
However, like the Gödel t-norm, when there is a gradient, it is large and will not cause vanishing or exploding gradients.

\begin{exmp}
Using the \luk\ t-norm and t-conorm in Equation \ref{eq:exmp_tsnorm} gives rise to the following computation
\begin{align*}
     \frac{\partial f(a, b, c)}{\partial a} &= \indic{\max(a + b - 1, 0) + \max(c - a, 0) < 1}\cdot\left(\indic{a+b > 1} - \indic{c-a > 0}\right).
\end{align*}
Choosing random values to initialize $a, b$ and $c$, gradient descent is able to find a (global) optimum in about 83.5\% of the initializations. 
\end{exmp}

\subsection{Yager T-Norm}
\label{sec:norm_yager}

\dualfigure{\imgT T_yager2.pdf}{\imgT S_yager2.pdf}{Left: The Yager t-norm. Right: The Yager t-conorm. For both, $p=2$.}{fig:yager2_norm}

The family of Yager t-norms \pcite{yagerGeneralClassFuzzy1980} is $T_Y(a, b)=\max(1 - \left((1-a)^p+ (1 - b)^p\right)^{\frac{1}{p}}, 0)$ and the family of Yager t-conorms is $S_Y(a, b) = \min( \left(a^p + b^p\right)^{\frac{1}{p}}, 1)$ for $p\geq 0$. We plot these for $p=2$ in Figure \ref{fig:yager2_norm}. The derivatives are given by
\begin{align}
        \frac{\partial T_{Y}(a, b)}{\partial a} &=  \begin{cases}
        \left((1 - a)^p+ (1 - b)^p\right)^{\frac{1}{p}-1}\cdot(1 - a)^{p-1} & \text{if } (1-a)^p + (1-b)^p < 1, \\
        0 & \text{if }(1-a)^p + (1-b)^p > 1,
      \end{cases} \\
      \frac{\partial S_{Y}(a, b)}{\partial a}&=  \begin{cases}
        \left(a^p + b^p\right)^{\frac{1}{p}-1}\cdot a^{p-1} \quad \quad \quad \quad \quad \quad \quad \ \ & \text{if } a^p + b^p < 1, \\
        0 &\text{if } a^p + b^p > 1
      \end{cases}
      \label{eq:deriv_s_yager}
\end{align}

\dualfigure{\imgT dT_yager2.pdf}{\imgT dS_yager2.pdf}{Left: The derivative of the Yager t-norm. Right: The derivative of the Yager s-norm. For both, $p=2$.}{fig:yager2_deriv}

We plot these derivatives in Figure \ref{fig:yager2_deriv}, showing for each a vanishing derivative on a non-negligible section of the domain. Using the method described in footnote \ref{footnote:yager_aggregator} (Section \ref{sec:aggr_yager}), Mathematica finds a closed form expression for the fraction of inputs for which the Yager t-norm is nonvanishing as $\frac{\sqrt{\pi } 4^{-1/p} \Gamma \left(\frac{1}{p}\right)}{p \Gamma
   \left(\frac{1}{2}+\frac{1}{p}\right)}$. 
Observe that when $p\neq 1$, the derivative of $T_Y$ is undefined at $a=b=1$ and the derivative of $S_Y$ is undefined at $a=b=0$. 
This requires care in the implementation to prevent numerical issues. 
For $p>1$, the lower of the two truth values has a higher derivative for the t-norm, while for the t-conorm, the higher of the two truth values has a higher derivative. As $p$ increases, $T_Y$ and $S_Y$ will behave more like $T_G$ and $S_G$. Note that when $p< 1$, the t-norm will have higher derivatives for higher inputs as the derivative has a singularity at $\lim_{a\rightarrow 1}=\infty$ ($b<1$).

\dualfigure{\imgT T_product.pdf}{\imgT S_product.pdf}{Left: The product t-norm. Right: The product t-conorm.}{fig:prod_norm}

\subsection{Product T-Norm}
\label{sec:product_real_logic}
The product t-norm and t-conorm, visualized in Figure \ref{fig:prod_norm}, are $T_P(a, b) = a\cdot b$ and $S_P(a, b) = a + b - a\cdot b$. Their derivatives are
\begin{equation}
    \frac{\partial T_P(a, b)}{\partial a} = b, \quad \frac{\partial S_P(a, b)}{\partial a} = 1-b.
\end{equation}
The derivative of the t-norm is 0 only when $a=b=0$, and similarly when $a=b=1$ for the t-conorm. The derivative of the t-norm can be interpreted as follows: `If we wish to increase $a\otimes b$, $a$ should be increased in proportion to $b$.' This is not a sensible learning strategy: If both $a$ and $b$ are small, in which case the conjunction is most certainly not satisfied, the derivative will be low instead of high. The derivative of the t-conorm is more intuitive, as it says `If we wish to increase $a\oplus b$, $a$ should be increased in proportion to $1-b$'. If $b$ is not yet true, we definitely want at least $a$ to be true.

\begin{exmp}
By using the product t-norm and t-conorm in Equation \ref{eq:exmp_tsnorm}, we get
\begin{equation*}
    \frac{\partial f(a, b, c)}{\partial a} = (1-(1-a)\cdot c) \cdot b - (1 - a\cdot b)\cdot c 
\end{equation*}
As explained, increase $a$ in proportion to $b$ if it is not true that $c$ and $\neg a$ are true, and decrease $a$ in proportion to $c$ if it is not true that $a$ and $b$ are true. 
\end{exmp}

\subsection{Summary}
The Gödel t-norm and t-conorm are simple and effective, having strong derivatives almost everywhere. However, they can be quite brittle by making very binary choices. The \luk t-norm and t-conorm also have strong derivatives, but vanish on half of the domain. The Yager family of t-norms and t-conorms also vanish on a significant part of its domain. The derivative of the t-norm is larger for lower values, which is a sensible learning strategy. This is not the case for the product t-norm, where the derivative is dependent on the other input value. However, the product t-conorm is intuitive, and corresponds to the intuition that if one input is not true, the other one should be. 



\section{Implication}


Finally, we consider what functions are suitable for modelling the implication. We will start by discussing the particular challenges associated with the implication operator.

\subsection{Challenges of the Material Implication}
\label{sec:implication_challenges}
A significant proportion of background knowledge is written as universally quantified implications. Examples of such statements are `all humans are mortal', `laptops consist of a screen, a processor and a keyboard' and `only humans wear clothes'. These formulas are of the form  $\forall x\ \phi(x) \rightarrow \psi(x)$, where we call $\phi(x)$ the antecedent and $\psi(x)$ the consequent. 

The implication is used in two well known rules of inference from classical logic. 
\emph{Modus ponens} inference says that if $\forall x\ \phi(x) \rightarrow \psi(x)$ and we know that $\phi(x)$ is true, then $\psi(x)$ should also be true. 
\emph{Modus tollens} inference, or \emph{contraposition}, says that if $\forall x\ \phi(x) \rightarrow \psi(x)$ and we know that $\psi(x)$ is false, then $\phi(x)$ should also be false, as otherwise $\psi(x)$ should also have been.

Unlike sequences of conjunctions where each of the formulas should be true, when the agent predicts a scene in which an implication is false, the supervisor has multiple choices.
Consider the implication `all ravens are black'. 
There are 4 categories for this formula: \textit{black ravens} (BR), \textit{non-black non-ravens} (NBNR), \textit{black non-ravens} (BNR) and \textit{non-black ravens} (NBR). 
Assume our agent observes an NBR, then there are four options to consider.
\begin{enumerate}
    \item \textit{Modus Ponens} (MP): The antecedent is true, so by modus ponens, the consequent is also true. We trust the agent's observation of a raven and believe it was a black raven (BR).
    \item \textit{Modus Tollens} (MT): The consequent is false, so by modus tollens, the antecedent is also false. We trust the agent's observation of a non-black object and believe that it was not a raven (NBNR).
    \item \textit{Distrust}: We think the agent is wrong in both observations, and conclude it was a black non-raven (BNR).
    \item \textit{Exception}: We trust the agent in observing a non-black raven (NBR) and ignore the fact that its observation goes against the background knowledge that ravens are black.\footnote{This option is not completely ludicrous as white ravens do in fact exist. However, they are rare.} 
\end{enumerate}
The distrust option seems somewhat useless. The exception option can be correct, but we cannot know when there is an exception from the agent's observations alone. 

We can assume there are far more non-black objects which are not ravens, than there are ravens. 
Thus, from a statistical perspective, it is most likely that the agent observed an NBNR. 
This shows the imbalance associated with the implication, which was first noted in \tcite{vankriekenSemisupervisedLearningUsing2019} for the Reichenbach implication. 
It is quite similar to the \emph{class imbalance problem} in Machine Learning \pcite{japkowiczClassImbalanceProblem2002} in the sense that one has far more \textit{contrapositive} examples than positive examples of the background knowledge.

This problem is closely related to the Raven paradox \pcite{hempelStudiesLogicConfirmation1945,vranasHempelRavenParadox2004,vankriekenSemisupervisedLearningUsing2019} from the field of confirmation theory which ponders what evidence can confirm a statement like `ravens are black'. 
It is usually stated as follows:
\begin{enumerate}
    \item Premise 1: Observing examples of a statement contributes positive evidence towards that statement.
    \item Premise 2: Evidence for some statement is also evidence for all logically equivalent statements.
    \item Conclusion: Observing examples of non-black non-ravens is evidence for `all ravens are black'.
\end{enumerate}
The conclusion follows since `non-black objects are non-ravens' is logically equivalent to `ravens are black'. 
For \dfl a similar thing happens. 
When we correct the observation of an NBR to a BR, the difference in truth value is equal to when we correct it to NBNR. 
More precisely, representing `ravens are black' as $I(a, b)$, where, for example, $I(1, 1)$ corresponds to BR:
\begin{align*}
&A(x_1, ..., I(1, 0), ..., x_n) - A(x_1, ..., I(1, 1), ..., x_n) = \\
&A(x_1, ..., I(1, 0), ..., x_n) - A(x_1, ..., I(0, 0), ..., x_n)
\end{align*}
as $I(0, 0) = I(1, 1) = 1$.  
When one agent observes a thousand BR's and a single NBR, and another agent observes a thousand NBNR's and a single NBR, their truth value for `ravens are black' is equal. 
The first agent has seen many ravens of which only a single was not black. 
The second only observed non ravens, and a single raven that was not black. 
Intuitively, the first agent's beliefs seem more in line with the background knowledge. 
We will now proceed to analyze a number of implication operators in light of this discussion.



\subsection{Analyzing the Implication Operators}
In this section, we choose to take the negation of the derivative with respect to the antecedent as it makes it easier to compare them: all fuzzy implications are monotonically decreasing with respect to the antecedent. 

\begin{definition}
A fuzzy implication $I$ is 
\textit{contrapositive differentiable symmetric} if $\dmp{a}{c} = \dmtp{1-c}{1-a}$ for all $a, c \in [0, 1]$.
\end{definition}
A consequence of contrapositive differentiable symmetry is that if $c=1-a$, then derivatives with respect to the antecedent and consequent are each other's negation since $\dmp{a}{c}=\dmtp{1 - c}{1-a} = \dmtp{1 -(1 - a)}{c} = \dmt{a}{c}$. This could be seen as the `distrust' option which increases the consequent and negated antecedent equally.

\begin{proposition}
If a fuzzy implication $I$ is $N$-contrapositive symmetric (that is, for all $a, c\in [0,1]$, $I(a, c)=I(N(c), N(a))$), where $N$ is the classical negation, it is also contrapositive differentiable symmetric. 
\end{proposition}
\begin{proof}
    Say we have an implication $I$ that is $N_C$-contrapositive symmetric. 
    Because $I$ is $N_C$-contrapositive symmetric, $I(1-c, 1-a) = I(a, c)$. 
    Thus, $\dmtp{1-c}{1-a} = -\frac{\partial I(a, c)}{\partial 1-c}=\frac{\partial I(a, c)}{\partial c}$. 
\end{proof}
By this proposition all S-implications are contrapositive differentiable symmetric. This property says there is no difference in how the implication handles the derivatives with respect to the consequent and antecedent.
\begin{proposition}
\label{prop:diff_left_neutral}
If an implication $I$ is left-neutral (that is, for all $c\in [0, 1]$, $I(1, c)=c$), then $\dmp{1}{c} = 1$. If, in addition, $I$ is contrapositive differentiable symmetric, then $\dmt{a}{0} = 1$.
\end{proposition}
\begin{proof}
    First, assume $I$ is left-neutral. Then for all $c\in[0,1]$, $I(1, c) = c$. Taking the derivative with respect to $c$, it is clear that $\dmp{1}{c} = 1$. Next, assume $I$ is contrapositive differentiable symmetric. 
    Then, $\dmp{1}{c} = \dmtp{1 - c}{1-1} = \dmt{1-c}{0} = 1$. As $1-c\in[0, 1]$, $\dmt{a}{0}=1$.
\end{proof}

All S-implications and R-implications are left-neutral, but only S-implications are all also contrapositive differentiable symmetric.
The derivatives of R-implications vanish when $a\leq c$, that is, on no less than half of the domain. 
This is not necessarily a bad property, although this depends highly on the sort of application we use DFL for. 
If, for example, $a=0.499$ and $c=0.5$ for the implication `ravens are black', this expresses a state of uncertainty, and there is probably more that can be learned! However, this will not be possible since the implication vanishes.


\subsection{Gödel-based Implications}
\label{sec:godel_implications}
\begin{figure}[h]
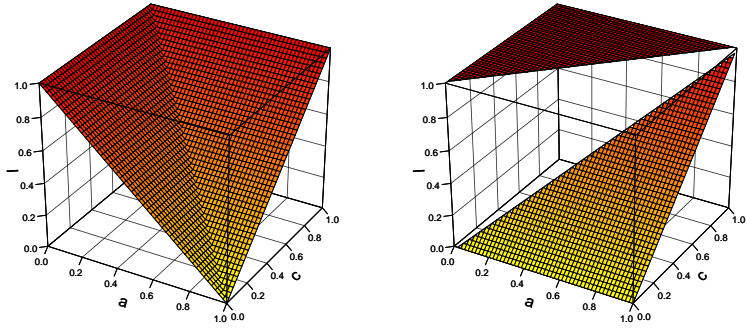

    \centering
    \begin{subfigure}[b]{\graphwidth}
    \includegraphics[width=\linewidth]{\imgT I_kleene.pdf}
    \end{subfigure}
    \begin{subfigure}[b]{\graphwidth}
    \includegraphics[width=\linewidth]{\imgT I_godel.pdf}
    \end{subfigure}%
    \caption[Left: The Kleene Dienes implication. Right: The Gödel implication]{Left: The Kleene Dienes implication. Right: The Gödel implication. Plots in this section are rotated so that the smallest value is in the front to help understand the shape of the functions. In particular, plots of the derivatives of the implications are rotated 180 degrees compared to the implications themselves.}
    \label{fig:kd_godel}
    \end{figure}

Implications based on the Gödel t-norm make strong discrete choices and are single-passing. As $I_{KD}(a, c) = \max(1-a, c)$, the derivatives are
\begin{equation}
    \dmpa{I_{KD}}{a}{c} =  \begin{cases}
        1, & \text{if } 1 - a < c \\
        0, & \text{if } 1 - a > c
      \end{cases}, \quad
      \dmta{I_{KD}}{a}{b}=  \begin{cases}
        1, & \text{if } 1-a > c \\
        0, & \text{if } 1-a < c
      \end{cases}   .
\end{equation}
Or, simply put, if we are more confident in the truth of the consequent than in the truth of the negated antecedent, increase the truth of the consequent. Otherwise, decrease the truth of the antecedent. This decision can be somewhat arbitrary and does not take into account the imbalance of modus ponens and modus tollens.

The Gödel implication is a simple R-implication: $I_G(a, c) =\begin{cases}
        1, & \text{if } a \leq c \\
        c, & \text{otherwise}
      \end{cases}$. 
Its derivatives are:
\begin{equation}
\label{eq:god_imp_deriv}
    \dmpa{I_G}{a}{c} =  \begin{cases}
        1, & \text{if } a > c \\
        0, & \text{otherwise}
      \end{cases}, \quad
      \dmta{I_G}{a}{b}= 0.
\end{equation}
These two implications are shown in Figure \ref{fig:kd_godel}. The Gödel implication increases the consequent whenever $a > c$, and the antecedent is never changed. This makes it a poorly performing implication in practice. For example, consider $a=0.1$ and $c=0$. Then the Gödel implication increases the consequent, even if the agent is fairly certain that neither is true. Furthermore, as the derivative with respect to the negated antecedent is always 0, it can never choose the modus tollens correction, which, as we argued, is actually often the best choice. 



\subsection{\luk\ and Yager-based Implications}
The \luk\ implication is both an S- and an R-implication. It is given by $I_{LK}(a, c) = \min(1-a + c, 1)$ and has the simple derivatives 
\begin{equation}
\label{eq:impl_deriv_luk}
    \dmpa{I_{LK}}{a}{c} = \dmta{I_{LK}}{a}{c} =  \begin{cases}
        1, & \text{if } a > c \\
        0, & \text{otherwise.}
      \end{cases}
\end{equation}
Whenever the implication is not satisfied because the antecedent is higher than the consequent, simply increase the negated antecedent and the consequent until it is lower. This could be seen as the `distrust' choice as both observations of the agent are equally corrected, and so does not take into account the imbalance between modus ponens and modus tollens cases. The derivatives of the Gödel implication $I_G$ are equal to those of $I_{LK}$ except that $I_G$ always has a zero derivative for the negated antecedent.

\dualfigure{\imgT I_yager_s_2.pdf}{\imgT I_yager_r_2.pdf}{Left: The Yager S-implication. Right: The Yager R-implication. For both, $p=2$.}{fig:yager-r-s-impl} 

The Yager S-implication is given as $I_Y(a, c) = \min\left(\left((1 - a)^p + c^p\right)^{\frac{1}{p}}, 1\right),\ p> 0$.
We plot $I_Y$ for $p=2$ in Figure \ref{fig:yager-r-s-impl}. $p=1$ is $I_{LK}$,  $p=0$ is $I_{DP}$, and $p=\infty$ is $I_{KD}$. 
The derivatives are computed as 
\begin{align}
    \dmpa{I_Y}{a}{c} &=  \begin{cases}
        \left((1-a)^p + c^p\right)^{\frac{1}{p}-1}\cdot c^{p-1}, \quad \quad \ \ & \text{if } (1-a)^p + c^p \leq 1, \\
        0, & \text{otherwise}
      \end{cases}\\
    \dmta{I_Y}{a}{c}&=  \begin{cases}
        \left((1-a)^p + c^p\right)^{\frac{1}{p}-1}\cdot (1-a)^{p-1},  & \text{if } (1-a)^p + c^p \leq 1, \\
        0, &\text{otherwise.}
      \end{cases}
\end{align}
\dualfigure{\imgT dMP_yager_s_2.pdf}{\imgT dMT_yager_s_2.pdf}{Plots of the derivatives of the Yager S-implication for $p=2$.}{fig:deriv_yager_sI2}
We plot these derivatives for $p=2$ in Figure \ref{fig:deriv_yager_sI2}. 
For all $p$, $\lim_{c\rightarrow 0}\dmpa{I_Y}{1}{c}=1$. 
Furthermore, for $p>1$, $\lim_{a\rightarrow 1} \left. \dmpa{I_Y}{a}{c}\right\rvert_{c=0}=0$ and for $p < 1$, $\lim_{a\rightarrow 1}\left. \dmpa{I_Y}{a}{0}\right\rvert_{c=0}=\infty$. 
For $p>1$, $I_Y$ can be understood as an increasingly less smooth version of the Kleene-Dienes implication $I_{KD}$. Lastly, this derivative, like those for $T_Y$ and $S_Y$ (Section \ref{sec:norm_yager}), is nonvanishing for a fraction of $\frac{\sqrt{\pi } 4^{-1/p} \Gamma \left(\frac{1}{p}\right)}{p \Gamma
   \left(\frac{1}{2}+\frac{1}{p}\right)}$ of the input space.

The Yager R-implication is found (Appendix \ref{appendix:yagerrimpl}) as $
    I_{T_Y}(a, c) = \begin{cases}
        1, & \text{if } a \leq c \\
        1 - \left((1 - c)^p - (1 - a)^p\right)^{\frac{1}{p}}, & \text{otherwise.}
      \end{cases}$
We plot $I_{T_Y}$ for $p=2$ in Figure \ref{fig:yager-r-s-impl}. 
As expected, $p=1$ reduces to $I_{LK}$, $p=0$ reduces to $I_{WB}$ and $p=\infty$ reduces to $I_G$. It is contrapositive symmetric only for $p=1$. The derivatives of this implication are 
\begin{align}
    \dmpa{I_{T_Y}}{a}{c} &= \begin{cases}
    ((1 - c)^p - (1 - a)^p)^{\frac{1}{p}-1}\cdot (1 - c)^{p-1} , & \text{if } a > c \\
    0, &\text{otherwise,}
    \end{cases}\\
    \dmta{I_{T_Y}}{a}{c}&= \begin{cases}
    ((1 - c)^p - (1 - a)^p)^{\frac{1}{p}-1}\cdot (1 - a)^{p-1}, & \text{if } a > c \\
    0, &\text{otherwise.}
    \end{cases}
\end{align}

We plot these in Figure \ref{fig:yager-r-deriv}. Note that if $p> 1$, for all $c<1$ it holds that $\lim_{a\downarrow c}\dmpa{I_{T_Y}}{a}{c} = \lim_{a\downarrow c}\dmta{I_{T_Y}}{a}{c} = \infty$ as when $a$ approaches $c$ from above, $(1 - c)^p - (1 - a)^p$ approaches 0, giving a singularity as $0^{\frac{1}{p} - 1}$ is undefined. This collection of singularities makes the training unstable in practice. 

\dualfigure{\imgT dMP_yager_r_2.pdf}{\imgT dMT_yager_r_2.pdf}{Plots of the derivatives of the Yager R-implication for $p=2$.}{fig:yager-r-deriv}

\subsection{Product-based Implications}
\label{sec:prod-implication}
\dualfigure{\imgT I_reichenbach.pdf}{\imgT I_goguen.pdf}{Left: The Reichenbach implication. Right: The Goguen implication.}{fig:goguen-i}

\dualfigure{\imgT dMT_logged_reichenbach.pdf}{\imgT dMT_RMSE1_reichenbach.pdf}{Left: The antecedent derivative of the Reichenbach implication with the log-product aggregator. Right: The antecedent derivative of the Reichenbach implication with the RMSE aggregator.}{fig:reichenbach-aggregators}

The product S-implication, also known as the Reichenbach implication, is given by $I_{RC}(a, c) = 1 - a + a \cdot c$. We plot it in Figure \ref{fig:goguen-i}. Its derivatives are given by:
\begin{equation}
    \dmpa{I_{RC}}{a}{c} = a, \quad \dmta{I_{RC}}{a}{c} = 1-c.
\end{equation}
These derivatives closely follow the modus ponens and modus tollens rules. When the antecedent is high, increase the consequent, and when the consequent is low, decrease the antecedent. However, around $(1-a)=c$, the derivative is equal and the `distrust' option is chosen. This can result in counter-intuitive behaviour. For example, if the agent predicts 0.6 for $\pred{raven}$ and 0.5 for $\pred{black}$ and we use gradient descent until we find a maximum, we could end up at 0.3 for $\pred{raven}$ and 1 for $\pred{black}$. We would end up increasing our confidence in $\pred{black}$ as $\pred{raven}$ was high. However, because of additional modus tollens reasoning, $\pred{raven}$ is barely true. 

Furthermore, if the agent most of the time predicts values around $a=0,\ c=0$ as a result of the modus tollens case being the most common, then a majority of the gradient decreases the antecedent as $\left.\dmta{I_{RC}}{a}{0}\right\rvert_{a=0}= 1$. We identify two methods that counteract this behavior. We introduce the second method in Section \ref{sec:sigm_implication}. 
The first method for counteracting the `corner' behavior notes that different aggregators change how the derivatives of the implications behave when their truth value is high.
For instance, we find that the derivatives with respect to the negated antecedent when using the log-product aggregator and RMSE aggregator are
\begin{alignat}{3}
    \frac{\partial\log \circ A_P(I_{RC}(a_1, c_1), ..., I_{RC}(a_n, c_n))}{\partial 1- a_i} &= \frac{1-c_i}{1 - a_i + a_i \cdot c_i}&&=\frac{\neg c}{a\rightarrow c}, \\
    \frac{\partial A_{RMSE}(I_{RC}(a_1, c_1), ...,I_{RC}(a_n, c_n))}{\partial 1 - a_i} &= \frac{(1 - c_i)(a_i - a_i\cdot c_i)}{\sqrt{n\sum_{j=1}^n (a_j - a_j\cdot c_j)^2}}&&=\frac{\neg c_i(\neg(a_i\rightarrow c_i))}{\sqrt{n
    \sum_{j=1}^n (\neg(a_j\rightarrow c_j))^2}}.
\end{alignat}
We plot these functions in Figure \ref{fig:reichenbach-aggregators}. 
For the RMSE aggregator 
we choose $n=2$ and $a_1,\ c_1$ so that $(a_1 - a_1 \cdot c_1)^2=0.9$. 
Note that the derivative with respect to the negated antecedent using the RMSE aggregator is 0 in $a_i=0$, $c_i=0$ as then $a_i - a_i \cdot c_i = 0$, and using the log-product aggregator, the derivative is 1. By differentiable contrapositive symmetry, the consequent derivative is 0 when using both aggregators. This shows that when using the RMSE aggregator, the derivatives will vanish at the corners $a=0,\ c=0$ and $a=1,\ c=1$, while when using the log-product aggregator, one of $a$ and $c$ will still have a gradient.

\dualfigure{\imgT log_dMP_goguen.pdf}{\imgT log_dMT_goguen.pdf}{The derivatives of the Goguen implication. Note that we plot these in log scale.}{fig:deriv_goguen}

The R-implication of the product t-norm is the Goguen implication and given by $I_{GG}(a, c) = \begin{cases} 1, &\text{if } a\leq c \\ \frac{c}{a}, &\text{otherwise.} \end{cases}$.
We plot this implication in Figure \ref{fig:goguen-i}. 
The derivatives of $I_{GG}$ are
\begin{equation}
    \dmpa{I_{GG}}{a}{c} = \begin{cases}
    0, &\text{if } a\leq c \\
    \frac{1}{a}, &\text{otherwise}
    \end{cases}, \quad \dmta{I_{GG}}{a}{c} = \begin{cases}
    0, &\text{if } a\leq c \\
    \frac{c}{a^2}, &\text{otherwise}
    \end{cases}.
\end{equation}

We plot these in Figure \ref{fig:deriv_goguen}. This derivative is not very useful. First of all, both the modus ponens and modus tollens derivatives increase with $\neg a$. This is opposite of the modus ponens rule as when the antecedent is \textit{low}, it increases the consequent most. For example, if $\pred{raven}$ is 0.1 and $\pred{black}$ is 0, then the derivative with respect to $\pred{black}$ is 10, because of the singularity when $a$ approaches 0.

\subsubsection{Sigmoidal Implications}
\label{sec:sigm_implication}
For the second method for tackling the corner problem, we introduce a new class of fuzzy implications formed by transforming other fuzzy implications using the sigmoid function and translating it so that the boundary conditions still hold. The derivation, along with several proofs of properties, can be found in Appendix \ref{appendix:sigm}.

\begin{definition}
\label{deff:sigmoidal}
If $I$ is a fuzzy implication, then the $I$-sigmoidal implication $\sigma_I$ is given for $s>0$ and $b_0\in \mathbb{R}$ as
\begin{equation}
    \sigma_I(a, c) = \frac{1 + e^{-s(1+b_0)}}{e^{-b_0 s}-e^{-s(1+b_0)}}\cdot 
   \left(\left(1 + e^{-b_0s}\right) \cdot \sigma\left(s\cdot\left(I(a, c) + b_0 \right)\right) - 1\right) 
\end{equation}
    where $\sigma(x)=\frac{1}{1+e^x}$ denotes the sigmoid function.
\end{definition}
Here $b_0$ is a parameter that controls the position of the sigmoidal curve and $s$ controls the `spread' of the curve. $\sigma_I$ is the function $ \sigma\left(s\cdot\left(I(a, c) + b_0\right)\right)$ linearly transformed so that its codomain is the closed interval $[0, 1]$. For the common value of $b_0=-\frac{1}{2}$, a simpler form exists: 
\begin{equation}
    \sigma_I(a, c)=\frac{1}{e^{\frac{s}{2}}-1}\cdot 
  \left(\left(1 + e^{\frac{s}{2}}\right) \cdot \sigma\left(s\cdot\left(I(a, c) - \frac{1}{2}\right)\right) - 1\right). 
\end{equation}
Next, we give the derivative of $\sigma_I$. Substituting $d=\frac{1 + e^{-s\cdot(1 + b_0)}}{e^{-s\cdot b_0}-e^{-s\cdot(1 + b_0)}}$ and $h=\left(1 + e^{-s\cdot b_0}\right)$, we find
\begin{equation}
\label{eq:deriv_sigm}
    \frac{\partial \sigma_I(a, c)}{\partial I(a, c)}=d\cdot h\cdot s\cdot  \sigma\left(s\cdot\left(I(a, c) + b_0\right)\right)\cdot(1 -  \sigma\left(s\cdot\left(I(a, c) + b_0\right)\right)).
\end{equation}
This keeps the properties of the original function but smoothens the gradient for higher values of $s$. 
As the derivative of the sigmoid function is positive, this derivative vanishes only when the derivative of $I$ vanishes.

\dualfigure{\imgT logdMP_reichenbach.pdf}{\imgT logdMP_probsum9.pdf}{The consequent derivatives of the log-Reichenbach and log-Reichenbach-sigmoidal (with $s=9$) implications. The figure is plotted in log scale.}{fig:log_dmp}

\dualfigure{\imgT dmp_probsum9.pdf}{\imgT dmt_probsum9.pdf}{The derivatives of the Reichenbach-sigmoidal implication for $s=9$.}{fig:deriv_rcsigm}

\begin{figure}
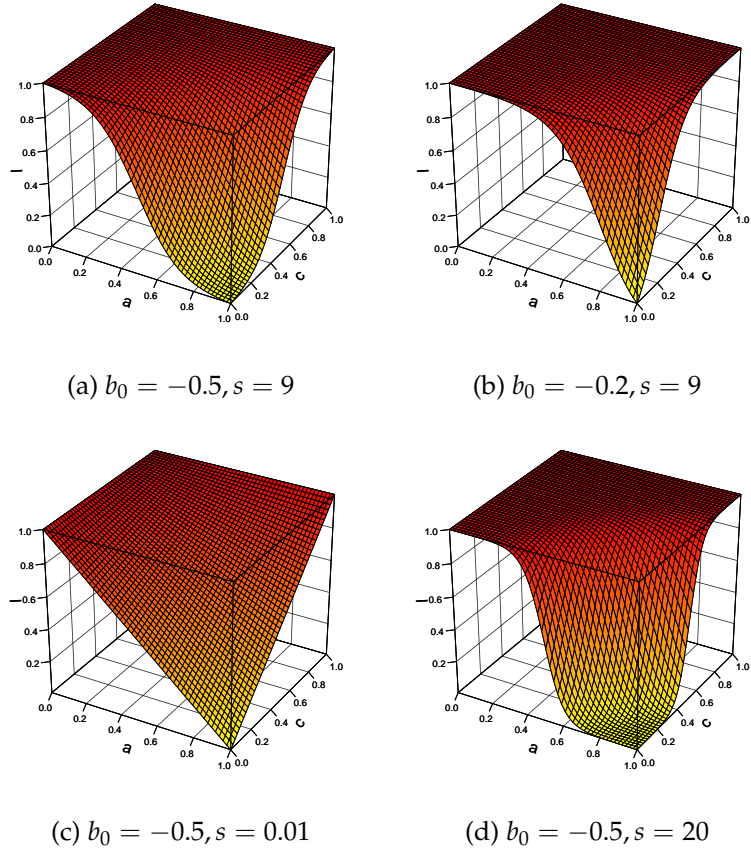

    \centering
    \begin{subfigure}[b]{\graphwidth}
    \includegraphics[width=\linewidth]{\imgT min05I_probsum9.pdf}
    \caption{$b_0=-0.5, s=9$}
    \label{fig:probsum_sigm_s_9-0.5}
    \end{subfigure}
    \begin{subfigure}[b]{\graphwidth}
    \includegraphics[width=\linewidth]{\imgT min02I_probsum9.pdf}
    \caption{$b_0=-0.2, s=9$}
    \label{fig:probsum_sigm_s_9-0.2}
    \end{subfigure}
    \\
    \begin{subfigure}[b]{\graphwidth}
    \includegraphics[width=\linewidth]{\imgT min05I_probsum001.pdf}
    \caption{$b_0=-0.5, s=0.01$}
    \label{fig:probsum_sigm_s_0.01}
    \end{subfigure}
    \begin{subfigure}[b]{\graphwidth}
    \includegraphics[width=\linewidth]{\imgT min05I_probsum20.pdf}
    \caption{$b_0=-0.5, s=20$}
    \label{fig:probsum_sigm_s_20}
    \end{subfigure}%
    \caption{The Reichenbach-sigmoidal implication for different values of $b_0$ and $s$.}
    \label{fig:probsum_sigm_s}
    \end{figure}

We plot the derivatives for the Reichenbach-sigmoidal implication $\sigma_{I_{RC}}$ in Figure \ref{fig:deriv_rcsigm}. As expected by Proposition \ref{prop:sigm_contrapos}, it is  differentiable contrapositive symmetric. Compared to the derivatives of the Reichenbach implication 
it has a small gradient in all corners. 
When using the log-product aggregator, the derivative of the antecedent with respect to the total valuation is divided by the truth of the implication. In Figure \ref{fig:log_dmp} we compare the consequent derivative of the normal Reichenbach implication with the Reichenbach-sigmoidal implication when using the $\log$ function. Clearly, for both there is a singularity at $a=1,\ c=0$, as then the implication is 0 and so the derivative of the log function becomes infinite. A significant difference is that the sigmoidal variant is less `flat' than the normal Reichenbach implication. This can be useful, as this means there is a larger gradient for values of $c$ that make the implication less true. In particular, the gradient at the modus ponens case ($a=1,\ c=1$) and the modus tollens case ($a=0,\ c=0$) are far smaller, which could help balancing the effective total gradient by solving the `corner' problem of the Reichenbach implication we brought up in Section \ref{sec:prod-implication}. These derivatives are smaller for higher values of $s$. 

In Figure \ref{fig:probsum_sigm_s} we plot the Reichenbach-sigmoidal implication for different values of the hyperparameters $b_0$ and $s$. Comparing \ref{fig:probsum_sigm_s_9-0.5} and \ref{fig:probsum_sigm_s_9-0.2} we see that larger values of $b_0$ move the sigmoidal shape so that its center is at lower input values. Note that for $s=0.01$ in Figure \ref{fig:probsum_sigm_s_0.01}, the plotted function is indiscernible from the plot of the Reichenbach implication in Figure \ref{fig:goguen-i} as the interval on which the sigmoid acts is extremely small and the sigmoidal transformation is almost linear. For very high values of $s$ like in \ref{fig:probsum_sigm_s_20} we see that the `S' shape is much thinner, and a larger part of the domain has a low derivative.

\subsection{Summary}
We analyzed several fuzzy implications from a theoretical perspective, while keeping the challenges caused by the material implication in mind. As a result of this analysis, we find that popular R-implications, in particular the Gödel implication, the Yager R-implication and the Goguen implication, will not work well in a differentiable setting. The other analyzed implications seem to have more intuitive derivatives, but may have other practical issues like non-smoothness.

\section{Experimental setup}
\label{chapter:experiments}


\label{sec:mnist}
To get insights in the behavior of these operators in practice, we next perform a series of simple experiments to analyze them. 
We discuss experiments using the MNIST dataset of handwritten digits \pcite{lecunMNISTHandwrittenDigit2010} to investigate the behavior of different fuzzy operators introduced in this paper. 
The goal of the experiments is not to show that our method is state of the art for the problem of semi-supervised learning on MNIST, but rather to be able to get insights into how fuzzy operators behave in a differentiable setting.\footnote{Code is available at \url{https://github.com/HEmile/differentiable-fuzzy-logics}.}
\subsection{Measures}
\label{sec:mnist_ldr}
To investigate the performance of the different configurations of \dfl, we first introduce several useful metrics. These give us insight into how different operators behave. In this section, we assume we are dealing with formulas of the form $\varphi=\forall x_1, ..., x_m\ \phi\rightarrow \psi$.

\begin{definition}
The \textit{consequent magnitude} $\mpmag$ and the \textit{antecedent magnitude} $\mtmag$ for a knowledge base $\corpus$  is defined as the sum of the partial derivatives of the consequent and antecedent with respect to the \dfl loss:
\begin{equation}
    \mpmag= \sum_{\varphi\in\corpus}\sum_{\instantiation\in\instantiations_\varphi}\frac{\partial \val(\mu, \varphi)}{\partial \val(\mu, \psi)}, \quad \mtmag= \sum_{\varphi\in\corpus}-\sum_{\instantiation\in\instantiations_\varphi}\frac{\partial \val(\mu, \varphi)}{\partial \val(\mu, \phi)},
\end{equation}
where $\instantiations_\varphi$ is the set of instances of the universally quantified formula $\varphi$ and $\psi$ and $\phi$ are evaluated under instantiation $\instantiation$.

The \textit{consequent ratio} $\mpratio$ is the sum of consequent magnitudes divided by the sum of consequent and antecedent magnitudes: $\mpratio = \frac{\mpmag}{\mpmag + \mtmag}$.
\end{definition}


\begin{definition}
Given a \textit{labeling function} $l$ that returns the truth value of a formula according to the data for instance $\instantiation$, the \textit{consequent and antecedent correctly updated magnitudes} are the sum of partial derivatives for which the consequent or the negated antecedent is true:
\begin{equation}
    \mpcorupdate = \sum_{\varphi\in\corpus}\sum_{\instantiation\in\instantiations_\varphi}  l(\psi, \mu) \cdot \frac{\partial \val(\mu, \varphi)}{\partial \val(\mu, \psi)}, \quad \mtcorupdate = \sum_{\varphi\in\corpus}-\sum_{\instantiation\in\instantiations_\varphi}l(\neg \phi, \mu) \cdot \frac{\partial \val(\mu, \varphi)}{\partial \val(\mu, \phi)}.
\end{equation}
\end{definition}
That is, if the consequent is true in the data, we measure the magnitude of the derivative with respect to the consequent. 
To evaluate these quantities, we define ratios similar to a precision metric:
\begin{definition}
The \textit{correctly updated ratio} for consequent and antecedent are defined as
\begin{equation}
    \mpupdateratio = \frac{\sum_{\varphi\in\corpus}\mpcorupdate_\varphi}{\sum_{\varphi\in\corpus}\mpmag_\varphi}, \quad    \mtupdateratio = \frac{\sum_{\varphi\in\corpus}\mtcorupdate_\varphi}{\sum_{\varphi\in\corpus}\mtmag_\varphi}.
\end{equation}
\end{definition}
These quantify what fraction of the updates are going in the right direction. When these ratios approach 1, \dfl will always increase the truth value of the consequent or negated antecedent correctly.\footnote{It can still change the truth value of a ground atom wrongly if $\phi$ or $\psi$ are not atomic formulas.} Otherwise, we are increasing truth values of subformulas that are wrong. 
Ideally, we want these measures to be high.

\subsection{Formulas}
We use a knowledge base $\corpus$ of universally quantified logic formulas. There is a predicate for each digit, that is $\pred{zero},\ \pred{one}, ..., \pred{eight}$ and  $\pred{nine}$. For example, $\pred{zero}(x)$ is true whenever $x$ is a handwritten digit labeled with 0. 
We have two sets of formulas where we learn an additional binary predicate.

\subsubsection{The \pred{same} problem}
The \pred{same} problem is a simple problem to test different operators for implication and universal aggregation.
We use the binary predicate $\pred{same}$ that is true whenever both its arguments are the same digit. We next describe the formulas we use. 
\begin{enumerate}
\item $ \forall x, y\ \pred{zero}(x)\otimes \pred{zero}(y) \rightarrow \pred{same}(x, y), ..., \forall x, y\ \pred{nine}(x)\otimes \pred{nine}(y) \rightarrow \pred{same}(x, y) $. 
If both $x$ and $y$ are handwritten zeros, for example, then they represent the same digit. 
\item $ \forall x, y\ \pred{zero}(x) \otimes \pred{same}(x, y) \rightarrow \pred{zero}(y), ..., \forall x, y\ \pred{nine}(x) \otimes \pred{same}(x, y) \rightarrow \pred{nine}(y) $. 
If $x$ and $y$ represent the same digit and one of them represents zero, then the other one does as well. 
\item $ \forall x, y\ \pred{same}(x, y) \rightarrow \pred{same}(y, x) $. 
This formula encodes the symmetry of the $\pred{same}$ predicate. 
\end{enumerate}
We next note what the values of $\mpupdateratio$ and $\mtupdateratio$ roughly are for each of the used formulas if we were to pick at random.
\begin{enumerate}
\item $ \forall x, y\ \pred{zero}(x)\otimes \pred{zero}(y) \rightarrow \pred{same}(x, y), ..., \forall x, y\ \pred{nine}(x)\otimes \pred{nine}(y) \rightarrow \pred{same}(x, y) $. 
For this formula, $\mpupdateratio\geq \frac{1}{10}$ as it is the distribution of $\pred{same}(x, y)$\footnote{It is slightly more than $\frac{1}{10}$ because we are using a minibatch of examples. Therefore, the reflexive pairs (i.e., $\pred{same}(x, x)$) are common.} and $\mtupdateratio\leq \frac{99}{100}$ as it is 1 minus the probability that both $x$ and $y$ are $\pred{zero}$. The modus ponens case is true in more than $\frac{1}{100}$ cases, the modus tollens casein less than $\frac{9}{10}$ cases and the `distrust' option in more than $\frac{9}{100}$ cases.

\item $ \forall x, y\ \pred{zero}(x) \otimes \pred{same}(x, y) \rightarrow \pred{zero}(y), ..., \forall x, y\ \pred{nine}(x) \otimes \pred{same}(x, y) \rightarrow \pred{nine}(y) $. 
For this formula, $\mpupdateratio = \frac{1}{10}$ as it is the probability that a digit represents zero and $\mtupdateratio\leq \frac{99}{100}$. The modus ponens cases is true in more than $\frac{1}{100}$ cases, the modus tollens in $\frac{9}{10}$ cases and the `distrust' option in $\frac{9}{100}$ cases.

\item $ \forall x, y\ \pred{same}(x, y) \rightarrow \pred{same}(y, x) $. 
As this is a bi-implication, $\mpupdateratio\geq\frac{1}{10}$ and $\mtupdateratio\leq \frac{9}{10}$. The `distrust' option is not possible in this formula.
\end{enumerate}

From this, we can see that a set of operators is better than random guessing for the consequent updates if $\mpupdateratio>0.1$. It is more difficult to say what the value of $\mtupdateratio$ should be to be as good as random guessing, as the probabilities are upper bounded with the lowest bound at $0.9$. We can only say that we know a set of operators to be better than random if $\mtupdateratio > 0.99$.

\subsubsection{The \pred{sum9} problem}
In the second problem we use the binary predicate $\pred{sum9}$ that is true whenever its arguments sum to 9. We use this problem to test existential quantification, conjunction and disjunction. The formulas are
\begin{enumerate}
    \item $\forall x \exists y\  \pred{sum9}(x, y)$. For each digit, there is another such that their sum is 9.\footnote{We sample minibatches of 64 digits, which means there is a negligible probability that there exists a digit in the minibatch for which there is no match (0.0117, to be precise).} 
    \item $\forall x, y\ \pred{sum9}(x, y) \rightarrow (\pred{zero}(x) \otimes \pred{nine}(y)) \oplus (\pred{one}(x) \otimes \pred{eight}(y)) \oplus \dots \oplus (\pred{nine}(x) \otimes \pred{zero}(y))$: This formula defines the $\pred{sum9}$ predicate.
\end{enumerate}
\subsection{Experimental Methodology}

We split the MNIST dataset so that 1\% of it is labeled and 99\% is unlabeled. Given a handwritten digit $\bx$ labeled with digit $y$, $p_\btheta(y|\bx)$ computes the distribution over the 10 possible labels. We use 2 convolutional layers with max pooling, the first with 10 and the second with 20 filters. Then follows two fully connected hidden layers with 320 and 50 nodes and a softmax output layer. The probability that $\pred{same}(\bx_1, \bx_2)$ for two handwritten digits $\bx_1$ and $\bx_2$ holds is modeled by $p_\btheta(\pred{same}|\bx_1, \bx_2)$. This takes the 50-dimensional embeddings of $\bx_1$ and $\bx_2$ of the fully connected hidden layer $e_{\bx_1}$ and $e_{\bx_2}$. These are used in a network architecture called a Neural Tensor Network \pcite{socherReasoningNeuralTensor2013}:

\begin{equation}
    p_\btheta(\pred{same}|\bx_1, \bx_2)=\sigma\left(u^\intercal \tanh\left(e_{\bx_1}^\intercal W^{[1:k]}e_{\bx_2} + V \begin{bmatrix} e_{\bx_1} \\ e_{\bx_2} \end{bmatrix} + b\right)\right).
\end{equation}
$W^{[1:k]}\in\mathbb{R}^{d\times d\times k}$ is used for the bilinear tensor product, $V\in \mathbb{R}^{k\times 2d}$ is used for a the concatenated embeddings and $b\in\mathbb{R}^k$ is used as a bias vector. We use $k=50$ for the size of the hidden layer. $u\in\mathbb{R}^{k}$ is used to compute the output logit, which goes through the sigmoid function $\sigma$ to get the confidence value. 

The loss function we use is split up in three parts, the first over the unlabeled dataset $\dataset_u$, and the two others over the labeled dataset $\dataset_l$:
\begin{align}
    w_{\dfl} \cdot \loss_{\dfl} (\langle \dataset_u, \eta, \btheta \rangle , \corpus)  -\sum_{\bx, y\in \dataset_l} \log p_\btheta(y|\bx) 
    -  \sum_{\substack{\bx_1, y_1, \bx_2, y_2\\\in \dataset_l\times\dataset_l}}
    \log p_\btheta(\pred{same}=\boldsymbol{1}_{y_1=y_2}|\bx_1, \bx_2)
\end{align}
The first term is the \dfl loss which is weighted by the \textit{\dfl weight} $w_{\dfl}$. 
The second is the supervised cross entropy loss with a batch size of 64. 
The third is the supervised binary cross entropy loss used to learn recognize $\pred{same}(x, y)$.\footnote{It is possible to not use this loss term and learn the $\pred{same}$ predicate using just the formulas, although this is more challenging and only works with a good set of fuzzy operators.} 
This loss is $\log p_\btheta(\pred{sum9}=\boldsymbol{1}_{y_1 + y_2=9}|\bx_1, \bx_2)$ for the \pred{sum9} problem. As there are far more negative examples than positive examples, we undersample the negative examples.
Note that the two supervised losses can also be seen as a universal aggregation over logical facts using the log-product aggregator: $-A_{\log T_P}(p_\btheta(y_1|\bx_1), ..., p_\btheta(y_{|\dataset_l|}|\bx_{|\dataset_l|}))$. 
For optimization, we used ADAM \pcite{kingmaAdamMethodStochastic2017} with a learning rate of 0.001. 

\section{Results}
We ran experiments for many combinations of operators with the aim of showing that the discussed insights are present in practice.
We report the accuracy of recognizing digits in the test set, the consequent ratio $\mpratio$, and the consequent and antecedent correctly updated ratios $\mpupdateratio$ and $\mtupdateratio$. 
We train for at most 70.000 iterations (or until convergence).  
The purely supervised baseline has a test accuracy of $95.18\%\pm 0.204$, and runs for about 35 minutes. 
Semi-supervised methods should improve upon this baseline to be useful. 
Our implementation including \dfl runs for 1 hour and 52 minutes.

\subsection{Symmetric Configurations}
First, we consider several \textit{symmetric} configurations, where the conjunction is a t-norm $T$, disjunction the dual t-conorm of $T$, universal aggregation the extended t-norm $A_T$, existential aggregation the extended t-conorm $E_S$ and the implication either is the S-implication based on the t-conorm or the R-implication based on the t-norm. 
For example, for $T_P$ we use $T_P$ for conjunction, $S_P$ for disjunction, $\logprod=\log\circ A_{T_P}$ for aggregation and $I_{RC}$ for implication. 
Symmetric configurations will retain many equivalence relations in fuzzy logic, unlike when one would choose arbitrary configuration of operators. 
\subsubsection{Symmetric Configurations on \pred{same} problem}

\begin{table}
    \centering
    \begin{tabular}{l...c}
     & \multicolumn{4}{c}{S-Implications} \\
    \hline
                       & \mc{Accuracy}   & \mc{$\mpratio$} & \mc{$\mpupdateratio$} & \mc{$\mtupdateratio$} \\
    \hline
    $T_G$              & 95.3            & 0.32            & 0.31                   & 0.83                           \\
    $T_P$              & \bft{96.5}      & 0.08            & 0.72                  & \textbf{0.99}  \\
    $T_{LK}$           & 94.9            & 0.5             & \bft{0.86}            & 0.12                          \\
    $T_Y,\ p=1.5$      & 95.2            &                 &                        &                            \\
    $T_Y,\ p=2$        & 77.7            & 0.20            & 0.51                  & 0.75                          \\
    $T_Y,\ p=20$       & 95.6            & 0.02            & 0.54                  & 0.75 \\
    $T_{Nm}$           & 95.2            &                &                      &  \\
    \hline
    \end{tabular}
    \caption[Results on the \pred{same} problem for several symmetric configurations with S-implications]{Results on the \pred{same} problem for several symmetric configurations with S-implications. For all, $w_{\dfl}=1$ except for $T_P$, for which $w_{\dfl}=10$.}
    \label{table:mnist_symmetric1}
\end{table}
\begin{table}
    \centering
    \begin{tabular}{l....}
     & \multicolumn{4}{c}{R-Implications} \\
    \hline
                       & \mc{Accuracy}   & \mc{$\mpratio$} & \mc{$\mpupdateratio$} & \mc{$\mtupdateratio$}  \\
    \hline
    $T_G$              & 95.0            & 1               & 0.11                  & - \\ 
    $T_P$              & 94.8            & 0.62            & 0.04                  & 0.96     \\
    $T_{LK}$           & 94.9            & 0.5             & \bft{0.86}            & 0.12                \\  
    $T_Y,\ p=1.5$      & 95.2            &                 &                       &  \\
    $T_Y,\ p=2$        & 95.0            & 0.62            & 0.61                  & 0.46 \\
    $T_Y,\ p=20$       & 95.5            & 0.53            & 0.01                  & \bft{0.99}        \\
    $T_{Nm}$           & 95.2            &                &                      &    \\
    \hline
    \end{tabular}
    \caption[Results on the \pred{same} problem for several symmetric configurations with R-implications]{Results on the \pred{same} problem for several symmetric configurations with R-implications. For all, $w_{\dfl}=1$ except for $T_P$, for which $w_{\dfl}=10$.}
    \label{table:mnist_symmetric2}
\end{table}

All configurations are run with $w_{\dfl}=1$ except for $T_P$ which is run using $w_{\dfl}=10$. The results on the $\pred{same}$ problem can be found in Tables \ref{table:mnist_symmetric1} and \ref{table:mnist_symmetric2}.
One general observation that can be made is that S-implications seem to work much better than R-implications. The only configuration with R-implications that outperform the supervised baseline is $T_Y$, $p=20$, but here the S-implication performs similar to it.
We hypothesize this is because the derivatives of R-implications vanish whenever $a \leq c$. 

The Gödel t-norm performs on par with the supervised baseline. This is because the min aggregator is single-passing, like the configuration as a whole. 
The single instance which receives a derivative might just be an exception as argued in Section \ref{sec:aggr_min} and evident from the low values of $\mpupdateratio$ and $\mtupdateratio$. 

The \luk\ t-norm performs worse than the supervised baseline. 
Since $A_{LK}$ either has a derivative of 0 or 1 everywhere, the total gradient is very large when it does not vanish. 
By the definition of $I_{LK}$, $\mpratio=\frac{1}{2}$ as the consequent and negated antecedent derivatives are equal (see Equation \ref{eq:impl_deriv_luk}). $\mpupdateratio$ is very low with only 0.01, which is worse than random guessing. As half of the gradient is MP reasoning, that half is nearly always incorrect.
The performance of the Yager t-norm seems highly dependent on the choice of the parameter $p$. For $p=20$ the top performance is quite a bit higher than the baseline. 
The lower the value of $p$, the more likely it is that the derivative of the universal aggregator vanishes. 
However, for $p=2$, the results are even worse than the \luk\ t-norm, which corresponds to $p=1$, while the derivative with $p=1.5$ simply vanishes throughout the whole run.

The product t-norm performs best and also has the highest values for $\mpupdateratio$ and $\mtupdateratio$. 
To a large extend this is because the log-product aggregator is very effective. 
Finally, the Nilpotent t-norm performs exactly like the supervised baseline since the derivative of the universal aggregator vanished during the complete training run. 

\subsubsection{Symmetric configurations on the \pred{sum9} problem}
\begin{table}
    \centering
    \begin{tabular}{l....}
    \hline
                   & \mc{Accuracy}                  & \mc{$\mpratio$} & \mc{$\mpupdateratio$} & \mc{$\mtupdateratio$} \\ \hline
    $T_G$          & 95.2                           & 0.31            & 0.44                  & 0.78             \\
    $T_P$          & \bft{96.1}                     & 0.13            & 0.80                  & 0.95             \\
    $T_{LK}$       & 95.2                           &                 &                       &                 \\
    $T_Y,\ p=1.5$    & 95.2                           &                 &                       &              \\
    $T_Y,\ p=2$    & 95.2                           &                 &                       &              \\
    $T_Y,\ p=20$   & 95.5                          & 0.99            & 0.82                  & 0.71             \\
    $T_{Nm}$       & 95.2                           &                 &                       &   \\
    \hline
    \end{tabular}
    \caption{Results on the \pred{sum9} problem for several symmetric configurations using the S-implication. $w_{\dfl} = 1$ except for $T_P$ with $w_{\dfl}=10$.}
    \label{table:sum9_symmetric}
\end{table}
In addition to the \pred{same} problem, we also run the symmetric configurations on the \pred{sum9} problem to be able to also take into account how existential quantification and disjunction behave. 
The results are in Table \ref{table:sum9_symmetric}.
These closely reflect the results for the $\pred{same}$ problem. 
Again, the only configurations that clearly outperform the supervised baseline are the product t-norm and the Yager t-norm with $p=20$, with the product t-norm being the most promising candidate.
Furthermore, in addition to the Nilpotent t-norm, the derivatives of the \luk\ t-norm and Yager t-norm with $p=2$  vanish throughout the whole run.

\subsection{Individual operators}
\label{sec:experiment_individual}
We also perform several experiments where we investigate the contribution of specific fuzzy operators without regard to whether the resulting configurations are sensible in a logical sense.
We do this to better understand how each operator contributes to the learning process.
Throughout this section, we fix the universal aggregation operator to the log-product aggregator $A_{\log T_P}$, the existential aggregation operator to the generalized mean $E_{GM}$ with $p=1.5$, the conjunction and disjunction to $T_Y$ and $S_Y$, also with $p=1.5$, and the implication to the Reichenbach-sigmoidal implication with $s=9$ and $b=-0.5$. 
We select these because of their promise in initial experiments.  
\subsubsection{Aggregation}
\begin{table}
    \centering
    \begin{tabular}{l....}
    & \multicolumn{4}{c}{Universal aggregation} \\
    \hline
                       & \mc{Accuracy}   & \mc{$\mpratio$} & \mc{$\mpupdateratio$} & \mc{$\mtupdateratio$} \\
    \hline
    $A_{T_G}$          & 86.6            & 0.37            & 0.65                  & 0.45  \\                          
    $A_{\log T_P}$     & \bft{96.3}      & 0.14            & 0.53                  & 0.96   \\
    $A_{T_{LK}}$       & 78.2            & 0.00            & \bft{0.94}            & \textbf{0.99} \\
    $A_{T_Y},\ p=1.5$  & 79.5            & 0.00            & 0.77                  & 0.98 \\                            
    $A_{T_Y},\ p=2$    & 83.3            & 0.00            & 0.83                  & 0.98 \\                            
    $A_{T_Y},\ p=20$   & 84.0            & 0.00            & 0.81                  & 0.98 \\
    $A_{GME},\ p=1.5$  & 96.1            & 0.43            & 0.43                  & 0.76 \\
    $A_{RMSE}$         & 96.2            & 0.46            & 0.42                  & 0.72 \\
    $A_{GME},\ p=20$   & 95.5            & 0.45            & 0.38                  & 0.70 \\
    $A_{T_{Nm}}$       & 79.8            & 0.34            & 0.50                  & 0.58  \\ 
    \hline
    \end{tabular}
    \caption{Results on the \pred{same} problem, varying the universal aggregator. For all, $w_{\dfl}=10$.}
    \label{table:uni_aggregation}
\end{table}
\begin{table}
    \centering
    \begin{tabular}{l....}
     & \multicolumn{4}{c}{Existential aggregation} \\
    \hline
                       & \mc{Accuracy}   & \mc{$\mpratio$} & \mc{$\mpupdateratio$} & \mc{$\mtupdateratio$} \\
    \hline
    $E_{S_G}$          & 95.3            & 0.38            & 0.60                  & 0.66 \\
    $E_{S_P}$          & 96.1            & 0.15            & 0.68                  & 0.94 \\
    $E_{S_{LK}}$       & 95.9            & 0.17            & \bft{0.76}            & 0.87 \\
    $E_{S_Y}$          & 95.9            & 0.21            & 0.64                  & 0.85 \\
    $E_{S_Y}$          & 96.3            & 0.27            & 0.59                  & 0.81 \\
    $E_{S_Y}$          & 96.3            & 0.37            & 0.62                  & 0.70 \\
    $E_{GM}$           & \bft{96.9}      & 0.29            & 0.13                  & \bft{0.95} \\
    $E_{GM}$           & 96.7            & 0.29            & 0.14                  & 0.94 \\
    $E_{GM}$           & 96.4            & 0.39            & 0.64                  & 0.70 \\
    $E_{S_{Nm}}$       & 95.5            & 0.31            & 0.58                  & 0.78 \\
    \hline
    \end{tabular}
    \caption{Results on the \pred{sum9} problem, varying the existential aggregator.}
    \label{table:exi_aggregation}
\end{table}

Table \ref{table:uni_aggregation} shows the results when varying the universal aggrator in the \pred{same} problem. The log-product operator with $w_{\dfl}=10$ is the best universal aggregator, with $A_{RMSE}$ trailing behind it slightly. 
Other generalized mean errors are also effective. 
We also see that the single-passing aggregators (minimum and Nilpotent minimum aggregators) and aggregators that vanish on a large part of their domain (Yager-based and Nilpotent minimum) all perform poorly. 
However, curiously the Yager-based aggregators have very high $\mpupdateratio$ and $\mtupdateratio$.

The generalized means have the best result for existential aggregation (Table \ref{table:exi_aggregation}), with $p=1.5$ performing best. 
They manage to properly select the inputs that make the existential quantifier true by softly increasing the largest inputs.
Unlike with universal aggregation, the Yager existential aggregator is also a decent choice. 
The maximum aggregator and Nilpotent maximum aggregator only somewhat outperform the supervised baseline, however, which again is due to them being single-passing. 
\subsubsection{Conjunction and Disjunction}
\begin{table}
    \centering
    \begin{tabular}{l....}
    \hline
                   & \mc{Accuracy}                  & \mc{$\mpratio$} & \mc{$\mpupdateratio$} & \mc{$\mtupdateratio$} \\ \hline
    $T_G$          & 20.4                           & 0.47            & 0.15                  & 0.90             \\
    $T_P$          & 20.3                           & 0.46            & 0.16                  & 0.92             \\
    $T_{LK}$       & \bft{97.0}                     & 0.31            & 0.15                       & 0.93                \\
    $T_Y,\ p=1.5$  & \bft{97.0}                     & 0.29            & 0.11                  & \bft{0.96}             \\
    $T_Y,\ p=2$    & 96.8                           & 0.30                & 0.11              & \bft{0.96}             \\
    $T_Y,\ p=20$   & 94.9                           & 0.66            & 0.16                  & 0.86             \\
    $T_{Nm}$       & 96.9                           & 0.31            & \bft{0.17}                  & 0.93   \\
    \hline
    \end{tabular}
    \caption{Results on the \pred{sum9} problem, varying the t-norm and t-conorm together. }
    \label{table:tnorms}
\end{table}
In Table \ref{table:tnorms}, we compare different t-norms together with their corresponding t-conorms. 
Here, it is the operators that vanish on a large part of their domain that work best, namely Yager t-norms and the Nilpotent minimum. 
These seem to work much better than when used in aggregation since the amount of inputs is much smaller, reducing the probability that the derivative vanishes. 
The product t-norm and Gödel t-norm, which corresponds to weak conjunction and disjunction, seem to do perform very poorly. 
In this table, there is a clear relation where lower $\mpratio$ seem to perform better. 
It is likely lower in the Yager t-norm since the disjunction in the consequent will very often be 1, which happens when the sum in the Yager t-norm hits the boundary. 
\subsubsection{Implications}
\label{sec:experiment_implications}
In Tables \ref{table:s_implications} and \ref{table:R-implications}, we compare different fuzzy implications on the \pred{same} problem. The Reichenbach implication and Yager S-implication work well, both having an accuracy around 97\%. 
The Kleene Dienes and Yager R-implications surpass the baseline as well. 
In these experiments, the sigmoidal-Reichenbach implication, which we run with $s=9$ and $b=-\frac{1}{2}$ performs as well as the normal Reichenbach implication. 
However, we find in \ref{sec:mnist_rcsigmoidal} that with the SGD algorithm, this implication outperforms the Reichenbach implication, reaching 97.3 accuracy. 

\begin{table}
    \centering
    \begin{tabular}{l|cccc}
    & \multicolumn{4}{c}{S-Implications} \\
    \hline
    & Accuracy & $\mpratio$ & $\mpupdateratio$ & $\mtupdateratio$ \\
    \hline
    $I_{KD}$ & 95.7 & 0.08 & \bft{0.82} & 0.98 \\
    $I_{RC}$ & \bft{96.3} & 0.09 & 0.73 & \textbf{0.98} \\
    $I_{LK}$ & 96.0 & 0.5 & 0.07 & 0.93 \\
    $I_Y,\ p=1.5$ & 96.1 & 0.14 & 0.77 & 0.96 \\
    $I_Y,\ p=2$ & 96.2 & 0.13 & 0.82 & 0.97 \\
    $I_Y,\ p=20$ & 95.1 & 0.57 & 0.48 & 0.98 \\
    $I_{FD}$ & 96.1 & 0.19 & 0.68 & 0.91 \\
    $\sigma_{I_{RC}}$ & \bft{96.3} & 0.14 & 0.53 & 0.96 \\
    \hline
    \end{tabular}
    \caption{Results on the \pred{same} problem for several symmetric configurations with S-implications.}
    \label{table:s_implications}
\end{table}

\begin{table}
    \centering
    \begin{tabular}{l|c|c|c|c}
     & \multicolumn{4}{c}{R-Implications} \\
    \hline
                       & Accuracy   & $\mpratio$ & $\mpupdateratio$ & $\mtupdateratio$  \\
    \hline
    $I_G$              & 90.5       & 1          & 0.05             &  \\ 
    $I_{GG}$           & 93.6       & 0.99       & 0.00             & 0.87     \\
    $I_{LK}$           & 96.0       & 0.5        & 0.07             & 0.93                   \\  
    $I_{RY},\ p=1.5$ & 96.1       & 0.14       & 0.14             & 0.66 \\
    $I_{RY},\ p=2$           & 95.7       & 0.13       & 0.64             & 0.38 \\
    $I_{RY},\ p=20$           & 96.0       & 0.14       & 0.65             & 0.38        \\
    $I_{FD}$           & 96.1       & 0.19       & 0.68             & 0.91   \\
    \hline
    \end{tabular}
    \caption{Results on the \pred{same} problem for several symmetric configurations with R-implications.}
    \label{table:R-implications}
\end{table}

As argued in Sections \ref{sec:godel_implications} and \ref{sec:prod-implication}
, the Gödel implication and Goguen implication have worse performance than the supervised baseline by making many incorrect modus ponens inferences. 
While the derivatives of $I_{LK}$ and $I_G$ only differ in that $I_G$ disables the derivatives with respect to negated antecedent, $I_{LK}$ performs among the better while $I_G$ is the worst test implication, suggesting that the derivatives with respect to the negated antecedent are required to successfully applying \dfl. 
Note that S-implications tend to perform better than R-implications, in particular for the Gödel t-norm and the product t-norm. 
This could be because they inherently balance derivatives with respect to the consequent and negated antecedent by being contrapositive differentiable symmetric.

%

\subsubsection{Influence of Individual Formulas}
\label{sec:mnist_formulas_experiments}
\begin{table}

\centering
\begin{tabular}{l....}
\hline
Formulas                         & \mc{Accuracy} & \mc{$\mpratio$} & \mc{$\mpupdateratio$} & \mc{$\mtupdateratio$}               \\
\hline
(1) \& (2) & \bft{97.1}    & 0.05            & 0.54                  & \bft{0.99} \\
(2) \& (3)  & 95.9          & 0.12            & \bft{0.75}            & 0.95                           \\
(1) \& (3)  & 96.3          & 0.15            & 0.52                  & 0.98                           \\
(1)                     & 95.6          & 0.05            & 0.59                  & \bft{1.00} \\
(2)                     & 95.2          & 0.03            & \bft{0.78}            & 0.99                           \\
(3)                      & \bft{95.8}    & 0.19            & 0.64                  & 0.95               \\
\hline
\end{tabular}
\caption[The results leaving some formulas of the \pred{same} problem out]{The results using $\sigma_{I_{RC}}$ for the implication with $s=9$ and $b_0=-\frac{1}{2}$, $T_Y, p=2$ for the conjunction and $\logprod$ for the aggregator with $w_{\dfl}=10$, leaving some formulas of the \pred{same} problem out. The numbers indicated the formulas that are present during training.  }
\label{table:mnist_formula_experiments}
\end{table}
Finally, we compare what the influence of the different formulas of the \pred{same} problem are in Table \ref{table:mnist_formula_experiments}. Removing the reflexivity formula (3) does not largely impact the performance. The biggest drop in performance is by removing formula (1) that defines the $\pred{same}$ predicate. Using only formula (1) gets slightly better performance than only using formula (2), despite the fact that no positive labeled examples can be found using formula (1) as the predicates $\pred{zero}$ to $\pred{nine}$ are not in its consequent. Since 95\% of the derivatives are with respect to the negated antecedent, this formula contributes by finding additional counterexamples. Furthermore, improving the accuracy of the $\pred{same}$ predicate improves the accuracy on digit recognition: Just using the reflexivity formula (3) has the highest accuracy when used individually, even though it does not use the digit predicates.

\subsection{Reichenbach-Sigmoidal Implication}
\label{sec:mnist_rcsigmoidal}
The newly introduced Reichenbach-sigmoidal implication $\sigma_{I_{RC}}$ is a promising candidate for the choice of implication as we have argued in Section \ref{sec:sigm_implication}. 
To get a better understanding of this implication, we investigate the effect of its parameters in the \pred{same} problem. 
We fix the aggregator to the log-product, the conjunction operator to the Yager t-norm with $p=2$, and use a \dfl weight of $w_{\dfl}=10$.
\begin{figure}
\centering
\begin{subfigure}[b]{0.49\linewidth}
\includegraphics[width=\linewidth]{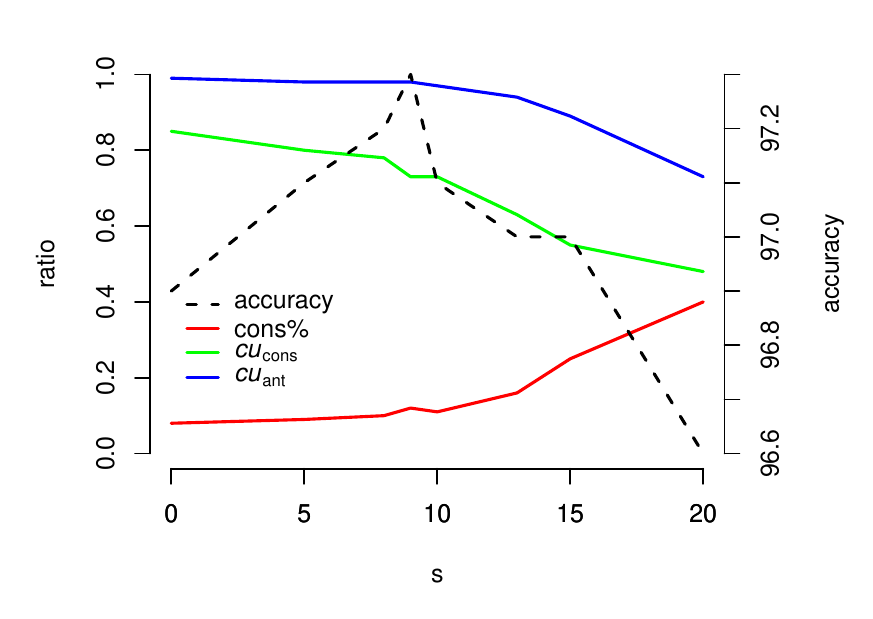}
\end{subfigure}
\begin{subfigure}[b]{0.49\linewidth}
\includegraphics[width=\linewidth]{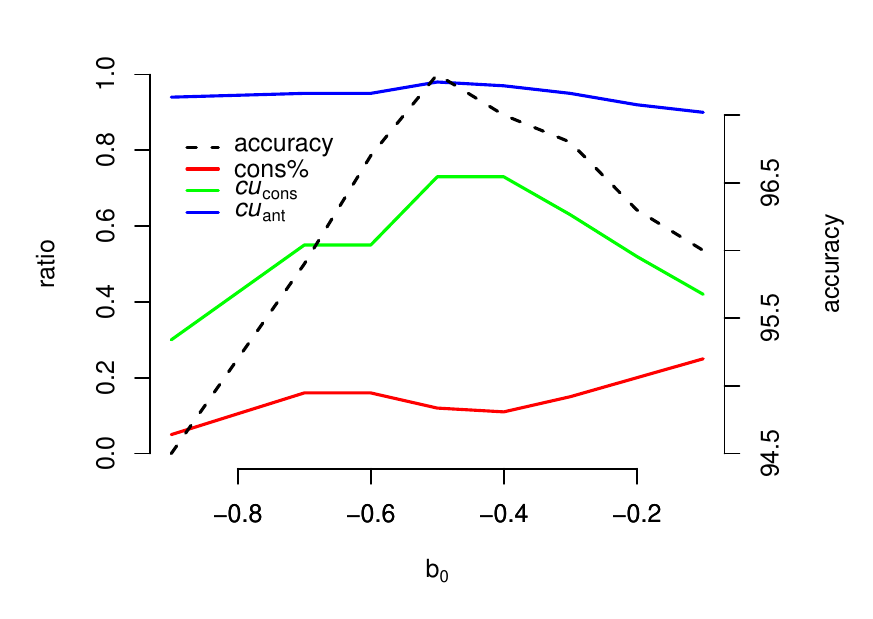}
\end{subfigure}%
\caption[The results using the Reichenbach-sigmoidal implication $\sigma_{I_{RC}}$]{The results using the Reichenbach-sigmoidal implication $\sigma_{I_{RC}}$, the $\logprod$ aggregator, $T_Y$ with $p=2$ and $w_{\dfl}=10$. Left shows the results for various values of $s$, keeping $b_0$ fixed to -0.5, and right shows the results for various values of $b_0$, keeping $s$ fixed to 9.}
\label{fig:mnist_s_experiments}
\end{figure}

On the left plot of Figure \ref{fig:mnist_s_experiments} we find the results when we experiment with the parameter $s$, keeping $b_0$ fixed to $-\frac{1}{2}$. Note that when $s$ approaches 0 the Reichenbach-sigmoidal implication is $I_{RC}$. The value of 9 gives the best results, with 97.3\% accuracy. 
Interestingly enough, there seem to be clear trends in the values of $\mpratio$,\ $\mpupdateratio$ and $\mtupdateratio$. Increasing $s$ seems to increase  $\mpratio$. This is because the antecedent derivative around the corner $a=0,\ c=0$ will be low, as argued in Section \ref{sec:sigm_implication}. When $s$ increases, the corners will be more smoothed out. 
Furthermore, both $\mpupdateratio$ and $\mtupdateratio$ decrease when $s$ increases. This could again be because around the corners the derivatives become small. Updates in the corner will likely be correct as the model is already confident about those. For a higher value of $s$, most of the gradient magnitude is at instances on which the model is less confident. We note that the same happened when using the RMSE aggregator and the product t-norm. Regardless, the best parameter value clearly is not the one for which the values of $\mpupdateratio$ and $\mtupdateratio$ are highest, namely the Reichenbach implication itself.

On the right plot of Figure \ref{fig:mnist_s_experiments} we experiment with the value of $b_0$. Clearly, $-\frac{1}{2}$ works best, having the highest accuracy and $\mpupdateratio$.



\subsection{Analysis}

We plot the accuracy of the different configurations with respect to $\mpcorupdate$ and $\mtcorupdate$ in Figures \ref{fig:cucons_accuracy} and \ref{fig:cuant_accuracy}. 
The blue dots represent runs on the \pred{same} problem, while the red dots represent runs on the \pred{sum9} problem. Figure \ref{fig:cuant_accuracy} shows a positive correlation, suggesting that it is vital to the learning process that updates going into the antecedent are correct. 
Although there seems to be a slight positive correlation in Figure \ref{fig:cucons_accuracy} for the \pred{same} problem, it is not as pronounced. Furthermore, it seems that for the \pred{sum9}, this correlation is negative instead, as the configurations with the highest accuracy  have low values of $\mpcorupdate$. 

\begin{figure}
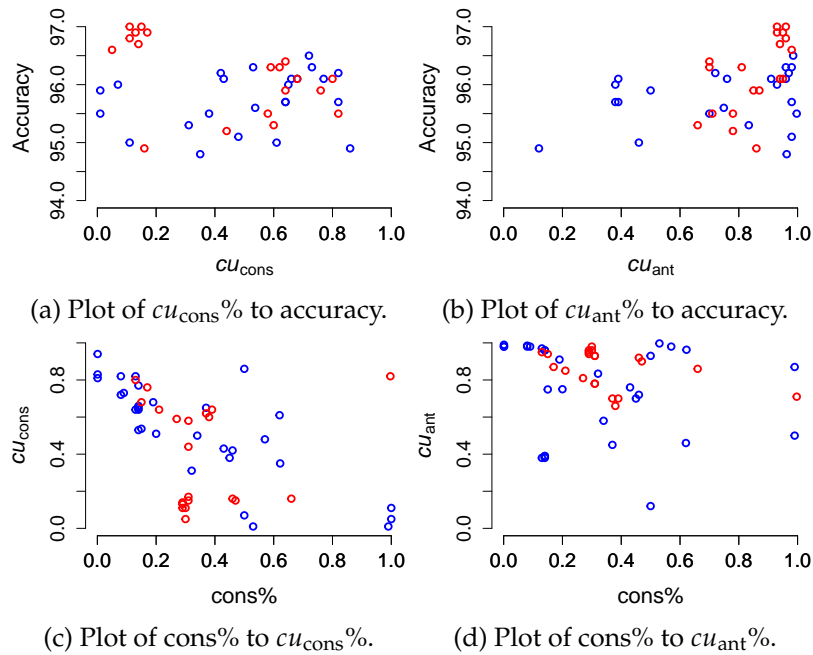

\centering
\begin{subfigure}[b]{\graphwidth}
\includegraphics[width=\linewidth]{\imgE cucons_acc.pdf}
\caption{Plot of $\mpupdateratio$ to accuracy.}
\label{fig:cucons_accuracy}
\end{subfigure}
\begin{subfigure}[b]{\graphwidth}
\includegraphics[width=\linewidth]{\imgE cuant_acc.pdf}
\caption{Plot of $\mtupdateratio$ to accuracy.}
\label{fig:cuant_accuracy}
\end{subfigure}%
\\
\begin{subfigure}[b]{\graphwidth}
\includegraphics[width=\linewidth]{\imgE consp_cucons.pdf}
\caption{Plot of $\mpratio$ to $\mpupdateratio$.}
\label{fig:mpmt_cucons}
\end{subfigure}
\begin{subfigure}[b]{\graphwidth}
\includegraphics[width=\linewidth]{\imgE consp_cuant.pdf}
\caption{Plot of $\mpratio$ to $\mtupdateratio$.}
\label{fig:mpmt_cuant}
\end{subfigure}
\caption[Several of the analytical measures plotted against each other]{We plot several of the analytical measures to find their relations. Blue dots represent runs on the \pred{same} problem while red dots represent runs on the \pred{sum9} problem.}
\label{fig:analyze_measures}
\end{figure}

We plot all experimental values of $\mpratio$ to the values of $\mpupdateratio$ and $\mtupdateratio$ in Figures \ref{fig:mpmt_cucons} and \ref{fig:mpmt_cuant}. 
For both, there seems to be a negative correlation. 
Apparently, a larger consequent ratio decreases the correctness of the updates. 
In \ref{sec:mnist_rcsigmoidal} we find, when experimenting with the value of $s$, that this could be because for lower values of $\mpratio$, a smaller portion of the reasoning happens in the corners around $a=0,\ c=0$ and $a=1,\ c=1$, and more for instances that the agent is less certain about. 
Since all S-implications have strong derivatives at both these corners (Proposition \ref{prop:diff_left_neutral}), this phenomenon is likely present in other S-implications.

This all suggests we need to properly balance the contribution of updates to the antecedent and consequent. 
Since usually, as reasoned in Section \ref{sec:implication_challenges}, derivatives with respect to the antecedent are more common, this balance should be reflected in the experimental ratio between these updates. 

\subsection{Conclusions}
We have run experiments on many configurations of hyperparameters to explore what works and what does not. 
The only well performing fully symmetric option is the product t-norm with the Reichenbach implication. 
If we are willing to forego symmetry, we find that the choice of the aggregators is the most important factor for performance. 
For universal aggregation, we recommend the log-product aggregator, while for existential quantification we recommend the generalized mean with a value of $p$ of somewhere between 1 and 2. 
We found that it is especially important to choose aggregation operators that do not vanish on a large part of their domain, and that are not single-passing. 
In our experiments, a well tuned sigmoidal-Reichenbach implication coupled with vanilla SGD proved to be the most effective fuzzy implication.
In general, we recommend choosing S-implications above R-implications.
For conjunction and disjunction, we recommend tuning the Yager t-norm, although this value can be dependent on the complexity of the formulas to prevent the derivative from vanishing during the whole run. 

Although \dfuzz significantly improves on the supervised baseline and is thus suited for semi-supervised learning, it is not currently competitive with other methods like Ladder Networks \pcite{rasmusSemisupervisedLearningLadder2015}, which has an accuracy of 98.9\% for 100 labeled pictures and 99.2\% for 1000. 




\section{Related Work}
\label{chapter:related_work}
\dfuzz falls into the discipline of Statistical Relational Learning \pcite{getoorIntroductionStatisticalRelational2007}, which concerns models that can reason under uncertainty and learn relational structures like graphs. 

\subsection{\dfuzz}
\label{sec:rel-dfuzz}
Special cases of \dfl have been researched in several papers under different names. Logic Tensor Networks (LTN) \pcite{badreddineLogicTensorNetworks2022,serafiniLogicTensorNetworks2016} implements function symbols and uses neural model to interpret predicates. LTN is applied to weakly supervised learning on Scene Graph Parsing \pcite{donadelloLogicTensorNetworks2017} and transfer learning in Reinforcement Learning \pcite{badreddineInjectingPriorKnowledge2019}.

Semantic-based regularization (SBR) \pcite{diligentiSemanticbasedRegularizationLearning2017} applies \dfl to kernel machines. They use R-implications and the mean aggregator. \tcite{senCollectiveClassificationNetwork2008} applies SBR to collective classification by predicting using a trained deep learning model, and then optimizes the \dfl loss to find new truth values. This ensures predictions are consistent with the formulas during test-time. 

\tcite{marraLearningTnormsTheory2019} 
uses t-norm Fuzzy Logics, where the R-implication is used alongside weak disjunction. By using t-norms based on generator functions, the satisfiability computation can be simplified and generalizations of common loss functions can be found. \tcite{marraConstraintbasedVisualGeneration2018} applies \dfl to image generation. It uses the product t-norm, the log-product aggregator and the Goguen implication. By using function symbols that represent generator neural networks, they create constraints that are used to create a semantic description of an image generation problem. \tcite{rocktaschelInjectingLogicalBackground2015} uses the product t-norm and Reichenbach implication for relation extraction by using an efficient matrix embedding of the rules. \tcite{guoJointlyEmbeddingKnowledge2016} extends this to link prediction and triple classification by using a margin-based ranking loss for implications. 

\tcite{demeesterLiftedRuleInjection2016} uses a regularization technique equivalent to the \luk\ implication. Instead of using existing data, it finds a loss function which does not iterate over objects, yet can guarantee that the rules hold. This is very scalable, but can only model simple implications. A promising approach is using adversarial sets \pcite{minerviniAdversarialSetsRegularising2017}, which is a set of objects from the domain that do not satisfy the knowledge base. These are probably the most informative objects. It uses gradient descent to find objects that \textit{minimize} the satisfiability. The parameters of the deep learning model are then updated so that it predicts consistent with the knowledge base on this adversarial set. A benefit of this approach is that it does not have to iterate over instances that already satisfy the constraints. Adversarial sets are applied to natural language interpretation in \pcite{minerviniAdversariallyRegularisingNeural2018}. Both papers use the \luk\ implication and Gödel t-norm and t-conorm. They are not able to infer new labels on existing unlabeled data as they use artificial data, but these methods are not orthogonal and can be used jointly. 


\subsection{Neurosymbolic methods using Fuzzy Logic Operators}
Posterior regularization \pcite{ganchevPosteriorRegularizationStructured2010,huHarnessingDeepNeural2016} is a framework for weakly-supervised learning on structured data. 
It projects the output of a deep learning model to a `rule-regularized subspace' to make it consistent with the knowledge base. This output is used as a label for the deep learning model to imitate. Unlike this paper, it does not compute derivatives over the computation of the satisfaction of the knowledge base. \tcite{marraIntegratingLearningReasoning2019} and \tcite{danieleKnowledgeEnhancedNeural2019} instead use gradient descent for the projection. Therefore, unlike earlier methods for posterior regularization, derivatives with respect to the operators are used. They learn relative formula weights jointly with the parameters of the deep learning model. 

\tcite{arakelyanComplexQueryAnswering2021} uses t-norms, t-conorms and existential quantification to answer queries by finding what entity embedding has the highest truth value of a given query. This search is done using gradient descent. 
By comparing what entity embedding best fits the optimized entity embedding, the authors can answer complex FOL queries. The authors either use the product or the Gödel t-norms. 

Another recent work which employs fuzzy logic operators in a neurosymbolic setting is Logical Neural Networks \pcite{riegelLogicalNeuralNetworks2020}. 
This work stands orthogonal to our work, as the foremost distinction is that they employ logics on the low-level (i.e., logical connectives as neurons and neural activation functions) while we employ it on the higher level (i.e., in defining the loss function). 
They limit their work to the propositional level for simplification purposes, although they argue that extending it to relational level is straightforward.

$\partial$ILP \pcite{evansLearningExplanatoryRules2018} is a differentiable inductive logic programming that uses the product t-norm and t-conorm to do differentiable inference. The Neural Theorem Prover \pcite{rocktaschelEndtoendDifferentiableProving2017} does differentiable proving of queries and combines different proof paths using the Gödel t-norm and t-conorm.  \tcite{sourekLiftedRelationalNeural2015} also introduces a method for differentiable query proving, with learnable weights for formulas. They use operators inspired by fuzzy logic and transformed by the sigmoid function. 

There is a vast literature on Fuzzy Neural Networks \pcite{jangANFISAdaptivenetworkbasedFuzzy1993,jangNeurofuzzySoftComputing1997,linNeuralnetworkbasedFuzzyLogic1991} that replace standard neural network neurons with neurons based on fuzzy logic. Some neurons use fuzzy logic operators which are differentiated through if the networks are trained using backpropagation. 

\subsection{\dprob}
\label{sec:related_probabilistic}
Some approaches use probabilistic logics instead of fuzzy logics and interpret predicates probabilistically. 
As deep learning classifiers can model probability distributions, probabilistic logics could be a more natural choice than fuzzy logics. 
DeepProbLog \pcite{manhaeveDeepProbLogNeuralProbabilistic2018,manhaeveNeuralProbabilisticLogic2021} is a probabilistic logic programming language with neural predicates that compute the probabilities of ground atoms. 
It supports automatic differentiation which can be used to back-propagate from the loss at a query predicate to the deep learning models that implement the neural predicates, similar to \dfl. 
It also supports probabilistic rules which can handle exceptions to rules.
We compare another differentiable probabilistic logic called Semantic Loss \pcite{xuSemanticLossFunction2018} in Appendix \ref{appendix:prl-sl} and show similarities between it and \dfl using operators based on the product t-norm. This similarity suggests that many practical problems that \dpfl has are also present in Semantic Loss, as shown empirically in follow-up work \pcite{heReducedImplicationbiasLogic2022}. %
As inference is exponential in the size of the grounding for probabilistic logics, both approaches use an advanced compilation technique \pcite{darwicheSDDNewCanonical2011} to make inference feasible for larger problems. Recent work \pcite{ahmedSemanticStrengtheningNeurosymbolic2023} has studied interpolating between \dfl with the product t-norm and Semantic Loss.

\label{chapter:conclusions}

\section{Discussion}

This paper presented theoretical results of \dfuzz operators and then evaluated their behavior on semi-supervised learning.
We now discuss some problems with deploying solutions using \dfl.

\dfl can be seen as a form of multi-objective optimization \pcite{hwangMultipleObjectiveDecision2012}. In the \dfl loss (Equation \ref{eq:lossrl}) we sum up the valuations of different formulas, each of which is a separate objective. 
Each of these objectives can be weighted differently, resulting in wildly varying loss landscapes. 
Having so many objectives requires significant hyperparameter tuning. A method capable of learning relative formula weights jointly like \pcite{marraIntegratingLearningReasoning2019, danieleKnowledgeEnhancedNeural2019, sourekLiftedRelationalNeural2018}, could solve this problem.

A second challenge is related to the class imbalance problem \pcite{japkowiczClassImbalanceProblem2002, budaSystematicStudyClass2018a}. We argued in Section \ref{sec:implication_challenges} that for a significant portion of common-sense background knowledge, the modus tollens case is by far the most common. 
Our \pred{same} problem indeed showed that most well-performing implications have a far larger derivative with respect to the negated antecedent than to the consequent. 
This imbalance will only increase for more complex problems. 
However, simply removing derivatives with respect to the antecedent does not seem to be the solution. A reason for this could be that those are usually correct, unlike derivatives with respect to the consequent. In fact, we found in \ref{sec:mnist_formulas_experiments} that the formula in which the digits are in the antecedent performs better on its own than the formula in which the digits are in the consequent, even though the model could not learn from any new positive examples. 

Although we have focused on experimenting with the accuracy of the derivatives of the implication, it should be noted that the derivatives of the disjunction operator make a choice as well. For example, if the agent observes a walking object and the supervisor knows that only humans and animals can walk, how is the supervisor supposed to choose whether it is a human or an animal? Here, similar imbalances exist in the different possible classes: There might be more images of humans than of animals.

Further, we pose whether it is more important that we choose operators based on the performance on the task at hand, or based on its logical properties. 
The best configuration uses operators based on both the product and Yager t-norms. 
The product t-norm is the only viable symmetric choice in our experiments. 
The largest benefit of a `symmetric' choice of operators is that the truth value of formulas that are logically equivalent in classical logic will be equal. 
This makes it easier to analyze how the background knowledge will behave and does not require putting it in a particular form.  

As a final remark, noteworthy is the interpretation of truth values. As aforementioned, the logic we use is fuzzy logic which was originally aimed to address logical reasoning in the presence of vagueness rather than probabilistic uncertainty. The truth values derived using fuzzy operators, therefore, are not probabilistic (see, for instance, \cite[p.~4]{hajekMetamathematicsFuzzyLogic1998})
\footnote{Indeed, when reasoning about belief, using fuzzy logic semantics instead of probabilistic logic semantics straight out-of-the-box, can yield undesirable results: Consider an event $a$ where $p(a)$ (probability of $a$) is 0.5. Now consider a disjunction, where $p(a \vee a)$ has the value $0.5$. However, in \luk{} logic, $S(a, a)$ would yield 1. }.

However, since a considerably large amount of problems addressed by machine learning literature is probabilistic (as it has mathematical origins in statistics), the classification task used in our running example is also of probabilistic origin. 
With this choice we also aimed to respect the recent literature: Applications of fuzzy operators on a general set of problems which are not necessarily fuzzy is not uncommon in neurosymbolic AI. 
Examples include \pcite{serafiniLogicTensorNetworks2016}, \pcite{riegelLogicalNeuralNetworks2020}, \cite{arakelyanComplexQueryAnswering2021}, among others cited in Section \ref{chapter:related_work}.

\section{Conclusion}
We analyzed \dfuzz in order to understand how reasoning using logical formulas behaves in a differentiable setting. We examined how the properties of a large amount of different operators affect \dfl. 
We have found substantial differences between the properties of a large number of such \dfuzz operators, and we showed that many of them, including some of the most popular operators, are highly unsuitable for use in a differentiable learning setting. 
By analyzing aggregation functions, we found that the log-product aggregator and the RMSE aggregator have convenient connections to both fuzzy logic and machine learning and can deal with outliers. 
Next, we analyzed conjunction and disjunction operators and found several strong candidates. 
In particular, the Gödel t-norm and t-conorm are a simple choice, and that the Yager t-norm and the product t-conorm have intuitive derivatives.

We noted an interesting imbalance between derivatives with respect to the negated antecedent and the consequent of the implication. Because the modus tollens case is much more common, we conclude that a large part of the useful inferences on the MNIST experiments are made by decreasing the antecedent, or by `modus tollens reasoning'. Furthermore, we found that derivatives with respect to the consequent often increase the truth value of something that is false as the consequent is false in the majority of times. Therefore, we argue that `modus tollens reasoning' should be embraced in future research. As a possible solution to problems caused by this imbalance, we introduced a smoothed fuzzy implication called the Reichenbach-sigmoidal implication. 

Experimentally, we found that the product t-norm is the only t-norm that can be used as a base for all choices of operators. The product t-conorm and the Reichenbach implication have intuitive derivatives  that correspond to inference rules from classical logic, and the log-product aggregator is the most effective universal aggregation operator. 

In order to gain the largest improvements over a supervised baseline however, we had to abandon the normal symmetric configurations of norms, where t-norms, t-conorms, implications and the aggregation operators satisfy the usual algebraic relations. Instead, we had to resort to non-symmetric configurations where operators based on different t-norms are combined. 
The Reichenbach-sigmoidal implication performs best in our experiments. 
Its hyperparameters can be tweaked to decrease the imbalance of the derivatives with respect to the negated antecedent and consequent. 
For existential quantification, we found that the general mean error performs best, and for conjunction and disjunction the family of Yager t-norms and the Nilpotent minimum has the highest final accuracy. 

We believe a proper empirical comparison of different methods that introduce background knowledge through logic could be useful to properly understand the details, performance, possible applications and challenges of each method. 
Secondly, we believe more work is required in using background knowledge to help deep models train on real-world problems. One research direction would be to develop methods that can properly deal with exceptions. An approach in which formula importance weights can be learned could be used to distinguish between relevant and irrelevant formulas in the background knowledge, and probabilistic instead of fuzzy logics could be a more natural fit. 
Lastly, additional research on the vast space of fuzzy logic operators might find more properties that are useful in \dfl.

\newpage

\chapter[Iterative Local Refinement]{Refining Neural Network Predictions using Background Knowledge}
\label{ch:lrl}
\begin{paperbase}
	This chapter is based on the Machine Learning Journal article \cite{danieleRefiningNeuralNetwork2023a}. Alessandro Daniele and Emile van Krieken are shared first authors of this paper and agreed on the inclusion of this article in this dissertation. 
\end{paperbase}

\begin{abstract}
Recent work has shown learning systems can use logical background knowledge to compensate for a lack of labeled training data. 
Many methods work by creating a loss function that encodes this knowledge. However, often the logic is discarded after training, even if it is still helpful at test time. Instead, we ensure neural network predictions satisfy the knowledge by refining the predictions with an extra computation step. We introduce differentiable \emph{\boost\ functions} that find a corrected prediction close to the original prediction. 
We study how to effectively and efficiently compute these \boost\ functions. Using a new algorithm called Iterative Local Refinement (ILR), we combine \boost\ functions to find \boosted\ predictions for logical formulas of any complexity. ILR finds \boost s on complex SAT formulas in significantly fewer iterations and frequently finds solutions where gradient descent can not. Finally, ILR produces competitive results in the MNIST addition task.
\end{abstract}

\section{Introduction}
Recent years have shown promising examples of using symbolic background knowledge in learning systems: From training classifiers with weak supervision signals \citep{manhaeveDeepProbLogNeuralProbabilistic2018}, generalizing learned classifiers to new tasks \citep{roychowdhuryRegularizingDeepNetworks2021}, compensating for a lack of good supervised data \citep{diligentiSemanticbasedRegularizationLearning2017,donadelloLogicTensorNetworks2017}, to enforcing the structure of outputs through a logical specification \citep{xuSemanticLossFunction2018}. The main idea underlying these integrations of learning and reasoning, often called neurosymbolic integration, is that background knowledge can complement the neural network when one lacks high-quality labeled data \citep{giunchigliaDeepLearningLogical2022}. Although pure deep learning approaches excel when learning over \emph{vast} quantities of data with \emph{gigantic} amounts of compute \citep{chowdheryPaLMScalingLanguage2022,rameshHierarchicalTextConditionalImage2022}, we cannot afford this luxury for most tasks. 

Many neurosymbolic methods, such as the one discussed in Chapter \ref{ch:dfl}, work by creating a differentiable loss function that encodes the background knowledge (Figure \ref{fig:LRL-comparison}a). However, often the logic is discarded after training, even though this background knowledge could still be helpful at test time \citep{roychowdhuryRegularizingDeepNetworks2021,giunchigliaROADRAutonomousDriving2022}. Instead, we ensure we constrain the neural network with the background knowledge, both during train time and test time, by correcting its output to satisfy the background knowledge (Figure \ref{fig:LRL-comparison}b). In particular, we consider how to make such corrections while being as close as possible to the original predictions of the neural network.

We study how to effectively and efficiently correct the neural network by ensuring its predictions satisfy the symbolic background knowledge. In particular, we consider fuzzy logics formed using functions called t-norms \citep{klementTriangularNorms2000,rossFuzzyLogicEngineering2010}. Prior work has shown how to use a gradient ascent-based optimization procedure to find a prediction that satisfies this fuzzy background knowledge \citep{diligentiSemanticbasedRegularizationLearning2017,roychowdhuryRegularizingDeepNetworks2021}. However, a recent model called KENN \citep{danieleKnowledgeEnhancedNeural2019} shows how to compute the correction analytically for a fragment of the G\"{o}del logic. 

To extend this line of work, we introduce the concept of \emph{\boost\ functions} and derive \boost\ functions for many fuzzy logic operators. \Boost\ functions are functions that find a prediction that satisfies the background knowledge while staying close to the neural network's original prediction. Using a new algorithm called \emph{Iterative Local Refinement} (ILR), we can combine \boost\ functions for different fuzzy logic operators to efficiently find \boost s for logical formulas of any complexity. Since \boost\ functions are differentiable, we can easily integrate them as a neural network layer. In our experiments, we compare ILR with an approach using gradient ascent. We find that ILR finds optimal \boost s in significantly fewer iterations. Moreover, ILR often produces results that stay closer to the original predictions or better satisfy the background knowledge.
Finally, we evaluate ILR on the MNIST Addition task~\citep{manhaeveDeepProbLogNeuralProbabilistic2018} and show how to combine ILR with neural networks to solve neurosymbolic tasks.

In summary, our contributions are:
\begin{enumerate}
\item We formalize the concept of minimal \boost\ functions in Section \ref{sec:minimal-boost-function}.
\item We introduce the ILR algorithm in Section \ref{sec:ILR}, which uses the minimal \boost\ functions for individual fuzzy operators to find \boost s for general logical formulas.
\item We discuss how to use ILR for neurosymbolic AI in Section \ref{sec:neuro-symbolic}, where we exploit the fact that ILR is a differentiable algorithm.
\item We analytically derive minimal \boost\ functions for individual fuzzy operators constructed from the G\"{o}del, \luk, and product t-norms in Section \ref{sec:basic-t-norm}. 
\item We discuss a large class of t-norms for which we can analytically derive minimal \boost\ functions in Section \ref{sec:general-analysis}.
\item We compare ILR to gradient descent approaches and show it finds \boost s on complex SAT formulas in significantly fewer iterations and frequently finds solutions where gradient descent can not. 
\item We apply ILR to the MNIST Addition task~\citep{manhaeveDeepProbLogNeuralProbabilistic2018} to test how ILR behaves when injecting knowledge into neural network models.
\end{enumerate}

\label{sec:lrl}
\begin{figure}
    \includegraphics[width=\linewidth]{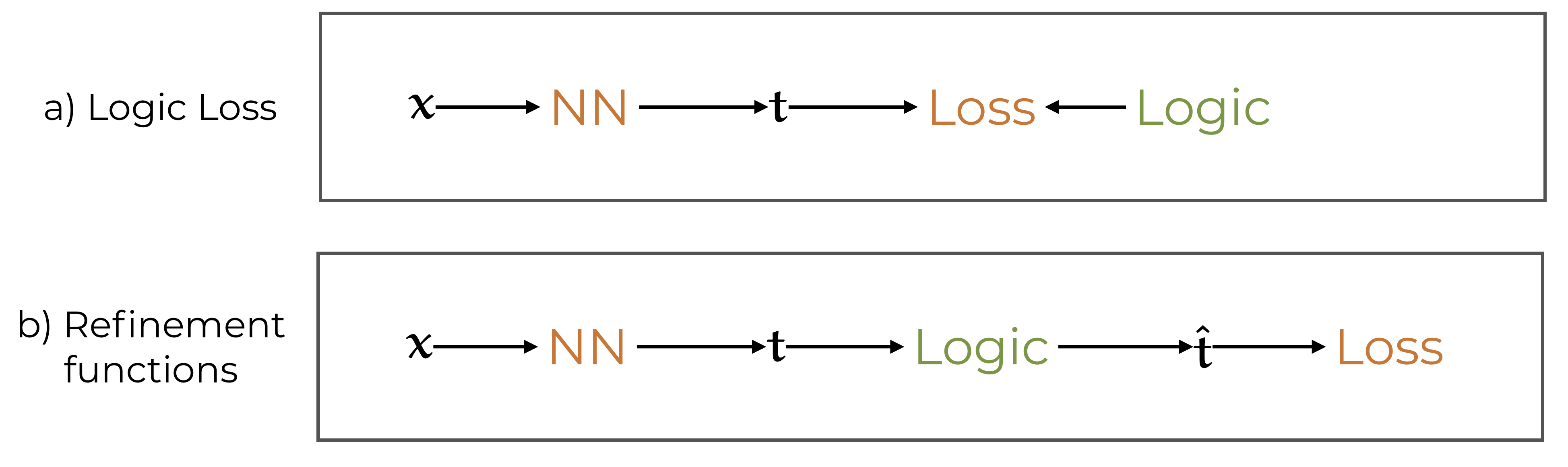}
    \caption[Comparing different approaches for constraining neural networks with background knowledge]{Comparing different approaches for constraining neural networks with background knowledge. Loss-based approaches include LTN, SBR, and Semantic Loss, while KENN, CCN(h), and SBR-CC are representatives for \boost\ functions. $\bx$ represents a high-dimensional input for a neural network, $\truth$ represents the initial predictions of this neural network and $\boostv$ represents the refined prediction that incorporates the background knowledge.}
    \label{fig:LRL-comparison}
\end{figure}

\section{Related work}
ILR falls into a larger body of work that attempts to integrate background knowledge expressed as logical formulas into neural networks. For an overview, see \cite{giunchigliaDeepLearningLogical2022}. Figure \ref{fig:LRL-comparison} shows two categories that most methods fall in. The first only use background knowledge during training in the form of a loss function~\citep{badreddineLogicTensorNetworks2022,xuSemanticLossFunction2018,diligentiSemanticbasedRegularizationLearning2017,fischerDL2TrainingQuerying2019,yangInjectingLogicalConstraints2022,vankriekenAnalyzingDifferentiableFuzzy2022}. The second considers the background knowledge as part of the model and enforces the knowledge at test time~\citep{danieleKnowledgeEnhancedNeural2019,wangSATNetBridgingDeep2019,giunchigliaMultiLabelClassificationNeural2021a,ahmedSemanticProbabilisticLayers2022,hoernleMultiplexNetFullySatisfied2022,dragoneNeuroSymbolicConstraintProgramming2021}. ILR is a method in the second category. We note that these approaches can be combined~\citep{giunchigliaROADRAutonomousDriving2022,roychowdhuryRegularizingDeepNetworks2021}.

First, we discuss approaches that construct loss functions from the logical formulas (Figure \ref{fig:LRL-comparison}a). These loss functions measure when the deep learning model violates the background knowledge, such that minimizing the loss function amounts to \say{correcting} such violations \citep{vankriekenAnalyzingDifferentiableFuzzy2022}. 
While these methods show significant empirical improvement, they do not guarantee that the neural network will satisfy the formulas outside the training data. 
LTN and SBR \citep{badreddineLogicTensorNetworks2022,diligentiSemanticbasedRegularizationLearning2017} use fuzzy logic to provide compatibility with neural network learning, while Semantic Loss \citep{xuSemanticLossFunction2018} uses probabilistic logics. It is possible to extend the formalization of \boost\ functions to probabilistic logics by defining a suitable notion of minimality. One example is the KL-divergence between the original and \boosted\ distributions over ground atoms. 

Among the methods where knowledge is part of the model, KENN inspired ILR \citep{danieleKnowledgeEnhancedNeural2019,KENN_rel}. KENN is a framework that injects knowledge into neural networks by iteratively refining its predictions. It uses 
a relaxed version of the G\"{o}del t-conorm obtained through a relaxation of the argmax function, which it applies 
in logit space. 
Closely related to both ILR and KENN is CCN(h) \citep{giunchigliaMultiLabelClassificationNeural2021a}, which we see as computing the minimal \boost\ function for stratified normal logic programs under G\"{o}del t-norm semantics. We discuss this connection in more detail in Section \ref{sec:godel-t-norm}.

The loss-function-based method SBR also introduces a procedure for using the logical formulas at test time in the context of collective classification~\citep{diligentiSemanticbasedRegularizationLearning2017,roychowdhuryRegularizingDeepNetworks2021}. Unlike KENN \citep{danieleKnowledgeEnhancedNeural2019}, these approaches do not enforce the background knowledge during training but only use it as a test time procedure. In particular, \cite{roychowdhuryRegularizingDeepNetworks2021} shows that doing these corrections at test time improves upon just using the loss-function approach. Unlike our analytic approach to \boost\ functions, SBR finds new predictions using a gradient descent procedure very similar to the algorithm we discuss in Section \ref{sec:gradient-descent}. We show it is much slower to compute than ILR.

Another method closely related to ILR is the neural network layer SATNet \citep{wangSATNetBridgingDeep2019}, which has a setup closely related to ours. However, SATNet does not have a notion of minimality and uses a different underlying logic constructed from a semidefinite relaxation. DeepProbLog \citep{manhaeveDeepProbLogNeuralProbabilistic2018} also is a probabilistic logic, but unlike Semantic Loss is used to derive new statements through proofs and cannot directly be used to correct the neural network on predictions that do not satisfy the background knowledge. Instead, ILR can be used to inject constraints on the output of a neural network, and to prove new statements starting from the neural network predictions.

Finally, some methods are limited to equality and inequality constraints rather than general symbolic background knowledge \citep{fischerDL2TrainingQuerying2019,hoernleMultiplexNetFullySatisfied2022}. DL2 \citep{fischerDL2TrainingQuerying2019} combines these constraints into a real-valued loss function, while MultiplexNet \citep{hoernleMultiplexNetFullySatisfied2022} adds the knowledge as part of the model. However, MultiplexNet requires expressing the logical formulas as a DNF formula, which is hard to scale.

\section{Fuzzy evaluation operators}
Logical formulas $\varphi$ can be evaluated using compositions of fuzzy operators. We assume $\varphi$ is a propositional logic formula, but we note the evaluation procedure can be extended to grounded first-order logical formulas on finite domains. For instance, \cite{KENN_rel} introduced a technique for propositionalizing universally quantified formulas of predicate logic in the context of KENN. Moreover, this technique can be extended to existential quantification by treating it as a disjunction.  We assume a set of (free) propositions $\predicates= \{P_1, ..., P_n\}$ and constants $\constants = \{C_1, ..., C_m\}$, where each constant is a proposition with a fixed truth value $C_i\in [0, 1]$.

\begin{definition}
\label{def:evaluation}
If $T$ is a t-norm, $S$ a t-conorm and $I$ a fuzzy implication, then the \emph{fuzzy evaluation operator} $f_\varphi:[0, 1]^n\rightarrow [0,1]$ of the formula $\varphi$ with propositions $\predicates$ and constants $\constants$ is a function of truth vectors $\truth$ and given as
\begin{align}
    \op_{P_i}(\truth) &= \truths_i\\
    \op_{C_j}(\truth) &= C_j \\
    \op_{\neg \phi}(\truth) &= 1-f_{\phi}(\truth) \\
    \op_{\bigwedge_{j=1}^m \phi_j}(\truth) &= T(\op_{\phi_1}(\truth), ..., \op_{\phi_m}(\truth)) \\
    \op_{\bigvee_{j=1}^m \phi_j}(\truth) &= S(f_{\phi_1}(\truth), ..., f_{\phi_m}(\truth)) \\
    \op_{\phi\rightarrow \psi}(\truth) &= I(f_\phi(\truth), f_\psi(\truth)),
\end{align}
where we match the structure of the formula $\varphi$ in the subscript $f_\varphi$. 
\end{definition}

\section{Minimal Fuzzy \Boost\ Functions}
\label{sec:minimal-boost-function}
We will next define (fuzzy) \boost\ functions, which consider how to change the input arguments of fuzzy operators such that the output of the operators is a given truth value. 
\boost\ functions prefer changes to the input arguments that are as small as possible. 
We will introduce several definitions to facilitate studying this concept. The first is an optimality criterion. 

\begin{definition}[Fuzzy \boost\ function]
    Let $\op_\varphi: [0, 1]^n\rightarrow [0,1]$ be a fuzzy evaluation operator. Then $\boostv: [0, 1]^n$ is called a \emph{\boosted\ (truth) vector} for the \emph{\boost\ value} $\revis{\varphi}\in[0, 1]$ if $\op_\varphi(\boostv) = \revis{\varphi}$.

    Furthermore, let $\revismin{\varphi}=\min_{\boostv\in[0, 1]^n} \op_\varphi(\boostv)$ and $\revismax{\varphi}=\max_{\boostv \in [0, 1]^n} \op_\varphi(\boostv)$. 
    Then $\boostf: [0, 1]^{n}\times[0, 1]\rightarrow [0, 1]^n$ is a \emph{(fuzzy) \boost\ function}\footnote{The concept of \boost\ functions is closely related to the concept of \emph{Fuzzy boost function} in the KENN paper \citep{danieleKnowledgeEnhancedNeural2019}.}
    for $\op_\varphi$ if for all $\truth \in [0, 1]^n$, 
    \begin{enumerate}
    \item for all $\revis{\varphi}\in [\revismin{\varphi}, \revismax{\varphi}]$, $\boostf(\truth, \revis{\varphi})$ is a \boosted\ vector for $\revis{\varphi}$;
    \item for all $\revis{\varphi} < \revismin{\varphi}$, $\boostf(\truth, \revis{\varphi})=\boostf(\truth, \revismin{\varphi})$;
    \item for all  $\revis{\varphi} > \revismax{\varphi}$, $\boostf(\truth, \revis{\varphi})=\boostf(\truth, \revismax{\varphi})$. 
    \end{enumerate}
    \end{definition}




A \boost\ function for $f_\varphi$ changes the input truth vector in such a way that the new output of $f_\varphi$ will be $\revis{\varphi}$. 
Whenever $\revis{\varphi}$ is high, we want the \boosted\ vector to satisfy the formula $\varphi$, while if $\revis{\varphi}$ is low, we want it to satisfy its negation. When $\revis{\varphi}=1$, the constraint created by the formula is a hard constraint, while if it is in $(0, 1)$, this constraint is soft. 
We require bounding the set of possible $\revis{\varphi}$ by $\revismin{\varphi}$ and $\revismax{\varphi}$ since if there are constants $C_i$, or if $\varphi$ has no satisfying (discrete) solutions, there can be formulas such that there can be no \boosted\ vectors $\boostv$ for which $f_\varphi(\boostv)$ equals 1. 

Next, we introduce a notion of minimality of \boost\ functions. The intuition behind this concept is that we prefer the new output, the \boosted\ vector $\boostv$, to stay as close as possible to the original truth vector $\truth$. Therefore, we assume we want to find a truth vector near the neural network's output that satisfies the background knowledge.
\begin{definition}[Minimal \boost\ function]
    Let $\minboost$ be a \boost\ function for operator $\op_\varphi$. $\minboost$ is a \emph{minimal} \boost\ function with respect to some norm $\|\cdot\|$ if for each $\truth\in[0, 1]^n$ and $\revis{\varphi}\in [\revismin{\varphi}, \revismax{\varphi}]$, there is no \boosted\ vector $\boostv'$ for $\revis{\varphi}$  such that $\|\minboost(\truth, \revis{\varphi}) - \truth\| > \|\boostv' - \truth\|$.
\end{definition}


For a particular fuzzy evaluation operator $\op_\varphi$, finding the minimal \boost\ function corresponds to solving the following optimization problem:

\begin{equation}
    \label{eq:optim-problem}
    \begin{aligned}
        \textrm{For all } \quad & \truth\in [0,1]^n, \revis{\varphi} \in [\revismin{\varphi}, \revismax{\varphi}]  \\
        \min_{\boostv} \quad & \|\boostv - \truth\|  \\
        \textrm{such that } \quad & \op_\varphi(\boostv) = \revis{\varphi},  \\
        & 0 \leq  \boosts_i \leq 1
    \end{aligned}
\end{equation}

For some $f_\varphi$ we can solve this problem analytically using the Karush-Kuhn-Tucker (KKT) conditions. However, while $\|\cdot\|$ is convex, $f_\varphi$ (usually) is not. Therefore, we can not rely on efficient convex solvers. Furthermore, for strict t-norms, finding exact solutions to this problem is equivalent to solving PMaxSAT when $\revis{\varphi}=1$ \citep{diligentiSemanticbasedRegularizationLearning2017,giunchigliaROADRAutonomousDriving2022}, hence this problem is NP-complete. In Sections \ref{sec:general-analysis} and \ref{sec:class-analysis}, we will derive minimal \boost\ functions for a large amount of individual fuzzy operators analytically. These results are the theoretical contribution of this chapter. We first discuss in Section \ref{sec:ILR} a method called ILR for finding general solutions to the problem of finding minimal \boost\ functions. ILR uses the analytical minimal \boost\ functions of individual fuzzy operators in a forward-backward algorithm. Then, in Section \ref{sec:neuro-symbolic}, we discuss how to use this algorithm for neurosymbolic AI.

\section{Iterative Local Refinement}
\label{sec:ILR}
\begin{figure}
    \includegraphics[width=\linewidth]{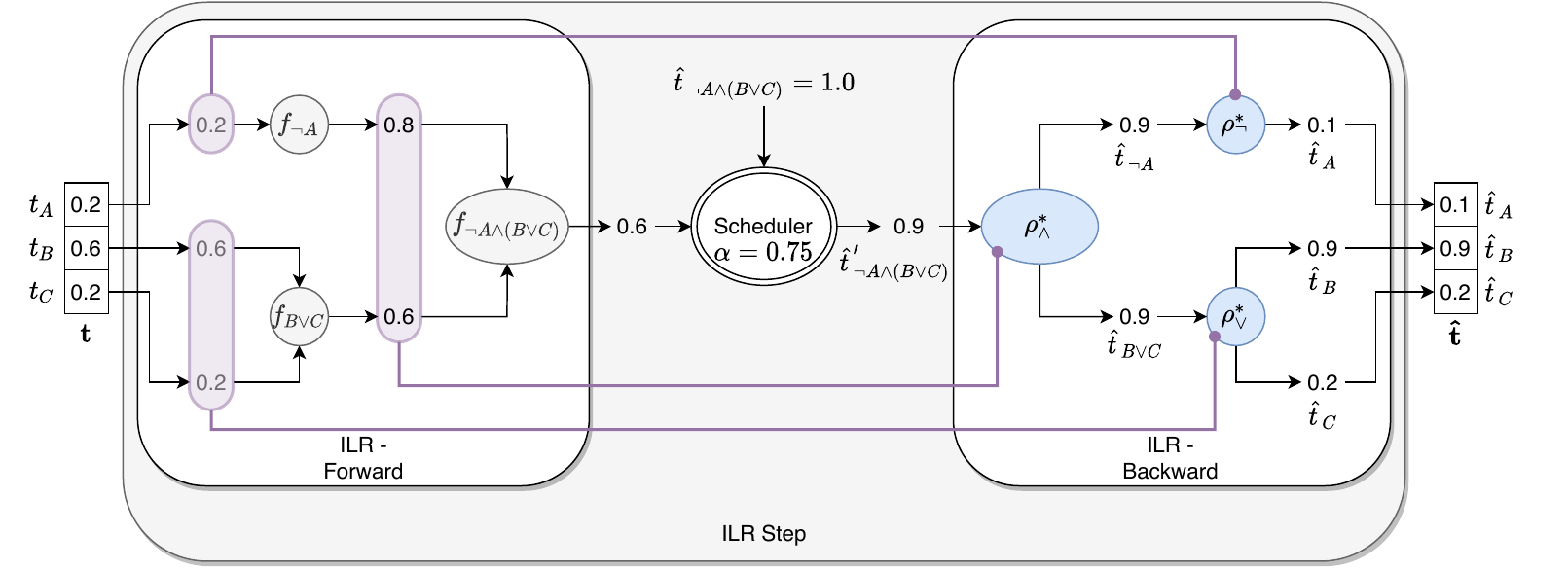}
    \caption[Visualization of one step of ILR for the G\"{o}del logic and formula $\phi = \lnot A \land (B \lor C)$]{Visualization of one step of ILR for the G\"{o}del logic and formula $\phi = \lnot A \land (B \lor C)$. In the forward pass (left), ILR computes the truth value of $\phi$. In the backward pass (right), ILR traverses the computational graph of the forward step in reverse to calculate the \boosted\ vector $\boostv$. ILR substitutes each fuzzy operator of the forward pass with the corresponding \boost\ function. Each \boost\ function receives as input the initial truth values used by the fuzzy operator in the forward step (purple lines) and the target value for the corresponding subformula. The scheduler calculates the target value $\hat{t}'_{\lnot A \land (B \lor C)}$ for the entire formula, which ILR calls between the forward and backward steps.}
    \label{fig:ILR}
\end{figure}

We introduce a fast, iterative, differentiable but approximate algorithm called \emph{Iterative Local Refinement (ILR)} that finds minimal \boost\ functions for general formulas. ILR is a forward-backward algorithm acting on the computation graph of formulas. First, it traverses the graph from its leaves to its root to compute the current truth values of subformulas. Then, it traverses the graph back from its root to the leaves to compute new truth values for the subformulas. ILR makes use of analytical minimal \boost\ functions to perform this backward pass.  ILR is a differentiable algorithm if the fuzzy operators and their corresponding minimal \boost\ functions are differentiable as it computes compositions of these functions.

Algorithm~\ref{alg:main} contains the pseudocode of ILR, and Figure~\ref{fig:ILR} presents an example of a single step (lines 3 to 7 of the algorithm) for the formula $\varphi = \lnot A \land (B \lor C)$ under the G\"{o}del semantics.

First, ILR computes the truth value of the formula in the forward pass, as shown on the left side of Figure~\ref{fig:ILR}. ILR saves the truth vectors of intermediate subformulas in $\truth_{\mathsf{sub}}$, which are presented in Figure~\ref{fig:ILR} as the numbers inside the purple shapes. Then, ILR calls a scheduler to determine the right target value for the formula $\varphi$. The target value is $\hat{t}'_{\varphi} = \alpha\cdot(1-0.6)=0.9$ for our example. 
The scheduling mechanism smooths the updates ILR makes. We implement this in line \ref{alg-line:scheduling} of Algorithm~\ref{alg:main}. It works by choosing a different \boosted\ value at each iteration: The difference between the current truth value and the \boosted\ value is multiplied by a scheduling parameter $\alpha$, which we choose to be either 0.1 or 1 (no scheduling). While usually not necessary, for some formulas, the scheduling mechanism allowed for finding better solutions.

Following the scheduler, ILR computes the backward step  in rows from 13 to 19 in Algorithm~\ref{alg:main}. It changes the input truth vector $\truth$ based on the formula $\varphi$. Note that the formula $\varphi$ in Figure~\ref{fig:ILR} is a conjunction of two subformulas ($\varphi_1 = \lnot A$ and $\varphi_2 = B \lor C$). ILR applies refinement functions recursively by treating the subformulas as literals: We give the truth values of $\varphi_1$ and $\varphi_2$ we saved in the forward pass, as inputs to the refinement function. In the example, we use the \boost\ function for the  G\"{o}del t-norm. 

The \boost\ function updates the truth values of $\varphi_1$ and $\varphi_2$.
%
%
Then, we interpret these new values as the \emph{target} truth values for the formulas $\varphi_1$ and $\varphi_2$. This allows us to apply the refinement proccedure recursively. For instance, in Figure~\ref{fig:ILR}, the refined truth values $\hat{t}_{\lnot A}$ and $\hat{t}_{B \lor C}$ can be interpreted as the target truth values for $\lnot A$ and $B \lor C$, respectively. Then, by applying the refinement functions for negation\footnote{Note that the minimal refinement function for the negation is trivial since there is only a feasible solution. For this reason, we omitted it from our analysis} and t-conorm, we can obtain the truth values of $A$, $B$ and $C$.

One choice in ILR is how to combine the results from different subformulas. Indeed, when a proposition appears in multiple subformulas, it can be assigned multiple different \boosted\ values. As an example, suppose the formula of Figure~\ref{fig:ILR} was $\varphi = \lnot A \land (B \lor A)$, with the proposition $C$ replaced by $A$. While similar to the previous formula, $A$ is repeated twice. Consequently, the algorithm produces two different refined values for $A$. 
 We found the heuristic in line \ref{alg-line:combine} generally works well, which takes the $\boosts_j$ with the largest absolute value. We also explored two other heuristics. In the first, we averaged the different \boosted\ values, but this took significantly longer to converge. The second heuristic we explored was the smallest absolute value, which frequently did not find solutions. Another choice is the convergence criterion. A simple option is to stop running the algorithm whenever it has stopped getting closer to the \boosted\ value for a couple of iterations. In our experiments, we observed that ILR monotonically decreases the distance to the \boosted\ value, after which it gets stuck on a single local optimum or oscillates between two local minima.

ILR is not guaranteed to find a \boosted\ vector $\boostv$ such that $\op_\varphi(\boostv)=\revis{\varphi}$. This is easy to see theoretically because, for many fuzzy logics like the product and G\"{o}del logics, $\revis{\varphi}=1$ corresponds to the PMaxSAT problem, which is NP-complete \citep{diligentiSemanticbasedRegularizationLearning2017,giunchigliaROADRAutonomousDriving2022}, while ILR has linear time complexity. 
However, this is traded off by 1) being highly efficient, usually requiring only a couple of iterations for convergence, and 2) not having any hyperparameters to tune, except arguably for the combination function. Furthermore, ILR usually converges quickly in neurosymbolic settings since background knowledge is very structured, and the solution space is relatively dense. These settings are unlike the randomly generated SAT problems we study in Section \ref{sec:results-sat}. These contain little structure the ILR algorithm can exploit. 


\begin{algorithm}
\caption{Iterative Local Refinement}\label{alg:main}
\begin{algorithmic}[1]
    \Require{$\varphi, \revis{\varphi}, \truth$, $\alpha\in (0, 1]$}
    \State $\truth' \gets \truth$
    \While{not converged}
        \State $\truth_{\mathsf{sub}}\gets \{\}$
        \For{subformula $\phi$ of $\varphi$}
            \State $\truth_{\mathsf{sub}}[\phi]\gets \op_\phi(\truth')$ \Comment{Forward pass using Definition \ref{def:evaluation}}
        \EndFor
        \State $\revis{\varphi}' \gets f_\varphi(\truth) + \alpha \cdot (\revis{\varphi} - f_\varphi(\truth))$ \label{alg-line:scheduling}
        \State $\truth' \gets$ \Call{Backward}{$\varphi$, $\revis{\varphi}'$, $\truth_{\mathsf{sub}}$}
    \EndWhile
    \State \Return $\truth'$
    \Function{Backward}{$P_i$, $\revis{P_i}$, $\truth_{\mathsf{sub}}$}
        \State \Return $[\truths_1, \dots, \revis{P_i}, \dots, \truths_n]^\top$ \Comment{$\truth$ except at position $i$.}
    \EndFunction
    \Function{Backward}{$\neg \phi$, $\revis{\neg \phi}$, $\truth_{\mathsf{sub}}$}
        \State \Return \Call{Backward}{$\phi$, $1-\revis{\neg \phi}$, $\truth_{\mathsf{sub}}$}
    \EndFunction
    \Function{Backward}{$\bigwedge_{i=1}^m\phi_i$, $\revis{\varphi}$, $\truth_{\mathsf{sub}}$} \label{alg-line:backward-t-norm-start}
        \State $\boostv_{\wedge}\gets \minboost_T(\left[\truth_{\mathsf{sub}}[\phi_1], ..., \truth_{\mathsf{sub}}[\phi_m]\right]^ \top, \revis{\varphi})$ \Comment{Minimal \boost\ function}
        \State $\boostv \gets \boldsymbol{0}$
        \For{$i\gets 1$ to $m$}
            \State $\boostv' \gets$ \Call{Backward}{$\phi_i$, $\boosts_{\wedge, i}$, $\truth_{\mathsf{sub}}$}
            \State $\boosts_j \gets$ \algorithmicif\ $\vert \boosts_j \vert > \vert \boosts'_j \vert$ \algorithmicthen\ $\boosts_j$ \algorithmicelse\ $\boosts' _j$ for all $j\in \{1, ..., n\}$ \label{alg-line:combine}
        \EndFor
        \State \Return $\boostv$ \label{alg-line:backward-t-norm-end}
    \EndFunction
\end{algorithmic}
\end{algorithm}

\section{Neuro-Symbolic AI using ILR}
\label{sec:neuro-symbolic}
The ILR algorithm can be added as a module after a neural network $g$ to create a neurosymbolic AI model. The neural network predicts (possibly some of) the initial truth values $\truth$. Since both the forward and backward passes of ILR are differentiable computations, we can treat ILR as a constrained output layer \citep{giunchigliaDeepLearningLogical2022}. For instance, in Figure~\ref{fig:ILR}, the input $\truth$ could be generated by the neural network, and we provide supervision directly on the predictions $\boostv$. With ILR, the predictions, i.e., the \boosted\ vector $\boostv$, take the background knowledge into account while staying close to the original predictions made by the neural network. Loss functions like cross-entropy can use $\boostv$ as the prediction. We train the neural network $g$ by minimizing the loss function with gradient descent and backpropagating through the ILR layer. 

One strength of ILR is the flexibility of the \boost\ values $\revis{\varphi_i}$ for each formula $\varphi_i$. These can be set to 1 to treat $\varphi_i$ as a hard constraint that always needs to be satisfied. Alternatively, \boost\ values can be trained as part of a larger deep learning model. Since ILR is a differentiable layer, we can compute gradients of the \boost\ values. This procedure allows ILR to learn what formulas are useful for prediction. For instance, in Figure~\ref{fig:ILR}, $\revis{\lnot A \land (B \lor C)}$ can either be given or act as a parameter of the model that is learned together with the neural network parameters.

We give an example of the integration of ILR with a neural network in Figure~\ref{fig:ILR_MNIST}, where we use ILR for the MNIST Addition task proposed by~\cite{manhaeveDeepProbLogNeuralProbabilistic2018}. In this task, we have access to a training set composed of triplets $(x, y, z)$, where $x$ and $y$ are images of MNIST~\citep{lecunMNISTHandwrittenDigit2010} handwritten digits, and $z$ is a label representing an integer in the range $\{0,...,18\}$, corresponding to the sum of the digits represented by $x$ and $y$. The task consists of learning the addition function and a classifier for the MNIST digits, with supervision only on the sums. To achieve this, knowledge consisting of the rules of addition is given. For instance, the rule
$
Is(x, 3) \land Is(y, 2) \to Is(x+y, 5)
$ states that the sum of 3 and 2 is 5.

The architecture of the model presented in Figure~\ref{fig:ILR_MNIST} consists of a neural network (a CNN) that performs digit recognition on the inputs $x$ and $y$. After this step, ILR predicts a truth value for each possible sum. Notice that we define the CNN outputs $\boldsymbol{C}_x,\boldsymbol{C}_y\in[0, 1]^{10}$ as constants, i.e., ILR does not change the predictions of the digits. Moreover, the initial prediction for the truth vector of possible sums $\truth_{x+y}\in[0, 1]^{19}$ is the zero vector. This allows ILR to act as a proof-based method. Indeed, similarly to DeepProbLog~\citep{manhaeveDeepProbLogNeuralProbabilistic2018}, the architecture proposed in Figure~\ref{fig:ILR_MNIST} uses the knowledge in combination with the predictions of the neural network to derive truth values for new statements (the sum of the two digits). We apply the loss function to the final predictions $\hat{\truth}_{x+y}$. During learning, the error is back-propagated through the entire model, reaching the CNN, which learns to classify the MNIST images from indirect supervision.

We present the results obtained by ILR in Section~\ref{sec:MNIST_exp}, and compare its performance with other neurosymbolic AI frameworks.

\begin{figure}
    \includegraphics[width=\linewidth]{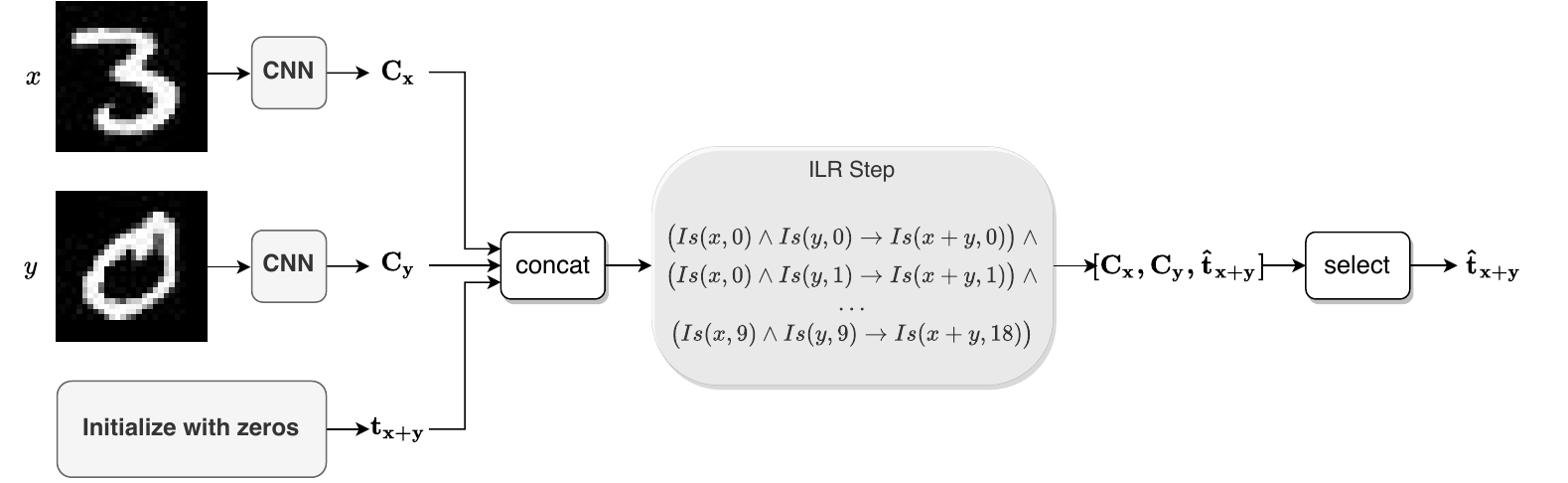}
    \caption[Neurosymbolic architecture based on ILR for the MNIST Addition task]{Neurosymbolic architecture based on ILR for the MNIST Addition task. A CNN takes two images of MNIST digits, returning their classification. The CNN predictions are concatenated with a vector of zeros, representing the initial prediction for the Addition task. We perform an ILR step to update the sum of the two numbers, which is the final output of the model.}
    \label{fig:ILR_MNIST}
\end{figure}

\section{Analytical minimal \boost\ functions}
\label{sec:general-analysis}
Having introduced the ILR algorithm, we next study the problem of finding minimal \boost\ functions for individual fuzzy operators. We need these in closed form to compute the ILR algorithm, as ILR uses them during the backward pass. This section first discusses several transformations of minimal \boost\ functions and gives the minimal \boost\ functions of the basic t-norms G\"{o}del, \luk\, and product. In Section \ref{sec:class-analysis}, we investigate a large class of t-norms for which we have closed-form formulas for the minimal \boost\ functions.
\subsection{General results}
We first provide several basic results on minimal \boost\ functions for fuzzy operators. In particular, we will consider formulas such as $\varphi=\bigwedge_{i=1}^n P_i \bigwedge_{i=1}^m C_i$, that is, conjunctions of propositions and constants. 
As an abuse of notation, from here on, we will refer to $\revismin{\varphi}$ and $\revismax{\varphi}$ when evaluated by the t-norm $T$ as $\revismin{T}$ and $\revismax{T}$ and will do so also for other fuzzy operators.
We find using Definition \ref{deff:tnorm} that for some t-norm $T$,  $\revismin{T} = 0$ and $\revismax{T} = T(\truthc)$, where $\truthc$ is the  values of the constants $C_1, ..., C_m$ as a truth vector, while for some t-conorm $S$, $\revismin{S} = S(\truthc)$ and $\revismax{S}=1$. Note that for $m=0$, $\revismax{T}=1$ and $\revismin{S}=0$. Next, we find some useful transformations of minimal \boost\ functions to derive new results:

\begin{proposition}
    \label{prop:dual-t-conorm}
    Consider the formulas $\phi=\bigwedge_{i=1}^n P_i\bigwedge_{i=1}^m C_i$ and $\psi= \neg (\bigvee_{i=1}^n P_i \bigvee_{i=1}^m C_i)$. Assume $\minboost_\phi$ is a minimal \boost\ function for $\op_\phi$ evaluated using t-norm $T$. Consider $\op_{\psi}(\truth)$ evaluated using dual t-conorm $S$ of $T$. Then $\minboost_{\psi}(\truth, \revis{\psi})=\boldsymbol{1}-\minboost_{\phi}(\boldsymbol{1}-\truth, \revis{\psi})$ is a minimal \boost\ function for $\op_{\psi}$.
\end{proposition}
\begin{proof}
    First, note $\op_{\psi}(\truth) = 1-S(\truth, \truthc) = 1-(1-T(\boldsymbol{1}-\truth, \boldsymbol{1}-
\truthc))=T(\boldsymbol{1}-\truth, \boldsymbol{1}-\truthc)$. Consider $\truth'=\boldsymbol{1}-\truth$. By the assumption of the proposition, $\minboost_\phi(\truth', \revis{\phi})$ is a minimal \boost\ function for $T(\truth', \boldsymbol{1}-\truthc)=T(\boldsymbol{1}-\truth, \boldsymbol{1}-\truthc)=\op_\psi(\truth)$. Furthermore, note that 
    \begin{align*}
        \op_\psi(\minboost_{\psi}(\truth, \revis{\psi}))=T(\boldsymbol{1}-\minboost_{\psi}(\truth, \revis{\psi}), \boldsymbol{1}-\truthc) =T(\minboost_\phi(\truth', \revis{\psi}), \boldsymbol{1}-\truthc) = \revis{\psi}
    \end{align*}
\end{proof}
An analogous argument can be made for $\phi'=\bigvee_{i=1}^n P_i\bigvee_{i=1}^m C_i$ and $\psi=\neg(\bigwedge_{i=1}^n P_i\bigwedge_{i=1}^m C_i)$ to show that, given minimal \boost\ function $\minboost_{\phi'}$ of dual t-conorm $S$, the minimal \boost\ function for $\op_{\psi}(\truth)$ is $\minboost_{\psi}(\truth, \revis{\psi})=\boldsymbol{1}-\minboost_{\phi}(\boldsymbol{1}-\truth, \revis{\psi})$.

We will use this result to simplify the process of finding minimal \boost\ functions for the t-norms and dual t-conorms. For example, assume we have a minimal \boost\ function $\minboost_T$ for $\revis{T}\in [T(\truth), \revismax{T}]$. Let $S$ be the corresponding dual t-conorm. Then, we can change the constraint $S(\boostv, \truthc)=\revis{S}$ in Equation \ref{eq:optim-problem} to the equivalent constraint $\boldsymbol{1}-S(\boostv, \truthc)=\boldsymbol{1}-\revis{S}$. We then use Proposition \ref{prop:dual-t-conorm} to find the minimal \boosted\ vector for $\revis{S}\in[\revismin{S}, S(\truth)]$ as $\boldsymbol{1}-\minboost_T(\boldsymbol{1}-\truth, 1-\revis{S})$. 

\begin{proposition}
    \label{prop:s-implication}
    Consider the formulas $\phi= P_1 \vee P_2$ and $\psi= \neg P_1 \vee P_2$. Assume $\minboost_\phi$ is a minimal \boost\ function for $\op_\phi$ evaluated using the t-conorm $S$, and define $\truth'=[1-\truths_1, \truths_2]$. Then $\minboost_{\psi}(\truth, \revis{\psi})=\left[1-\minboost_{\phi}(\truth', \revis{\psi})_1, \minboost_{\phi}(\truth', \revis{\psi})_2\right]^\top$ is a minimal \boost\ function for $\op_{\psi}$.
\end{proposition}
\begin{proof}
    First, note $\op_{\psi}(\truth) = S(1-\truth_1, \truth_2)$. By the assumption of the proposition, $\minboost_\phi(\truth', \revis{\psi})$ is a minimal \boost\ function for $S(\truth')=\op_\psi(\truth)$. Furthermore, note that 
    \begin{align*}
        \op_\psi(\minboost_{\psi}(\truth, \revis{\psi})) &= S(1-\minboost_\psi(\truth', \revis{\psi})_1, \minboost_\psi(\truth', \revis{\psi})_2)   \\
        &= S(1-(1-\minboost_\phi(\truth', \revis{\psi})_1), \minboost_\phi(\truth', \revis{\psi})_2)  = S(\minboost_\phi(\truth', \revis{\psi})) = \revis{\psi}.
    \end{align*}
\end{proof}
Similar to the previous proposition, this proposition gives us a simple procedure for finding the minimal \boost\ functions for the S-implication of some t-conorm.

\subsection{Basic T-norms}
\label{sec:basic-t-norm}
In this section, we introduce the minimal \boost\ functions for the t-norms and t-conorms of the three main fuzzy logics (G\"{o}del, \luk, and Product). In particular, we consider when these t-norms and t-conorms can act on both propositions and constants, that is, $\varphi=\bigwedge_{i=1}^n \truths_i \bigwedge_{i=1}^m C_i$, which is evaluated with $T(\truth, \truthc)$. We present the main results with simple examples. 

\subsubsection{G\"{o}del t-norm}
\label{sec:godel-t-norm}
In this section, we derive minimal \boost\ functions for the G\"{o}del t-norm and t-conorm for the family of $p$-norms.

\begin{proposition}
    The minimal \boost\ function of the G\"{o}del t-norm for $\revis{T_G}\in [0, \min_{i=1}^m C_i]$ is
\begin{equation}
    \minboost_{T_G}(\truth, \revis{T_G})_i=\begin{cases}
        \revis{T_G}  & \text{if } \revis{T_G} \geq T_G(\truth) \text { and } \truths_i < \revis{T_G}, \\
        \revis{T_G}  & \text{if } \revis{T_G} < T_G(\truth) \text { and } i=\arg\min_{j=1}^n \truths_j, \\
        \truths_i & \text {otherwise,}
    \end{cases}
\end{equation} 
The minimal \boost\ function of the G\"{o}del t-conorm and $\revis{S_G} \in [\max_{i=1}^m C_i, 1]$ is 
\begin{equation}
    \minboost_{S_G}(\truth, \revis{S_G})_i=\begin{cases}
        \revis{S_G} & \text{if } \revis{S_G} \geq S_G(\truth) \text { and } i=\arg\max_{j=1}^m \truths_j, \\
        \revis{S_G} & \text{if } \revis{S_G} < S_G(\truth) \text { and } \truths_i > \revis{S_G}, \\
        0 & \text {otherwise.}
    \end{cases}
\end{equation} 
\end{proposition}
\begin{proof}
    Assume the minimal refinement function of the t-norm is not minimal and that $\revis{T_G}\geq T_G(\truth)$. Then there is a \boosted\ vector $\boostv$ for $T_G$, $\truth\in [0, 1]^n$ and $\revis{T_G} \in [T_G(\truth), \min_{i=1}^m]$ such that $\boostv\neq \minboostv$ while $\|\boostv - \truth\|_p < \|\minboostv - \truth\|_p$, where $\minboostv=\minboost_{T_G}(\truth, \revis{T_G})$. Since $T_G(\boostv)=\revis{T_G}$, for all $i\in \{1, ..., n\}$, $\boosts_i\geq \revis{T_G}$ and so necessarily for all $i$ such that $\truths_i< \revis{T_G}$, $\boosts_i \geq \revis{T_G}$. Since there is some $i$ such that $\boosts_i\neq \minboosts_i$, either $\truths_i < \revis{T_G}$ and then necessarily $\boosts_i > \minboosts_i$, or $\boosts_i \geq \revis{T_G}$ but $\boosts_i \neq \minboosts_i=\truths_i$. In either case, since $\|\cdot \|_p$ is strictly convex in each argument with minimum at $\truth$, $\|\boostv - \truth\|_p > \|\minboostv - \truth\|_p$, hence $\boostv$ could not have smaller norm. 

    A derivation for increasing the G\"{o}del t-conorm, assuming $\revis{S_G} \geq S_G(\truth)$, was first presented in \cite{danieleKnowledgeEnhancedNeural2019}. The remaining cases follow from Proposition \ref{prop:dual-t-conorm}. 
\end{proof}

\begin{proposition}
    \label{prop:godel-impl}
    A minimal \boost\ function of the G\"{o}del implication $R_G(\truths_1, \truths_2)=\begin{cases}\truths_2 & \text{if } \truths_1 > \truths_2, \\ 1 & \text{otherwise.}\end{cases}$ for $\revis{R_G}\in[\revismin{R_G}, \revismax{R_G}]$ is 
    \begin{equation}
            \minboost_{R_G}(\truths_1, \truths_2, \revis{R_G}) = \begin{cases}
                [\max(\revis{R_G} + \epsilon, \truths_1), \revis{R_G}]^\top  & \text{if } \revis{R_G} < 1 \\
                [\truths_1, \max(\truths_1, \truths_2)]^\top & \text{otherwise.} 
            \end{cases}
    \end{equation}
    where $\epsilon$ is an arbitrarily small positive number.
\end{proposition}
\begin{proof}
    First, assume $\revis{R_G} <  1$. To ensure $R_G(\truths_1, \truths_2)=\revis{R_G}$, we require $\truths_2 =\revis{R_G}$ as is clear from the definition. However, we also require $\truths_1 > \revis{R_G}$. If $\truths_1$ is already larger, we can leave it to ensure minimality. Otherwise, we require it to be at least infinitesimally bigger, that is $\revis{R_G} + \epsilon$. 

    Next, assume $\revis{R_G}=1$. If $\truths_1\leq \truths_2$, then the implication is already 1 and we do not need to revise anything. Otherwise, setting it equal to any value between $\truths_2$ and $\truths_1$ is minimal. 
\end{proof}

\begin{figure}
    \includegraphics[width=\linewidth]{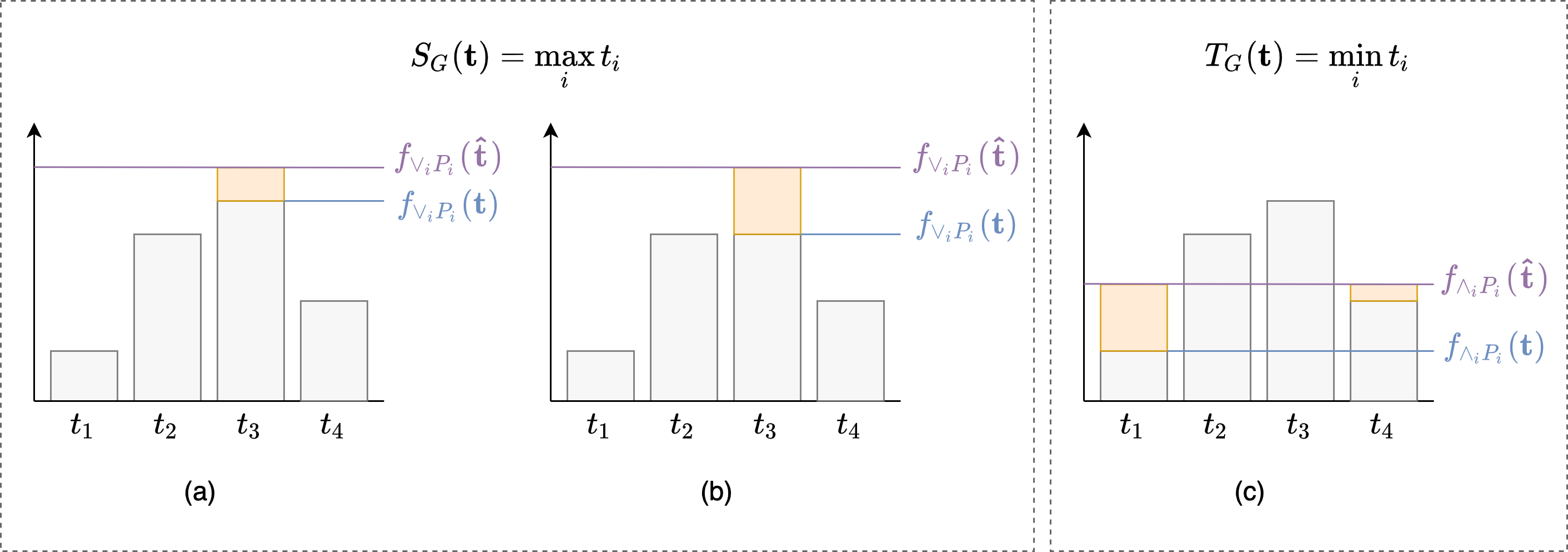}
    \caption[G\"{o}del minimal \boost\ functions]{G\"{o}del minimal \boost\ functions. The grey bars represent the initial truth vectors $\truth$; the light blue and purple lines indicate the initial truth value of the formula and the revision value $\revis{\varphi}$, and the orange bars are the corresponding minimal \boosted\ vectors. (a) t-conorm; (b) t-conorm with two literals with same truth value; (c) t-norm.}
    \label{fig:Godel}
\end{figure}

The bar plot in Figure~\ref{fig:Godel}(a) shows an example for the G\"{o}del t-conorm with four literals. The minimal \boosted\ vector is represented with the orange boxes, while the initial and \boost\ values of the entire formula are represented as a blue and purple line respectively. Here, our goal is to increase the value of the t-conorm, i.e., the maximum value. Increasing other literals up to $\revis{\varphi}$ would require longer orange bars and bigger values for the L$_p$ norm. Figure~\ref{fig:Godel}(b) represents when multiple literals have the largest truth value. Here, only one should be increased\footnote{In our experiments, we choose randomly.}. Finally, Figure~\ref{fig:Godel}(c) shows the \boosted\ vector for the G\"{o}del t-norm. Since the smallest truth value should be at least $\revis{\varphi}$, we simply ensure all truth values are at least $\revis{\varphi}$.

Our results are closely related to that of \cite{giunchigliaMultiLabelClassificationNeural2021a}, which considers hard constraints, i.e., $\revis{\varphi}=1$. In the hierarchical multi-label classification setting, the authors introduce an output layer that ensures predictions satisfy a set of hierarchy constraints. This layer corresponds to applications of the minimal \boost\ function for the G\"{o}del implication with $\revis{R_G}=1$. Furthermore, \cite{giunchigliaMultiLabelClassificationNeural2021a} introduces CCN(h). This method considers an output layer that ensures predictions satisfy background knowledge expressed in a stratified normal logic program. The authors introduce an iterative algorithm that computes the minimal solution for such programs. This algorithm is related to that of ILR in Section \ref{sec:ILR}. However, their formalization differs somewhat from ours, and future work could study whether these results also hold for our formalization of minimal \boost\ functions and if they can be extended to any value of $\revis{\varphi}$. Finally, \cite{giunchigliaMultiLabelClassificationNeural2021a} introduces a loss function compensating for gradient bias introduced by the constrained output layer. 

\subsubsection{\luk\ t-norm}
\label{seq:lukasiewicz}
In this section, we derive minimal \boost\ functions for the \luk\ t-norm and t-conorm, for the family of $p$-norms. We will start using the following notation here: $\truth^\uparrow$ refers to the truth values $\truths_i$ sorted in ascending order, while $\truth^\downarrow$ refers to the truth values sorted in descending order.


\begin{proposition}
 Let $\revis{T_L}\in[0, \max(\|\truthc\|_1 - (m - 1), 0)]$ and define $\lukincrease_K=\frac{\revis{T_L}+ m + K -1-\|\truthc\|_1 - \sum_{i=1}^ K\truths^\uparrow_{i}}{K}$. Let $K^ *$ be the largest integer $1\leq K\leq n$ such that $\lukincrease_{K}<1-\truths^ \uparrow_{K}$. Then the minimal \boost\ vector of the \luk\ t-norm is 
\begin{equation}
    \minboost_{T_L}(\truth, \revis{T_L})_i=\begin{cases}
        \truths_i + \lukincrease_{K^*} & \text{if } \revis{T_L} > T_L(\truth) \text{ and } \truths_i \leq \truths^\uparrow_{K^*}, \\
        1 & \text{if } \revis{T_L} > T_L(\truth) \text{ and } \truths_i > \truths^\uparrow_{K^*}, \\
        \truths_i - \frac{\max(\|\truth\|_1 + \|\truthc\|_1 + 1 - n - \revis{T_L}, 0)}{n} & \text{otherwise.}
    \end{cases} 
\end{equation}
Let $\revis{S_L} \in [\min(\|\truthc\|_1, 1), 1]$ and define $\lambda_K =  \frac{\|\truth\|_1 + \|\truthc\|_1  - \revis{S_L}}{K}$. Let  $K^*$ be the largest integer $1\leq K\leq n$ such that $\lambda_K < \truths^\downarrow_{K}$. Then the minimal \boost\ function of the \luk\ t-conorm is

\begin{equation}
    \minboost_{S_L}(\truth, \revis{S_L})_i=\begin{cases}
        \truths_i + \frac{\max(\revis{S_L}-\|\truth\|_1 - \|\truthc\|_1, 0)}{n} & \text{if } \revis{S_L} > S_L(\truth), \\
        \truths_i - \lambda_{K^*} & \text{if } \revis{S_L} < S_L(\truth) \text { and } \truths_i \geq \truths^\downarrow_{K^*}, \\
        0 & \text{otherwise.}
    \end{cases}
\end{equation} 
\end{proposition}
\begin{proof}
    \textbf{T-norm.} Assume $\revis{T_L}\geq T_L(\truth)$. We will prove this using the KKT conditions, which are both necessary and sufficient for minimality for the \luk\ t-norm since it is affine when the max constraint is not active. We drop the $p$-root in the norm since it is a strictly monotonically increasing function.  The Lagrangian and corresponding derivative is
    \begin{align*}
        \ell&=\sum_{i=1}^n \vert\boosts_i -\truths_i\vert^p + \lambda(\max(\|\boostv\|_1 + \|\truthc\|_1 - (m+n-1), 0) - \revis{T_L}) +\sum_{i=1}^ n \gamma_i(\boosts_i -1)\\
        \frac{\partial \ell}{\partial  \boosts_i} &= p(\boosts_i - \truths_i)^{p-1} + \lambda \frac{\partial }{\partial \boosts_i} \max(\|\boostv\|_1 + \|\truthc\|_1 - (m+n-1), 0) + \gamma_i =0.
    \end{align*}
    We note that we drop the absolute signs since $T_L$ is strictly monotonically increasing function and $\revis{T_L} \geq T_L(\truth, \truthc)$. 
    Assuming $\revis{T_L} > 0$, $T_L(\boostv, \truthc)=\revis{T_L}$ can only be true if the first argument of $\max$ is chosen. Then for all $i, j\in \{1, ..., n\}$, $p(\boosts_i - \truths_i)^ {p-1} + \gamma_i=p(\boosts_j - \truths_j)^ {p-1} +  \gamma_j$. Define $I$ as the set of $K^ *$ smallest $\truths_i$. 
    \begin{itemize}
        \item \emph{Primal feasibility:} For all $i\in I$, $\minboost_{T_L}(\truth, \revis{T_L})_i=\lukincrease_{K^ *}\leq 1$ by definition. For all $i\in \{1, ..., n\}\setminus I$, $\minboost_{T_L}(\truth, \revis{T_L})_i=1-\truths_i$. Furthermore, 
        \begin{align*}
        T_L(\minboost_{T_L}(\truth, \revis{T_L}), \truthc) &=\max(\sum_{i=1}^{K^*} (\truths^\uparrow_i + \lukincrease_{K^ *}) +\sum_{i=K^* +1}^n 1 + \|\truthc\|_1 - n - m + 1, 0) \\
        &= \max(\sum_{i=1}^{K^*} \truths^\uparrow_i + K^ * \lukincrease_{K^ *} + n - K^ * + \|\truthc\|_1 - n - m + 1, 0) \\
        &= \max(\sum_{i=1}^{K^*}\truths_i^\uparrow + \revis{T_L} + m + K^ * -1 -\|\truthc\|_1 - \sum_{i=1}^{K^*} \truths^\uparrow_i  \\
        &- K^ *   + \|\truthc\|_1  - m+ 1, 0)  =\revis{T_L}
        \end{align*}
        \item \emph{Complementary Slackness:} Clearly, for all $i\in I$, we require $\gamma_i=0$. For all $i\in \{1, ..., n\}\setminus I$, $\minboost_{T_L}(\truth, \revis{T_L}t)_i -1 = 1 - 1 = 0$. 
        \item \emph{Dual feasibility:} 
        For all $i\in I$, $\gamma_i=0$. For $i\in \{1, ..., n\}\setminus I$, consider some $j\in I$ and note that $p(\boosts_i - \truths_i)^ {p-1} + \gamma_i=p(\boosts_j - \truths_j)^ {p-1} +  \gamma_j$. Filling in $\boostv$, we find $\gamma_i= p\lukincrease_{K^ *}^ {p-1} - p(1-\truths_i)^ {p-1}$. This is nonnegative if $\lukincrease_{K^ *}\geq 1-\truths_i$. First, we show $\lukincrease_{K^*} \geq \lukincrease_{K^*+1}$. Write out their definitions, multiply by $K^*(K^*+1)$ and remove common terms. Then,
        \begin{align*}
            \revis{T_L} + m - 1 - \|\truthc\|_1 - \sum_{i=1}^{K^*}\truths_{i}^\uparrow &\geq -K^* \truths^\uparrow_{K^*+1}\\
            \revis{T_L} + m + K^*+1 -\|\truthc\|_1 - \sum_{i=1}^{K^*+1}\truths_{i}^\uparrow &\geq (K^*+1)(1-\truths_{K^*+1}^\uparrow) \\
            \lukincrease_{K^*+1}&\geq 1-\truths_{K^*+1}^\uparrow.
        \end{align*}
        $\lukincrease_{K^*+1}\geq 1-\truths_{K^*+1}^\uparrow$ is true by the construction in the proposition. Therefore, 
        \begin{equation*}
            \lukincrease_{K^*}\geq\lukincrease_{K^*+1}\geq 1-\truths^\uparrow_{K^*+1}\geq 1-\truths_i,
        \end{equation*}
        proving dual feasibility.
    \end{itemize}

    \textbf{T-conorm.} We do not add multipliers for the constraints on $\boosts_i$, and show critical points adhere to these constraints. The Lagrangian is
    \begin{equation}
        \ell=\sum_{i=1}^n (\boosts_i - \truths_i)^p + \lambda(\min(\|\boostv\|_1 + \|\truthc\|_1, 1) - \revis{S_L})
    \end{equation}
    Note that $\revismax{S_L}=1$. 
    Taking the derivative to $\boosts_i$, we find
    \begin{align*}
        \frac{\partial \ell}{\partial \boosts_i} =  p\cdot (\boosts_i - \truths_i)^{p-1} + \lambda\frac{\partial}{\partial \boosts_i} \min(\min(\|\boostv\|_1 + \|\truthc\|_1, 1)=0
    \end{align*}
    Assume $\revis{S_L}\neq S_L(\truth)$, this gives three cases for all $i\in \{1, ..., n\}$:
    \begin{enumerate}
        \item If $\|\truth\|_1 + \|\truthc\| \geq 1$ and $\revis{S_L}=1$, then since $\boosts_i \geq \truths_i$, $\frac{\partial}{\partial \boosts_i} \min(\|\boostv\|_1 + \|\truthc\|_1, 1)=\frac{\partial}{\partial \boosts_i}1=0$, and so $\boosts_i=\truths_i$.
        \item If $\|\truth\|_1 + \|\truthc\| \geq 1$, then $\revismin{S_L} =\revismax{S_L}= 1$, and again $\boosts_i=\truths_i$. 
        \item Otherwise, it must be that $\|\boostv\|_1 + \|\truthc\|_1\leq 1$ and so $\frac{\partial}{\partial \boosts_i} \min(\|\boostv\|_1 + \|\truthc\|_1, 1)=\frac{\partial}{\partial \boosts_i}\|\boostv\|_1=1$, and therefore $p\cdot(\boosts_i - \truths_i)^{p-1}=-\lambda$. Since the equality holds for all $i\in \{1, ..., n\}$, we find $p\cdot(\boosts_i - \truths_i)^{p-1}=p\cdot(\boosts_j - \truths_j)^{p-1}$ for all $i, j\in \{1, ..., n\}$. As we are only interested in real nonnegative solutions,  we find that $\boosts_i - \truths_i=\boosts_j - \truths_j=\delta$. 
        Since $\|\boostv\|_1 + \|\truthc\|_1  = \|\truth\|_1 + \|\truthc\|_1 + n \delta  = \revis{S_L}$, we find
        \begin{equation*}
            \delta = \frac{\revis{S_L}-\|\truth\|_1 - \|\truthc\|_1}{n}, \quad \boosts_i=\truths_i + \delta.
        \end{equation*}
        Note that $\boosts_i\geq \truths_i$, since by assumption $\revis{S_L}\geq S_L(\truth, \truthc)$, and $\boosts_i \leq 1$ since by $\revis{S_L} \leq \revismax{S_L}\leq1$, $\delta = \frac{\revis{S_L}-\|\truth\|_1 - \|\truthc\|_1}{n} \leq \frac{1-\|\truth\|_1 - \|\truthc\|_1}{n}\leq \frac{1-\truths_i}{n}\leq 1-\truths_i$, that is, the constraints of Equation \ref{eq:optim-problem} are satisfied.

    \end{enumerate}

    The remaining cases follow from Proposition \ref{prop:dual-t-conorm}. 
\end{proof} 

Although slightly obfuscated, these \boost\ functions simply increase each of the literals equally, while properly dealing with constraints on the truth values. We explain this using Figure~\ref{fig:Lukasiewciz_p2}, where the optimal solution corresponds to a vector that, from the original truth values $\truth$, is perpendicular to the contour line of the operator at the value $\revis{\varphi}$. Moreover, the figure also provides some intuition for our proofs. The stationary points of the Lagrangian correspond to the points where the constraint function (blue circumference) tangentially touches the contour line of the \boosted\ value (orange line).

\begin{figure}
    \includegraphics[width=\linewidth]{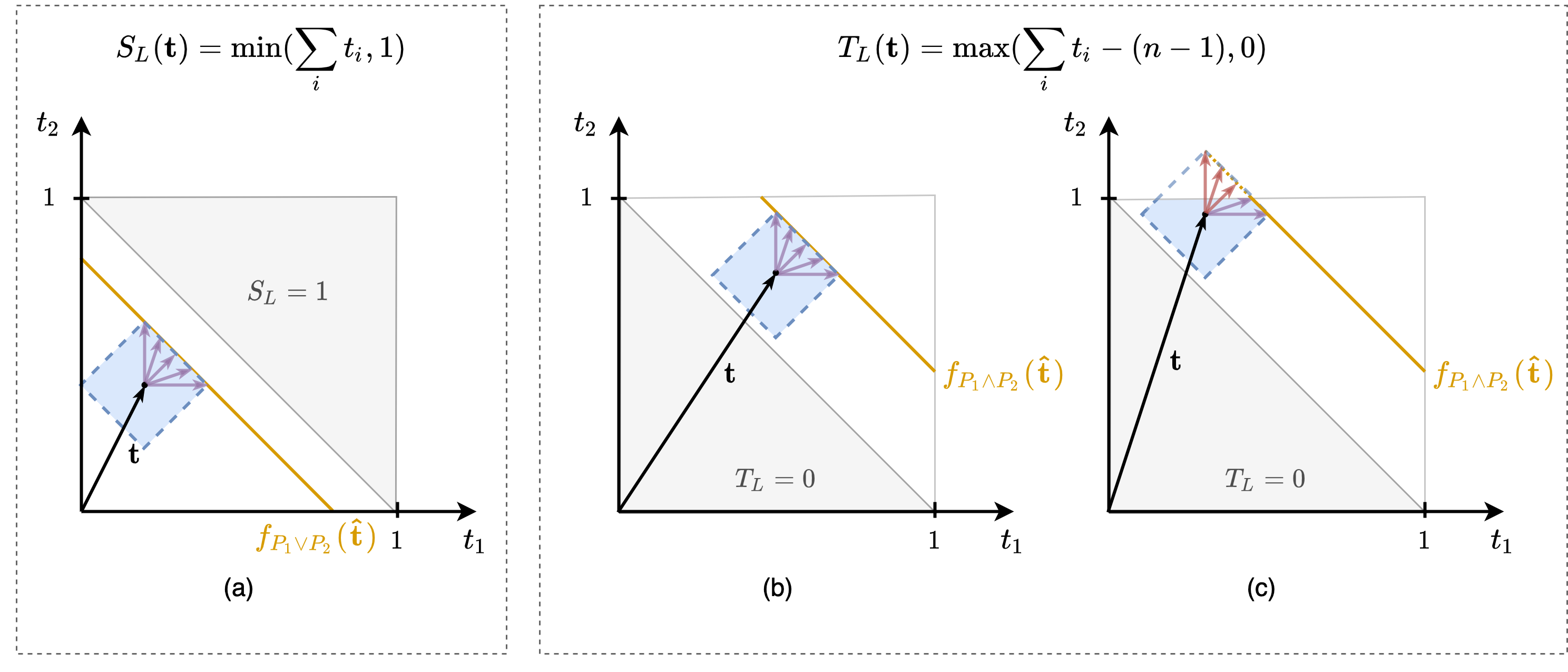}
    \caption[\luk\ minimal \boost\ functions]{\luk\ minimal \boost\ functions. The orange line corresponds to the contour line of the $S_L$ and $T_L$ at the value $\revis{\varphi}$. The dotted blue circumference corresponds to a set of points at an equal distance from $\truth$. (a) t-conorm; (b) t-norm; (c) t-norm in the limit case.}
    \label{fig:Lukasiewciz_p2}
\end{figure}

The change applied by the \boost\ function is proportional to the \boost\ value $\revis{}$. Computing these \boost\ functions requires finding $K^*$, which can be done efficiently in log-linear time using a sort on the input truth values and a binary search.

The residuum of the \luk\ t-norm is equal to its S-implication formed using $S_L(1-a, c)$, and so its minimal \boost\ function can be found using Proposition \ref{prop:s-implication}. 

The \luk\ logic is unique in containing large convex and concave fragments \citep{gianniniConvexLogicFragment2019}. In particular, any CNF formula interpreted using the weak conjunction (Godel t-norm) and \luk\ t-conorm is concave, allowing for efficient maximization using a quadratic program of a slightly relaxed variant of the problem in Equation \ref{eq:optim-problem}. \cite{gianniniConvexLogicFragment2019} studies this property in a setting similar to ours in the context of collective classification. Future work could study using this convex fragment to find minimal \boost\ functions for more complex formulas. 

\subsubsection{Product t-norm}
To present the three basic t-norms together, we give the closed-form \boost\ function for the product t-norm with the $L_1$ norm. Our proof is a special case of the general results on a large class of t-norms we will discuss in Section \ref{sec:general-analysis}. In particular, the product t-norm is a strict, Schur-concave t-norm with an additive generator. It is an example of a t-norm for which we can find a closed-form \boost\ function for the $L_1$ norm using Propositions \ref{prop:additive-generator} and \ref{prop:dual-t-conorm}. First, we show the minimal \boost\ function for the product t-norm.

\begin{equation}
   \minboost_{T_P}(\truth, \revis{T_P})_i= \begin{cases}
       \sqrt[n-K^*]{\frac{\revis{T_P}}{\prod_{j=1}^{K^*}\truths_j^\downarrow \prod_{j=1}^mC_j}} & \text{if } T_P(\truth, \truthc) > \revis{T_P} \text{ and } \truths_i \leq \truths^\downarrow_{K^*+1}, \\
       \sqrt{\frac{\revis{T_P}}{\prod_{j\neq i}\truths_i^\downarrow \prod_{i=1}^mC_i}} & \text{if } T_P(\truth, \truthc) < \revis{T_P} \text{ and } i=\arg\min_{j=1}^n \truths_j,   \\
       \truths_i & \text{otherwise.}
   \end{cases} 
\end{equation}

Next, we present the result for the product t-conorm:
\begin{equation}
    \minboost_{S_P}(\truth, \revis{S_P})_i= \begin{cases}
        1-\sqrt{\frac{1-\revis{S_P}}{\prod_{j\neq i}1-\truths_i^\downarrow \prod_{i=1}^m1-C_i}} & \text{if } S_P(\truth, \truthc) < \revis{S_P} \text{ and } i=\arg\min_{j=1}^n \truths_j,   \\
        1-\sqrt[n-K^*]{\frac{1-\revis{S_P}}{\prod_{j=1}^{K^*}1-\truths_j^\downarrow \prod_{j=1}^m1-C_j}} & \text{if } S_P(\truth, \truthc) > \revis{S_P} \text{ and } \truths_i \leq \truths^\downarrow_{K^*+1}, \\
        \truths_i & \text{otherwise.}
    \end{cases} 
 \end{equation}

This \boosted\ function increases all the literals smaller than a certain threshold up to the threshold itself, where we assume $\revis{T_P}$ is greater than the initial truth value. In fact, like the other t-norms in the class discussed in Section \ref{sec:class-analysis}, it is similar to the G\"{o}del t-norm in that it increases all literals above some threshold to the same value. Similarly, the \boost\ function for the t-conorm increases the highest literal. Figure~\ref{fig:Product} gives an intuition behind this behavior.
 
 \begin{figure}
    \includegraphics[width=\linewidth]{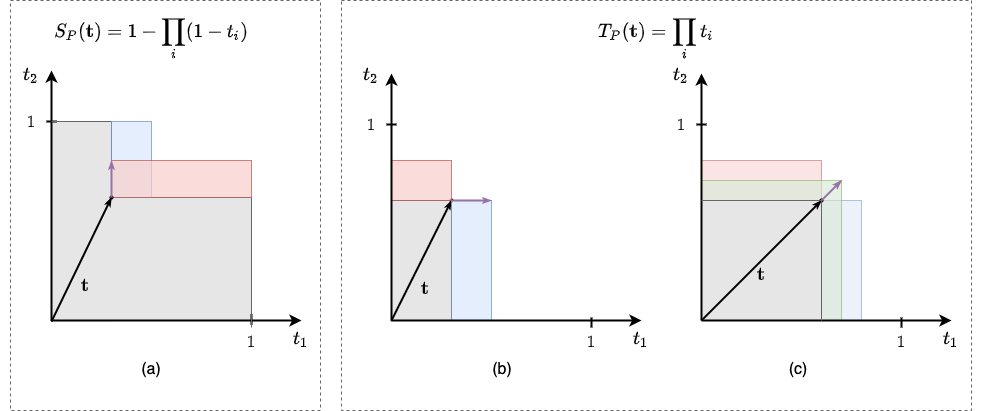}
    \caption[Product minimal \boost\ functions]{Product minimal \boost\ functions. The grey areas represent the truth value of the operator associated with the initial vector $\truth$. Red and blue areas represent the \boosted\ values when increasing a single literal.
    (a) t-conorm; (b) t-norm; (c) t-norm when multiple literals have the same truth value. The green area represents the improvement obtained by increasing both literals equally.}
    \label{fig:Product}
\end{figure}

Finally, the residuum minimal \boost\ function can be found with $\minboost_{I_P}(\truths_1, \truths_2, \revis{I_P}) = [\truths_1, \frac{\revis{I_P}}{\truths_1}]^\top$. 

We also studied the minimal \boost\ function for the $L_2$-norm, but concluded that the result is a $2n$th degree polynomial with no simple closed-form solutions. For details, see the original paper \cite{danieleRefiningNeuralNetwork2023a}(Appendix D). 


\section{Additive generators, strict core monotonicity and a dual problem}
To be able to adequately discuss and prove theorems about a general class of t-norms, we first have to provide some more background on the theory of t-norms, introduce a particular notion of monotonicity on vectors, and discuss a dual problem of Equation \ref{eq:optim-problem}. 

\begin{definition}
    A t-norm $T$ is \emph{Archimedean} if for all $x, y\in (0, 1)$, there is an $n$ such that
    $T(\underbrace{x,\dots,x}_{n\times})<y$.
  
    A continuous t-norm $T$ is \emph{strict} if, in addition, for all $x\in (0, 1)$, $0 < T(x, x) < x$. 
  \end{definition}

\subsection{Additive generators}
The study of t-norms frequently involves the study of their \emph{additive generator} \citep{klementTriangularNorms2000,klementTriangularNormsPosition2004}, which are univariate functions that construct t-norms, t-conorms, and residuums. 

\begin{definition}
    A function $\add: [0, 1]\rightarrow [0, \infty]$ such that $\add(1)=0$ is an \emph{additive generator} if it is strictly decreasing, right-continuous at 0, and if for all $t_1, t_2\in [0, 1]$, $\add(t_1)+\add(t_2)$ is either in the range of $\add$ or in $[\add(0^{+}), \infty ]$. 
\end{definition}
\begin{theorem}
    If $\add$ is an additive generator, then the function $T: [0, 1]^n\rightarrow [0, 1]$ defined as 
    \begin{equation}
        \label{eq:additive-generator}
        T(\truth) = \add^{-1}(\min(\add(0^+), \sum_{i=1}^n \add(\truths_i)))
    \end{equation}
    is a t-norm.
\end{theorem}
Using Equation \ref{eq:additive-generator}, the function $g$ acts like an invertible function. It transforms truth values into a new space that can be seen as measuring `untruthfulness'. $\sum_{i=1}^n \add(\truths_i)$ can be seen as a measure of the `untruth' of the conjunction. T-norms constructed in this way are necessarily Archimedean, and each continuous Archimedean t-norm has an additive generator. $T_P$, $T_L$ and $T_D$ have an additive generator, but $T_G$ and $T_N$ do not. Furthermore, if $\add(0^+)=\infty$, $T$ is strict and we find $T(\truth)=\add^{-1}(\sum_{i=1}^n \add(\truths_i))$. 
The residuum constructed from continuous t-norms with an additive generator can be computed using $\add^{-1}(\max(\add(c)-\add(a), 0))$ \citep{jayaramFuzzyImplications2008}. 

\subsubsection{Schur-concave t-norms}
We will frequently consider the class of Schur-concave t-norms, with their dual t-conorms and residuums formed from these Schur-concave t-norms. We denote with $\truth^\downarrow$ the truth vector $\truth$ sorted in descending order, and with $\truth^ \uparrow$ as $\truth$ sorted in ascending order.
\begin{definition}
    \label{def:schur-concave}
    A vector $\truth\in\mathbb{R}^n$ is said to \emph{majorize} another vector $\truthalt \in\mathbb{R}^ n$, denoted $\truth\succ\truthalt$, if $\sum_{i=1}^n \truths_{i}=\sum_{i=1}^n {\truthsalt_i}$ and if for each $i\in\{1, ..., n\}$ it holds that $\sum_{j=1}^i \truths_{ j}^\downarrow \geq \sum_{j=1}^i \truthsalt_{j}^\downarrow$. 
\end{definition}
\begin{definition}
    A function $[0, 1]^ n\rightarrow [0,1]$ is called \emph{Schur-convex} if for all $\truth, \truthalt\in [0, 1]^n$, $\truth\succ \truthalt$ implies that $f(\truth) \geq f(\truthalt)$. Similarly, a \emph{Schur-concave} function has that $\truth \succ \truthalt$ implies that $f(\truth) \leq f(\truthalt)$. 
\end{definition}
The dual t-conorm of a Schur-concave t-norm is Schur-convex. The three basic and continuous t-norms $T_G$, $T_P$ and $T_L$ are Schur-concave. There are also non-continuous Schur-concave t-norms, such as the Nilpotent minimum \citep{takaciSchurconcaveTriangularNorms2005,vankriekenAnalyzingDifferentiableFuzzy2022}. The drastic t-norm is an example of a t-norm that is not Schur-concave \citep{takaciSchurconcaveTriangularNorms2005}. This class includes all quasiconcave t-norms since all symmetric quasiconcave functions are also Schur-concave \citep[see p98, Prop. C.3]{marshallSchurConvexFunctions2011}. Therefore, this class constitutes a significant class of relevant t-norms. For a more precise characterization of Schur-concave t-norms, see \cite{takaciSchurconcaveTriangularNorms2005,alsinaSchurConcaveTNormsTriangle1984}. 

\subsection{Strict cone monotonicity}
\label{appendix:cone-monotonicity}
We consider strict cone-monotonicity \citep{vandykeConeMonotonicityStructure2013,clarkeSubgradientCriteriaMonotonicity1993}, which is a weak notion of strict monotonicity for higher dimensions. This intuitively means that there is always some direction we can move in to increase the value of the t-norm. Since t-norms are already non-decreasing in each argument, this implies there is no point where the t-norm is \say{flat} in all directions.
\begin{definition}
    \label{def:cone-monotone}
    A set $K\subset [0, 1]^n$ is a \emph{(convex) cone} if for every $s>0$ and $\truth\in K$ such that $s\truth\in [0, 1]^n$, also $s\truth \in K$. 

    A fuzzy evaluation operator $\op_\varphi$ is \emph{strictly cone-increasing} at $\truth\in [0, 1]^n$ if there is a nonempty cone $K(\truth)$ such that $\truth'-\truth \in K$ implies $\op_\varphi(\truth) < \op_\varphi(\truth')$.
\end{definition}

Strict cone-monotonicity is a weak notion of strict monotonicity in the sense that all t-norms that are strictly increasing in each argument are strictly cone-increasing, but the reverse need not be true. 

\begin{proposition}
    \label{prop:cone-monotone}
    If $f_\varphi$ is non-decreasing and strictly cone-increasing at $\truth \in [0, 1]^n$, there exist a nonempty cone $K'(\truth)\subseteq K(\truth)$ such that $\truth'-\truth\in K'(\truth)$ implies $\truth'_i \geq \truths_i$ for all $i\in \{0, ..., n\}$.
\end{proposition}

\begin{proof}
    Assume otherwise. Consider some $\truth'$ such that $s(\truth'-\truth)\in K(\truth)$ for $s>0$. By assumption, there is some $i\in \{0, ..., n\}$ such that $\truth'_i<\truths_i$. Consider $\hat{\truth}$ equal to $\truth'$ except that $\hat{\truth}_i=\truths_i$ for such $i$. Since $f_\varphi$ is non-decreasing in each argument, $f_\varphi(\hat{\truth})\geq f_\varphi(\truth') > f_\varphi(\truth)$, then clearly $s(\hat{\truth}-\truth)$ for $s>0$ forms the cone $K'(\truth)$.
\end{proof}

\subsection{Dual problem}
Next, we will investigate a dual problem  for the problem in Equation \ref{eq:optim-problem} that will allow us to prove some theorems:
\begin{equation}
    \label{eq:optim-problem-max}
    \begin{aligned}
        \textrm{For all } \quad & \truth\in [0,1]^n, u \in [0, \infty):  \\
        \max_{\boostv} \quad & \op_\varphi(\boostv)   \\
        \textrm{such that } \quad & \|\boostv - \truth\| = u,  \\
        & 0\leq  \boosts_i \leq 1. 
    \end{aligned}
\end{equation}
That is, instead of finding the $\boostv$ closest to $\truth$ with \boost\ value $\revis{\varphi}$, we find the largest \boosted\ value attainable with a fixed budget $u$. We need to be precise when solutions of this dual problem coincide with the problem in Equation \ref{eq:optim-problem}. Here we use the concept of strict cone-monotonicity introduced in the previous section.

\begin{theorem}
    \label{theorem:max-to-min}
   A solution $\minboostv$ for some $\op_\varphi$, $\truth$ and $u\geq 0$ of Equation \ref{eq:optim-problem-max} is also a solution to Equation $\ref{eq:optim-problem}$ for $\truth$ and $\revis{\varphi}=\op_{\varphi}(\minboostv)\geq\op_{\varphi}(\truth)$ if $\op_\varphi$ is non-decreasing in all arguments and strictly cone-increasing at each $\truth'\in [0, 1]^n$ such that $\op_\varphi(\truth')=\revis{\varphi}$, and if $\|\cdot \|$ is strictly increasing in all arguments.
\end{theorem}
\begin{proof}
    Assume otherwise, and suppose a solution $\boostv$ for Equation \ref{eq:optim-problem} exists such that $\op_\varphi(\boostv)=\revis{\varphi}$ while $\| \boostv - \truth\| < \| \minboostv  - \truth\|=u$. Since $\op_\varphi$ is non-decreasing in all arguments and $\revis{\varphi} \geq \op_\varphi(\truth)$, $\boostv - \truth$ and $\minboostv  - \truth$ are nonnegative. By Proposition \ref{prop:cone-monotone} there is some cone $K(\boostv)$ that contains a line segment $\boldsymbol{\epsilon}(s)=s(\truth'-\boostv)$ such that for all $s>0$, $\op_\varphi(\boostv) < \op_\varphi(\boostv + \mathbf{\epsilon}(s))$ and for all $i\in \{0, ..., n \}$, $0 \leq \boldsymbol{\epsilon}(s)_i$. Therefore, necessarily there is some $i$ such that $0 < \boldsymbol{\epsilon}(s)_i$. Since $\|\cdot \|$ is strictly increasing on nonnegative vectors and continuous (since it is a norm), necessarily, there are some $s>0$ such that $\|\boostv+\boldsymbol{\epsilon}(s)\|=u$. However, this is in contradiction with the premise that $\minboostv$ is a solution of Equation \ref{eq:optim-problem-max}, as $\op_\varphi(\boostv+\epsilon(s)) > \op_\varphi(\minboostv)$. 
\end{proof}

Since $f_\varphi\in[0, 1]^n \rightarrow [0, 1]$, $\op_\varphi$ cannot satisfy the conditions of Theorem \ref{theorem:max-to-min} when $\revis{\varphi}=1$. For all $\revis{\varphi}\in[0, 1)$, however, both the G\"{o}del and product t-norms and t-conorms are strictly cone-increasing. The \luk{} t-norm satisfies the conditions for $\revis{\varphi}\in (0, 1)$, since it has flat regions for $\revis{\varphi}=0$. The same reasoning can be made for the nilpotent minimum and drastic t-norms \citep{vankriekenAnalyzingDifferentiableFuzzy2022}. Furthermore, all t-norms with an additive generator are strictly cone-increasing on $\revis{\varphi} \in (0, 1)$, as are all strict t-norms.

\section{A general class of t-norms with analytical minimal \boost\ functions}
\label{sec:class-analysis}
In this section, we will introduce and discuss a general class of t-norms that have analytic solutions to the problem in Equation \ref{eq:optim-problem} to find their corresponding minimal \boost\ functions. We can find those for the t-norm, the t-conorm, and the residuum.

\subsection{Minimal \boost\ functions for Schur-concave t-norms}
We now have the background to discuss several useful and interesting results on Schur-concave t-norms. First, we present two results that characterize Schur-concave minimal \boost\ functions.  
\begin{theorem}
    \label{theorem:schur-concave-t-norm}
    Let $T$ be a Schur-concave t-norm that is strictly cone-increasing at $\revis{T}$ and let $\|\cdot \|$ be a strict norm. Then there is a minimal \boosted\ vector $\minboostv$ for $\truth$ and $\revis{T}$ such that whenever $\truths_i> \truths_j$, then $\minboosts_i - \truths_i\leq \minboosts_j - \truths_j$. 
\end{theorem}
\begin{proof}
    Assume there is a minimal \boosted\ vector $\boostv\neq\minboostv$ which has some $\boosts_i - \truths_i > \boosts_j - \truths_j$ while $\truths_i > \truths_j$. Consider $\boostv'$ equal to $\boostv$ except that $\boosts'_i=\boosts_j - \truths_j + \truths_i$ and $\boosts'_j=\boosts_i - \truths_i + \truths_j$ such that by symmetry $\|\boostv - \truth\|=\|\boostv' - \truth\|$. Define $\newmax'=\max( \boosts_i', \boosts_j')$ and $\newmin' = \min(\boosts_i', \boosts_j')$. Clearly, $\boosts_i> \newmax'\geq\newmin'> \boosts_j$. 
    We will show $\boostv$ majorizes $\boostv'$ by checking the condition of Definition \ref{def:schur-concave} for any $k\in \{1, ..., n\}$. 
    \begin{enumerate}
        \item If $\boosts^\downarrow_k> \boosts_i$, then all elements are equal and $\sum_{l=1}^ k \boosts^\downarrow_l=\sum_{l=1}^ k \boosts'^\downarrow _l$. 
        \item If $\boosts_i \geq \boosts^\downarrow_k > \newmax'$, then $\sum_{l=1}^ {k}\boosts^\downarrow_l=\sum_{l=1}^ {k-1} \boosts'^\downarrow_l + \boosts_i \geq \sum_{l=1}^ k \boosts'^\downarrow_l$.
        \item If $\newmax' \geq \boosts^\downarrow_k > \newmin'$, then $\sum_{l=1}^k \boosts^\downarrow_l > \sum_{l=1}^k \boosts'^\downarrow_l$, since by removing common terms we get $\boosts_i > \newmax'$. 
        \item If $\newmin'\geq \boosts^\downarrow_k > \boosts_j$, then removing all common terms in the sums, we are left with $\boosts_i + \boosts^\downarrow_k > \newmin'+\newmax'$. Note $\newmin' + \newmax' = \boosts_j + \truths_i - \truths_j  + \boosts_i + \truths_j-\truths_i=\boosts_i + \boosts_j$. Subtracting $\boosts_i$ from both sides, we are left with $\boosts^\downarrow_k > \boosts_j$, which is true by assumption.
        \item If $\newmin \geq \boosts^\downarrow_k$, then removing common terms, we are left with $\newmax + \newmin=\boosts_i + \boosts_j$.
    \end{enumerate}
    Therefore, $\boostv$ majorizes $\boostv'$, and so by Schur concavity, $T(\boostv, \truthc)\leq T(\boostv', \truthc)$, noting that the additional truth vector $\truthc$ will not influence the majorization result since it is applied at both sides. By Theorem \ref{theorem:max-to-min}, either 1) $T(\boostv, \truthc)< T(\boostv', \truthc)$, so $\boostv$ could not have been minimal, leading to a contradiction, or 2) $T(\boostv, \truthc)=T(\boostv', \truthc)$ and both $\boostv$ and $\boostv'$ are minimal.
\end{proof}
We note that we can make this argument in the other direction to show that any Schur-convex t-conorm will have a minimal \boosted\ vector such that $\truths_i> \truths_j$ implies $\minboosts_i \geq \minboosts_j$. Furthermore, if we know that a t-norm has a unique minimal \boost\ function, we can use this theorem to infer a useful ordering on how it changes the truth values. 

Next, we will consider the $L_1$ norm $\sum_{i=1}^n\vert \boosts_i - \truths_i\vert$, for which we can find general solutions for the t-norm, t-conorm and R-implication when the t-norm is Schur-concave. 
\begin{proposition}
    Let $\truth\in [0, 1]^n$ and let $T$ be a Schur-concave t-norm that is strictly cone-increasing at $\revis{T}\in [T(\truth, \truthc), \revismax{T}]$. Then there is a value $\lambda\in [0, 1]$ such that the vector $\minboostv$,
    \begin{equation}
        \label{eq:minboost-t-norm-schur-concave}
        \minboosts_i = \begin{cases}
            \lambda, & \text{if } \truths_i < \lambda, \\
            \truths_i, & \text{otherwise,}
        \end{cases}
    \end{equation}
    is a minimal \boosted\ vector for $T$ and the $L_1$ norm at $\truth$ and $\revis{T}$.
\end{proposition}
\begin{proof}
    Assume otherwise. Then, using Theorem \ref{theorem:max-to-min}, there must be a \boosted\ vector $\boostv$ such that $\|\boostv - \truth\|_1=\|\minboostv - \truth\|_1$ but $T( \boostv, \truthc) > T(\minboostv, \truthc)$. Since $\revis{T} \in [T(\truth, \truthc), \revismax{T}]$, we can assume $\boosts_i\geq \truths_i$. 
    We define $\pi^*(i)$ as the permutation in descending order of $\minboostv$. Furthermore, let $k$ be the smallest $j$ such that $\truths^ \downarrow_j < \lambda$. 

    Since $\|\boostv\|_1=\|\minboostv\|_1$, by assumption of equal $L_1$ norms of $\boostv$ and $\minboostv$, we will prove for all $i\in \{1, ..., n\}$ that $\boostv$ majorizes $\minboostv$. 
    \begin{itemize}
        \item If $i < k$, then $\sum_{j=1}^i\boosts^\downarrow_j\geq\sum_{j=1}^i \boosts_{\pi^ *(j)} \geq \sum_{j=1}^i \truths_{\pi^*(j)}=\sum_{j=1}^i\minboosts^\downarrow_j$. The first inequality follows from the fact that there is no ordering of $\boostv$ that will have a higher sum than in descending order. 
        \item If $i\geq k$, then clearly $\minboosts^\downarrow_i=\lambda$. Furthermore, $\sum_{j=1}^i \minboosts^\downarrow_j=\sum_{j=1}^k\truths^\downarrow_j+(i-k)\lambda$. We will distinguish two cases:
        \begin{enumerate}
            \item $\boosts^\downarrow_i \geq \lambda$. Then for all $j\in \{k, ..., i\}$, $\boosts^\downarrow_j\geq \lambda$. Furthermore, from the previous result, $\sum_{j=1}^{k-1}\boosts^\downarrow_j \geq \sum_{j=1}^ {k-1}\minboosts^\downarrow_j$ and so clearly $\sum_{j=1}^i \boosts^\downarrow_j \geq \sum_{j=1}^ i\minboosts^\downarrow_i$.
            \item $\boosts^\downarrow_i < \lambda$. Then for all $j>i$, $\boosts^\downarrow_j\leq \boosts^\downarrow_i< \lambda$, and so $\sum_{j=i+1}^n \boosts^\downarrow_j \leq \sum_{j=i+1}^ n \boosts^\downarrow_i=(n-i)\boosts^\downarrow_i<(n-i)\lambda$. Using this, we note that 
            \begin{equation*}
                \|\minboostv\|_1=\sum_{j=1}^k\truths^\downarrow_j+(n-k)\lambda=\|\boostv\|_1=\sum_{j=1}^i\boosts^\downarrow_j+\sum_{j=i+1}^n\boosts^\downarrow_j<\sum_{j=1}^i \boosts^\downarrow_j + (n-i)\lambda.
            \end{equation*} 
            Then, subtracting $(n-i)\lambda$ from the inequality, we find
            \begin{align*}
                \sum_{j=1}^i\boosts^\downarrow_j> \sum_{j=1}^k\truths^\downarrow_j+(n-k)\lambda-(n-i)\lambda=\sum_{j=1}^k\truths^\downarrow_j+(i-k)\lambda=\sum_{j=1}^i \minboosts^\downarrow_j
            \end{align*}
        \end{enumerate}
    \end{itemize}
    And so, $\boostv$ majorizes $\minboostv$, and by Schur concavity of $T$, $T(\boostv, \truthc) \leq T(\minboostv, \truthc)$ leading to a contradiction. 
\end{proof}

We found this result rather surprising: It is optimal for a large class of t-norms and the $L_1$ norm to increase the lower truth values to some value $\lambda$. In this sense, these solutions are very similar to that of the G\"{o}del \boost\ functions. The value of $\lambda$ depends on the choice of t-norm and $T(\minboostv, \truthc)$ is a non-decreasing function of $\lambda$. We show in Section \ref{sec:additive-generators} how to compute these. 

We have a similar result for the \boost\ functions of Schur-convex t-conorms. This proposition shows that, under the $L_1$ norm, it is optimal to increase only the largest literal, just like with the G\"{o}del t-norm.

\begin{proposition}
    Let $\truth\in [0, 1]^n$ and let $S$ be a Schur-convex t-conorm that is strictly cone-increasing at $\revis{S}\in [S(\truth, \truthc), 1]$. Then there is a value $\lambda\in [0, 1]$ such that the vector $\minboostv$,
    \begin{equation}
        \minboosts_i = \begin{cases}
            \lambda & \text{if } i={\arg\max}_{i\in D}\truths_i, \\
            \truths_i, & \text{otherwise,}
        \end{cases}
    \end{equation}
    is a minimal \boosted\ vector for $S$ and the $L_1$ norm at $\truth$ and $\revis{S}$.
\end{proposition}
\begin{proof}
    Assume otherwise. Then, using Theorem \ref{theorem:max-to-min}, there must be a \boosted\ vector $\boostv\neq \minboostv$ such that $\|\boostv - \truth\|_1=\|\minboostv - \truth\|_1=\lambda-\truth^\downarrow_1$ but $S(\boostv, \truthc) > S(\minboostv, \truthc)$. 
    Let $\pi(i)$ be the permutation in descending order of $\boostv$. 

 
    Consider any $k\in \{1, ..., n\}$. Then $\sum_{i=1}^k \minboosts^\downarrow_i=\sum_{i=1}^k\truths^\downarrow_i+(\lambda - \truths^\downarrow_1)$, while $\sum_{i=1}^k \boosts^\downarrow_j=\sum_{i=1}^k\truths_{\pi(i)}+\sum_{i=1}^k(\boosts_{i} - \truths_{\pi(i)})$. There is no permutation with higher sum than in descending order, so $\sum_{i=1}^k\truths_{\pi(i)} \leq \sum_{i=1}^k\truths^\downarrow_i$. Furthermore, since $\|\boostv - \truth\|_1=\lambda - \truths^\downarrow_1$, $\sum_{i=1}^k(\boosts_i - \truths_{\pi(i)})\leq \lambda - \truths^\downarrow_1$. Therefore, $\sum_{i=1}^k \boosts^\downarrow_i \leq \sum_{i=1}^k \minboosts^\downarrow_i$, that is, $\minboostv$ majorizes $\boostv$, and by Schur convexity of $S$, $S(\minboostv, \truthc) \geq S(\boostv, \truthc)$. 
 \end{proof}

\subsection{Closed forms using Additive Generators}
\label{sec:additive-generators}
Where the previous section gives general results on the form or \say{shape} of minimal \boost\ functions for t-norms and t-conorms under the $L_1$ norm, we still need to figure out what the value of $\lambda$ is for a particular $\revis{\varphi}$. Luckily, additive generators will do the job here. 
\begin{proposition}
    \label{prop:additive-generator}
    Let $T$ be a Schur-concave t-norm with additive generator $g$ and let $0<\revis{T}\in [T(\truth, \truthc), \revismax{T}]$. 
    Let $K\in \{0, ..., n-1\}$ denote the number of truth values such that $\minboosts_i=\truths_i$ in Equation \ref{eq:minboost-t-norm-schur-concave}.
    Then using

    \begin{equation}
        \lambda_K = g^ {-1}\left(\frac{1}{n-K}\left(g(\revis{T}) -\sum_{i=1}^K g(\truths^\downarrow_i) - \sum_{i=1}^m g(C_i)\right)\right)
    \end{equation}

    in Equation $\ref{eq:minboost-t-norm-schur-concave}$ gives $T(\minboostv, \truthc)=\revis{T}$ if $\minboostv\in [0, 1]^n$. 
\end{proposition}
\begin{proof}
    Using Equations \ref{eq:additive-generator} and \ref{eq:minboost-t-norm-schur-concave}, we find that 
    \begin{align*}
        T(\minboostv, \truthc)=g^{-1}\left(\min\left(g(0^+), \sum_{i=1}^K g(\truths^\downarrow_i) + \sum_{i=K+1}^n g(\lambda_K) + \sum_{i=1}^m g(C_i)\right)\right)=\revis{T}
    \end{align*}
    Since $\revis{T} > 0$, we can remove the $\min$, since $\revis{T}>0$ will require that $\sum_{i=1}^K g(\truths^\downarrow_i) + (n-K) g(\lambda_K) + \sum_{i=1}^m g(C_i)>g(0^+)$. We apply $g$ to both sides of the equation, which is allowed since $g$ is a bijection. Thus 
    \begin{align*}
        g(\revis{T}) &= \sum_{i=1}^K g(\truths^\downarrow_i) + (n-K) g(\lambda_K) + \sum_{i=1}^m g(C_i)  \\
        g(\lambda_K) &= \frac{1}{n-K}\left( g(\revis{T}) - \sum_{i=1}^K g(\truths^\downarrow_i) - \sum_{i=1}^m g(C_i) \right) \\
        \lambda_K &= g^ {-1}\left(\frac{1}{n-K}\left(g(\revis{T}) -\sum_{i=1}^K g(\truths^\downarrow_i) - \sum_{i=1}^m g(C_i)\right)\right),
    \end{align*}
    where in the last step we apply $g^{-1}$.
\end{proof}
$g(\revis{T})$ can be seen as the `untruth'-value in $g$-space that $\minboostv$ should attain. Since we have $n-K$ truth values that we can move freely, we need to make sure that their `untruth'-value in $g$-space is $g(\revis{T})/(n-K)$. However, we also need to handle the truth values we cannot change freely, which is why those are subtracted from $g(\revis{T})$. 

We should note that this does not yet give a procedure for computing the correct $K\in \{0, ..., n-1\}$. The intuition here is that we should find an $K$ such that $\truths_i \geq \lambda_K$ for the $K$ largest values, and $\truths_i < \lambda_K$ for the remaining $n-K$. Like with computing the $K^*$ for the \boost\ function for the \luk\ t-norm (Section \ref{seq:lukasiewicz}), we can do this in logarithmic time after sorting $\truth$, but we choose to compute $\lambda_K$ for each $K\in \{0, n-1\}$ in parallel. 

We can similarly find a closed form for the t-conorms: 
\begin{align}
    \lambda &= 1 - g^{-1}\left(g(1-\revis{S}) - \sum_{i\neq j}g(1-\truths_i) -  \sum_{i=1}^m g(1-C_i)\right)
\end{align}
\begin{proof}
    Let $j={\arg\max}_{i=1}^n \truths_i$. 
\begin{align*}
    S(\minboostv) = 1-g^{-1}(\min (g(0^+), \sum_{i=1}^n g(1-\minboosts_i) + \sum_{i=1}^m g(1-C_i))) &= \revis{S} \\
    \min(g(0^ +), g(1-\lambda)+ \sum_{i\neq j} g(1-\truths_i) + \sum_{i=1}^m g(1-C_i)) &= g(1-\revis{S})
\end{align*}
If $\revis{S} < 1$, or if $g(0)$ is well defined, then we can ignore the $\min$:
\begin{align*}
    g(1-\lambda) &= g(1-\revis{S}) - \sum_{i\neq j} g(1-\truths_i)) -  \sum_{i=1}^m g(1-C_i) \\
    \lambda &= 1 - g^{-1}\left(g(1-\revis{S}) - \sum_{i\neq j}g(1-\truths_i) -  \sum_{i=1}^m g(1-C_i)\right) 
\end{align*}
\end{proof}

\begin{proposition}
    Let $\truths_1, \truths_2\in [0,1]$ and let $T$ be a strict Schur-concave t-norm with additive generator $g$. Consider its residuum $R(\truths_1, \truths_2)=\sup \{z\vert T(\truths_1, z)\leq \truths_2\}$ that is strictly cone-increasing at $0<\revis{R}\in [R(\truths_1, \truths_2), \revismax{R}]$. Then there is a value $\lambda\in [0, 1]$ such that $\minboostv=[\truths_1, g^{-1}(g(\revis{R}) + g(\truths_1))]^\top$ is a minimal \boosted\ vector for $R$ and the $L_1$ norm at $\truth$ and $\weight$. 
\end{proposition}
\begin{proof}
    We will assume $\truths_1 > \truths_2$, as otherwise $R(\truths_1, \truths_2)=1$ for any residuum, which necessarily means $\revis{R}=1$ and so $\minboostv=\truth$. 
    Assume $\minboostv$ is not minimal. Since $R$ is strictly cone increasing at $\revis{R}$, by Theorem \ref{theorem:max-to-min}\footnote{This theorem has to be adjusted for the fact that fuzzy implications are non-increasing in the first argument. It can be applied by considering $1-\truths_1$.} there must be some $\boostv$ such that $\|\boostv-\truth\| = \|\minboostv - \truth\|=\lambda-\truths_2$ but $R(\boosts_1, \boosts_2)>R(\minboosts_1, \minboosts_2)$. Since $R$ is non-decreasing in the first argument and non-increasing in the second, we consider $\boostv=[\truths_1 - \epsilon, \lambda - \epsilon]^\top$ for $\epsilon>0$. 
    
    The residuum constructed from continuous t-norms with an additive generator can be computed as $R(\truths_1, \truths_2)=\add^{-1}(\max(\add(\truths_2)-\add(\truths_1), 0))$. Since we assumed $R(\minboosts_1, \minboosts_2) < R(\boosts_1, \boosts_2)$, applying $g$ to both sides,
    \begin{align*}
       \max(g(\lambda) - g(\truths_1), 0) &> \max(g(\lambda - \epsilon) - g(\truths_1 - \epsilon), 0)\\
       g^{-1}(g(\lambda) + g(\truths_1 - \epsilon)) &< g^{-1}(g(\lambda-\epsilon) + g(\truths_1))\\
       T(\lambda, \truths_1 - \epsilon) & < T(\lambda - \epsilon, \truths_1)
    \end{align*}
    where in the second step we assume $\lambda \leq \truths_1$, that is, we are not setting new consequent larger than the antecedent, as otherwise we could find a smaller \boosted\ vector by setting it to exactly $\truths_1$. In the last step we use that $T$ is strict, as then $T(\truths_1, \truths_2)=g^{-1}(g(\truths_1) + \truths_2))$. We now use the majorization as $\lambda + \truths_1 - \epsilon = \lambda - \epsilon + \truth$.

    Since $\lambda \leq \truths_1$, surely $\truths_1 > \lambda - \epsilon$. Then there are two cases: 
    \begin{enumerate}
    \item $\lambda \geq \truths_1 -\epsilon$. Then $\truths_1 \geq \lambda$ as assumed.
    \item $\truths_1 - \epsilon \geq \lambda$. Then clearly $\truth \geq \truths_1 - \epsilon$ as $\epsilon > 0$. 
    \end{enumerate}
    Therefore $[\lambda - \epsilon, \truths_1]^\top$ majorizes $[\lambda, \truth- \epsilon]^\top$, and by Schur concavity $T(\lambda, \truth - \epsilon) \geq T(\lambda - \epsilon, \truths_1)$ which is a contradiction.
\end{proof}

Here, we find that for this class of residuums, increasing the consequent (the second argument of the implication) is minimal for the $L_1$ norm. This update reflects modus ponens reasoning: When the antecedent is true, increase the consequent. As we have argued in \cite{vankriekenAnalyzingDifferentiableFuzzy2022}, this could cause issues in many machine learning setups: Consider the modus tollens correction instead decreases the antecedent. For common-sense knowledge, this is more likely to reflect the true state of the world.

\section{Experiments}

We performed experiments on two tasks. The first one does not involve learning. Instead, we aim to solve SAT problems. This experiment allows assessing whether ILR can enforce complex and unstructured knowledge. The second experiment is on the MNIST Addition task~\citep{manhaeveDeepProbLogNeuralProbabilistic2018} to test ILR in a neurosymbolic setting and assess its ability to learn from data.

\subsection{Experiments on 3SAT problems}
With this experiment, we aim to determine how quickly ILR finds a \boosted\ vector and how minimal this vector is. We test this on formulas of varying complexity to analyze for what problems each algorithm performs well\footnote{Code available at \href{https://github.com/DanieleAlessandro/IterativeLocalRefinement}{https://github.com/DanieleAlessandro/IterativeLocalRefinement}}.

\subsubsection{Setup}
We perform experiments on SATLIB \citep{hoosSATLIBOnlineResource2000}, a library of randomly generated 3SAT problems. 3SAT problems are formulas in the form $\bigwedge_{i=1}^c \bigvee_{j=1}^3 l_{ij}$, where $l_{ij}$ is a literal that is either $P_k$ or $\neg P_k$ and where $P_k\in\{P_1, ..., P_n\}$ is an input proposition. In particular, we consider uf20-91 of satisfiable 3SAT problems with $n=20$ propositions and $c=91$ disjunctive clauses. For this, we select the \boosted\ value $\revis{\varphi}$ to be 1. 
We uniformly generate initial truth values for the propositions $\truth\in [0,1]^d$~\footnote{Each run used the same initial value for each algorithm to have a fair comparison.}. To allow experimenting with formulas of varying complexity, we introduce a simplified version of the task which uses only the first 20 clauses.

We use three metrics to compare ILR with a gradient descent baseline described in Section \ref{sec:gradient-descent}. The first is speed: How many iterations does it take for each algorithm to converge? Since both algorithms have similar computational complexities, we will use the number of iterations for this. The second is satisfaction: Is the algorithm able to find a solution with truth value $\revis{\varphi}$?
Finally, we consider minimality: How close to the original prediction is the \boosted\ vector $\boostv$? Note that the \boost\ function for the product logic is only optimal for the $L_1$ norm, while
for G\"{o}del and \luk, the \boost\ function is optimal for all L$_p$ norms, including $L_1$. Moreover, the results of $L_1$ and $L_2$ are very similar. Therefore, we use the $L_1$ as a metric for minimality for each t-norm.

\subsubsection{Gradient descent baseline}
\label{sec:gradient-descent}
We compare ILR to gradient descent with the following loss function 
\begin{equation}
    \mathcal{L}(\hat{\bz}, \truth, \revis{\varphi})= \|\op_\varphi(\sigma(\hat{\bz})) - \revis{\varphi}\|_2 + \beta \|\sigma(\hat{\bz}) - \truth\|_p.
\end{equation}
Here $\boostv = \sigma(\hat{\bz})$ is a real-valued vector $\hat{\bz}\in \mathbb{R}^n$ transformed to $\boostv\in [0, 1]^n$ using the sigmoid function $\sigma$ to ensure the values of $\boostv$ remain in $[0,1]^n$ during gradient descent. The first term minimizes the distance between the current truth value of the formula $\varphi$ and the \boost\ value. In contrast, the second term is a regularization term that minimizes the distance between the \boosted\ vector and the original truth value $\truth$ in the $L_p$ norm. $\beta$ is a hyperparameter that trades off the importance of this regularization term. 

This method for finding \boosted\ vectors is very similar to the collective classification method introduced in SBR \citep{diligentiSemanticbasedRegularizationLearning2017,roychowdhuryRegularizingDeepNetworks2021}. The main difference is in the $L_p$ norms chosen, as we use squared error for the first term instead of the $L_1$ norm. 
Gradient descent is a steepest descent method that takes steps minimizing the $L_2$ norm. Therefore, it can also be seen as a method for finding minimal \boost\ functions given the $L_2$ norm. The coordinate descent algorithm is the corresponding steepest descent method for the $L_1$ norm. Future work could compare how coordinate descent performs for finding minimal \boost\ functions for the $L_1$ norm. We suspect it will be much slower than gradient descent-based methods as it can only change a single truth value each iteration.  

We found that ADAM \citep{kingmaAdamMethodStochastic2017} significantly outperformed standard gradient descent in all metrics, and we chose to use it throughout our experiments. Furthermore, inspired by the analysis of the derivatives of aggregation operators in \cite{vankriekenAnalyzingDifferentiableFuzzy2022}, we slightly change the formulation of the loss function for the \luk\ t-norm and product t-norm.
The \luk\ t-norm will have precisely zero gradients for most of its domain. Therefore, we remove the $\max$ operator when evaluating the $\bigwedge$ in the SAT formula, so it has nonzero gradients. For the product t-norm, the gradient will also approach 0 because of the large set of numbers between $[0, 1]$ that it multiplies. As suggested by \cite{vankriekenAnalyzingDifferentiableFuzzy2022}, we instead optimize the logarithm of the product t-norm:
\begin{equation*}
    \mathcal{L}_P(\hat{\bz}, \truth, \revis{\varphi})= \| \sum_{i=1}^c \log f_{\bigvee_{j=1}^3}(\sigma(\truth)) - \log \revis{\varphi}\|_2 + \beta \|\sigma(\hat{\bz}) - \truth\|_1.
\end{equation*}

\subsubsection{Results}
\label{sec:results-sat}
\begin{figure}
    \includegraphics[width=\linewidth]{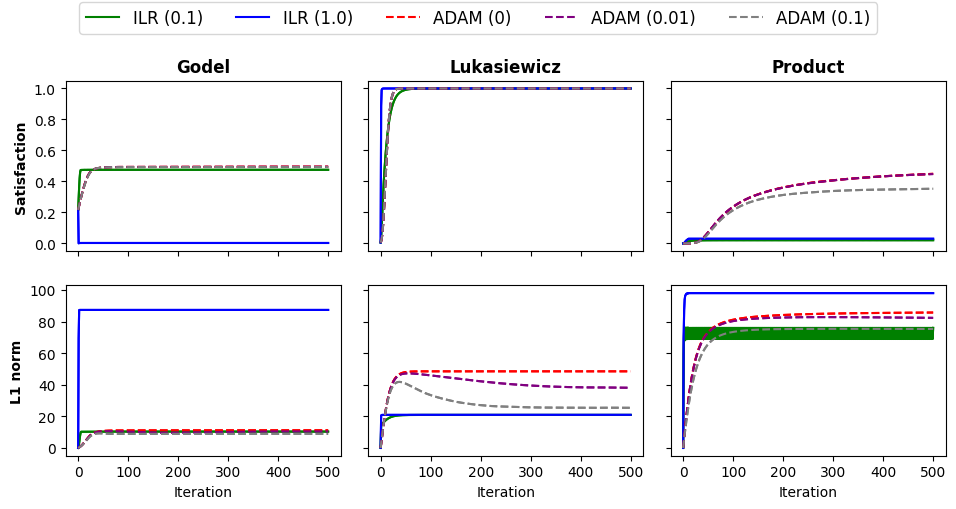}
    \caption[Comparison of ILR with ADAM on uf20-91 in SATLIB]{Comparison of ILR with ADAM on uf20-91 in SATLIB. The target \Boosted\ value is 1.0. The x-axis corresponds to the number of iterations, while the y-axis is the value of $\revis{\varphi}$ in the first row of the grid and the $L_1$ norm in the second row. The number in parentheses represents the schedule parameter for ILR and the regularization parameter $\beta$ for ADAM. Note that ADAM's plots often almost perfectly overlap.}
    \label{fig:results_91}
\end{figure}

In Figure~\ref{fig:results_91}, we show the results obtained by ILR and ADAM on the three t-norms (one for each grid column). 
We observe that ILR with schedule parameter $\alpha=0.1$ has a smoother plot than ILR with $\alpha=1.0$, which converges faster: In our experiments, the number of steps until convergence was always between 2 and 5. For both values of the scheduling parameters, ILR outperforms ADAM in terms of convergence speed.

When comparing satisfaction and minimality, the behaviour differs based on the t-norm. In the case of \luk, all methods find feasible solutions to the optimization problem. Furthermore, in terms of minimality (i.e., $L_1$ norm), ILR finds better solutions than ADAM.

For the G\"{o}del logic, no method can reach a feasible solution. Here, ILR with schedule parameter $\alpha=1$ performs very poorly, obtaining worse solutions than the original truth values. On the other hand, with $\alpha = 0.1$, it performs as well as ADAM for both metrics but with faster convergence.

Finally, for the product logic, ILR fails to increase the satisfaction of the formula to the \boosted\ value. However, ADAM can find much better solutions, getting the average truth value to around 0.5. Still, it is far from reaching a feasible solution. Nonetheless, we recommend using ADAM for complicated formulas in the product logic.

However, we argue that in the context of Neural-Symbolic Integration, the provided knowledge is usually relatively easy to satisfy. With 91 clauses, there are few satisfying solutions in this space of $2^ {21}$ possible binary solutions. However, background knowledge usually does not constrain the space of possible solutions as heavily as this. For this reason, we propose a simplified formula, where we only use 20 out of 91 clauses. Figure~\ref{fig:results_20} shows the results for this setting. We see that ILR with no scheduling ($\alpha=1$) finds feasible solutions for all t-norms. ILR finds solutions for the G\"{o}del t-norm where ADAM cannot find any, while for \luk\ and product, it finds solutions in much fewer iterations and with a lower $L_1$ norm. Hence, we argue that for knowledge bases that are less constraining, ILR without scheduling is the best choice.

\begin{figure}
    \includegraphics[width=\linewidth]{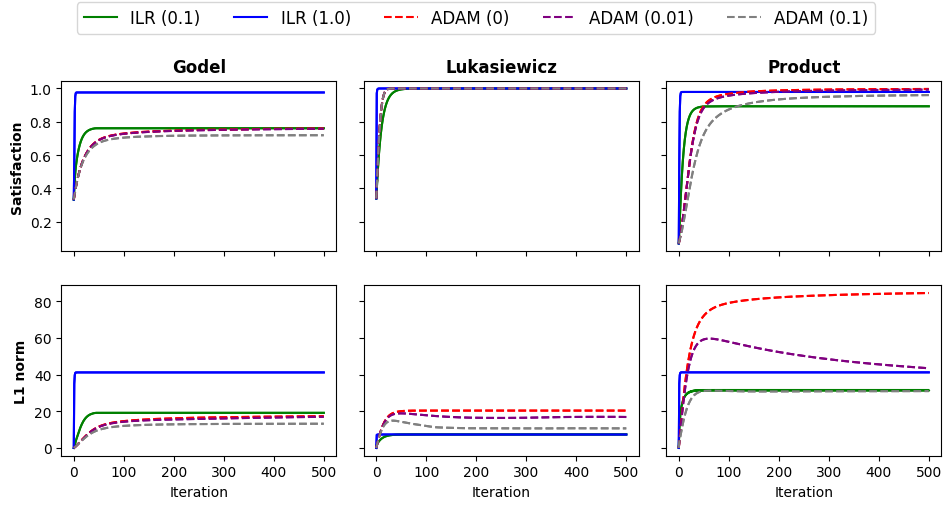}
    \caption{Comparison of ILR with ADAM on the uf20-91 with 20 clauses. Target value 1.0.}
    \label{fig:results_20}
\end{figure}

\subsection{Experiments on MNIST Addition}
\label{sec:MNIST_exp}
The experiments on the SATLIB benchmark show how well ILR can enforce knowledge in highly constrained settings. However, as already mentioned, in neurosymbolic AI, the background knowledge is typically much simpler. SAT benchmarks often only have a few solutions, heavily limiting what predictions the neural network can make.
Moreover, previous experiments only tested ILR where initial truth vectors are random, and we did not have any neural networks or learning. 

To evaluate the performance of ILR in neurosymbolic settings, we implemented the architecture of Figure~\ref{fig:ILR_MNIST}. Here, the task is to learn a classifier for handwritten digits while only receiving supervision on the sums of pairs of digits.

\subsubsection{Setup}

We follow the architecture of Figure~\ref{fig:ILR_MNIST}. We use the neural network proposed by~\cite{manhaeveDeepProbLogNeuralProbabilistic2018}, which is a network composed of two convolutional layers, followed by a MaxPool layer, followed by a fully connected layer with ReLU activation function and a fully connected layer with softmax activation. We use the G\"{o}del t-norm and corresponding minimal \boost\ functions. Note that G\"{o}del implication can only increase the consequent and can never decrease the antecedents. For this reason, ILR converges in a single step.

We set both $\alpha$ and target value $\hat{t}$ to one, meaning that we ask ILR to satisfy the entire formula in one step. We use the ADAM optimizer and a learning rate of 0.01, with the cross-entropy loss function. However, since the outputs of the ILR step do not sum to one, we cannot directly apply it to the \boosted\ vector ILR computes. To overcome this issue, we add a logarithm followed by a softmax as the last layers of the model. If the sum of the \boosted\ vector is one, the composition of the logarithm and softmax functions corresponds to the identity function. Moreover, these two layers are monotonic increasing functions and preserve the order of the \boosted\ vector.

We use the dataset defined in~\cite{manhaeveDeepProbLogNeuralProbabilistic2018} with 30000 samples, and also run the experiment using only 10\% of the dataset (3000 samples). We run ILR for 5 epochs on the complete dataset, and 30 epochs on the small one. We repeat this experiment 10 times.
We are interested in the accuracy obtained in the test set for the addition task.
We ran the experiments on a MacBook Pro (2016) with a 3,3 GHz Dual-Core Intel Core i7.

\subsubsection{Results}

\begin{table}[]
    \centering
    \begin{tabular}{l|l|l}
         & 30000 & 3000 \\
         \hline
        DeepProblog~\citep{manhaeveDeepProbLogNeuralProbabilistic2018} & 97.20 $\pm$ 0.45 & 92.18 $\pm$ 1.57 \\
        LTN~\citep{badreddineLogicTensorNetworks2022} & 96.78 $\pm$ 0.5 & 92.15 $\pm$ 0.75 \\
        ILR & 96.67 $\pm$ 0.45 & 93.38 $\pm$ 1.70\\
    \end{tabular}
    \caption[ILR results on the MNIST addition task.]{Results on the MNIST addition task. We report the accuracy of predicting the sum (in \%) on the test set with 30000 and 3000 samples. DeepProbLog results are taken from~\cite{badreddineLogicTensorNetworks2022}. LTN results have been obtained by replicating the experiments of~\cite{badreddineLogicTensorNetworks2022}.}
    \label{tab:results}
\end{table}

ILR can efficiently learn to predict the sum, reaching results similar to state of the art, requiring, on average, 30 seconds per epoch. However, sometimes ILR got stuck in a local minimum during training, where the accuracy reached was close to 50\%. It is worth noticing that LTN suffers from the same problem~\citep{badreddineLogicTensorNetworks2022}, with results strongly dependent on the initialization of the parameters. 
To better understand this local minimum, we analyzed the confusion matrix.
\begin{figure}
    \centering
    \includegraphics[width=7cm]{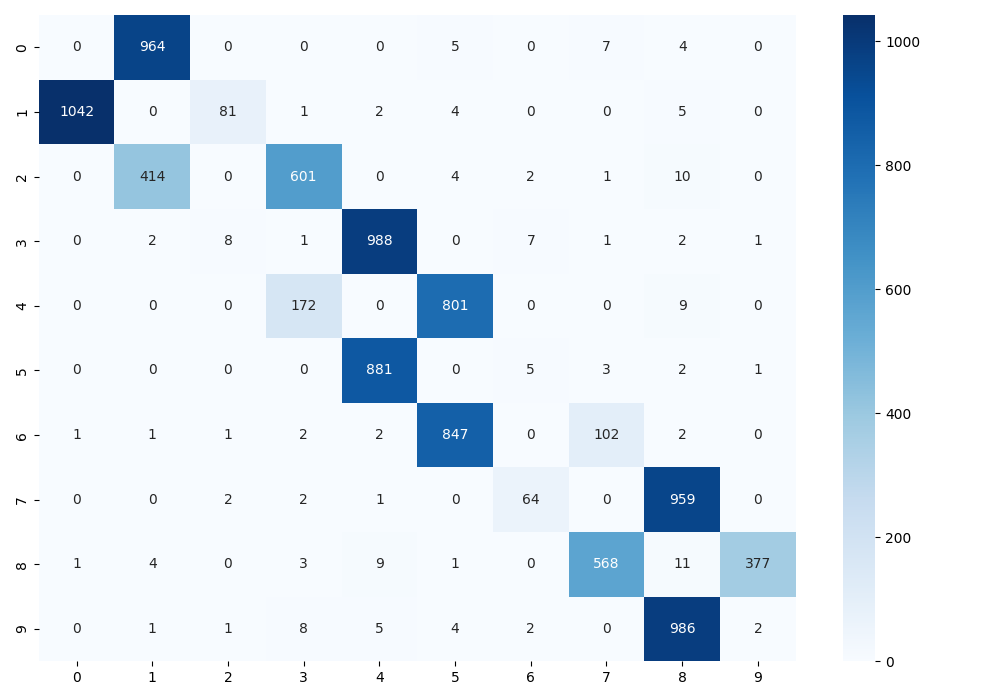}
    \caption{Confusion matrix on the MNIST classification for a local minimum}
    \label{fig:confusion}
\end{figure}
Figure~\ref{fig:confusion} shows one of the confusion matrices for a model stuck in the local minimum: the CNN recognizes each digit either as the correct digit minus one or plus one. Then, our model obtains the correct prediction in close to 50\% of the cases. For example, suppose the digits are a 3 and a 5. The 3 is classified as a 2 or a 4, while the 5 is classified as a 4 or a 6. If the model predicts 2 and 6 or 4 and 4, it returns the correct sum (8). Otherwise, it does not. We believe that in these local minima, there is no way for the model to change the digit predictions without increasing the loss, and the model remains stuck in the local minimum.

Table~\ref{tab:results} shows the results in terms of accuracy of ILR, LTN~\citep{badreddineLogicTensorNetworks2022} and DeepProblog~\citep{manhaeveDeepProbLogNeuralProbabilistic2018}. To calculate the accuracy, we follow~\cite{badreddineLogicTensorNetworks2022} and select only the models that do not stop in a local minimum. Notice that this problem is rare for ILR (once every 30 runs) and happens more frequently with LTN (once every 5 runs).

\section{Conclusion and Future Work}
We analytically studied a large class of minimal fuzzy \boost\ functions. We used \boost\ functions to construct ILR, an efficient algorithm for general formulas. Another benefit of these analytical results is to get a good intuition into what kind of corrections are done by each t-norm.
In our experimental evaluation of this algorithm, we found that our algorithm converges much faster and often finds better solutions than the baseline ADAM, especially for less constraining problems. However, we conclude that for complicated formulas and product logic, ADAM finds better results. Finally, we assess ILR on the MNIST Addition task and show it can be combined with a neural network, providing results similar to two of the most prominent methods for neurosymbolic AI.

There is a lot of opportunity for future work on \boost\ functions. We will study how the \boost\ functions induced by different t-norms perform in practical neurosymbolic integration settings. On the theoretical side, possible future work could be considering analytical \boost\ functions for certain classes of complex formulas. Furthermore, there are many classes of t-norms and norms for which finding analytical \boost\ functions is an open problem. Another promising avenue for research is designing specialized loss functions that handle biases in the gradients arising from combining constrained output layers with cross-entropy loss functions \citep{giunchigliaMultiLabelClassificationNeural2021a}. We also want to highlight the possibility of extending the work on fuzzy \boost\ functions to probabilistic \boost\ functions, using a notion of minimality such as the KL-divergence.


\newpage
\thispagestyle{empty}
\refstepcounter{parts}\label{part:2}
\vspace*{\fill} 
\begin{center}
    \Huge Part II:\\
    \underline{Probabilistic Neurosymbolic Learning}
\end{center}
\addcontentsline{toc}{chapter}{Part II: Probabilistic Neurosymbolic Learning}
\vspace*{\fill}

\chapter[Storchastic]{Storchastic: A Framework for General Stochastic Automatic Differentiation}
\label{ch:storchastic}
\begin{paperbase}
	This chapter is based on the NeurIPS 2021 article \cite{vankriekenStorchasticFrameworkGeneral2021}.
\end{paperbase}

\begin{abstract}
    Modelers use automatic differentiation (AD) of computation graphs to implement complex deep learning models without defining gradient computations.
    Stochastic AD extends AD to stochastic computation graphs with sampling steps, which arise when modelers handle the intractable expectations common in reinforcement learning and variational inference.
    However, current methods for stochastic AD are limited: 
    They are either only applicable to continuous random variables and differentiable functions, or can only use simple but high variance score-function estimators.
    To overcome these limitations, we introduce \emph{Storchastic}, a new framework for AD of stochastic computation graphs.
    \emph{Storchastic} allows the modeler to choose from a wide variety of gradient estimation methods at each sampling step, to optimally reduce the variance of the gradient estimates.
    Furthermore, \emph{Storchastic} is provably unbiased for estimation of any-order gradients, and generalizes variance reduction techniques to any-order derivative estimates. 
    Finally, we implement \emph{Storchastic} as a PyTorch library  at \url{github.com/HEmile/storchastic}.
\end{abstract}

\section{Introduction}
    One of the driving forces behind deep learning is automatic differentiation (AD) libraries of complex computation graphs. 
    Deep learning modelers are relieved by accessible AD of the need to implement complex derivation expressions of the computation graph. 
    However, modelers are currently limited in settings where the modeler uses intractable expectations over random variables \citep{mohamedMonteCarloGradient2020, correiaEfficientMarginalizationDiscrete2020}. 
    Two common examples are reinforcement learning methods using policy gradient optimization \citep{williamsSimpleStatisticalGradientfollowing1992,lillicrapContinuousControlDeep2016, mnihAsynchronousMethodsDeep2016} 
    and latent variable models, especially when inferred using amortized variational inference \citep{mnihNeuralVariationalInference2014, kingmaAutoencodingVariationalBayes2014, ranganathBlackBoxVariational2014,rezendeStochasticBackpropagationApproximate2014}.
    Typically, modelers estimate these expectations using Monte Carlo methods, that is, sampling,
    and resort to gradient estimation techniques \citep{mohamedMonteCarloGradient2020} to differentiate through the expectation.

    A popular approach for stochastic AD is reparameterization~\citep{kingmaAutoencodingVariationalBayes2014}, which is both unbiased and has low variance, but is limited to continuous random variables and differentiable functions. 
    The other popular approach~\citep{williamsSimpleStatisticalGradientfollowing1992,schulmanGradientEstimationUsing2015,foersterDiCEInfinitelyDifferentiable2018} analyzes the computation graph and then uses the score function estimator to create a \emph{surrogate loss} that provides gradient estimates when differentiated. 
    While this approach is more general as it can also be applied to discrete random variables and non-differentiable functions, naive applications of the score function will have high variance, which leads to unstable and slow convergence.
    Furthermore, this approach is often implemented incorrectly \citep{foersterDiCEInfinitelyDifferentiable2018}, which can introduce bias in gradients.
    
    \begin{figure}
        \includegraphics[width=\linewidth]{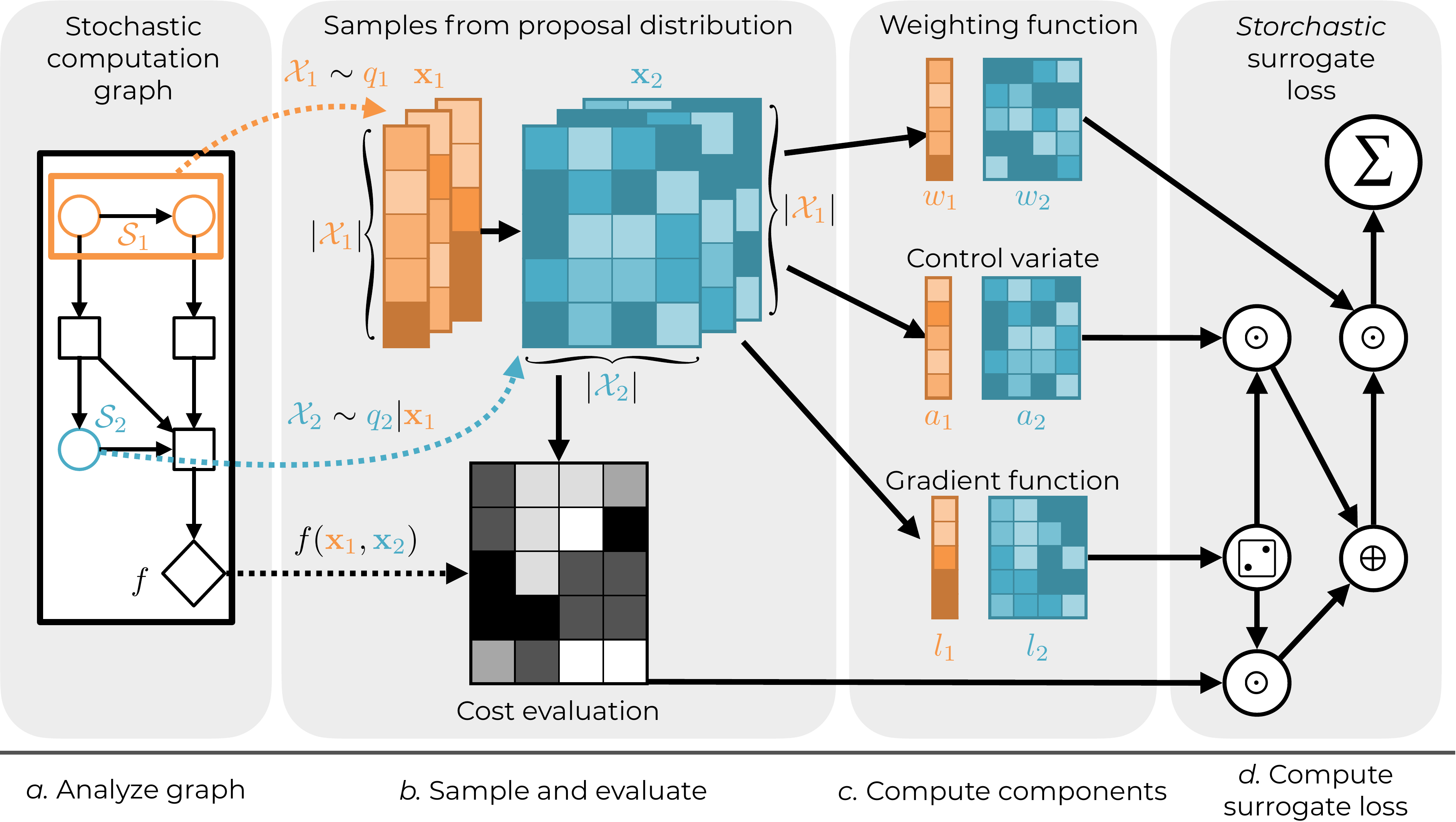}
        \caption[An illustration of the (parallelized) \emph{Storchastic} loss computation]{An illustration of the (parallelized) \emph{Storchastic} loss computation. 
        \emph{a.} Assign the stochastic nodes of the input stochastic computation graph  (SCG) into two topologically sorted partitions. 
        \emph{b.} Evaluate the SCG. We first sample the set of values $\mathcal{X}_1$ from the \tproposal{}. For each of the samples $\bx_i\in\mathcal{X}_1$, we then sample a set of samples $\mathcal{X}_2$. The rows in the figure indicate different samples in $\mathcal{X}_1$, while the columns indicate samples in $\mathcal{X}_2$. The different samples are used to evaluate the cost function $f$ $|\mathcal{X}_1|\cdot |\mathcal{X}_2|$ times. 
        \emph{c.} Compute the weighting function, \tadditive{} and \tmultipl{} for all samples. 
        \emph{d.} Using those components and the cost function evaluation, compute the \emph{storchastic} surrogate loss, mimicking Algorithm \ref{alg:storchastic}. $\odot$ refers to element-wise multiplication, $\oplus$ to element-wise summation and $\sum$ for summing the entries of a matrix. 
        }
        \label{fig:surrogate-loss}
    \end{figure}

    We therefore develop a new framework called \emph{Storchastic} to support deep learning modelers. 
    They can use \emph{Storchastic} to focus on defining stochastic deep learning models without having to worry about complex gradient estimation implementations.
    \emph{Storchastic} extends DiCE \citep{foersterDiCEInfinitelyDifferentiable2018} to other gradient estimation techniques than basic applications of the score function. 
    It defines a surrogate loss by decomposing gradient estimation methods into four components: The \tproposal{}, weighting function, \tmultipl{} and \tadditive{}.  
    We can use this decomposition to get insight into how gradient estimators differ, and use them to further reduce variance by adapting components of different gradient estimators.

    Our main contribution is a framework with a formalization and a proof that, if the components satisfy certain conditions, performing $n$-th order differentiation on the \emph{Storchastic} surrogate loss gives unbiased estimates of the $n$-th order derivative of the stochastic computation graph.
    We show these conditions hold for a wide variety of gradient estimation methods for first order differentiation.
    For many score function-based methods like RELAX \citep{grathwohlBackpropagationVoidOptimizing2018}, MAPO \citep{liangMemoryAugmentedPolicy2018} and the unordered set estimator \citep{koolEstimatingGradientsDiscrete2020}, the conditions also hold for any-order differentiation. 
    In \emph{Storchastic}, we only have to prove these conditions locally. 
    This means that modelers are free to choose the gradient estimation method that best suits each sampling step, while guaranteeing that the gradient remains unbiased. 
    \emph{Storchastic} is the first stochastic AD framework to incorporate the measure-valued derivative \citep{pflugSamplingDerivativesProbabilities1989,heidergottMeasurevaluedDifferentiationMarkov2008,mohamedMonteCarloGradient2020} and SPSA \citep{spallMultivariateStochasticApproximation1992, bhatnagarStochasticRecursiveAlgorithms2013}, and the first to guarantee variance reduction of any-order derivative estimates through \tadditive{}s.

    In short, our contributions are the following:
    \begin{enumerate}
        \item We introduce \emph{Storchastic}, a new framework for general stochastic AD that uses four gradient estimation components, in Section \ref{sec:requirements}-\ref{sec:surrogate}.
        \item We prove Theorem \ref{thrm:storchastic-informal}, which provides conditions under which \emph{Storchastic} gives unbiased any-order derivative estimates in Section \ref{sec:storchastic-conditions}. To this end, we introduce a mathematical formalization of forward-mode evaluation in AD libraries in Section \ref{sec:forward-mode}.
        \item We derive a technique for extending variance reduction using \tadditive{}s to any-order derivative estimation in Section \ref{sec:var-reduction}.
        \item We implement \emph{Storchastic} as an open source library for PyTorch, Section \ref{sec:implementation}.
    \end{enumerate}

\section{Background}
We use capital letters $\node, \detnode, S_1, ..., S_k$ for nodes in a graph, calligraphic capital letters $\stochastic, \deterministic$ for sets and non-capital letters for concrete computable objects such as functions $f$ and values $\vals{i}$.
\subsection{Stochastic Computation Graphs}
We start by introducing Stochastic Computation Graphs (SCGs) \citep{schulmanGradientEstimationUsing2015}, which is a formalism for stochastic AD.
A \textit{Stochastic Computation Graph} (SCG) is a directed acyclic graph (DAG) $\scg=\left(\nodes, \edges\right)$ where nodes $\nodes$ are partitioned in \textit{stochastic nodes} $\stochastic$ and \textit{deterministic nodes} $\deterministic$. 
We define the set of \textit{parameters} $\parameters\subseteq\deterministic$ such that  all $\theta\in\parameters$ have no incoming edges, and the set of \textit{cost nodes} $\costs\subseteq\deterministic$ such that all $c \in \costs$ have no outgoing edges. 

The set of \textit{parents} $\pa(\node)$ is the set of incoming nodes of a node $\node\in \nodes$, that is $\pa(\node)=\{\nodealt\in\nodes|(\nodealt, \node)\in \edges\}$. 
Each \textit{stochastic node} $\stochnode\in\stochastic$ represents a random variable with \emph{sample space} $\sspace_\stochnode$ and probability distribution $p_\stochnode$ conditioned on its parents.
    Each \textit{deterministic node} $\detnode$ represents a (deterministic) function $\func_\detnode$ of its parents. 
%

$\nodealt$ \emph{influences} $\node$, denoted $\nodealt\influences \node$, if there is a directed path from $\nodealt$ to $\node$. 
We denote with $\nodes_{\before\node}=\{\nodealt\in\nodes|\nodealt\influences\node \}$ the set of nodes that influence $\node$.
    The \textit{joint probability} of all random variables $\bx_\stochastic\in\prod_{\stochnode\in\stochastic}\sspace_\stochnode$ is defined as 
    $p(\bx_\stochastic)=\prod_{\stochnode\in\stochastic}p_\stochnode(\bx_\stochnode|\bx_{\pa(\stochnode)})$, where $\vals{\pa(\stochnode)}$ is the set of values of the nodes $\pa(\stochnode)$.
    The \textit{expected value} of a deterministic node $\detnode\in\deterministic$ is its expected value over sampling stochastic nodes that influence that node, that is, 
    \begin{equation}
    \mathbb{E}[\detnode]=\mathbb{E}_{\stochastic_{\before \detnode}}[\func_\detnode(\pa(F))]=\int_{\sspace_{\stochastic_{\before\detnode}}} p(\vals{\before\stochastic}) \func_\detnode(\vals{\pa(\detnode)})d\vals{\stochastic_{\before \detnode}}.
    \end{equation}
\subsection{Problem Statement}
In this paper, we aim to define a \emph{surrogate loss} that, when differentiated using an AD library, gives an unbiased estimate of the $n$-th order derivative of a parameter $\theta$ with respect to the expected total cost $\nabla_\theta^{(n)}\mathbb{E}[\sum_{\cost\in\costs}\cost]$.
This gradient can be written as $\sum_{\cost\in\costs}\nabla_\theta^{(n)}\mathbb{E}[\cost]$, and we focus on estimating the gradient of a single cost node $\nabla_\theta^{(n)}\mathbb{E}[\cost]$. 
 
\subsection{Example: Discrete Variational Autoencoder}
\label{sec:example-vae}
\begin{figure}
    \centering
    \includegraphics[width=0.6\linewidth]{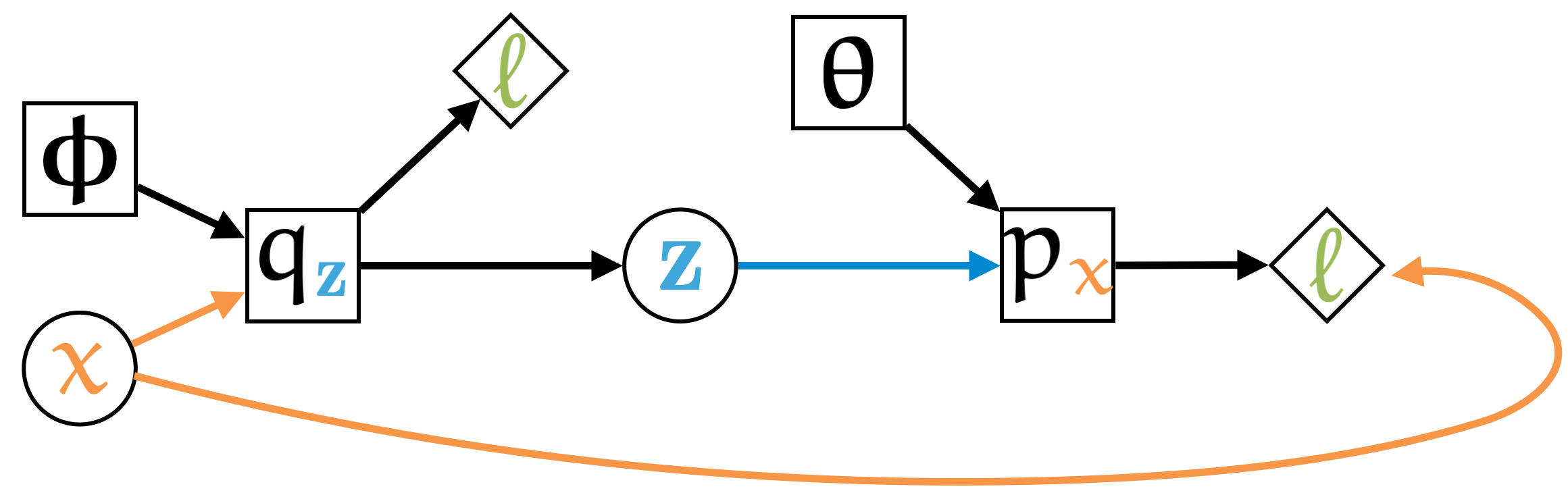}
    \caption{A Stochastic Computation Graph representing the computation of the losses of an VAE with a discrete latent space. }
    \label{fig:VAE}
\end{figure}
Next, we introduce a running example: A variational autoencoder (VAE) with a discrete latent space \citep{kingmaAutoencodingVariationalBayes2014, jangCategoricalReparameterizationGumbelsoftmax2017}. 
First, we represent the model as an SCG: The \textit{deterministic} nodes are $\mathcal{N}=\{ \phi, \theta,  q_{z}, p_{x}, \ell_{KLD}, \ell_{Rec}\}$ and the \textit{stochastic} nodes are $\mathcal{S}=\{x, z\}$. 
These are connected as shown in Figure \ref{fig:VAE}. 
The \textit{parameters} are $\Theta = \{\theta, \phi\}$ which respectively are the parameters of the variational posterior $q$ and the model likelihood $p$, and the \textit{cost nodes} $\mathcal{C}=\{\ell_{KLD}, \ell_{Rec}\}$ that represent the KL-divergence between the posterior and the prior, and the `reconstruction loss', or the model log-likelihood after decoding the sample $z$. 
Finally, $q_z$ represents the parameters of the multivariate categorical distribution of the amortized variational posterior $q_\phi(z|x)$.
This SCG represents the equation
\begin{equation}
    \mathbb{E}_{x, z}[\ell_{KLD} + \ell_{Rec}]
    = \mathbb{E}_x[\ell_{KLD}] + \mathbb{E}_{x, z\sim q_\phi(z|x)}[\ell_{Rec}].
\end{equation}
The problem we are interested in is estimating the gradients of these expectations with respect to the parameters. 
Since $x$ is not influenced by the parameters, we have $\nabla_\theta^{(n)} \mathbb{E}_x[\ell_{KLD}]=0$ and $\nabla_\phi^{(n)} \mathbb{E}_x[\ell_{KLD}]=\mathbb{E}_x[\nabla_\phi^{(n)} \ell_{KLD}]$. 
The second term is more challenging.
We can move the gradient with respect to $\theta$ in, since $z$ is not influenced by $\theta$: $\nabla_\theta^{(n)}\mathbb{E}_{x, z\sim q_\phi(z|x)}[\ell_{Rec}] = \mathbb{E}_{x, z\sim q_\phi(z|x)}[\nabla_\theta^{(n)}\ell_{Rec}]$.
However, we cannot compute $\nabla_\phi^{(n)}\mathbb{E}_{x, z\sim q_\phi(z|x)}[\ell_{Rec}]$ without gradient estimation methods. This is because sampling from $q_\phi(z|x)$ is dependent on $\phi$. 
Furthermore, since we are dealing with a discrete stochastic node, we cannot apply the reparameterization method here without introducing bias.


\subsection{Formalizing AD libraries and DiCE}
\label{sec:forward-mode}
To be able to properly formalize and prove the propositions in this paper, we introduce the `forward-mode' operator that simulates forward-mode evaluation using AD libraries. 
This operator properly handles the common `stop-grad' operator, which ensures that its argument is only evaluated during forward-mode evaluations of the computation graph.
It is implemented in Tensorflow and Jax with the name \texttt{stop\_gradient} \citep{abadiTensorFlowLargescaleMachine2015,bradburyJAXComposableTransformations2018} and in PyTorch as \texttt{detach} or \texttt{no\_grad} \citep{paszkePyTorchImperativeStyle2019}. 
`stop-grad' is necessary to define surrogate losses for gradient estimation, which is why it is essential to properly define it.
For formal definitions of the following operators and proofs we refer the reader to Appendix \ref{sec:appendix-forward-mode}. 
\begin{definition}[informal]
    The \emph{stop-grad} operator $\bot$ is a function such that $\nabla_x\bot(x)=0$.
    The \emph{forward-mode} operator $\forward{}$, which is denoted as an arrow above the argument it evaluates, acts as an identity function, except that $\forward{\bot(a)}=\forward{a}$.
    Additionally, we define the \magicbox{} operator as $\magic(x)=\exp(x-\bot(x))$.
\end{definition}
Importantly, the definition of $\forward{}$ implies that $\forward{\nabla_x f(x)}$ does not equal $\nabla_x \forward{f(x)}$ if $f$ contains a stop-grad operator. 
\magicbox{}, which was first introduced in \cite{foersterDiCEInfinitelyDifferentiable2018}, is particularly useful for creating surrogate losses that remain unbiased for any-order differentiation. 
It is defined such that $\forward{\magic(x)}=1$ and $\nabla_x\magic(f(x))=\magic(f(x))\nabla_x f(x)$. 
This allows injecting multiplicative factors to the computation graph only when computing gradients.

Making use of \magicbox{}, DiCE \citep{foersterDiCEInfinitelyDifferentiable2018} is an estimator for automatic $n$th-order derivative estimation that defines a surrogate loss using the score function: 
\begin{equation}
    \label{eq:DiCE}
    \nabla_\theta^{(n)}\mathbb{E}[\sum_{\cost\in\costs}\cost] = \mathbb{E}\bigg[\forward{\nabla_\theta^{(n)} \sum_{\cost\in\costs}\magic\big(\sum_{\stochnode \in \stochastic_{\before\cost}}\log p(\vals{\stochnode}|\vals{\pa(\stochnode)})\big)\cost}\bigg].
\end{equation}
DiCE correctly handles the credit assignment problem: The score function is only applied to the stochastic nodes that influence a cost node.
It also handles pathwise dependencies of the parameter through cost functions.
However, it has high variance since it is based on a straightforward application of the score function.



\section{The \emph{Storchastic} Framework}
\label{sec:storchastic}
In this section, we introduce \emph{Storchastic}, a framework for general any-order gradient estimation in SCGs that gives modelers the freedom to choose a suitable gradient estimation method for each stochastic node. 
First, we present 5 requirements that we used to develop the framework in Section \ref{sec:requirements}. 
\emph{Storchastic} deconstructs gradient estimators into four components that we present in Section \ref{sec:components}. 
We use these components to introduce the \emph{Storchastic} surrogate loss in Section \ref{sec:surrogate}, and give conditions that need to hold for unbiased estimation in Section \ref{sec:storchastic-conditions}.
In Section \ref{sec:var-reduction} we discuss variance reduction, in Section \ref{sec:estimators} we discuss several estimators that fit in \emph{Storchastic}, and in Section \ref{sec:implementation} we discuss our PyTorch implementation.
An overview of our approach is outlined in Figure \ref{fig:surrogate-loss}.

\subsection{Requirements of the \emph{Storchastic} Framework}
\label{sec:requirements}
First, we introduce the set of requirements we used to develop \emph{Storchastic}.
\begin{enumerate}
    \item \label{req:plug} Modelers should be able to choose a different gradient estimation method for each stochastic node. 
    This allows for choosing the method best suited for that stochastic node, or adding background knowledge in the estimator.
    \item \label{req:estimators} \emph{Storchastic} should be flexible enough to allow implementing a wide range of reviewed gradient estimation methods, including score function-based methods with complex sampling techniques \citep{yinARSMAugmentREINFORCESwapMergeEstimator2019,koolEstimatingGradientsDiscrete2020,liangMemoryAugmentedPolicy2018} or \tadditive{}s \citep{grathwohlBackpropagationVoidOptimizing2018, tuckerREBARLowvarianceUnbiased2017}, and other methods such as measure-valued derivatives \citep{heidergottMeasurevaluedDifferentiationMarkov2008,pflugSamplingDerivativesProbabilities1989} and SPSA \citep{spallMultivariateStochasticApproximation1992} which are missing AD implementations \citep{mohamedMonteCarloGradient2020}.
    \item \label{req:surrogate} \emph{Storchastic} should define a \emph{surrogate loss} \citep{schulmanGradientEstimationUsing2015}, which gives gradients of the SCG when differentiated using an AD library. This makes it easier to implement gradient estimation methods as modelers get the computation of derivatives for free.
    \item \label{req:higher-order} Differentiating the surrogate loss $n$ times should give estimates of the $n$th-order derivative, which are used in for example reinforcement learning \citep{furmstonApproximateNewtonMethods2016, foersterLearningOpponentLearningAwareness2018} and meta-learning \citep{finnModelAgnosticMetaLearningFast2017, liMetaSGDLearningLearn2017}.
    \item \label{req:variance} Variance reduction methods through better sampling and \tadditive{}s should generalize in higher-order derivative estimation.
    \item \label{req:unbiased} \emph{Storchastic} should be provably unbiased. To reduce the effort of developing new methods, researchers should only have to prove a set of local conditions that generalize to any SCG. 
\end{enumerate}

\subsection{Gradient Estimators in Storchastic}
\label{sec:components}
Next, we introduce each of the four components and motivate why each is needed to ensure Requirement \ref{req:estimators} is satisfied.
First, we note that several recent gradient estimators, like MAPO \citep{liangMemoryAugmentedPolicy2018}, unordered set estimator \citep{koolEstimatingGradientsDiscrete2020} and self-critical baselines \citep{koolAttentionLearnSolve2019,rennieSelfCriticalSequenceTraining2017} act on sequences of stochastic nodes instead of on a single stochastic node. 
Therefore, we create a partition $\stochastic_1, ..., \stochastic_k$ of $\stochastic_{\before\cost}$ topologically ordered by the influence relation, and define the shorthand $\vals{i}:= \vals{\stochastic_i}$.
For each partition $\stochastic_i$, we choose a \emph{gradient estimator}, which is a 4-tuple $\gradestim$. Here, $\fproposalcond{i}$ is the \emph{\tproposal{}}, $\fweight{i}$ is the \emph{weighting function}, $\fmultipl{i}$ is the \emph{\tmultipl{}} and $\additive{i}$ is the \emph{\tadditive{}}. 

\textbf{Proposal distribution} Many gradient estimation methods in the literature do not sample a single value $\vals{i}\sim p(\vals{i}|\vals{<i})$, but sample, often multiple, values from possibly a different distribution. 
Some instances of sampling schemes are taking multiple i.i.d. samples, importance sampling \citep{mahmoodWeightedImportanceSampling2014} which is very common in off-policy reinforcement learning, sampling without replacement \citep{koolEstimatingGradientsDiscrete2020}, memory-augmented sampling \citep{liangMemoryAugmentedPolicy2018} and antithetic sampling \citep{yinARMAugmentREINFORCEMergeGradient2019}.
Furthermore, measure-valued derivatives \citep{mohamedMonteCarloGradient2020,heidergottMeasurevaluedDifferentiationMarkov2008} and SPSA \citep{spallMultivariateStochasticApproximation1992} also sample from different distributions by comparing the performance of two related distributions.
To capture this, the \tproposal{} $\fproposalcond{i}$ samples a \emph{set} of values $\sampleset_i=\{\vals{i, 1}, ..., \vals{i, m}\}$ where each $\vals{i, j}\in \sspace_{\stochastic_i}$. 
The sample is conditioned on $\vals{<i}=\cup_{\stochnode\in\stochastic_i} \vals{\pa(\stochnode)}$, the values of the parent nodes of the stochastic nodes in $\stochastic_i$. 
This is illustrated in Figure \ref{fig:surrogate-loss}.b.

\textbf{Weighting function} When a gradient estimator uses a different sampling scheme, we have to weight each individual sample to ensure it remains a valid estimate of the expectation. 
For this, we use a nonnegative weighting function $\weights{i}: \sspace_{\stochastic_i}\rightarrow \mathbb{R}^+$.  
Usually, this function is going to be detached from the computation graph, but we allow it to receive gradients as well to support implementing expectations and gradient estimation methods that compute the expectation over (a subset of) values \citep{koolEstimatingGradientsDiscrete2020, liangMemoryAugmentedPolicy2018,liuRaoBlackwellizedStochasticGradients2019}.


\textbf{Gradient function} The \tmultipl{} is an unbiased gradient estimator together with the weighting function. 
It distributes the empirical cost evaluation to the parameters of the distribution. 
In the case of score function methods, this is the log-probability. For measure-valued derivatives and SPSA we can use the parameters of the distribution itself.

\textbf{Control variate} Modelers can use \tadditive{}s to reduce the variance of gradient estimates \citep{greensmithVarianceReductionTechniques2004,mohamedMonteCarloGradient2020}. 
It is a function that has zero-mean when differentiated.
Within the context of score functions, a common control variate is a baseline, which is a function that is independent of the sampled value. 
We also found that LAX, RELAX, and REBAR (Appendix \ref{sec:relax}), and the GO gradient \citep{congGOGradientExpectationbased2019} (Appendix \ref{sec:gogradient}) have natural implementations using a \tadditive{}. 
We discuss how we implement \tadditive{}s in \emph{Storchastic} in Section \ref{sec:var-reduction}.

\begin{example}
As an example, we show how to formulate the score function with the leave-one-out baseline \citep{mnihVariationalInferenceMonte2016,koolBuyREINFORCESamples2019} in \emph{Storchastic}.
This method samples $m$ values with replacement and uses the average of the other values as a baseline. 
\begin{itemize}
    \item \textbf{Proposal distribution}: We use $m$ samples with replacement, which can be formulated as $q(\sampleset_{i}|\vals{<i})=\prod_{j=1}^m p(\vals{i, j}|\vals{<i})$.
    \item \textbf{Weighting function}: Since samples are independent, we use $\fweight{i}=\frac{1}{m}$. 
    \item \textbf{Gradient function}: The score-function uses the log-probability $\fmultipl{i}=\log p(\vals{i}|\vals{<i})$.
    \item \textbf{Control variate}: We use $\fadditive{i, j}=(1-\magic(\fmultipl{i}))\frac{1}{m-1}\sum_{ j'\neq j} f_\cost(\vals{<i}, \vals{i, j})$,
    where $\frac{1}{m-1}\sum_{ j'\neq j} f_\cost(\vals{<i}, \vals{i, j})$ is the leave-one-out baseline.
    $(1-\magic(\fmultipl{i}))$ is used to ensure the baseline will be subtracted from the cost before multiplication with the \tmultipl{}. 
    It will not affect the forward evaluation since $\forward{1-\magic(\fmultipl{i})}$ evaluates to 0.
\end{itemize}
\end{example}

\subsection{The \emph{Storchastic} Surrogate Loss}
\label{sec:surrogate}
As mentioned in Requirement \ref{req:surrogate}, we would like to define a \emph{surrogate loss}, which we will introduce next.
Differentiating this loss $n$ times, and then evaluating the result using an AD library, will give unbiased estimates of the $n$-th order derivative of the parameter $\theta$ with respect to the cost $\cost$.
Furthermore, according to Requirement \ref{req:plug}, we assume the modeler has chosen a gradient estimator $\gradestim$ for each partition $\stochastic_i$, which can all be different.
Then the Storchastic surrogate loss is
%
\begin{align}
    \SL= \sum_{\itersample{1}} \fweight{1}&\Big[\fadditive{1} + 
        \sum_{\itersample{2}} \fweight{2}\Big[ \magic(\fmultipl{1})\fadditive{2} + \dots  \\
        + \sum_{\itersample{k}} \fweight{k}  &\Big[ \magic(\sum_{j=1}^{k-1} \fmultipl{j})\fadditive{k} + \magic\big( \sum_{i=1}^k \fmultipl{i} \big)\cost  \Big] \dots \Big] \Big]   \label{eq:surrogate-loss},\\
    \text{where } \sampleset_1 \sim \fproposal{1}&, \sampleset_2 \sim \fproposalcond{2}, ..., \sampleset_k \sim \fproposalcond{k}.
\end{align}
When this loss is differentiated $n$ times using AD libraries, it will produce unbiased estimates of the $n$-th derivative, as we will show later.
\begin{algorithm}
\begin{algorithmic}[1]
    \Function{estimate\_gradient}{$n$, $\theta$} 
        \State $\SL \gets $ \Call{surrogate\_loss}{$1$, $\{\}$, 0} \Comment{Compute surrogate loss} 
        \State \Return $\forward{\nabla_\theta^{(n)} \SL}$ \Comment{Differentiate and use AD library to evaluate surrogate loss}
    \EndFunction
    \State
    \Function{surrogate\_loss}{$i$, $\vals{<i}$, $L$} 
        \If {$i=k+1$}
            \State \Return $\magic\big(L) f_\cost(\vals{\leq k})$ \Comment{Use \magicbox{} to distribute cost}
        \EndIf
        \State $\sampleset_{i} \sim \fproposalcond{i}$ \Comment{Sample from \tproposal{}}
        \State $\mathtt{sum} \gets 0$
        \For {$\itersample{i}$} \Comment{Iterate over options in sampled set }
            \State $A \gets \magic(L)\fadditive{i}$ \Comment{Compute \tadditive{}}
            \State $L_i\gets L + \fmultipl{i}$ \Comment{Compute gradient function} 
            \State $\hat{\cost} \gets$ \Call{surrogate\_loss}{$i+1$, $\vals{\leq i}$, $L_i$} \Comment{Compute surrogate loss for $\vals{i}$}
            \State $\mathtt{sum} \gets \mathtt{sum} + \fweight{i}(\hat{\cost} + A)$ \Comment{Weight and add}
        \EndFor
        \State \Return $\mathtt{sum}$
    \EndFunction
\end{algorithmic}
\caption{The \emph{Storchastic} framework: Compute a Monte Carlo estimate of the $n$-th order gradient given $k$ gradient estimators $\gradestim$. }
\label{alg:storchastic}
\end{algorithm}
To help understand the \emph{Storchastic} surrogate loss and why it satisfies the requirements, we will break it down using Algorithm \ref{alg:storchastic}. 
The \textproc{estimate\_gradient} function computes the surrogate loss for the SCG, and then differentiates it $n\geq 0$ times using the AD library to return an estimate of the $n$-th order gradient, which should be unbiased according to Requirement \ref{req:higher-order}.
If $n$ is set to zero, this returns an estimate of the expected cost.

The \textproc{surrogate\_loss} function computes the equation using a recursive computation, which is illustrated in Figure \ref{fig:surrogate-loss}.b-d.
It iterates through the partitions and uses the gradient estimator to sample and compute the output of each component.
It receives three inputs: The first input $i$ indexes the partitions and gradient estimators, the second input $\vals{<i}$ is the set of previously sampled values for partitions $\stochastic_{<i}$, and $L$ is the sum of \tmultipl{}s of those previously sampled values.
In line 8, we sample a set of values $\sampleset_i$ for partition $i$ from $\fproposalcond{i}$.
In lines 9 to 14, we compute the sum over values $\vals{i}$ in $\sampleset_i$, which reflects the $i$-th sum of the equation.
Within this summation, in lines 11 and 12, we compute the \tmultipl{} and \tadditive{} for each value $\vals{i}$.
We will explain in Section \ref{sec:var-reduction} why we multiply the \tadditive{} with the \magicbox{} of the sum of the previous \tmultipl{}.

In line 13, we go into recursion by moving to the next partition. 
We condition the surrogate loss on the previous samples $\vals{<i}$ together with the newly sampled value $\vals{i}$.
We pass the sum of \tmultipl{}s for later usage in the recursion. 
Finally, in line 14, the sample performance and the \tadditive{} are added in a weighted sum. 
The recursion call happens for each $\itersample{i}$, meaning that this computation is exponential in the size of the sampled sets of values $\sampleset_i$.
For example, the surrogate loss samples $|\sampleset_1|$ times from $\proposalcond{2}$, one for each value $\itersample{1}$.
However, this computation can be trivially parallelized by using tensor operations in AD libraries.
An illustration of this parallelized computation is given in Figure \ref{fig:surrogate-loss}.

Finally, in line 7 after having sampled values for all $k$ partitions, we compute the cost, and multiply it with the \magicbox{} of the sum of \tmultipl{}s.
This is similar to what happens in the DiCE estimator in Equation \eqref{eq:DiCE}.
\emph{Storchastic} can be extended to multiple cost nodes by computing surrogate losses for each cost node, and adding these together before differentiation.
For stochastic nodes that influence multiple cost nodes, the algorithm can share samples and gradient estimation methods to reduce overhead.


\subsection{Conditions for Unbiased Estimation}
\label{sec:storchastic-conditions}
We next introduce our main result that shows \emph{Storchastic} satisfies Requirements \ref{req:higher-order} and \ref{req:unbiased}, namely the conditions the gradient estimators should satisfy such that the \emph{Storchastic} surrogate loss gives estimates of the $n$-th order gradient of the SCG. 
A useful part of our result is that, in line with Requirement \ref{req:unbiased}, only local conditions of gradient estimators have to be proven to ensure estimates are unbiased.
Our result gives immediate generalization of these local proofs to any SCG.
\begin{manualtheorem}{1}
   \label{thrm:storchastic-informal}
   Evaluating the $n$-th order derivative of the \emph{Storchastic} surrogate loss in Equation \eqref{eq:surrogate-loss} using an AD library is an unbiased estimate of $\nabla_\theta^{(n)} \mathbb{E}[\cost]$ under the following conditions. First, all functions $f_\detnode$ corresponding to deterministic nodes $\detnode$ and all probability measures $p_\stochnode$ corresponding to stochastic nodes $\stochnode$ are \emph{identical under evaluation}. 
   Secondly, for each gradient estimator $\gradestim$, $i=1, ..., k$, all the following hold for $m=0, ..., n$:
   \begin{enumerate}
       \item  $\mathbb{E}_{\fproposalcond{i}}[\sum_{\itersample{i}} \forward{\nabla^{(m)}_\theta \fweight{i} \magic(\fmultipl{i})f(\vals{i})}]= \forward{\nabla_\theta^{(m)} \mathbb{E}_{\stochastic_i}[f(\vals{i})] }$ for any deterministic function $f$;
       \item $\mathbb{E}_{\fproposalcond{i}}[\sum_{\itersample{i}} \forward{\nabla^{(m)}_\theta \fweight{i} \fadditive{i}}]=0$;
       \item  for $n\geq m>0$, $\mathbb{E}_{\fproposalcond{i}}[\sum_{\itersample{i}} \forward{\nabla_\theta^{(m)} \fweight{i}}]=0$;
       \item $\forward{\fproposalcond{i}} = \fproposalcond{i}$, for all permissible $\sampleset_i$.
   \end{enumerate}
\end{manualtheorem}
The first condition defines a local surrogate loss for single expectations of any function under the \tproposal{}. 
The condition then says that this surrogate loss should give an unbiased estimate of the gradient for all orders of differentiation $m=0, ..., n$. 
Note that since 0 is included, the forward evaluation should also be unbiased.
This is the main condition used to prove unbiasedness of the \emph{Storchastic} framework, and can be proven for the score function and expectation, and for measure-valued derivatives and SPSA for zeroth and first-order differentiation. 

The second condition says that the \tadditive{} should be 0 in expectation under the \tproposal{} for all orders of differentiation. 
This is how \tadditive{}s are defined in previous work \citep{mohamedMonteCarloGradient2020}, and should usually not restrict the choice. 
The third condition constrains the weighting function to be 0 in expectation for orders of differentiation larger than 0.
Usually, this is satisfied by the fact that weighting functions are detached from the computation graph, but when enumerating expectations, this can be shown by using that the sum of weights is constant.
The final condition is a regularity condition that says \tproposal{}s should not be different under forward mode.
We also assume that the SCG is \emph{identical under evaluation}. This means that all functions and probability densities evaluate to the same value with and without the forward-mode operator, even when differentiated. 
This concept is formally introduced in Appendix \ref{sec:appendix-forward-mode}.

A full formalization and the proof of Theorem \ref{thrm:storchastic} are given in Appendix \ref{seq:unbiasedness-proof}. 
The general idea is to rewrite each sampling step as an expectation, and then inductively show that the inner expectation $i$ over the \tproposal{} $\proposalcond{i}$ is an unbiased estimate of the $n$th-order derivative over $\stochastic_i$ conditional on the previous samples.
To reduce the multiple sums over \tmultipl{}s inside \magicbox{}, we make use of a property of \magicbox{} proven in Appendix \ref{sec:appendix-forward-mode}:
\begin{manualproposition}{\ref{prop:DiCE_multiply}}
   Summation inside a \magicbox{} is equivalent under evaluation to multiplication of the arguments in individual \magicbox{}es, ie:
   \begin{equation}
    \magic(l_1(x) + l_2(x)) f(x) \equivforward  \magic(l_1(x)) \magic(l_2(x)) f(x).
   \end{equation}
\end{manualproposition}
\emph{Equivalence under evaluation}, denoted $\equivforward$, informally means that, under evaluation of $\forward{}$, the two expressions and their derivatives are equal.
This equivalence is closely related to $e^{a+b}=e^a e^b$.

\subsection{Any-order variance reduction using \tadditive{}s}
\label{sec:var-reduction}
To satisfy Requirement \ref{req:variance}, we investigate implementing \tadditive{}s such that the variance of any-order derivatives is properly decreased.
This is challenging in general SCG's \citep{maoBaselineAnyOrder2019}, since in higher orders of differentiation, derivatives of \tmultipl{}s will interact, but naive implementations of \tadditive{}s only reduce the variance of the \tmultipl{} corresponding to a single stochastic node.
\emph{Storchastic} solves this problem similarly to the method introduced in \cite{maoBaselineAnyOrder2019}. 
In line 11 of the algorithm, we multiply the \tadditive{} with the sum of preceding \tmultipl{}s $\magic(L)$. 
We prove that this ensures every term of the any-order derivative will be affected by a \tadditive{} in Appendix \ref{sec:baselines}. 
This proof is new, since \cite{maoBaselineAnyOrder2019} only showed this for first and second order differentiation, not for general \tadditive{}s, and uses a slightly different formulation that we show misses some terms.

\begin{manualtheorem}{2}[informal]
Let $\Multipl{i}=\sum_{j=1} ^i \multipl{i}$. The \emph{Storchastic} surrogate loss of \eqref{eq:surrogate-loss} can equivalently be computed as
\begin{align}
    \SL\equivforward \sum_{\itersample{1}} &  
        \sum_{\itersample{2}}   \dots  
         \sum_{\itersample{k}} \prod_{i=1}^k \fweight{i}  \sum_{i=1}^k \magic(\Multipl{i-1})\Big(\fadditive{i} + (\magic(\multipl{i}) - 1)\cost \Big) + \cost.
\end{align}
\end{manualtheorem}
This gives insight into how \tadditive{}s are used in \emph{Storchastic}. 
They are added to the \tmultipl{}, but only during differentiation since $\forward{\magic(\Multipl{i}) - 1} = 0$. 
Furthermore, since both terms are multiplied with $\magic(\Multipl{i-1})$ (see line 11 of Algorithm \ref{alg:storchastic}), both terms correctly distribute over the same any-order derivative terms. 
By choosing a \tadditive{} of the form $\fadditive{i}=(1 - \magic(\multipl{i}))\cdot  \baseline{i}$, we recover baselines which are common in the context of score functions \citep{foersterDiCEInfinitelyDifferentiable2018, mohamedMonteCarloGradient2020}.
For the proof, we use the following proposition also proven in Appendix \ref{sec:baselines}:
\begin{proposition}
    For orders of differentiation $n>0$, 
    \begin{equation}
        \label{eq:baseline-generator}
        \forward{\nabla_\node^{(n)} \magic(\Multipl{k})} = \forward{\nabla_\node^{(n)} \sum_{i=1}^k \big(\magic(\multipl{i}) - 1\big) \magic(\Multipl{i-1})}.
    \end{equation}
\end{proposition}

\subsection{Gradient Estimation Methods}
\label{sec:estimators}
In Appendix \ref{sec:gradient-estimators} we show how several prominent examples of gradient estimation methods in the literature can be formulated using \emph{Storchastic}, and prove for what orders of differentiation the conditions hold.
Starting off, we show that for finite discrete random variables, we can formulate enumerating over all possible options using \emph{Storchastic}. 
The score function fits by mimicking DiCE \citep{foersterDiCEInfinitelyDifferentiable2018}. 
We extend it to multiple samples with replacement to allow using the leave-one-out baseline \citep{mnihVariationalInferenceMonte2016,koolBuyREINFORCESamples2019}. 
Furthermore, we show how importance sampling, sum-and-sample estimators such as MAPO \citep{liangMemoryAugmentedPolicy2018}, the unordered set estimator \citep{koolEstimatingGradientsDiscrete2020} and RELAX and REBAR \citep{grathwohlBackpropagationVoidOptimizing2018, tuckerREBARLowvarianceUnbiased2017} fit in \emph{Storchastic}.
We also discuss the antithetic sampling estimator ARM \citep{yinARMAugmentREINFORCEMergeGradient2019}.
Unfortunately, condition 2 only holds for this estimator for $n\leq 1$ since it relies on a particular property of the score function that holds only for first-order gradient estimation. 
In addition to score function based methods, we discuss the GO gradient, SPSA \citep{rubinsteinSimulationMonteCarlo2016} and Measure-Valued Derivative \citep{heidergottMeasurevaluedDifferentiationMarkov2008}, and show that the last two will only be unbiased for $n\leq 1$. 
Finally, we note that reparameterization \citep{kingmaAutoencodingVariationalBayes2014, rezendeStochasticBackpropagationApproximate2014} can be implemented by transforming the SCG such that the sampling step is outside the path from the parameter to the cost \citep{schulmanGradientEstimationUsing2015}.

\subsection{Implementation}
\label{sec:implementation}
We implemented \emph{Storchastic} as an open source PyTorch \citep{paszkePyTorchImperativeStyle2019} library
\footnote{Code is available at \url{github.com/HEmile/storchastic}.}. 
To ensure modelers can easily use this library, it automatically handles sets of samples as extra dimensions to PyTorch tensors which allows running multiple sample evaluations in parallel.
This approach is illustrated in Figure \ref{fig:surrogate-loss}.
By making use of PyTorch broadcasting semantics, this allows defining models for simple single-sample computations that are automatically parallelized using \emph{Storchastic} when using multiple samples.
The \emph{Storchastic} library has implemented most of the gradient estimation methods mentioned in Section \ref{sec:estimators}.
Furthermore, new gradient estimation methods can seamlessly be added.

\subsubsection{Example: Leave-one-out baseline in Discrete Variational Autoencoder}
\begin{figure}
\begin{lstlisting}[language=Python]
class ScoreFunctionLOO(storch.method.Method):
    def proposal_dist(self, distribution, amt_samples):
        return distr.sample((amt_samples,))

    def weighting_function(self, distribution, amt_samples):
        return torch.full(amt_samples, 1/amt_samples)

    def estimator(self, sample, cost):
        # Compute gradient function (log-probability)
        log_prob = sample.distribution.log_prob(tensor)
        sum_costs = storch.sum(costs.detach(), sample.name)
        # Compute control variate
        baseline = (sum_costs - costs) / (sample.n - 1)
        return log_prob, (1.0 - magic_box(log_prob)) * baseline
\end{lstlisting}
\caption{Implementing the score function with the leave-one-out baseline in the Storchastic library.}
\label{fig:list-loo}
\end{figure}
As a small case study, we show how to implement the score function with the leave-one-out baseline introduced in Section \ref{sec:components} for the discrete variational autoencoder introduced in Section \ref{sec:example-vae} in PyTorch using Storchastic. While the code listed is simplified, it shows the flexibility with which one can compute gradients in SCGs. 

We list in Figure \ref{fig:list-loo} how to implement the score function with the leave-one-out baseline.
Line 3 implements the proposal distribution, line 6 the weighting function, line 10 the gradient function and line 13 and 14 the control variate.
Gradient estimation methods in Storchastic all extend a common base class \texttt{storch.method.Method} to allow easy interoperability between different methods.

In Figure \ref{fig:list-vae}, we show how to implement the discrete VAE. The implementation directly follows the SCG shown in Figure \ref{fig:VAE}. In line 2, we create the \texttt{ScoreFunctionLOO} method defined in Figure \ref{fig:list-loo}. 
Then, we run the training loop: In line 6 we create the stochastic node $x$ by denoting the minibatch dimension as an independent dimension.
In line 8 we run the encoder with parameters $\phi$ to find the variational posterior $q_z$. We call the gradient estimation method in line 9 to get a sample of $z$. Note that this interface is independent of gradient estimation method chosen, meaning that if we wanted to compare our implemented method with a baseline, all that is needed is to change line 2. After the decoder, we compute the two costs in lines 12 and 13. Finally, we call Storchastic main algorithm in line 15 and run the optimizer.
\begin{figure}
\begin{lstlisting}[language=Python]
from vae import minibatches, encode, decode, KLD, binary_cross_entropy
method = ScoreFunctionLOO("z", 8)
for data in minibatches():
    optimizer.zero_grad()
    # Denote minibatch dimension as independent plate dimension
    data = storch.denote_independent(data.view(-1, 784), 0, "data")
    # Compute variational distribution given data, sample z
    q = torch.distributions.OneHotCategorical(logits=encode(data))
    z = method(q)
    # Compute costs, form the ELBO
    reconstruction = decode(z)
    storch.add_cost(KLD(q))
    storch.add_cost(binary_cross_entropy(reconstruction, data))
    # Storchastic backward pass, optimize
    ELBO = storch.backward()
    optimizer.step()
\end{lstlisting}
\caption{Simplified implementation of the discrete VAE using Storchastic.}
\label{fig:list-vae}
\end{figure}


\subsubsection{Discrete VAE Case Study Experiments}
\label{sec:experiments}
\begin{table}[]
    \begin{tabular}{l|ll|ll}
                    & \multicolumn{2}{l|}{$2^{20}$ latent} & \multicolumn{2}{l}{$10^{20}$ latent} \\
                    & Train       & Validation       & Train        & Validation        \\ \hline
    Score@1         & 191.3              & 191.9                  & 206.3             & 206.7                  \\
    ScoreLOO@5 \citep{koolBuyREINFORCESamples2019}     & 110.8              & 110.4                  & 111.2             & 110.4                  \\
    REBAR@1 \citep{tuckerREBARLowvarianceUnbiased2017}        & 220.0              & 1000                   & 155.6             & 154.9                  \\
    RELAX@1 \citep{grathwohlBackpropagationVoidOptimizing2018}        & 210.6              & 205.9                  & 202.5             & 201.7                  \\
    Unordered set@5 \citep{koolEstimatingGradientsDiscrete2020} & 117.1              & 138.4                  & 115.4             & 117.2                  \\
    Gumbel@1 \citep{jangCategoricalReparameterizationGumbelsoftmax2017}       & 107.0              & 106.6                  & 92.9              & 92.6                   \\
    GumbelST@1 \citep{jangCategoricalReparameterizationGumbelsoftmax2017}     & 113.0              & 112.9                  & 98.3              & 98.0                   \\
    ARM@1 \citep{yinARMAugmentREINFORCEMergeGradient2019}         & 131.3              & 130.8                  &                   &                        \\
    DisARM@1\citep{dongDisARMAntitheticGradient2020}         & 125.1              & 124.3                  &                   &                           
    \end{tabular}
    \caption[Test runs on MNIST VAE generative modeling]{Test runs on MNIST VAE generative modeling. We report the lowest train and validation ELBO over 100 epochs. The number after the `@' symbol denotes the amount of samples used to compute the estimator. We note that the ARM and DiSARM methods are specific for binary random variables, and do not evaluate it in the $10^{20}$ discrete VAE.}
    \label{tbl:vae}
\end{table}
We report test runs on MNIST \citep{lecunMNISTHandwrittenDigit2010} generative modeling using discrete VAEs in Table \ref{tbl:vae}. We use Storchastic to run 100 epochs on both a latent space of 20 Bernoulli random variables and 20 Categorical random variables of 10 dimensions, and report training and test ELBOs.
We run these on the gradient estimation methods currently implemented in the PyTorch library.

Although results reported are worse than similar previous experiments, we note that we only run 100 epochs (900 epochs in \cite{koolEstimatingGradientsDiscrete2020}) and we do not tune the methods. 
However, the results reflect the order expected from \cite{koolEstimatingGradientsDiscrete2020}, where score function with leave-one-out baseline also performed best, closely followed by the Unordered set estimator.
Furthermore, the Gumbel softmax \citep{jangCategoricalReparameterizationGumbelsoftmax2017,maddisonConcreteDistributionContinuous2017} still outperforms the other score-function based estimators, although the results in \cite{koolEstimatingGradientsDiscrete2020} suggest that with more epochs and better tuning, better ELBO than reported here can be achieved.

These results are purely presented as a demonstration of the flexbility of the Storchastic library: Only a single line of code is changed to be able to compare the different estimators.
A more thorough and fair comparison, also in different settings, is left for future work. 





\section{Related Work}
\label{sec:related-work}
The literature on gradient estimation is rich, with papers focusing on general methods that can be implemented in \emph{Storchastic} \citep{spallMultivariateStochasticApproximation1992,heidergottMeasurevaluedDifferentiationMarkov2008, grathwohlBackpropagationVoidOptimizing2018,yinARMAugmentREINFORCEMergeGradient2019,liangMemoryAugmentedPolicy2018, liuRaoBlackwellizedStochasticGradients2019, congGOGradientExpectationbased2019}, see Appendix \ref{sec:gradient-estimators}, and works focused on Reinforcement Learning \citep{williamsSimpleStatisticalGradientfollowing1992,lillicrapContinuousControlDeep2016, mnihAsynchronousMethodsDeep2016} or Variational Inference \citep{mnihVariationalInferenceMonte2016}. For a recent overview, see \cite{mohamedMonteCarloGradient2020}. 

The literature focused on SCGs is split into methods using reparameterization \citep{rezendeStochasticBackpropagationApproximate2014,kingmaAutoencodingVariationalBayes2014,figurnovImplicitReparameterizationGradients2018,maddisonConcreteDistributionContinuous2017, jangCategoricalReparameterizationGumbelsoftmax2017} and those using the score function \citep{schulmanGradientEstimationUsing2015}. Of those, DiCE \citep{foersterDiCEInfinitelyDifferentiable2018} is most similar to \emph{Storchastic}, and can do any-order estimation on general SCGs. 
DiCE is used in the probabilistic programming library Pyro \citep{binghamPyroDeepUniversal2019}. 
We extend DiCE to allow for incorporating many other gradient estimation methods than just basic score function. 
We also derive and prove correctness of a general implementation for control variates for any-order estimation which is similar to the one conjectured for DiCE in \cite{maoBaselineAnyOrder2019}.

\cite{parmasTotalStochasticGradient2018,xuBackpropQGeneralizedBackpropagation2019} and \cite{weberCreditAssignmentTechniques2019} study actor-critic-like techniques and bootstrapping for SCGs to incorporate reparameterization using methods inspired by deterministic policy gradients \citep{lillicrapContinuousControlDeep2016}. By using models to differentiate through, these methods are biased through model inaccuracies and thus do not directly fit into \emph{Storchastic}.
However, combining these ideas with the automatic nature of \emph{Storchastic} could be interesting future work. 


\section{Conclusion}
  We investigated general automatic differentiation for stochastic computation graphs. 
    We developed the \emph{Storchastic} framework, and introduced an algorithm for unbiased any-order gradient estimation that allows using a large variety of gradient estimation methods from the literature. 
    We also investigated variance reduction and showed how to properly implement \tadditive{}s such that it affects any-order gradient estimates.
    The framework satisfies the requirements introduced in Section \ref{sec:requirements}. 

    For future work, we are interested in extending the analysis of \emph{Storchastic} to how variance compounds when using different gradient estimation methods. 
    Furthermore, \emph{Storchastic} could be extended to allow for biased methods.
    We are also interested in closely analyzing the different components of gradient estimators, both from a theoretical and empirical point of view, to develop new estimators that combine the strengths of estimators in the literature.

\newpage

\chapter[A-NeSI]{A-NeSI: A Scalable Approximate Method for Probabilistic Neurosymbolic Inference}
\label{ch:anesi}

\begin{paperbase}
	This chapter is based on the NeurIPS 2023 article \citep{vankriekenANeSIScalableApproximate2023}.
\end{paperbase}

\begin{abstract}
We study the problem of combining neural networks with symbolic reasoning. Recently introduced frameworks for Probabilistic Neurosymbolic Learning (PNL), such as DeepProbLog, perform exponential-time exact inference, limiting the scalability of PNL solutions. We introduce \emph{Approximate Neurosymbolic Inference} (\textsc{A-NeSI}): a new framework for PNL that uses neural networks for scalable approximate inference. \textsc{A-NeSI} 1) performs approximate inference in polynomial time without changing the semantics of probabilistic logics; 2) is trained using data generated by the background knowledge; 3) can generate symbolic explanations of predictions; and 4) can guarantee the satisfaction of logical constraints at test time, which is vital in safety-critical applications. Our experiments show that \textsc{A-NeSI} is the first end-to-end method to solve three neurosymbolic tasks with exponential combinatorial scaling. Finally, our experiments show that \textsc{A-NeSI} achieves explainability and safety without a penalty in performance.
\end{abstract}

\section{Introduction}
Recent work in neurosymbolic learning combines neural perception with symbolic reasoning \citep{vanharmelenBoxologyDesignPatterns2019,marraStatisticalRelationalNeural2021}, using symbolic knowledge to constrain the neural network \citep{giunchigliaDeepLearningLogical2022}, to learn perception from weak supervision signals \citep{manhaeveNeuralProbabilisticLogic2021}, and to improve data efficiency \citep{badreddineLogicTensorNetworks2022,xuSemanticLossFunction2018}. 
Many neurosymbolic methods 
use a differentiable logic such as fuzzy logics (see \cite{badreddineLogicTensorNetworks2022,diligentiSemanticbasedRegularizationLearning2017} and Chapters \ref{ch:dfl} and \ref{ch:lrl}) or probabilistic logics \citep{manhaeveDeepProbLogNeuralProbabilistic2018,xuSemanticLossFunction2018,deraedtNeurosymbolicNeuralLogical2019}. 
We call the latter \emph{Probabilistic Neurosymbolic Learning (PNL)} methods. PNL methods add probabilities over discrete truth values to maintain all logical equivalences from classical logic, unlike fuzzy logics (Chapter \ref{ch:dfl}). However, performing inference requires solving the \emph{weighted model counting (WMC)} problem, which
is computationally exponential \citep{chaviraProbabilisticInferenceWeighted2008}, significantly limiting the kind of tasks that PNL can solve. 

We study how to scale PNL to exponentially complex tasks using deep generative modelling \citep{tomczakDeepGenerativeModeling2022}. Our method called \emph{Approximate Neurosymbolic Inference} (\textsc{A-NeSI}), introduces two neural networks that perform approximate inference over the WMC problem. The \emph{prediction model} predicts the output of the system, while the \emph{explanation model} computes which worlds (i.e. which truth assignment to a set of logical symbols) best explain a prediction. We use a novel training algorithm to fit both models with data generated using background knowledge: \textsc{A-NeSI} samples symbol probabilities from a prior and uses the symbolic background knowledge to compute likely outputs given these probabilities. We train both models on these samples. See Figure \ref{fig:framework} for an overview.

\begin{figure*}
	\includegraphics[width=\textwidth]{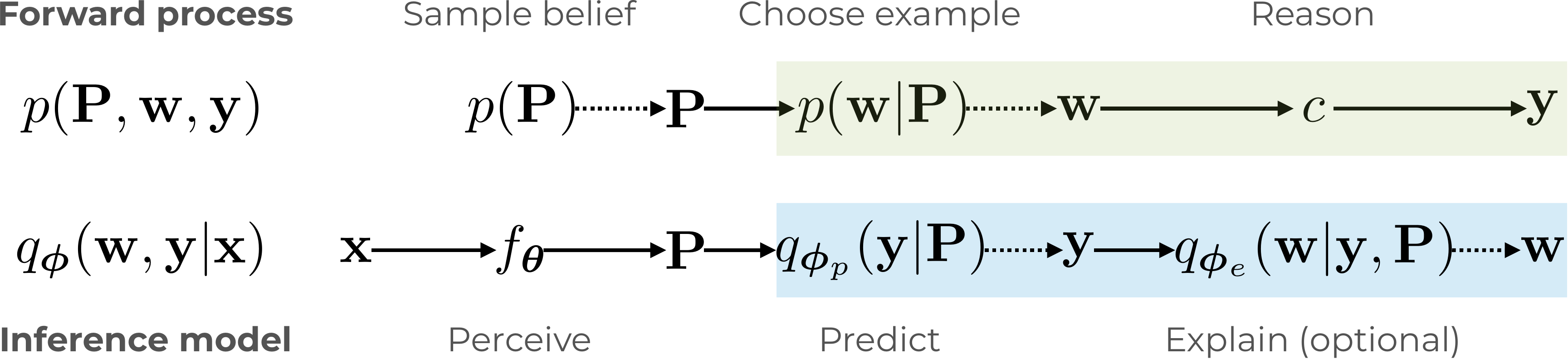}
	\caption[Overview of \textsc{A-NeSI}]{Overview of \textsc{A-NeSI}. The forward process samples a belief $\bP$ from a prior, then chooses a world $\bw$ for that belief. The symbolic function $c$ computes its output $\by$. The inference model uses the perception model $f_\btheta$ to find a belief $\bP$, then uses the prediction model $q_\paramP$ to find the most likely output $\by$ for that belief. Optionally, the explanation $q_\paramE$ model explains the output.}
	\label{fig:framework}
\end{figure*}

\textsc{A-NeSI} combines all benefits of neurosymbolic learning with scalability. Our experiments on the Multi-digit MNISTAdd problem \citep{manhaeveDeepProbLogNeuralProbabilistic2018} show that, unlike other approaches, \textsc{A-NeSI} scales almost linearly in the number of digits and solves MNISTAdd problems with up to 15 digits while maintaining the predictive performance of exact inference. Furthermore, it can produce explanations of predictions and guarantee the satisfaction of logical constraints using a novel symbolic pruner.

The chapter is organized as follows.
In Section \ref{sec:a-nesi}, we introduce \textsc{A-NeSI}. Section \ref{sec:inference-models} presents scalable neural networks for approximate inference. Section \ref{sec:train-inference} outlines a novel training algorithm using data generated by the background knowledge. Section \ref{sec:joint} extends \textsc{A-NeSI} to include an explanation model. Section \ref{sec:symbolic-pruning} extends \textsc{A-NeSI} to guarantee the satisfaction of logical formulas.
In Section \ref{sec:experiments-anesi}, we perform experiments on three Neurosymbolic tasks that require perception and reasoning. Our experiments on Multi-digit MNISTAdd show that \textsc{A-NeSI} learns to predict sums of two numbers with 15 handwritten digits, up from 4 in competing systems. Similarly, \textsc{A-NeSI} can classify Sudoku puzzles in $9\times 9$ grids, up from 4, and find the shortest path in $30\times 30$ grids, up from 12. 

\section{Problem setup}
\label{sec:problem-statement}
First, we introduce our inference problem. We will use the MNISTAdd task from \cite{manhaeveDeepProbLogNeuralProbabilistic2018} as a running example. 
In this problem, we must learn to recognize the sum of two MNIST digits using only the sum as a training label. Importantly, we do not provide the labels of the individual MNIST digits.

\subsection{Problem components}
We introduce four sets representing the spaces of the variables of interest.
\begin{enumerate}[leftmargin=*]
	\item $X$ is an input space. In MNISTAdd, this is the pair of MNIST digits $\bx=(\digit{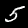}, \digit{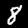})$.
	\item $W$ is a structured space of $\Wdim$ discrete choices. Its elements $\bw\in W$ are \emph{worlds}: symbolic representations of some $\bx\in X$. For $(\digit{A-NeSI/images/mnist_5.png}, \digit{A-NeSI/images/mnist_8.png})$, the correct world is $\bw=(5, 8)$. 
	\item $Y$ is a structured space of $\Ydim$ discrete choices. Elements $\by\in Y$ represent the output of the neurosymbolic system. Given world $\bw=(5, 8)$, the sum is $13$. We decompose the sum into individual digits, so $\by=(1, 3)$. 
	\item $\bP$ is a \emph{belief} that assigns probabilities to different worlds with $p(\bw|\bP)=\prod_{i=1}^\Wdim \bP_{i, w_i}$ 
\end{enumerate}
We assume access to some \emph{symbolic reasoning function} $\symfun: W\rightarrow Y$ that deterministically computes the (structured) output $\by$ for any world $\bw$. For MNISTAdd, $\symfun$ takes the two digits $(5, 8)$, sums them, and decomposes the sum by digit to form $(1, 3)$. 

\subsection{Weighted Model Counting}
\label{sec:weighted_model_counting}
Together, these components form the \emph{Weighted Model Counting (WMC)} problem \citep{chaviraProbabilisticInferenceWeighted2008}: 
\begin{equation}
	\label{eq:wmc}
	p(\by|\bP)=\mathbb{E}_{p(\bw|\bP)}[\symfun(\bw)=\by]
\end{equation}
The WMC counts the \emph{possible} worlds $\bw$\footnote{Possible worlds are also called `models'. We refrain from using this term to prevent confusion with `(neural network) models'.} for which $\symfun$ produces the output $\by$, and weights each possible world $\bw$ with $p(\bw|\bP)$. 
In PNL, we want to train a perception model $f_\btheta$ to compute a belief $\bP = f_\btheta(x)$ for an input $\bx\in X$ in which the correct world is likely. Note that \emph{possible} worlds are not necessarily \emph{correct} worlds: $\bw=(4, 9)$ also sums to 13, but is not a symbolic representation of $\bx=(\digit{A-NeSI/images/mnist_5.png}, \digit{A-NeSI/images/mnist_8.png})$.

Given this setup, we are interested in efficiently computing the following quantities:
\begin{enumerate}[leftmargin=*]
	\item $p(\by|\bP=f_\btheta(\bx))$: We want to find the most likely outputs for the belief $\bP$ that the perception network computes on the input $\bx$.
	\item $\nabla_\bP p(\by|\bP=f_\btheta(\bx))$: We want to train our neural network $f_\btheta$, which requires computing the gradient of the WMC problem. 
	\item $p(\bw|\by, \bP=f_\btheta(\bx))$: We want to find likely worlds given a predicted output and a belief about the perceived digits. The probabilistic logic programming literature calls $\bw^*=\arg\max_{\bw} p(\bw|\by, \bP)$ the \emph{most probable explanation (MPE)}  \citep{shterionovMostProbableExplanation2015}.
\end{enumerate}

\subsection{The problem with exact inference for WMC}
\label{sec:why-approximate}
The three quantities introduced in Section \ref{sec:weighted_model_counting} require calculating or estimating the WMC problem of Equation \ref{eq:wmc}. However, the exact computation of the WMC is \#P-hard, which is above NP-hard.
Thus, we would need to do a potentially exponential-time computation for each training iteration and test query. Existing PNL methods use probabilistic circuits (PCs) to speed up this computation \citep{yoojungProbabilisticCircuitsUnifying2020,kisaProbabilisticSententialDecision2014,manhaeveDeepProbLogNeuralProbabilistic2018,xuSemanticLossFunction2018}. PCs compile a logical formula into a circuit for which many inference queries are linear in the size of the circuit. PCs are a good choice when exact inference is required but do not overcome the inherent exponential complexity of the problem: the compilation step is potentially exponential in time and memory, and there are no guarantees the size of the circuit is not exponential. 

The Multi-digit MNISTAdd task is excellent for studying the scaling properties of WMC. We can increase the complexity of MNISTAdd exponentially by considering not only the sum of two single digits (called $N=1$) but the sum of two numbers with multiple digits. An example of $N=2$ would be $\digit{A-NeSI/images/mnist_5.png}\digit{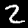}+\digit{A-NeSI/images/mnist_8.png}\digit{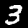}=135$. There are 64 ways to sum to 135 using two numbers with two digits. Contrast this to summing two digits: there are only 6 ways to sum to $13$. Each increase of $N$ leads to 10 times more options to consider for exact inference. Our experiments in Section \ref{sec:experiments-anesi} will show that approximate inference can solve this problem up to $N=15$. Solving this using exact inference would require enumerating around $10^{15}$ options for each query. 

\section{\textsc{A-NeSI}: Approximate Neurosymbolic Inference}
\label{sec:a-nesi}
Our goal is to reduce the inference complexity of PNL. To this end, in the following subsections, we introduce \emph{Approximate Neurosymbolic Inference} (\textsc{A-NeSI}). \textsc{A-NeSI} approximates the three quantities of interest from Section \ref{sec:problem-statement}, namely $p(\by|\bP)$, $\nabla_\bP p(\by|\bP)$ and $p(\bw|\by, \bP)$, using neural networks. We give an overview of our method in Figure \ref{fig:framework}.

\subsection{Inference models}
\label{sec:inference-models}
\textsc{A-NeSI} uses an \emph{inference model} $q_\bphi$ defined as 
\begin{equation}
	\label{eq:reverse}
	q_\bphi(\bw, \by|\bP) = q_\paramP(\by|\bP) q_\paramE(\bw|\by, \bP).
\end{equation}
We call $q_\paramP(\by|\bP)$ the \emph{prediction model} and $q_\paramE(\bw|\by, \bP)$ the \emph{explanation model}. The prediction model should approximate the WMC problem of Equation \ref{eq:wmc}, while the explanation model should predict likely worlds $\bw$ given outputs $\by$ and beliefs $\bP$. One way to model $q_\bphi$ is by considering $W$ and $Y$ as two sequences and defining an autoregressive generative model over these sequences:
\begin{align}
		q_\bphi(\by|\bP) = \prod_{i=1}^{\Ydim} q_\paramP(y_i|\by_{1:i-1}, \bP), \quad 
		q_\bphi(\bw|\by, \bP) = \prod_{i=1}^{\Wdim} q_\paramE(w_i|\by, \bw_{1:i-1}, \bP)	
	\label{eq:nim}
\end{align}
This factorization makes the inference model highly flexible. We can use simpler models if we know the dependencies between variables in $W$ and $Y$. The factorization is computationally linear in $\Wdim + \Ydim$, instead of exponential in $\Wdim$ for exact inference.
During testing, we use a beam search to find the most likely prediction. 



There are no restrictions on the architecture of the inference model $q_\bphi$. Any neural network with appropriate inductive biases and parameter sharing to speed up training can be chosen. For instance, CNNs over grids of variables, graph neural networks \citep{kipfSemiSupervisedClassificationGraph2017,schlichtkrullModelingRelationalData2018} for reasoning problems on graphs, or transformers for sets or sequences \citep{vaswaniAttentionAllYou2017}.

We use the prediction model to train the perception model $f_\btheta$ given a dataset $\mathcal{D}$ of tuples $(\bx, \by)$. Our novel loss function trains the perception model by backpropagating through the prediction model:
\begin{equation}
	\mathcal{L}_{Perc}(\mathcal{D}, \btheta)=-\log q_\paramP(\by|\bP=f_\btheta(\bx)), \quad \bx, \by \sim \mathcal{D}
\end{equation}
The gradients of this loss are biased due to the error in the approximation $q_\paramP$, but it has no variance outside of sampling from the training dataset. 

\subsection{Training the inference model}
\label{sec:train-inference}
We define two variations of our method. The \emph{prediction-only} variant (Section \ref{sec:pred-only}) uses only the prediction model $q_\paramP(\by|\bP)$, while the \emph{explainable} variant (Section \ref{sec:joint}) also uses the explanation model $q_\paramE(\bw|\by, \bP)$. 

Efficient training of the inference model relies on two design decisions. The first is a descriptive factorization of the output space $Y$, which we discuss in Section \ref{sec:output-factorization}. The second is using an informative \emph{belief prior} $p(\bP)$, which we discuss in Section \ref{sec:prior}. 

We first define the forward process that uses the symbolic function $\symfun$ to generate training data:
\begin{equation}
	\label{eq:forward}
	p(\bw, \by|\bP) = p(\bw|\bP) p(\by|\bw, \bP) = p(\bw|\bP) (\symfun(\bw)=\by)
\end{equation}
We take some example world $\bw$ and deterministically compute the output of the symbolic function $\symfun(\bw)$. Then, we compute whether the output $\symfun(\bw)$ equals $\by$. Therefore, $p(\bw, \by|\bP)$ is 0 if $\symfun(\bw) \neq \by$ (that is, $\bw$ is not a possible world of $\by$).

The belief prior $p(\bP)$ allows us to generate beliefs $\bP$ for the forward process. That is, we generate training data for the inference model using 
\begin{align}
	\label{eq:forward-gen}
	\begin{split}
	&p(\bP, \bw) = p(\bP) p(\bw|\bP) \\
	\text{where } &\bP, \bw \sim p(\bP, \bw), \quad \by = \symfun(\bw).
	\end{split}
\end{align}
The belief prior allows us to train the inference model with synthetic data. The prior and the forward process define everything $q_\bphi$ needs to learn. 

\subsubsection{Training the prediction-only variant}
\label{sec:pred-only}
\begin{figure}[tb]
	\begin{minipage}[t]{0.45\textwidth}
		\begin{algorithm}[H]
		\caption{Compute inference model loss}
		\label{alg:train-NIM}
		\begin{algorithmic}
		\State fit prior $p(\bP)$ on $\bP_1, ..., \bP_k$
		\State $\bP \sim p(\bP)$ 
		\State $\bw \sim p(\bw|\bP)$
		\State $\by \gets \symfun(\bw)$ 
		\State {\bfseries return} $\bphi + \alpha \nabla_\bphi\log q_\paramP(\by|\bP)$ 
		\end{algorithmic}
		\end{algorithm}
	\end{minipage}
	\hfill
	\begin{minipage}[t]{0.45\textwidth}
		\begin{algorithm}[H]
		\caption{\textsc{A-NeSI} training loop}
		\label{alg:train-A-NeSI}
		\begin{algorithmic}
		\State {\bfseries Input:} dataset $\mathcal{D}$, params $\btheta$, params $\bphi$
		\State \texttt{beliefs}$\gets []$
		\While {not converged}
		\State $(\bx, \by) \sim \mathcal{D}$
		\State $\bP \gets f_\btheta(\bx)$ 
		\State update \texttt{beliefs} with $\bP$
		\State $\bphi \gets$ {\bfseries Algorithm 2}(\texttt{beliefs}, $\bphi$)
		\State $\btheta \gets \btheta + \alpha \nabla_\btheta \log q_\paramP(\by|\bP)$
		\EndWhile
		\end{algorithmic}
		\end{algorithm}
\end{minipage}
\caption{The training loop of \textsc{A-NeSI}.}
\label{fig:training_loop}
\end{figure}
	
In the prediction-only variant, we only train the prediction model $q_\paramP(\by|\bP)$. We use the samples generated by the process in Equation \ref{eq:forward-gen}. We minimize the expected cross entropy between $p(\by|\bP)$ and $q_\paramP(\by|\bP)$ over the prior $p(\bP)$:
\begin{align}
	&\mathbb{E}_{p(\bP)}\left[-\mathbb{E}_{p(\by|\bP)}[\log q_\paramP(\by|\bP)]\right]  \\
	\label{eq:pred-only-deriv}
	= &-\mathbb{E}_{p(\bP)}\left[\mathbb{E}_{p(\bw, \by|\bP)}[\log q_\paramP(\by|\bP)]\right] \\
	=&-\mathbb{E}_{p(\bP, \bw)}\left[\log q_\paramP(\symfun(\bw)|\bP)\right]  \\
	{align}
	\label{eq:pred-only-loss}
	\mathcal{L}_{Pred}(\paramP) = & -\log q_\paramP(\symfun(\bw)|\bP), \quad \bP, \bw \sim p(\bP, \bw)
\end{align}
In line \ref{eq:pred-only-deriv}, we marginalize out $\bw$, and use the fact that $\by$ is deterministic given $\bw$. 
In the loss function defined in Equation \ref{eq:pred-only-loss}, we estimate the expected cross entropy using samples from $p(\bP, \bw)$. We use the sampled world $\bw$ to compute the output $\by=\symfun(\bw)$ and increase its probability under $q_\paramP$. Importantly, we do not need to use any data to evaluate this loss function. We give pseudocode for the full training loop in Figure \ref{fig:training_loop}. 

\subsubsection{Output space factorization}
\label{sec:output-factorization}
The factorization of the output space $Y$ introduced in Section \ref{sec:problem-statement} is one of the key ideas that allow efficient learning in \textsc{A-NeSI}. We will illustrate this with MNISTAdd.
As the number of digits $N$ increases, the number of possible outputs (i.e., sums) is $2\cdot 10^N-1$. Without factorization, we would need an exponentially large output layer. We solve this by predicting the individual digits of the output so that we need only $N\cdot 10+2$ outputs similar to  \cite{aspisEmbed2SymScalableNeuroSymbolic2022}. Furthermore, recognizing a single digit of the sum is easier than recognizing the entire sum: for its rightmost digit, only the rightmost digits of the input are relevant. 
	
Choosing the right factorization is crucial when applying \textsc{A-NeSI}. A general approach is to take the CNF of the symbolic function and predict each clause's truth value. However, this requires grounding the formula, which can be exponential. Another option is to predict for what objects a universally quantified formula holds, which would be linear in the number of objects. 

\subsubsection{Belief prior design}
\label{sec:prior}
How should we choose the $\bP$s for which we train $q_\bphi$? A naive method would use the perception model $f_\btheta$, sample some training data $\bx_1, ..., \bx_k\sim \mathcal{D}$ and train the inference model over $\bP_1=f_\btheta(\bx_1), ..., \bP_k=f_\btheta(\bx_k)$. However, this means the inference model is only trained on those $\bP$ occurring in the training data. Again, consider the Multi-digit MNISTAdd problem. For $N=15$, we have a training set of 2000 sums, while there are $2\cdot 10^{15}-1$ possible sums. By simulating many beliefs, the inference model sees a much richer set of inputs and outputs, allowing it to generalize.

A better approach is to fit a Dirichlet prior $p(\bP)$ on $\bP_1, ..., \bP_k$ that covers all possible combinations of numbers. We choose a Dirichlet prior since it is conjugate to the discrete distributions. For details, see Appendix \ref{appendix:dirichlet-prior}. During hyperparameter tuning, we found that the prior needs to be high entropy to prevent the inference model from ignoring the inputs $\bP$. Therefore, we regularize the prior with an additional term encouraging high-entropy Dirichlet distributions.


\subsubsection{Training the explainable variant}
\label{sec:joint}
The explainable variant uses both the prediction model $q_\paramP(\by|\bP)$ and the explanation model $q_\paramE(\bw|\by, \bP)$. 
When training the explainable variant, we use the idea that both factorizations of the joint should have the same probability mass, that is, $p(\bw, \by|\bP) = q_\bphi(\bw, \by|\bP)$. 
To this end, we use a novel \emph{joint matching} loss inspired by the theory of GFlowNets \citep{bengioFlowNetworkBased2021b,bengioGFlowNetFoundations2022}, in particular, the trajectory balance loss introduced by \cite{malkinTrajectoryBalanceImproved2022} which is related to variational inference \citep{malkinGFlowNetsVariationalInference2023}. For an in-depth discussion, see Appendix \ref{appendix:loss}. The joint matching loss is essentially a regression of $q_\bphi$ onto the true joint $p$ that we compute in closed form:
\begin{align}
\begin{split}
	\label{eq:joint-match}
	\mathcal{L}_{Expl}(\bphi)=\left(\log \frac{q_\bphi(\bw, \symfun(\bw)|\bP)}{p(\bw|\bP)}\right)^2, \quad \bP, \bw \sim p(\bP, \bw)
\end{split}
\end{align}
Here we use that $p(\bw, \symfun(\bw)|\bP)=p(\bw|\bP)$ since $c(\bw)$ is deterministic.
Like when training the prediction-only variant, we sample a random belief $\bP$ and world $\bw$ and compute the output $\by$. Then we minimize the loss function to \emph{match} the joints $p$ and $q_\bphi$.  We further motivate the use of this loss in Appendix \ref{appendix:gflownets}. Instead of a classification loss like cross-entropy, the joint matching loss ensures $q_\bphi$ does not become overly confident in a single prediction and allows spreading probability mass easier.

\subsection{Symbolic pruner}
\label{sec:symbolic-pruning}
\begin{figure*}
	\includegraphics[width=\textwidth]{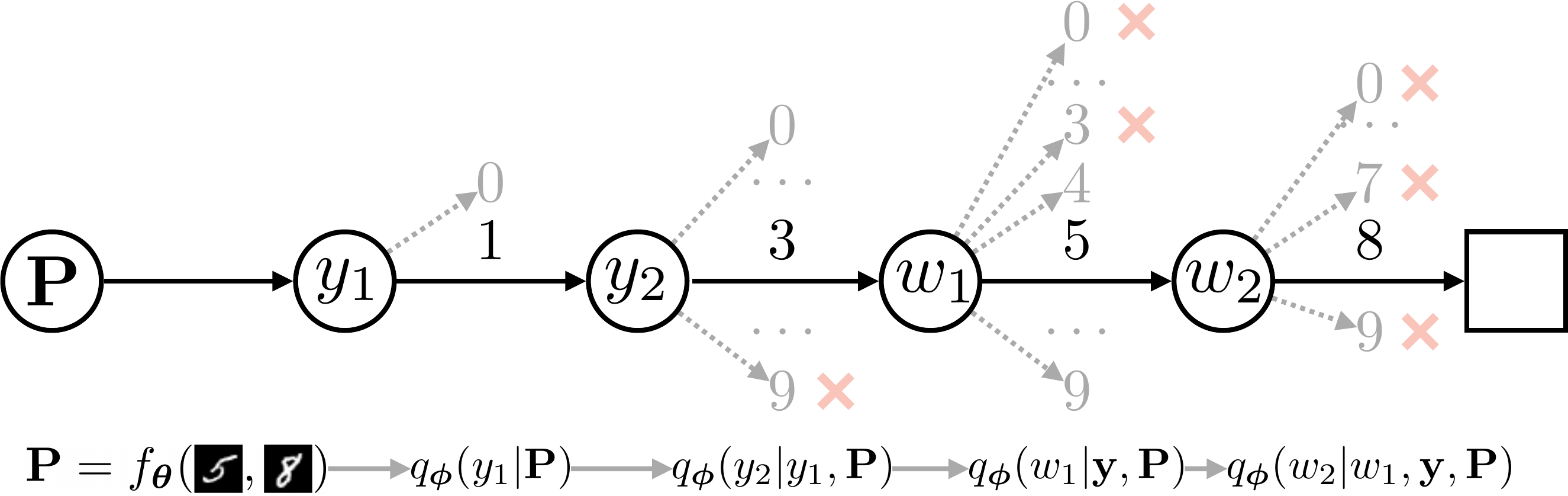}
	\caption[Example rollout sampling an explainable variant on input $\bx=(\digit{A-NeSI/images/mnist_5.png}, \digit{A-NeSI/images/mnist_8.png})$]{Example rollout sampling an explainable variant on input $\bx=(\digit{A-NeSI/images/mnist_5.png}, \digit{A-NeSI/images/mnist_8.png})$. For $y_2$, we prune option 9, as the highest attainable sum is $9+9=18$. For $w_1$, $\{0, ..., 3\}$ are pruned as there is no second digit to complete the sum. $w_2$ is deterministic given $w_1$, and prunes all branches but 8.}
	\label{fig:NIM_example}
\end{figure*}
An attractive option is to use symbolic knowledge to ensure the inference model only generates valid outputs. We can compute each factor $q_\bphi(w_i|\by, \bw_{1:i-1}, \bP)$ (both for world and output variables) using a neural network $\hat{q}_{\bphi, i}$ and a \emph{symbolic pruner} $\pruner_i$:
\begin{align}
	\begin{split}
	&q_\paramE(w_i|\by, \bw_{1:i-1}, \bP) = \frac{q_{w_i}\pruner_{w_i}}{\mathbf{q}\cdot \mathbf{\pruner}}, \\ 
	\text{where}\ &\mathbf{q} = \hat{q}_{\paramE, i}(\bw_{1:i-1}, \by, \bP), \quad
	\mathbf{\pruner} = \pruner_i(\bw_{1:i-1}, \by)	.
	\end{split}
\end{align}
The symbolic pruner sets the probability mass of certain branches $w_i$ to zero. Then, $q_\bphi$ is computed by renormalizing. If we know that by expanding $\bw_{1:i-1}$ it will be impossible to produce a possible world for $\by$, we can set the probability mass under that branch to 0: we will know that $p(\bw, \by)=0$ for all such branches. In Figure \ref{fig:NIM_example} we give an example for single-digit MNISTAdd. Symbolic pruning significantly reduces the number of branches our algorithm needs to explore during training. Moreover, symbolic pruning is critical in settings where verifiability and safety play crucial roles, such as medicine. We discuss the design of symbolic pruners $s$ in Appendix \ref{appendix:symbolic_pruner}.


\section{Experiments}
\label{sec:experiments-anesi}

We study three Neurosymbolic reasoning tasks to evaluate the performance and scalability of A-NeSI: Multi-digit MNISTAdd \citep{manhaeveApproximateInferenceNeural2021}, Visual Sudoku Puzzle Classification \citep{augustineVisualSudokuPuzzle2022} and Warcraft path planning. Code is available at \url{https://github.com/HEmile/a-nesi}.
We used the ADAM optimizer throughout. 

\textsc{A-NeSI} has two prediction methods. 1) \textbf{Symbolic prediction} uses the symbolic reasoning function $\symfun$ to compute the output: $\hat{\by}=\symfun({\arg\max}_{\bw} p(\bw|\bP=f_\btheta(\bx)))$. 2) \textbf{Neural prediction}  predicts with the prediction network $q_\paramP$ using a beam search: $\hat{\by}={\arg\max}_{\by} q_\paramP(\by|\bP=f_\btheta(\bx))$. 
In our studied tasks, neural prediction cannot perform better than symbolic prediction. We consider the prediction network adequately trained if it matches the accuracy of symbolic prediction.  

\subsection{Multi-digit MNISTAdd}
We evaluate \textsc{A-NeSI} on the Multi-Digit MNISTAdd task (Section \ref{sec:problem-statement}). For the perception model, we use the same CNN as in DeepProbLog \citep{manhaeveDeepProbLogNeuralProbabilistic2018}. 
The prediction model has $N+1$ factors $q_\paramP(y_i|\by_{1:i-1}, \bP)$, while the explanation model has $2N$ factors $q_\paramE(w_i|\by, \bw_{1, i-1}, \bP)$. We model each factor with a separate MLP. $y_i$ and $w_i$ are one-hot encoded digits, except for the first output digit $y_1$: it can only be 0 or 1. We used a shared set of hyperparameters for all $N$. Like \cite{manhaeveNeuralProbabilisticLogic2021,manhaeveApproximateInferenceNeural2021}, we take the MNIST \citep{lecunMNISTHandwrittenDigit2010} dataset and use each digit exactly once to create data. We follow \cite {manhaeveNeuralProbabilisticLogic2021} and require more unique digits for increasing $N$. Therefore, the training dataset will be of size $60000/2N$ and the test dataset of size $10000/2N$. For more details and a description of the baselines, see Appendix \ref{appendix:mnist_add}.

\begin{table*}
	\centering
	\begin{tabular}{ l | l l l l l }
		& N=1 & N=2 & N=4 & N=15\\
		\hline 
		 & \multicolumn{4}{c}{\textbf{Symbolic prediction}} \\
		LTN & 80.54 $\pm$ 23.33 & 77.54 $\pm$ 35.55 & T/O & T/O\\
		DeepProbLog & 97.20 $\pm$ 0.50 & 95.20 $\pm$ 1.70 & T/O & T/O\\
		DPLA*  & 88.90 $\pm$ 14.80 & 83.60 $\pm$ 23.70 & T/O & T/O\\
		DeepStochLog  & \textbf{97.90 $\pm$ 0.10} & \textbf{96.40 $\pm$ 0.10} & \textbf{92.70 $\pm$ 0.60} & T/O\\
		Embed2Sym  & 97.62 $\pm$ 0.29 & 93.81 $\pm$ 1.37 & 91.65 $\pm$ 0.57 & 60.46 $\pm$ 20.36\\
		\textsc{A-NeSI} (predict) & 97.66 $\pm$ 0.21 & 95.96 $\pm$ 0.38 & 92.56 $\pm$ 0.79 & 75.90 $\pm$ 2.21\\
		\textsc{A-NeSI} (explain) & 97.37 $\pm$ 0.32 & 96.04 $\pm$ 0.46 & 92.11 $\pm$ 1.06 & \textbf{76.84 $\pm$ 2.82}\\
		\textsc{A-NeSI} (pruning) & 97.57 $\pm$ 0.27 & 95.82 $\pm$ 0.33 & 92.40 $\pm$ 0.68 & 76.46 $\pm$ 1.39\\
		\textsc{A-NeSI} (no prior) & 76.54 $\pm$ 27.38 & 95.67 $\pm$ 0.53 & 44.58 $\pm$ 38.34 & 0.03 $\pm$ 0.09\\
		\hline 
		 & \multicolumn{4}{c}{\textbf{Neural prediction}} \\
		Embed2Sym  & 97.34 $\pm$ 0.19 & 84.35 $\pm$ 6.16 & 0.81 $\pm$ 0.12 & 0.00\\
		\textsc{A-NeSI} (predict) & \textbf{97.66 $\pm$ 0.21} & 95.95 $\pm$ 0.38 & \textbf{92.48 $\pm$ 0.76} & 54.66 $\pm$ 1.87\\
		\textsc{A-NeSI} (explain) & 97.37 $\pm$ 0.32 & \textbf{96.05 $\pm$ 0.47} & 92.14 $\pm$ 1.05 & \textbf{61.77 $\pm$ 2.37}\\
		\textsc{A-NeSI} (pruning) & 97.57 $\pm$ 0.27 & 95.82 $\pm$ 0.33 & 92.38 $\pm$ 0.64 & 59.88 $\pm$ 2.95\\
		\textsc{A-NeSI} (no prior) & 76.54 $\pm$ 27.01 & 95.28 $\pm$ 0.62 & 40.76 $\pm$ 34.29 & 0.00 $\pm$ 0.00\\
		\hline 
		Reference      & 98.01            & 96.06            & 92.27            & 73.97
		
	\end{tabular}
	\caption[Test accuracy of predicting the correct sum on the Multi-digit MNISTAdd task.]{Test accuracy of predicting the correct sum on the Multi-digit MNISTAdd task. ``T/O'' (timeout) represent computational timeouts. Reference accuracy approximates the accuracy of an MNIST predictor with $0.99$ accuracy using $0.99^{2N}$. 
	Bold numbers signify the highest average accuracy for some $N$ within the prediction categories.
	\emph{predict} is the prediction-only variant, \emph{explain} is the explainable variant, \emph{pruning} adds symbolic pruning (see Appendix \ref{appendix:MNISTAdd-pruner}), and \emph{no prior} is the prediction-only variant \emph{without} the prior $p(\bP)$.
 }
	\label{table:results}
\end{table*}

Table \ref{table:results} reports the accuracy of predicting the sum. For all $N$, \textsc{A-NeSI} is close to the reference accuracy, meaning there is no significant drop in the accuracy of the perception model as $N$ increases. For small $N$, it is slightly outperformed by DeepStochLog \citep{wintersDeepStochLogNeuralStochastic2022}, which can not scale to $N=15$. \textsc{A-NeSI} also performs slightly better than DeepProbLog, showing approximate inference does not hurt the accuracy. We believe the accuracy improvements compared to DeepProbLog to come from hyperparameter tuning and longer training times, as \textsc{A-NeSI} approximates DeepProbLog's semantics. 
With neural prediction, we get the same accuracy for low $N$, but there is a significant drop for $N=15$, meaning the prediction network did not perfectly learn the problem. However, compared to training a neural network without background knowledge (Embed2Sym with neural prediction), it is much more accurate already for $N=2$. Therefore, \textsc{A-NeSI}'s training loop allows training a prediction network with high accuracy on large-scale problems.

Comparing the different \textsc{A-NeSI} variants, we see the prediction-only, explainable and pruned variants perform quite similarly, with significant differences only appearing at $N=15$ where the explainable and pruning variants outperform the predict-only model, especially on neural prediction. 
However, when removing the prior $p(\bP)$, the performance degrades quickly. The prediction model sees much fewer beliefs $\bP$ than when sampling from a (high-entropy) prior $p(\bP)$. A second and more subtle reason is that at the beginning of training, all the beliefs $\bP=f_\btheta(\bx)$ will be uniform because the perception model is not yet trained. Then, the prediction model learns to ignore the input belief $\bP$. 

\begin{figure}
	\centering
	\includegraphics[width=0.95\textwidth]{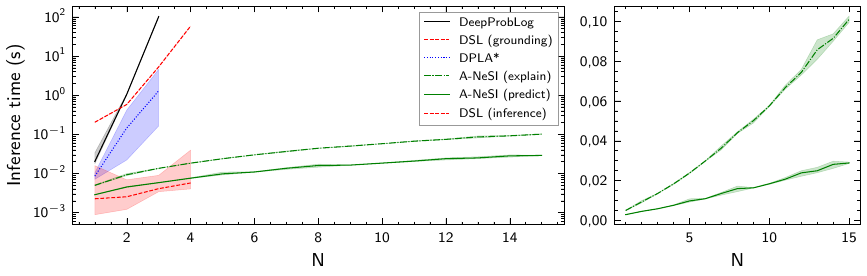}
        \vskip -3mm
	\caption[Inference time for a single input $\bx$ for different nr. of digits]{Inference time for a single input $\bx$ for different nr. of digits. The left plot uses a log scale, and the right plot a linear scale. DSL stands for DeepStochLog. We use the GM variant of DPLA*.  
	}
	\label{fig:timings}
\end{figure}

Figure \ref{fig:timings} shows the runtime for inference on a single input $\bx$. 
Inference in DeepProbLog (corresponding to exact inference) increases with a factor 100 as $N$ increases, and DPLA* (another approximation) is not far behind. Inference in DeepStochLog, which uses different semantics, is efficient due to caching but requires a grounding step that is exponential both in time and memory. We could not ground beyond $N=4$ because of memory issues. \textsc{A-NeSI} avoids having to perform grounding altogether: it scales slightly slower than linear. Furthermore, it is much faster in practice as parallelizing the computation of multiple queries on GPUs is trivial.

\subsection{Visual Sudoku Puzzle Classification}
In this task, the goal is to recognize whether an $N\times N$ grid of MNIST digits is a Sudoku. We have 100 examples of correct and incorrect grids from \cite{augustineVisualSudokuPuzzle2022}. We use the same MNIST classifier as in the previous section. We treat $\bw$ as a grid of digits in $\{1, \dots, N\}$ and designed an easy-to-learn representation of the label. For sudokus, all pairs of digits $\bw_i, \bw_j$ in a row, column, and block must differ. For each such pair, a dimension in $\by$ is 1 if different and 0 otherwise, and the symbolic function $\symfun$ computes these interactions from $\bw$. The prediction model is a single MLP that takes the probabilities for the digits $\bP_i$ and $\bP_j$ and predicts the probability that those represent different digits. 

More precisely, we see $\bx$ as a $\mathbb{R}^{N \times N\times 784}$ tensor of MNIST images, and beliefs $\bP$ as an $N\times N \times N$ grid of distributions over $N$ options. For each grid index $i, j$, the world variable $\bw_{i, j}$ corresponds to the digit at location $(i, j)$. Digits at different locations $(i, j)$ and $(i', j')$ need to differ if $i=i'$, $j=j'$ or if $(i, j)$ is in the same block as $(i', j')$. For each such pair of locations, a dimension in $\by$ indicates if the digits differ. 
We use a \emph{single} MLP $f_\btheta$ for the prediction model. For each pair, we compute $q_\paramP(\by_k|\bP)=f_\paramP(\bP_{i, j}, \bP_{i', j'})$, introducing the independence assumption that the digits at location $(i, j)$ and location $(i', j')$ being different does not depend on the digits at other locations. This is, clearly, wrong. However, we found it is sufficient to train the perception model accurately. 


Table \ref{table:visudo_results} shows the accuracy of classifying Sudoku puzzles. \textsc{A-NeSI} is the only method to perform better than random on $9\times 9$ sudokus. Exact inference cannot scale to $9\times 9$ sudokus, while we were able to run \textsc{A-NeSI} for 3000 epochs in 38 minutes on a single NVIDIA RTX A4000. For additional details, see Appendix \ref{appendix:visual_sudoku}.
\begin{table*}
	\centering
	\begin{tabular}{ l | l l }
		& N=4 & N=9 \\
		\hline 
		CNN & 51.50 $\pm$ 3.34 & 51.20 $\pm$ 2.20 \\
		Exact inference & 86.70 $\pm$ 0.50 & T/O \\
        NeuPSL & \textbf{89.7 $\pm$ 2.20} & 51.50 $\pm$ 1.37\\
		\textsc{A-NeSI} (symbolic prediction) & \textbf{89.70 $\pm$ 2.08} & \textbf{62.15 $\pm$ 2.08}\\
		\textsc{A-NeSI} (neural prediction) & \textbf{89.80 $\pm$ 2.11} & \textbf{62.25 $\pm$ 2.20}\\
	\end{tabular}
	\caption{Test accuracy of predicting whether a grid of numbers is a Sudoku.}
	\label{table:visudo_results}
\end{table*}

\subsection{Warcraft Visual Path Planning}

\begin{figure}[ht]
	\centering
	\begin{minipage}[b]{0.23\textwidth}
		\includegraphics[width=\linewidth]{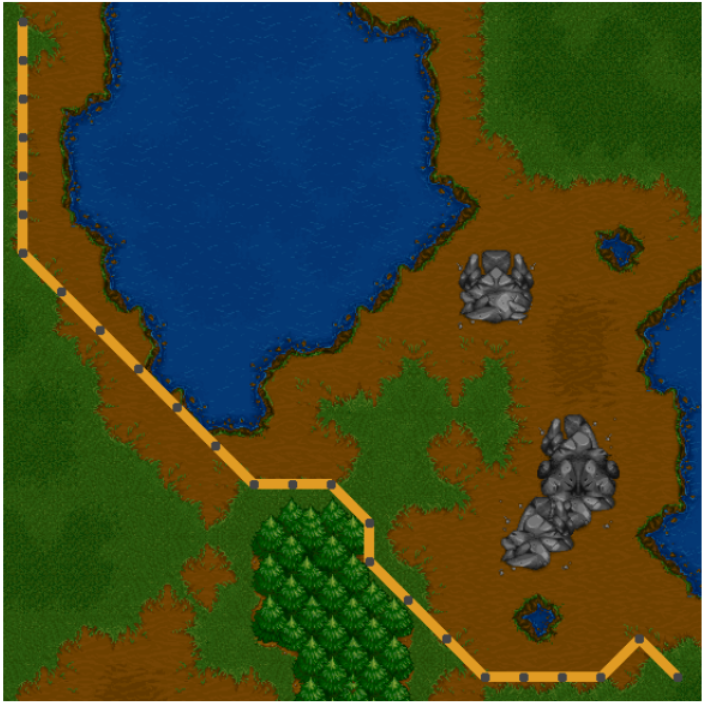}
		\caption{Warcraft task.}
		\label{fig:warcraft}
	\end{minipage}
	\hspace{0.00\textwidth}
	\begin{minipage}[b]{0.75\textwidth}

	\begin{table}[H]
	  \begin{tabular}{ l | l l }
	    & $12\times 12$ & $30 \times 30$ \\
	    \hline
	    ResNet18 \citep{pogancicDifferentiationBlackboxCombinatorial2020} & 23.0 $\pm$ 0.3 & 0.0 $\pm$ 0.0 \\
	    SPL \citep{ahmedSemanticProbabilisticLayers2022} & 78.2 & T/O \\
	    I-MLE \citep{niepertImplicitMLEBackpropagating2021} & 97.2 $\pm$ 0.5 & \textbf{93.7 $\pm$ 0.6} \\
	    \hline
	    RLOO \citep{koolBuyREINFORCESamples2019} & 43.75 $\pm$ 12.35 & 12.59 $\pm$ 16.38 \\
	    \textsc{A-NeSI} & 94.57 $\pm$ 2.27 & 17.13 $\pm$ 16.32\\
	    \textsc{A-NeSI} + RLOO & \textbf{98.96 $\pm$ 1.33} & 67.57 $\pm$ 36.76\\
	  \end{tabular}
	  \caption{Test accuracy of predicting the lowest-cost path on the Warcraft Visual Path Planning task.}
	  \label{table:warcraft_results}
	\end{table}
	\end{minipage}
\end{figure}

The Warcraft Visual Path Planning task \citep{pogancicDifferentiationBlackboxCombinatorial2020} is to predict a shortest path from the top left corner of an $N\times N$ grid to the bottom right, given an image from the Warcraft game (Figure \ref{fig:warcraft}). The perception model has to learn the cost of each tile of the Warcraft image. We use a small CNN that takes the tile $i, j$ (a $3\times 8 \times 8$ image) and outputs a distribution over 5 possible costs. 
We see $\bx$ as a $N\times N \times 3 \times 8 \times 8$ real tensor: The first two dimensions indicate the different grid cells, the third dimension indicates the RGB color channels, and the last two dimensions indicate the pixels in each grid cell. The world $\bw$ is an $N\times N$ grid of integers, where each integer indexes five different costs for traversing that cell. The symbolic function $\symfun$ takes the grid of costs $\bw$ and returns the shortest path from the top left corner $(1, 1)$ to the bottom right corner $(N, N)$ using Dijkstra's algorithm. We encode the shortest path as a sequence of actions to take in the grid, where each action is one of the eight (inter-)cardinal directions (down, down-right, right, etc.). The sequence is padded with the do-not-move action to allow for batching.

The prediction model is a ResNet18 model \citep{heDeepResidualLearning2016}, with an output layer of 8 options. It takes an image of size $6\times N \times N$ as input. The first 5 channels are the probabilities $\bP_{i, j}$, and the last channel indicates the current position in the grid. The 8 output actions correspond to the 8 (inter-)cardinal directions. We apply symbolic pruning (Section \ref{sec:symbolic-pruning}) to prevent actions that would lead outside the grid or return to a previously visited grid cell. We pretrained the prediction model by repeating Algorithm \ref{alg:train-NIM} on a fixed prior. See Appendix \ref{appendix:path_planning} for additional details.

Table \ref{table:warcraft_results} presents the accuracy of predicting a shortest path. \textsc{A-NeSI} is competitive on $12 \times 12$ grids but struggles on $30 \times 30$ grids. We believe this is because the prediction model is not accurate enough, resulting in gradients with too high bias. Still, \textsc{A-NeSI} finds short paths, is far better than a pure neural network, and can scale to $30\times 30$ grids, unlike SPL \citep{ahmedSemanticProbabilisticLayers2022}, which uses exact inference. Since both SPL and I-MLE \citep{niepertImplicitMLEBackpropagating2021} have significantly different setups (see Appendix \ref{appendix:path_planning_other_methods}), we added experiments using REINFORCE with the leave-one-out baseline (RLOO, \citep{koolBuyREINFORCESamples2019}) that we implemented with the Storchastic PyTorch library \citep{vankriekenStorchasticFrameworkGeneral2021}. We find that RLOO has high variance in its performance. Since RLOO is unbiased with high variance and \textsc{A-NeSI} is biased with no variance, we also tried running \textsc{A-NeSI} and RLOO simultaneously. Interestingly, this is the best-performing method on the $12 \times 12$ grid, and has competitive performance on $30\times 30$ grids, albeit with high variance (6 out 10 runs get to an accuracy between 93.3\% and 98.5\%, while the other runs are stuck around 25\% accuracy).

\section{Related work}
\textsc{A-NeSI} can approximate multiple PNL methods \citep{deraedtNeurosymbolicNeuralLogical2019}. DeepProbLog \citep{manhaeveDeepProbLogNeuralProbabilistic2018} performs symbolic reasoning by representing $\bw$ as ground facts in a Prolog program. It enumerates all possible proofs of a query $\by$ and weights each proof by $p(\bw|\bP)$. NeurASP \citep{yangNeurASPEmbracingNeural2020} is a PNL framework closely related to DeepProbLog, but is based on Answer Set Programming \citep{brewkaAnswerSetProgramming2011}. 
Some methods consider constrained structured output prediction \citep{giunchigliaDeepLearningLogical2022}. In Appendix \ref{appendix:background-knowledge}, we discuss extending \textsc{A-NeSI} to this setting. Semantic Loss \citep{xuSemanticLossFunction2018} improves learning 
with a loss function but 
does not guarantee that formulas are satisfied at test time. Like \textsc{A-NeSI}, Semantic Probabilistic Layers \citep{ahmedSemanticProbabilisticLayers2022} solves this with a layer that performs constrained prediction. 
These approaches perform exact inference using probabilistic circuits (PCs) \citep{yoojungProbabilisticCircuitsUnifying2020}. 
Other methods perform approximate inference by only considering the top-k proofs in PCs \citep{manhaeveNeuralProbabilisticLogic2021,huangScallopProbabilisticDeductive2021}. However, finding those proofs is hard, especially when beliefs have high entropy, and limiting to top-k significantly reduces performance. 
Other work considers MCMC approximations \citep{liSoftenedSymbolGrounding2023}. Using neural networks for approximate inference ensures computation time is constant and independent of the entropy of $\bP$ or long MCMC chains.  


Other neurosymbolic methods use fuzzy logics \citep{badreddineLogicTensorNetworks2022,diligentiSemanticbasedRegularizationLearning2017,danieleRefiningNeuralNetwork2023a,giunchigliaMultiLabelClassificationNeural2021a}, which are faster than PNL with exact inference. Although traversing ground formulas is linear time, the grounding is itself often exponential \citep{manhaeveApproximateInferenceNeural2021}, so the scalability of fuzzy logics often fails to deliver. \textsc{A-NeSI} is polynomial in the number of ground atoms and does not traverse the ground formula. Furthermore, background knowledge 
is often not fuzzy \citep{vankriekenAnalyzingDifferentiableFuzzy2022,grespanEvaluatingRelaxationsLogic2021}, and fuzzy semantics does not preserve classical equivalence. 


\textsc{A-NeSI} performs gradient estimation of the WMC problem. We can extend our method to biased but zero-variance gradient estimation by learning a distribution over function outputs (see Appendix \ref{appendix:gradient-estimation}). Many recent works consider continuous relaxations of discrete computation to make them differentiable \citep{petersenLearningDifferentiableAlgorithms2022,jangCategoricalReparameterizationGumbelsoftmax2017} but require many tricks to be computationally feasible. Other methods compute MAP states to compute the gradients \citep{niepertImplicitMLEBackpropagating2021,blondelEfficientModularImplicit2022,sahooBackpropagationCombinatorialAlgorithms2023} but are restricted to integer linear programs. 
The score function (or `REINFORCE') gives unbiased yet high-variance gradient estimates \citep{mohamedMonteCarloGradient2020,vankriekenStorchasticFrameworkGeneral2021}. Variance reduction techniques, such as memory augmentation \citep{liangMemoryAugmentedPolicy2018} and leave-one-out baselines \citep{koolBuyREINFORCESamples2019}, exist to reduce this variance. 

 



\section{Conclusion, Discussion and Limitations}
We introduced \textsc{A-NeSI}, a scalable approximate method for probabilistic neurosymbolic learning. We demonstrated that \textsc{A-NeSI} scales to combinatorially challenging tasks without losing accuracy.
\textsc{A-NeSI} 
can be extended to include explanations and hard constraints without loss of performance. 

However, there is no `free lunch': when is \textsc{A-NeSI} a good approximation? 
We discuss three aspects of learning tasks that could make it difficult to learn a strong and efficient inference model.
\begin{itemize}
\item \textbf{Dependencies of variables.} When variables in world $\bw$ are highly dependent, finding an informative prior is hard. We suggest then using a prior that can incorporate dependencies such as a normalizing flow \citep{rezendeVariationalInferenceNormalizing2015,chenVariationalLossyAutoencoder2022} or deep generative models \citep{tomczakDeepGenerativeModeling2022} over a Dirichlet distribution.
\item \textbf{Structure in symbolic reasoning function.} We studied reasoning tasks with a relatively simple structure. Learning the inference model will be more difficult when the symbolic function $\symfun$ is less structured. Studying the relation between structure and learnability is interesting future work. 
\item \textbf{Problem size.} \textsc{A-NeSI} did not perfectly train the prediction model in more challenging problems. 
 We expect its required parameter size and training time to increase with the problem size.
 \end{itemize}

Promising future avenues are 
studying if the explanation model produces helpful explanations \citep{shterionovMostProbableExplanation2015}, 
extensions to continuous random variables \citep{gutmannExtendingProbLogContinuous2011} (see Appendix \ref{appendix:gradient-estimation} for an example), and 
extensions to 
unnormalized distributions such as Markov Logic Networks\citep{richardsonMarkovLogicNetworks2006}, as well as (semi-) automated \textsc{A-NeSI} solutions for neurosymbolic programming languages like DeepProbLog \citep{manhaeveDeepProbLogNeuralProbabilistic2018}. 



\newpage

\chapter{Discussion and Conclusion}
In this dissertation, we studied the optimisation of neurosymbolic learning systems. Taking perspectives both from the fields of machine learning and symbolic AI, we both analysed existing methods (Chapter \ref{ch:dfl}) and developed new methods (Chapters \ref{ch:lrl}-\ref{ch:anesi}). We studied methods using fuzzy logic (Part \ref{part:1}) and probabilistic logics (Part \ref{part:2}).

\section{Reflection on Research Questions}
Next, we answer the research questions given in the introduction (Chapter \ref{ch:introduction}), refer to related follow-up work and propose specific open questions. We conclude with a more general discussion of the field of neurosymbolic learning and propose several directions for future research.

\subsection{Research question 1: \emph{Differentiable Fuzzy Logic Operators}}
\vspace{-0.2cm}
\emph{If we use fuzzy logic operators as the basis for loss functions, what happens when we differentiate this loss function?}
\medskip

\textbf{Answer.} In Chapter \ref{ch:dfl}, we analysed the derivatives of fuzzy logic operators to understand the learning process. We found that each of these derivatives reasons differently: For instance, the Gödel conjunction increases the lowest value, the Reichenbach implication increases the consequent in proportion to the antecedent, and the \luk{} disjunction increases all literals evenly until their sum is larger than 1. This analysis helped us understand how the fuzzy logic operators behave in many practical settings and allowed us to argue why and when certain fuzzy logic operators will not behave as expected. Additionally, we recommended what operators to use in practice, which we motivated theoretically and experimentally. Empirically, the best combination of operators uses operators that are not logically related, forgoing many logical equivalences. We found a surprising imbalance between the modus ponens and modus tollens gradients. In retrospect, this studied a latent worlds problem in neurosymbolic learning. Using these insights, we developed the Reichenbach-sigmoidal implication, which balanced the modus ponens and modus tollens gradients and empirically outperformed the other implications. 

\textbf{Related recent work.} The study of the behaviour of fuzzy logic operators in the context of neurosymbolic learning is continued in several research papers by other authors. \cite{grespanEvaluatingRelaxationsLogic2021} also studied fuzzy logic operators theoretically and empirically, clarifying what tautologies are preserved and comparing their performance in several learning settings. \cite{gianniniTnormsDrivenLoss2023} studied how to construct loss functions using fuzzy logic operators formed from additive generators (which we also discussed in Chapter \ref{ch:lrl}). \cite{heReducedImplicationbiasLogic2022} studied the problems caused by the raven paradox and proposes a new \emph{Reduced Implication-Bias Logic Loss}. \cite{wagnerNeuralSymbolicReasoningOpenWorld2022,wagnerReasoningWhatHas2022} studied the raven paradox from a data-driven analysis in an iterative LTN approach. From a more formal perspective, \cite{slusarzLogicDifferentiableLogics2023} define a \emph{Logic of Differentiable Logics} that encompass several of the logics formed from the fuzzy logic operators discussed, then discuss which of these logics have certain formal properties. \cite{badreddineLogicTensorNetworks2022} extended the recommendations on product fuzzy logics to develop \emph{Stable Product Real Logic}. Continuing this study, \cite{badreddineLogLTNDifferentiableFuzzy2023} developed \emph{logLTN}, a variation of Logic Tensor networks entirely computed in log-space. Finally, several papers used the recommendations for fuzzy logic operators given in this chapter to develop their neurosymbolic system \citep{carraroLogicTensorNetworks2023,lambertiEndtoendTrainingLogic2021,wuDifferentiableFuzzyMathcalALC2022}.

\textbf{Future work.} While this chapter and the follow-up work mentioned have extensively studied the use of fuzzy logic operators in neurosymbolic learning, there are still many open questions. So far, the operators have only been studied in a minimal number of settings, in particular where propositions are probabilistic rather than fuzzy. Furthermore, there is ample opportunity for studying additional properties of the derivatives of fuzzy logic operators, with the vast literature of fuzzy logic to take inspiration from \cite{klementTriangularNorms2000,klementTriangularNormsPosition2004}. Finally, studying the derivatives of \emph{compositions} of operators would provide an understanding of how operators interact in practical neurosymbolic systems.

\subsection{Research question 2: \emph{Iterative Local Refinement}} 
\vspace{-0.2cm}
\emph{How can we use fuzzy logic operators to develop neural network layers that enforce background knowledge?}
\medskip

\textbf{Answer.} In Chapter \ref{ch:lrl}, we defined an optimisation objective which, given a fuzzy logic operator and background knowledge, defines \emph{refinement functions} that correct neural network predictions to satisfy the background knowledge. Refinement functions can be used as a neural network layer to enforce background knowledge during training and testing. However, we found that closed-form solutions of the optimisation objective are challenging to find, in particular for more complex formulas. We found closed-form results for simple formulas for the Gödel and \luk t-norms and for a large class of fuzzy logic operators, including the product t-norm. For complex formulas, we developed an iterative algorithm called Iterative Local Refinement (ILR) that approximates refinement functions by using the closed-form results for simple formulas. We found that this iterative algorithm is effective in practice but is not guaranteed to converge to the global optimum.

\textbf{Future work.} While we provided the first closed-form formulas for some refinement functions, a large set of fuzzy logic operators has yet to be studied in this setting, some of which may have benefits compared to the studied operators. Furthermore, for the well-known product t-norm, we could not derive closed-form solutions even for simple formulas for norms other than the $L_1$ norm. A further study into this t-norm may result in a refinement function inheriting the beneficial properties of product fuzzy logic as recommended in Chapter \ref{ch:dfl}. Finally, while we have initial experiments on the effectiveness of ILR in the popular MNIST Addition \citep{manhaeveApproximateInferenceNeural2021} benchmark, there is little empirical work on the effectiveness of the ILR layer in practical neurosymbolic settings. 

\subsection{Research question 3: \emph{Storchastic}}
\vspace{-0.2cm}
\emph{How can we perform stochastic optimisation over an arbitrary mix of discrete and continuous computation?}
\medskip

\textbf{Answer.} In Chapter \ref{ch:storchastic}, we studied gradient estimation on general stochastic computation graphs (SCGs). SCGs allow for defining arbitrary mixes of discrete and continuous computation by defining distributions over function inputs and sampling from these distributions. We extended the DiCE gradient estimator \citep{foersterDiCEInfinitelyDifferentiable2018} to various gradient estimation methods developed for variance reduction. The resulting Storchastic framework defines a surrogate loss function to minimise using gradient descent. We proved that this loss function gives unbiased gradients if local gradient estimation methods are unbiased. Furthermore, we proved that variance reduction through control variates also applies to higher-order gradient estimation when implemented using the Storchastic surrogate loss. Finally, we implemented Storchastic as a library in PyTorch, which is available at \url{github.com/HEmile/storchastic}. 

\textbf{Related recent work.} Recent work studied automatic gradient estimation on compositions of arbitrary computation. \cite{aryaAutomaticDifferentiationPrograms2022} studies unbiased and low-variance gradient estimation for discrete computation. The authors develop a reparameterisation-based method that samples correlated noise to reduce variance. ADEV \citep{lewADEVSoundAutomatic2023} tackles the gradient estimation problem from a (both functional and probabilistic) programming perspective. The appendix of \cite{lewADEVSoundAutomatic2023} describes an extension of ADEV incorporating the four components defined by the Storchastic surrogate loss, thereby generalising Storchastic. However, unlike Storchastic, ADEV is forward-mode, which is much less efficient for most learning applications. Still, their functional programming perspective makes gradient estimation amenable to program analysis, allowing static analysis and program transformations. Finally, Proppo \citep{parmasProppoMessagePassing2022} also defines a PyTorch framework where users can implement arbitrary computation and general gradient estimation. Their design allows for gradient estimation methods that cannot be implemented in Storchastic. However, it does not generalise to higher-order gradient estimation. 

\textbf{Future work.} Future work could develop efficient integrations of Storchastic with general probabilistic programming frameworks \cite{vandemeentIntroductionProbabilisticProgramming2021} to provide variance reduction. While ADEV \cite{lewADEVSoundAutomatic2023} has taken a significant first step towards this goal, it is hard to parallelise due to being forward-mode. An extension of ADEV to backward-mode differentiation allows an implementation of Storchastic that benefits from the program analysis properties of ADEV. Furthermore, future work could apply the Storchastic framework to study what gradient estimation methods, particularly those for discrete random variables, are most effective in practice and how this relates to the structure of the SCGs. 

Optimisation in more complex structures is another significant open problem. Most neurosymbolic learning methods assume the simple perception-then-reasoning structure discussed in the other chapters. This structure effectively has only a single stochastic node. Why would reasoning not influence perception? As studied in \cite{vanharmelenBoxologyDesignPatterns2019}, many design patterns interleaving learning and reasoning can be of interest. One example is VAEL \citep{misinoVAELBridgingVariational2022}, which learns a discrete latent space with semantics. With Storchastic, we can optimise through arbitrary compositions of learning and reasoning. Still, developing such systems is an immense challenge for optimisation, symbol grounding and semantics.

\subsection{Research question 4: \emph{A-NeSI}} 
\vspace{-0.2cm}
\emph{How can we efficiently perform inference in probabilistic logics to scale probabilistic neurosymbolic systems?}
\medskip

\textbf{Answer.} In Chapter \ref{ch:anesi}, we introduced \emph{A-NeSI} for approximate inference in probabilistic logics. Our method trains two neural networks that approximate the weighted model count and its related most probable explanation problem (the posterior). A-NeSI scales to problem sizes far greater than those exact inference can solve without sacrificing accuracy. Additionally, we showed that the neural network modelling the posterior can explain simple decisions. Finally, we discussed an extension of A-NeSI that guarantees that generated solutions are consistent with the background knowledge.

\textbf{Future work.} A-NeSI provides many opportunities for future work. A significant limitation is requiring a prior over beliefs. Future methods could study better heuristics for picking or learning this prior. Furthermore, we only evaluated A-NeSI in particular neurosymbolic problems. A study of the accuracy and speed of A-NeSI in more common weighted model counting problems helps understand how broad the applicability of A-NeSI is for general approximate weighted model counting. Finally, A-NeSI provides ample design space for possible improvements. One example is improving the neural network architectures. We mostly used simple fully-connected networks. Instead, one might pick architectures with inductive biases that are better suited for reasoning, such as transformers \citep{vaswaniAttentionAllYou2017}, SATNet \citep{wangSATNetBridgingDeep2019} or GNNs \citep{velickovicNeuralExecutionGraph2020}. Another example is improving the training procedure and making it more sample efficient. Since we based A-NeSI on GFlowNets \citep{bengioFlowNetworkBased2021b}, future research could take inspiration from the many recent advances in studying GFlowNets \citep{bengioGFlowNetFoundations2022,madanLearningGFlowNetsPartial2022,huGFlowNetEMLearningCompositional2023}.


\section{Discussion and Future Directions}
Next, we would like to reflect on the broader state of Neurosymbolic Learning and its relation to optimisation, and provide our outlook on the future of the field. First, we discuss the two problems we discussed in the Introduction (Section \ref{sec:fundamental-nesy-problems}). We will discuss to what degree these problems are solved since the start of this PhD trajectory and what challenges remain. Finally, we will discuss major open directions for Neurosymbolic Learning. 

\subsection{Bridging the gap between discrete and continuous computation}
This dissertation broadly discussed two methods for the first fundamental problem (Section \ref{sec:discr-cont-bridge}): Bridging the gap between discrete and continuous computation: First, a continuous relaxation of the discrete computation, in this dissertation studied in the form of Fuzzy Logic (Part \ref{part:1}), but also employed in many problems outside logic such as discrete variational inference \citep{jangCategoricalReparameterizationGumbelsoftmax2017}, differentiable algorithms \citep{petersenLearningDifferentiableAlgorithms2022,groverStochasticOptimizationSorting2019}, learning to explain \citep{paulusGradientEstimationStochastic2020} and parsing \citep{niculaeDiscreteLatentStructure2023}. Second, the definition of distributions over discrete computation to allow the optimisation of expectations, in this dissertation studied in Storchastic (Chapter \ref{ch:storchastic}) and probabilistic logic (Chapter \ref{ch:anesi}). This definition falls under the umbrella of probabilistic modelling, which has appealing and accessible implementations in probabilistic (logic) programming \citep{vandemeentIntroductionProbabilisticProgramming2021,deraedtProbabilisticLogicProgramming2015}.

Within the more traditional field of deep learning, \say{make everything differentiable} is the common ethos. Indeed, continuous relaxations of discrete computation are currently the most popular approach to incorporating discrete-like computation into neural networks. For example, the highly successful (self-)attention layer \cite{vaswaniAttentionAllYou2017,bahdanauNeuralMachineTranslation2015} is a relaxed version of the discrete but far less popular (hard) attention layer \cite{serraOvercomingCatastrophicForgetting2018}. Hard attention allows for \emph{conditional computation}: Discrete choices, such as attending with a binary mask over a sequence, allow us to condition. We only need to compute the neural network over the attended sequence and can ignore the rest, saving a huge amount of computation. Continuous relaxation lacks conditional computation: We cannot branch on relaxed values, as they are continuous. Indeed, what does $\mathsf{if}(0.5)$ mean? 

Most continuous relaxations are motivated by the probabilistic interpretation of mixed continuous-discrete reasoning: They heuristically estimate the true expected value and are primarily concerned with the smoothness of the relaxation. These concerns were also central in our study of fuzzy logic operators in Chapter \ref{ch:dfl}. Relaxations are biased, often with few guarantees on how close the relaxed value and its gradients are to their true expected values. Furthermore, this bias accumulates as problem size and complexity increase. 

Focusing on fuzzy logics, few papers in the neurosymbolic learning community choose fuzzy logics for its model of fuzziness \citep{deraedtStatisticalRelationalNeurosymbolic2020,badreddineLogicTensorNetworks2022,diligentiSemanticbasedRegularizationLearning2017}. Instead, most authors (sometimes implicitly) use probabilistic propositions like the probability of observing the digit 5, while interpreting complex formulas with fuzzy logic operators. This is misguided, as the probability of a complex formula is not truth-functional, and any probabilistic interpretation is lost the moment fuzzy operators are employed. Furthermore, each of the different fuzzy logics has its peculiarities, as discussed in Chapter \ref{ch:dfl}, and choosing the right one for a particular task is quite a feat. However, fuzzy logics are easier to implement and understand, and naive implementations are much faster than probabilistic logics as weighted model counting is not required. 

Probabilistic approaches, in contrast, never need to relax the semantics of the discrete computation, allowing for straightforward conditional computation and exact or unbiased gradients. However, probabilistic approaches are perceived to be much more expensive than continuous relaxations. We challenge this notion and argue the reverse: While naive computation in probabilistic methods is much more expensive than in fuzzy logics, probabilistic methods offer far more tools for scaling. On the exact inference side, probabilistic circuits \citep{yoojungProbabilisticCircuitsUnifying2020} offer significant improvements over naive enumeration in terms of runtime and memory. On the approximate inference side, we can enjoy decades of research on sampling and variational inference, some of which we used in Chapters \ref{ch:storchastic} and \ref{ch:anesi}. Fuzzy logics, however, do not have a semantics amenable to approximate inference, and is limited to the size of the CNF formula modelling the problem, which can be rather large.  

Summarizing, we argue that probabilistic modelling and inference is a more principled approach to neurosymbolic learning than continuous relaxation. Since deep learning models are (usually) trained with probabilistic loss functions, we consider, like \cite{deraedtNeurosymbolicNeuralLogical2019}, probability theory as the unifying language between deep learning and symbolic computation. Furthermore, we believe scalability concerns can readily be dealt with using the many tools available in the probabilistic modelling toolbox. 

\subsection{The latent worlds problem: A conundrum for neurosymbolic learning?}
\label{sec:latent-worlds-problem-conc}
While we have many options for overcoming the discrete-continuous gap, the latent worlds problem introduced in Section \ref{sec:underconstraining} is poorly understood. The study of the raven paradox showed that fuzzy implications choose differently between modus ponens and modus tollens (Section \ref{sec:implication_challenges}). We also know that probabilistic approaches such as Semantic Loss exhibit the imbalance discussed (see Appendix \ref{appendix:prl-sl} and \cite{heReducedImplicationbiasLogic2022}). Why would a logical theory that filters invalid symbolic configurations have such a strong preference for a \emph{particular} world? For example, consider again the example from the introduction, where the neural network perceives that traffic lights are red and green simultaneously. 
Of course, the neurosymbolic learning algorithm should punish the neural network for this prediction, but should it also \emph{prefer} one of the red or the green lights? Or should it be agnostic to that choice, leaving such questions to the learning algorithm and the data? 

\cite{marconatoNotAllNeuroSymbolic2023} provides the first study of the many optima that neurosymbolic learning methods exhibit and argues that not just the (\say{deterministic}) worlds are optima, but also \emph{weighted mixtures} of worlds. However, fuzzy logics and probabilistic logics like DeepProblog and Semantic Loss assume the variables are independent when not conditioning on logical statements and rules. This independence assumption means these logics are not expressive enough to capture the optimal mixtures, as we prove in preliminary work \citep{vankriekenIndependenceAssumptionProbabilistic2023}. Instead, we prove that these logics converge to (partially) deterministic solutions. Should we not instead \emph{prefer} such mixtures, which express the uncertainty of the method about the correct world? Our preliminary work in \citep{vankriekenIndependenceAssumptionProbabilistic2023} theoretically shows that such mixtures are robust to reasoning shortcuts, but this comes at a significant computational cost.

Regardless, the latent worlds problem in its many forms is a significant fundamental challenge for neurosymbolic learning. However, recent efforts to study this problem formally, in particular \cite{marconatoNotAllNeuroSymbolic2023}, significantly increase our understanding: Whether solutions to this problem come from improving the knowledge or data, from improving the learning algorithm, the design of the neural network or from the logic itself (including the expressivity question presented in this section), there is ample opportunity for future work.


\subsection{Neurosymbolic Learning: Current challenges}
The last years have propelled neurosymbolic learning into an active research area. With MNIST Addition \citep{manhaeveNeuralProbabilisticLogic2021}, a problem surprisingly more interesting than it initially lets on, we now have an \say{MNIST of neurosymbolic AI}\footnote{In the sense that MNIST Addition is the minimal task a neurosymbolic method should be able to solve (almost) perfectly, while remaining far from proof that the method will work in more challenging tasks.}. MNIST Addition, and its many variants, show just how little labelled data we need to learn perception if we also have background knowledge. The fact that we can now solve such problems is a clear testament to neurosymbolic learning over regular deep learning methods: We can learn the grounding of a set of symbols merely from its symbolic relation to other symbols. Furthermore, evidence for the theoretical benefits of neurosymbolic methods has started emerging \citep{wangRegularizationInferenceLabel2023}, such as possible generalisation strengths.

\textbf{Challenge 1.} Nevertheless, most of the studied tasks in the literature are rather artificial: What kind of real-world settings do these perception-then-reasoning tasks represent? Few public real-world tasks and datasets require the interplay of perception and reasoning, a realisation shared in the neurosymbolic learning community. 
Luckily, in the \emph{learning with constraints} setting \citep{giunchigliaDeepLearningLogical2022}, some recent, more realistic datasets like ROAD-R \citep{giunchigliaROADRAutonomousDriving2022}, NeSy4VRD \cite{herronNeSy4VRDMultifacetedResource2023} and our Intelligraphs \citep{thanapalasingamIntelliGraphsDatasetsBenchmarking2023} explicitly introduce background knowledge on the symbols. Other tasks require reasoning on perceptions to compute the output, such as MNIST Addition. Unfortunately, few real-world datasets exist for this setting. One possible reason for this is that, in practice, we often \emph{do} have some labels for individual symbols. With a small quantity of labelled data, the performance of neurosymbolic models can significantly improve \citep{manhaeveNeuralProbabilisticLogic2021}. Just as \say{a little semantics goes a long way}\footnote{\url{https://www.cs.rpi.edu/~hendler/LittleSemanticsWeb.html}}, a \say{few labels go a long way}.

\textbf{Challenge 2.} We next want to discuss a common criticism of neurosymbolic methods: How do we get the background knowledge? Some research focuses on learning the knowledge itself. Most such methods in a neurosymbolic setting assume pre-trained perception models \citep{shindoAlphaILPThinking2023}. However, recent work has managed to learn both the perception model and the background knowledge simultaneously \citep{danieleDeepSymbolicLearning2022,barbieroInterpretableNeuralSymbolicConcept2023}. While impressive, these methods require far more data and higher memory complexity than those assuming background knowledge.   
Instead, we argue that the requirement of providing background knowledge is a strength of neurosymbolic methods: It is precisely by explicitly modelling expert knowledge that we can reduce data requirements, impose safety guarantees and improve interpretability. We have much background knowledge from decades or centuries of research in many fields. We should use this knowledge to our advantage, rather than asserting that \say{everything should be learned}. Instead, we can learn from data when formalised knowledge is hard to come by or represent.

\textbf{Challenge 3.} Accessible neurosymbolic learning tooling is currently missing. 
In the short term, neurosymbolic learning could benefit from a language and tooling for representing data and knowledge, leaving the semantics, inference and learning to the designer of the neurosymbolic method. Such a language improves accessibility to research in the field by immediately providing many neurosymbolic learning tasks and allows for more rapid development of new methods. 
In the long term, the development of expressive frameworks such as DeepProblog \citep{manhaeveNeuralProbabilisticLogic2021} and other probabilistic programming languages \citep{vandemeentIntroductionProbabilisticProgramming2021} can provide a much higher level of flexibility. 
These frameworks provide an enticing future vision for neurosymbolic learning: Specify and use all knowledge, algorithms and data available to us with powerful learning and inference.

\medskip
Altogether, we are excited about the future of neurosymbolic learning as a fusion of learning and reasoning, data and knowledge, and neural networks and logic. In this fusion, logic uses knowledge to reason what is \emph{possible}, and neural networks learn from data what is \emph{probable}.



\newpage
\thispagestyle{empty}
\vspace*{\fill} 
\begin{center}
    \Huge
    \underline{Appendices}
\end{center}
\addcontentsline{toc}{chapter}{Appendices}
\vspace*{\fill}

\RemoveLabels
\AddPubLabels
\appendix
\chapter[Fuzzy Logic Operators]{Analyzing Differentiable Fuzzy Logic Operators}

\section{Derivations of Used Functions}
\subsection{$p$-Error Aggregators}
\label{appendix:yager}
The unbounded Yager aggregator is
\begin{equation}
    A_{UY}(x_1, ..., x_n) = 1 - \left(\sum_{i=1}^n(1 - x_i)^p\right)^{\frac{1}{p}}, \quad p\geq 0.
\end{equation}
We can do an affine transformation $w\cdot A_{UY}(x_1, ..., x_n) - h$ on this function to ensure the boundary conditions in Definition \ref{deff:aggr} hold, namely $w\cdot A_{UY}(0, ..., 0) - h = 0$ and $w\cdot A_{UY}(1, ..., 1) - h = 1$
Solving for $h$, we find
\begin{align}
    w\cdot \left(1 - \left(\sum_{i=1}^n (1 - 0)^p\right)^{\frac{1}{p}}\right) &= h, \quad &w\cdot \left(1 - \left(\sum_{i=1}^n (1 - 1)^p\right)^{\frac{1}{p}}\right) - h &= 1 \notag \\
    w\cdot\left(1 -  n^{\frac{1}{p}}\right)  &= h, \quad &w\cdot \left(1 - 0\right) - h &= 1 \notag \\
    w - w\cdot \sqrt[p]{n} &= h, \quad &h &= w - 1.
     \label{eq:yageragg3}
\end{align}
Equating \ref{eq:yageragg3}, we find $w = \frac{1}{\sqrt[p]{n}}$, and so $h = \frac{1}{\sqrt[p]{n}} - 1$. Filling in and simplifying we find
\begin{align}
    A_{pE}(x_1, ..., x_n) &= \frac{1}{\sqrt[p]{n}}\cdot \left(1 -  \left(\sum_{i=1}^n(1 - x_i)^p\right)^{\frac{1}{p}} \right) - \left( \frac{1}{\sqrt[p]{n}} - 1 \right) \notag \\
    &= 1 -  \left(\frac{1}{n}\sum_{i=1}^n(1 - x_i)^p\right)^{\frac{1}{p}}
\end{align}
The derivation for the t-conorm $A_{UYS}(x_1, ..., x_n) = \left(\sum_{i=1}^n x_i^p\right)^{\frac{1}{p}},  p\geq 0$ is analogous.


\subsection{Sigmoidal Functions}
\label{appendix:sigm}
In Machine Learning, the logistic function or sigmoid function $\sigma(x) = \frac{1}{1 + e^{-x}}$ is a common activation function \pcite{goodfellowDeepLearning2016}(p.65-66). This inspired \pcite{sourekLiftedRelationalNeural2018} to introduce parameterized families of aggregation functions they call Max-Sigmoid activation functions. Their conjunction is $A'_{\sigma +\wedge}(x_1, ..., x_n) = \sigma\left(s\cdot\left(\sum_{i=1}^n x_i - n + 1 + b_0\right)\right)$ and their disjunction is $A'_{\sigma +\vee}(x_1, ..., x_n) = \sigma\left(s\cdot\left(\sum_{i=1}^n x_i + b_0\right)\right)$. 
We generalize this transformation for any function $f: [0, 1]^n\rightarrow \mathbb{R}$ that is symmetric and increasing: 
\begin{equation}
    A'_{\sigma f}(x_1, ..., x_n) = \sigma(s\cdot(f(x_1, ..., x_n) + b_0))
\end{equation}
This cannot be an aggregation function according to Definition \ref{deff:aggr} as $\sigma\in(0, 1)$ and so the boundary conditions $A'_\sigma(1, ..., 1)=1$ and $A'_\sigma(0, ..., 0)$ do not hold. We can solve this by adding two linear parameters $w$ and $h$, redefining $\sigma_f$ as
\begin{equation}
    \sigma_f(x_1, ..., x_n) = w\cdot\sigma(s\cdot(f(x_1, ..., x_n) + b_0)) - h
\end{equation}
For this, we need to make sure the lowest value of $f$ on the domain $[0, 1]^n$ maps to 0 and the highest to 1. For this, we define $inf_f=\inf\{f(x_1, ..., x_n)|x_1, ..., x_n\in[0, 1]\}$ and $sup_f=\sup\{f(x_1, ..., x_n)|x_1, ..., x_n\in[0, 1]\}$. This gives the following system of equations
\begin{align}
   \label{eq:sigmaggis0}
   A_\sigma(0, ..., 0) &= w\cdot\sigma(s\cdot(inf_f + b_0)) - h = 0 \\
   \label{eq:sigmaggis1}
   A_\sigma(1, ..., 1) &= w\cdot\sigma(s\cdot(sup_f + b_0)) - h = 1
\end{align}
First solve both equations for $w$:
\begin{align}
    w\cdot\sigma(s\cdot(inf_f + b_0)) - h &= 0 \quad & w\cdot\sigma(s\cdot(sup_f + b_0)) - h &= 1 \notag\\
    \frac{1}{1+e^{-s\cdot(inf_f + b_0)}} &= \frac{h}{w} \quad&\frac{1}{1+e^{-s\cdot(sup_f + b_0)}} &= \frac{1 + h}{w} \notag \\
    h\cdot (1 + e^{-s\cdot(inf_f + b_0)}) &= w \quad  &(1 + h)\cdot (1 + e^{-s\cdot(sup_f + b_0)}) &= w
    \label{eq:sigmeq1}
\end{align}
Now we can solve for $h$ by equating:
\begin{align*}
    (1 + h)\cdot 1 + e^{-s\cdot(sup_f + b_0)} &=  h\cdot 1 + e^{-s\cdot(inf_f + b_0)}\\
    h &= \frac{1 + e^{-s\cdot(sup_f + b_0)}}{1 + e^{-s\cdot(inf_f + b_0)}-1 + e^{-s\cdot(sup_f + b_0)}}
\end{align*}
We thus get the following formula:
\begin{align}
    A_\sigma(x_1, ..., x_n)
    &= \frac{1 + e^{-s\cdot(sup_f + b_0)}}{e^{-s\cdot(inf_f + b_0)}-e^{-s\cdot(sup_f + b_0)}}\cdot \\
    &\left(\left(1 + e^{-s\cdot(inf_f + b_0)}\right) \cdot \sigma\left(s\cdot\left(f(x_1, ..., x_n) + b_0\right)\right) - 1\right)
\end{align}

If $f$ is a fuzzy logic operator of which the outputs are all in $[0,1]$, the most straightforward choice of $b_0$ is $-\frac{1}{2}$. This translates the outputs of $f$ to $[-\frac{1}{2}, \frac{1}{2}]$ and so it uses a symmetric part of the sigmoid function. For some function $f\in [0, 1]^n\rightarrow [0,1]$, we find the following simplification, noting that the supremum of $f$ is $1$ and the infimum is $0$:
\begin{align}
        \sigma_{f}(x_1, ..., x_n) &= \frac{1 + e^{-s (1-\frac{1}{2})}}{e^{s(0-\frac{1}{2})}-e^{-s(1-\frac{1}{2})}}\cdot \notag \\
  &\left(\left(1 + e^{-s(0-\frac{1}{2})}\right) \cdot \sigma\left(s\cdot\left(f(x_1, ..., x_n) - \frac{1}{2}\right)\right) - 1\right) \\
  &= \frac{1 + e^{-\frac{s}{2}}}{e^{\frac{s}{2}}-e^{-\frac{s}{2}}}\cdot \frac{e^{\frac{s}{2}} - 1}{e^{\frac{s}{2}} - 1}\cdot 
  \left(\left(1 + e^{\frac{s}{2}}\right) \cdot \sigma\left(s\cdot\left(f(x_1, ..., x_n) - \frac{1}{2}\right)\right) - 1\right) \\
  &= \frac{e^{\frac{s}{2}} - e^{-\frac{s}{2}}}{(e^{\frac{s}{2}}-e^{-\frac{s}{2}})(e^{\frac{s}{2}} - 1)}\cdot 
  \left(\left(1 + e^{\frac{s}{2}}\right) \cdot \sigma\left(s\cdot\left(f(x_1, ..., x_n) - \frac{1}{2}\right)\right) - 1\right) \\
  &= \frac{1}{e^{\frac{s}{2}}-1}\cdot 
  \left(\left(1 + e^{\frac{s}{2}}\right) \cdot \sigma\left(s\cdot\left(f(x_1, ..., x_n) - \frac{1}{2}\right)\right) - 1\right) \\
  \label{eq:deriv_sigmoid_tnorms}
\end{align}

Next, we proof several properties of the sigmoidal implication. 

\begin{proposition}
\label{prop:sigm_mono_increase_f}
For all $a_1, c_1, a_2, c_2\in[0,1]$, 
\begin{enumerate}
    \item if $I(a_1, c_1) < I(a_2, c_2)$, then also $\sigma_I(a_1, c_1)< \sigma_I(a_2, c_2)$;
    \item if $I(a_1, c_1) = I(a_2, c_2)$, then also $\sigma_I(a_1, c_1)=\sigma_I(a_2, c_2)$.
\end{enumerate}
\end{proposition}
\begin{proof}
\begin{enumerate}
\item We note that $\sigma_I$ can be written as $\sigma_I(a, c) = w\cdot \sigma\left(s\cdot\left(I(a, c) +b_0\right)\right) - h$ for constants $w=\frac{\left(1 + e^{-s\cdot(1 + b_0)}\right)^2}{e^{-s\cdot b_0}-e^{-s\cdot(1 + b_0)}}$ and $h=\frac{1 + e^{-s\cdot(1 + b_0)}}{e^{-s\cdot b_0}-e^{-s\cdot(1 + b_0)}}$. As $s>0$, $-s\cdot b_0>-s\cdot(1 + b_0)$. Therefore, $e^{-s\cdot b_0}-e^{-s\cdot(1 + b_0)}>0$. Furthermore, as $e^{-\frac{s}{2}} > 0$ then certainly $\left(1 + e^{-\frac{s}{2}}\right)^2 > 0$. As both $e^{-s\cdot b_0}-e^{-s\cdot(1 + b_0)}>0$ and $\left(1 + e^{-\frac{s}{2}}\right)^2>0$, then also $w>0$. As $s>0$, $s\cdot\left(I(a_1, c_1) + b_0\right) < s\cdot\left(I(a_2, c_2) + b_0\right)$ as $I(a_1, c_1) < I(a_2, c_2)$. Next, note that the sigmoid function $\sigma$ is a monotonically increasing function. Using $w>0$ we find that $\sigma_I(a_1, c_1) = w\cdot\sigma(s\cdot\left(I(a_1, c_1) + b_0\right)) < w\cdot\sigma(s\cdot\left(I(a_2, c_2) + b_0\right))=\sigma_I(a_2, c_2)$.
\item $$
    \sigma_I(a_1, c_1) = w\cdot \sigma\left(s\cdot\left(I(a_1, c_1) +b_0\right)\right) - h =  w\cdot \sigma\left(s\cdot\left(I(a_2, c_2) +b_0\right)\right) - h = \sigma_I(a_2, c_2) 
$$
\end{enumerate}
\end{proof}

\begin{proposition}
\label{prop:sigm_supthenone}
$\sigma_I(a, c)$ is 1 if and only if $I(a, c) = 1$. Similarly, $\sigma_I(a, c)$ is 0 if and only if $I(a, c) = 0$. 
\end{proposition}
\begin{proof}
Assume there is some $a, c\in[0, 1]$ so that $I(a, c) = 1$. By construction, $\sigma_I(a, c)$ is 1 (see \ref{appendix:sigm}).

Now assume there is some $a_1, c_1\in[0, 1]$ so that $\sigma_I(a_1,c_1) = 1$. Now consider some $a_2, c_2$ so that $I(a_2, c_2) = 1$. By the construction of $\sigma_I$, $\sigma_I(a_2, c_2) = 1$. For the sake of contradiction assume $I(a_1, c_1) < 1 $. However, by Proposition \ref{prop:sigm_mono_increase_f} as $I(a_1,c_1) < I(a_2,c_2)$ then $\sigma_I(a_1,c_1) < \sigma_I(a_2,c_2)$ has to hold. This is in contradiction with $\sigma_I(a_1, c_1) = \sigma_I(a_2, c_2) = 1$ so the assumption that $I(a_1, c_1)< 1$ has to be wrong and $I(a_1, c_1)=1$.

The proof for $I(a,c)=0$ is analogous.
\end{proof}

\begin{proposition}
For all fuzzy implications $I$, $\sigma_I$ is also a fuzzy implication.
\end{proposition}
\begin{proof}
By Definition \ref{def:implication} $I(\cdot, c)$ is decreasing and $I(a, \cdot)$ is increasing. Therefore, by Proposition \ref{prop:sigm_mono_increase_f}.1, $\sigma_I(\cdot, c)$ is also decreasing and $\sigma_I(a, \cdot)$ is also increasing. Furthermore, $I(0, 0) = 1$, $I(1, 1) = 1$ and $I(1, 0)=0$. We find by Proposition \ref{prop:sigm_supthenone} that then also $\sigma_I(0, 0) = 1$, $\sigma_I(1, 1) = 1$ and $\sigma_I(1, 0) = 0$. 
\end{proof}

$I$-sigmoidal implications only satisfy left-neutrality if $I$ is left-neutral and $s$ approaches 0.

\begin{proposition}
\label{prop:sigm_contrapos}
If a fuzzy implication $I$ is contrapositive symmetric with respect to $N$, then $\sigma_I$ also is.
\end{proposition}
\begin{proof}
Assume we have an implication $I$ that is contrapositive symmetric and so for all $a, c\in[0, 1]$, $I(a, c) = I(N(c), N(a))$. By Proposition \ref{prop:sigm_mono_increase_f}.2, $\sigma_I(a, c) = \sigma_I(N(c), N(a))$. Thus, $\sigma_I$ is also contrapositive symmetric with respect to $N$.
\end{proof}
By this proposition, if $I$ is an S-implication, $\sigma_I$ is contrapositive symmetric and thus also contrapositive differentiable symmetric. 

\begin{proposition}
If $I$ satisfies the identity principle, then $\sigma_I$ also satisfies the identity principle.
\end{proposition}
\begin{proof}
Assume we have a fuzzy implication $I$ that satisfies the identity principle. Then $I(a, a) = 1$ for all $a$. By Proposition \ref{prop:sigm_supthenone} it holds that $\sigma_I(a, a)$ is also 1.
\end{proof}

\subsection{Nilpotent Aggregator}
\label{appendix:nilpotent}
\begin{proposition}
Equation \ref{eq:aggtnorm} is equal for the Nilpotent t-norm to
\begin{equation}
\label{eq:aggrnilp}
    A_{T_{nM}}(x_1, ..., x_n) = \begin{cases}
        \min(x_1, ..., x_n), &\text{if } x_i + x_j > 1;\ x_i \text{, } x_j \text{ are the two lowest values,}\\ 
        0, &\text{otherwise.}
    \end{cases}
\end{equation}
\end{proposition}
\begin{proof}

We will proof this by induction. Base case: Assume $n=2$. Then $A_{T_{nM}}(x_1, x_2) = T_{nM}(x_1, x_2)$. $x_1$ and $x_2$ are the two lowest values of $x_1, x_2$, so the condition in Equation \ref{eq:aggrnilp} would change to $x_1 + x_2 > 1$.

Inductive step: We assume Equation \ref{eq:aggrnilp} holds for some $n \geq 2$. Then by Equation \ref{eq:aggtnorm} \\$A_{T_{nM}}(x_1, ..., x_{n+1}) = T_{nM}(A_{T_{nM}}(x_1, ..., x_n), x_{n+1})$. Note that if $A_{T_{nM}}(x_1, ..., x_n)=0$ then $A_{T_{nM}}(x_1, ..., x_{n+1})$ is also 0 as $x_{n+1}\in[0, 1]$ and so $0 + x_{n+1} > 1$ can never hold. We identify three cases: 
\begin{enumerate}
\item If $x_{n+1}$ is the lowest value in $x_1, ..., x_{n+1}$, then $A_{T_{nM}}(x_1, ..., x_n)$ is either the second-lowest value in $x_1, ..., x_{n+1}$ or 0. If it is 0, the sum of the second and third-lowest values is not greater than 1, and so the sum of the two lowest values can neither be. If it is not, then $T_{nM}(A_{T_{nM}}(x_1, ..., x_n), x_{n+1})$ first compares if $A_{T_{nM}}(x_1, ..., x_n) + x_{n+1} > 1$, that is, if the sum of the two lowest values in $x_1, ..., x_{n+1}$ is higher than 1, and returns $x_{n+1}$ if this holds and 0 otherwise.

\item If $x_{n+1}$ is the second-lowest value in $x_1, ..., x_{n+1}$, then $A_{T_{nM}}(x_1, ..., x_n)$ is either the lowest value in $x_1, ..., x_{n+1}$ or 0. If it is 0, the sum of the first and third-lowest values is not greater than 1, and so the sum of the two lowest values can neither be. If it is not, then $T_{nM}(A_{T_{nM}}(x_1, ..., x_n), x_{n+1})$ first compares if $A_{T_{nM}}(x_1, ..., x_n) + x_{n+1} > 1$, that is, if sum of the two lowest values in $x_1, ..., x_{n+1}$ is higher than 1, and returns $A_{T_{nM}}(x_1, ..., x_n)$ if this holds and 0 otherwise.

\item If $x_{n+1}$ is neither the lowest nor second-lowest value in $x_1, ..., x_{n+1}$, then the sum $s$ of the two lowest values in $x_1, ..., x_n$ is also the sum of the two lowest values in $x_1, ..., x_{n+1}$. If $A_{T_{nM}}(x_1, ..., x_n)$ is 0, then $s$ can not have been greater than 1 and so $A_{T_{nM}}(x_1, ..., x_{n+1})$ is also 0. If it is not, then $s>1$ and  $A_{T_{nM}}(x_1, ..., x_n)$ is the lowest value and surely $A_{T_{nM}}(x_1, ..., x_n) + x_{n+1} > 1$ as $x_{n+1}$ is at least as large as the second-lowest value. 
\end{enumerate}
\end{proof}

By considering $E_{S_{nM}}(x_1, ..., x_n) = 1-A_{T_{nM}}(1-x_1, ..., 1-x_n)$, it is easy to see that 
\begin{equation}
E_{S_{nM}}(x_1, ..., x_n) = \begin{cases}
    \max(x_1, ..., x_n), &\text{if } x_i + x_j < 1;\ x_i \text{, } x_j \text{ are the two largest values,}  \\
    1, &\text{otherwise.}
\end{cases}
\end{equation}
\subsection{Yager R-Implication}
\label{appendix:yagerrimpl}
The Yager t-norm is defined as $T_Y(a, b) = 1 - \left((1 - a)^p + (1 - b)^p\right)^{\frac{1}{p}}$. The Yager R-implication then is defined (see Definition \ref{deff:r-implication}) as 
\begin{equation}
    I_{T_Y}(a, c) = \sup\{b\in [0, 1]|T_Y(a, b)\leq c\}
\end{equation}

When $a\leq c$, $I_{T_Y}=1$ as $T_Y(a, 1) = a\leq c$. Assuming $a>c$, we find by filling in 
\begin{equation}
    I_{T_Y}(a, c) = \sup\{b\in [0, 1]|1 - \left((1 - a)^p + (1 - b)^p\right)^{\frac{1}{p}}\leq c\}, \quad a > c
\end{equation}
\sloppy To get a closed-form solution of $I_{T_Y}$ we have to find the largest $b$ for which $1 - \left((1 - a)^p + (1 - b)^p\right)^{\frac{1}{p}}\leq c$. Solving this inequality for $b$, we find 
\begin{align}
    c &\geq 1 - \left((1 - a)^p + (1 - b)^p\right)^{\frac{1}{p}} \notag\\ 
    (1 - c)^p &\leq (1 - a)^p + (1 - b)^p \notag\\
    1 - b &\geq \left((1 - c)^p - (1-a)^p\right)^{\frac{1}{p}} \notag\\
    b &\leq 1 - \left((1 - c)^p - (1 - a)^p\right)^{\frac{1}{p}}.
\end{align}
If $a>c$, then $(1-c)^p > (1-a)^p$ and thus $(1 - c)^p - (1 - a)^p>0$. Furthermore, as $a, c\in [0, 1]$, $(1 - c)^p - (1 - a)^p\leq 1$. Therefore, it has to be true that $1 - \left((1 - c)^p - (1 - a)^p\right)^{\frac{1}{p}}\in[0, 1]$. The largest value $b\in[0, 1]$ for which the condition holds is then equal to $1 - \left((1 - c)^p - (1 - a)^p\right)^{\frac{1}{p}}$ as it is in $[0, 1]$ and satisfies the inequality. 

Combining this with the earlier observation that when $a\leq c$, $I_{T_Y} = 1$, we find the following R-implication:
\begin{equation}
    I_{T_Y}(a, c) = \begin{cases}
        1, & \text{if } a \leq c \\
        1 - \left((1 - c)^p - (1 - a)^p\right)^{\frac{1}{p}}, & \text{otherwise.}
      \end{cases}
\end{equation}

We plot $I_{T_Y}$ for $p=2$ in Figure \ref{fig:yager-r-s-impl}. 
As expected, $p=1$ reduces to the \luk\ implication. The derivatives of this implication are 
\begin{align}
    \dmpa{I_{T_Y}}{a}{c} &= \begin{cases}
    ((1 - c)^p - (1 - a)^p)^{\frac{1}{p}-1}\cdot (1 - c) , & \text{if } a > c \\
    0, &\text{otherwise,}
    \end{cases}\\
    \dmta{I_{T_Y}}{a}{c}&= \begin{cases}
    ((1 - c)^p - (1 - a)^p)^{\frac{1}{p}-1}\cdot (1 - a) , & \text{if } a > c \\
    0, &\text{otherwise.}
    \end{cases}
\end{align}

\section{\productlogic}
\label{appendix:prl-sl}
We compare \productlogic (\dpfl), which uses the product t- and t-conorm $T_P,\ S_P$, the Reichenbach implication $I_{RC}$ and the log-product aggregator $\logprod$, and Semantic Loss \pcite{xuSemanticLossFunction2018}. 
\begin{definition}
Let $\predicates$ be a set of predicates, $\objects$ the domain of discourse, $\interpretation$ an embedded interpretation of $\fol$ and $\corpus$ a knowledge base of background knowledge. The \textit{Semantic Loss} is defined as
\begin{equation}
\label{eq:semantic_loss}
    \loss_S(\btheta;\corpus) = -\log \sum_{\world\models \corpus} \prod_{\world \models \predP(o_1, ..., o_m)} \interpretation(\predP)(o_1, ..., o_m) \prod_{\world \models \neg \predP(o_1, ..., o_m)} \left( 1- \interpretation(\predP)(o_1, ..., o_m)\right),
\end{equation}
where $\world$ is a \textit{world} (or \textit{Herbrand interpretation}) that assigns a binary truth value to every ground atom and
where $\interpretation(\predP)(o_1, ..., o_m)$ is the probability of a ground atom.
\end{definition}

\dpfl is a single iteration of the \textit{loopy belief propagation} algorithm \pcite{murphyLoopyBeliefPropagation2013} for estimating Semantic Loss. 

\begin{proposition}
\label{prop:dpfl}
Let $\varphi$ be a closed formula so that each ground atom $\predP_1(o_{11}, ..., o_{1m})$ appears at most once in $\varphi$. Then it holds that $\loss_S(\btheta; \varphi)=-\valdfl(\{\}, \varphi)$ when using $T=T_P$, $S=S_P$, $I=I_{RC}$ and $A=\logprod$.
\end{proposition}
We omit the proof from this dissertation, and refer instead to the original paper at \cite{vankriekenAnalyzingDifferentiableFuzzy2022}(Appendix E.1).
As there are no loops in $\varphi$, the factor graph is a tree. Since $\valdfl(\{\}, \varphi)$ corresponds to a single iteration of loopy belief propagation, this is equal to regular belief propagation, which is an exact method for computing queries on probabilistic models \pcite{pearlProbabilisticReasoningIntelligent1988}. Clearly, this condition on $\varphi$ is very strong. Although loopy belief propagation is known to often be a good approximation empirically \pcite{murphyLoopyBeliefPropagation2013}, the degree to which \dpfl approximates Semantic Loss requires further research. However, if \dpfl approximates Semantic Loss well, it can be a strong alternative as it is not an exponential computation. However, it also means that most problems of \dpfl will also be present in Semantic Loss. For example, if we just have the formula $\forall\pred{raven}(x)\rightarrow\pred{black}(x)$, the grounding of the knowledge base will not contain repeated ground atoms, and thus Semantic Loss and \dpfl are equivalent and share difficulties related to the imbalance of modus ponens and modus tollens.

\chapter[Storchastic]{Storchastic: A Framework for General Stochastic Automatic Differentiation}
\newcommand{\sampledice}{\bx_{\stochnode, i_{\stochnode}}}
\newcommand{\gradest}{g\left(i, \left\{ \sample, \costresult_i \right\}_{i=1}^\amt, \bx_{\stochastic_{\before\stochnode}} \right)}
\newcommand{\sampleprob}{\mu_\stochnode(\bx_\stochnode|\bx_{\pa(\stochnode)})}
\newcommand{\sampleprobof}[1]{\mu_\stochnode(\bx_\stochnode=#1|\bx_{\pa(\stochnode)})}
\newcommand{\functionAPP}{\func_\detnode(\bx_{\pa(\detnode)})}

\section{Forward-mode evaluation}
\label{sec:appendix-forward-mode}
In this section, we define several operators that we will use to mathematically define operators used within deep learning
to implement gradient estimators.

To define these, we will need to distinguish how deep learning libraries evaluate their functions.
~\cite{foersterDiCEInfinitelyDifferentiable2018} handles this using a different kind of equality, denoted $\mapsto$.
Unfortunately, it is not formally introduced, making it unclear as to what rules are allowed with this equality.
For instance, they define the DiCE operator as
\[
\begin{array}{l}\text { 1. }\magic(f(\theta)) \mapsto 1 \\ \text { 2. } \nabla_{\theta} \magic(f(\theta))=\magic(f(\theta)) \nabla_\theta f(\theta)\end{array}
\]
However, without a clearly defined meaning of $\mapsto$ `equality under evaluation', it is unclear whether the following is allowed:
\[
    \nabla_{\theta} \magic(f(\theta)) \mapsto \nabla_{\theta} 1 = 0
\]
This would lead to a contradiction, as by definition
\[
    \nabla_{\theta} \magic(f(\theta)) = \magic(f(\theta)) \nabla_\theta f(\theta) \mapsto \nabla_\theta f(\theta).
\]
We first introduce an unambigiuous formulation for forward mode evaluation that does not allow such inconsistencies.

\begin{definition}
    The \emph{stop-grad} operator $\bot$ is a function such that $\nabla_x\bot(x)=0$.
    The \emph{forward-mode} operator $\forward{}$ is a function such that, for well formed formulas $a$ and $b$,
    \begin{enumerate}
        \item $\forward{\bot(a)}=\forward{a}$
        \item $\forward{a+b}=\forward{a}+\forward{b}$
        \item $\forward{a\cdot b}=\forward{a} \cdot \forward{b}$
        \item $\forward{a^b}=\forward{a}^{ \forward{b}}$
        \item $\forward{c}=c$, if $c$ is a constant or a variable.
        \item $\forward{\forward{a}}=\forward{a}$
    \end{enumerate}
    Additionally, we define the DiCE operator $\magic(x)=\exp(x-\bot(x))$
\end{definition}
When computing the results of a function $f(x)$, Deep Learning libraries instead compute $\forward{f(x)}$.
Importantly, $\forward{\nabla_x f(x)}$ does not always equal $\nabla_x\forward{f(x)}$.
For example, $\forward{\nabla_x \bot(f(x))}=\forward{0}=0$, while $\nabla_x \forward{\bot(f(x))}=\nabla_x \forward{f(x)}$.

In the last example, the derivative will first have to be rewritten to find a closed-form formula that does not contain the $\forward{}$ operator.
Furthermore, $\bot(f(x))$ only evaluates to a closed-form formula if it is reduced using derivation, or if it is enclosed in $\forward{}$.

We note that $\mathbb{E}_{p(x)}[\forward{f(x)}]=\forward{\mathbb{E}_{p(x)}[f(x)]}$ for both continuous and discrete distributions $p(x)$ if $\forward{p(x)}=p(x)$. This is easy to see for discrete distributions since these are weighted sums over an amount of elements. For continuous distributions we can use the Riemann integral definition. 
\begin{align}
    \mathbb{E}_{p(x)}[\forward{f(x)}] = \int p(x) \forward{f(x)}
\end{align}
\begin{proposition}
    1: $\forward{\magic(f(x))}=1$ and 2: $\nabla_x \magic(f(x))=\magic(f(x))\cdot \nabla_x f(x)$
\end{proposition}
\begin{proof}
    \begin{enumerate}
    \item 
        \begin{align*}
            \forward{\magic(f(x))} &= \forward{\exp(f(x)-\bot(f(x)))} \\
            &= \exp(\forward{f(x)}-\forward{\bot(f(x))})\\
            &= \exp(f(x)-f(x))=1
        \end{align*}
    \item
        \begin{align*}
        \nabla_x \magic(f(x)) &= \nabla_x\exp(f(x)-\bot(f(x))) \\
            &= \exp(f(x)-\bot(f(x))) \nabla_x (f(x)-\bot(f(x)))\\
            &= \magic(f(x))( \nabla_x f(x) - \nabla_x \bot(f(x)) \nabla_x f(x)) \\
            &= \magic(f(x)) (\nabla_x f(x) - 0\cdot\nabla_x f(x)) = \magic(f(x)) \nabla_x f(x)
        \end{align*}
    \end{enumerate} 
\end{proof}
Furthermore, unlike in the DiCE paper, with this notation $\forward{\nabla_x \magic(f(x))}$ unambiguously evaluates to $\forward{\nabla_x f(x)}$, as
$\forward{\nabla_x \magic(f(x))}=\forward{\magic(f(x))\nabla_x f(x)}=\forward{\magic(f(x))}\cdot \forward{\nabla_x f(x)}=\forward{\nabla_x f(x)}$.
Note that, although this is not a closed-form formula, by finding a closed-form formula for $\nabla_x f(x)$, this can be
reduced to $\nabla_x f(x)$.

\begin{proposition}
    \label{prop:DICE}
    For any two functions $f(x)$ and $l(x)$, it holds that  for all $n\in (0, 1, 2, ...)$,
    \begin{equation}
        \forward {\nabla^{(n)}_{x} \magic( l(x) f(x) ) } = \forward{ \g^{(n)}(x)}.
    \end{equation}
    where $\g^{(n)}(x)=\nabla_x \g^{(n-1)}(x)+\g^{(n-1)}(x)\nabla_x l(x)$ for $n>0$, and $\g^{(0)}(x)=f(x)$.
\end{proposition}

For this proof, we use a similar argument as in \cite{foersterDiCEInfinitelyDifferentiable2018}.

\begin{proof}
    

    First, we show that $\magic(l(x)) \g^{(n)}(x) = \nabla_x^{(n)}\magic(l(x))f(x)$.     We start off with the base case, $n=0$. Then, 
    $\magic(l(x))\g^{(0)}(x) = \magic(l(x))f(x) $.

    Next, assume the proposition holds for $n$, that is, $\magic(l(x)) \g^{(n)}(x) = \nabla_x^{(n)} \magic(\multipl{i})f(x)$. Consider $n+1$. 

    \begin{align}
        \magic(l(x)) \g^{(n+1)}(x) &= \magic(l(x)) ( \nabla_x \g^{(n)}(x) + g^{(n)}(x)\nabla_x l(x)) \\
        &= \nabla_x \magic(l(x))  \g^{(n)}(x) \\
        &= \nabla_x( \nabla_x^{(n)}(\magic(l(x))f(x)) ) \\
        &= \nabla_x^{(n + 1)}\magic(l(x))f(x)
    \end{align}

    Where from line 1 to 2 we use the DiCE proposition in the reversed direction. From 2 to 3 we use the inductive hypothesis.


    
    We use this result, $\magic(l(x)) \g^{(n)}(x) = \nabla_x^{(n)}\magic(l(x))f(x)$, to prove our proposition. Since $\forward{a}=1\cdot \forward{a}=\forward{\magic(\multipl{i})}\forward{a}=\forward{\magic(\multipl{i})a}$,
    \begin{align}
        \forward {\g^{(n)}(x)} = \forward{ \magic(l(x)) \g^{(n)}(x) }= \forward {\nabla_\node^{(n)} \magic(l(x))f(x) } 
    \end{align}
\end{proof}
\begin{definition}
    We say a function $f$ is \emph{identical under evaluation} if for all $n\in (0, 1, 2, ...)$, $\forward{\nabla_x^{(n)} f(x)} = \nabla_x^{(n)} f(x)$. Furthermore, we say that two functions $f$ and $g$ are \emph{equivalent under evaluation}, denoted $f\equivforward g$, if for all $n\in (0, 1, 2, ...)$, $\forward{\nabla_x^{(n)} f(x)} = \forward{\nabla_x^{(n)} g(x)}$.
\end{definition}

Every function that does not contain a stop-grad operator ($\bot$) is identical under evaluation, although  functions that are identical under evaluation can have stop-grad operators (for example, consider $f(x)\equivforward f(x) + \bot(f(x) - f(x))$). Note that $\forward{f(x)}=f(x)$ does not necessarily mean that $f$ is identical under evaluation, since for instance the function $f'(x)=\magic(2x)f(x)$ has $\forward{f'(x)}=f(x)$, but $\forward{\nabla_x f'(x)} = \forward{\magic(2x)(\nabla_x f(x) + 2)} = \nabla_x f(x) + 2 \neq \nabla_x f'(x) = \magic(2x)(\nabla_x f(x) + 2)$.

\begin{proposition}
    \label{prop:equiv-eval}
    If $f(x)$ and $l(x)$ are identical under evaluation, then all $\g^{(n)}(x)$  from $n=0, ..., n$ as defined in Proposition \ref{prop:DICE} are also identical under evaluation. 
\end{proposition}
\begin{proof}
    Consider $n=0$. Then $g^{(0)}(x) = f(x)$. Since $f(x)$ is identical under evaluation, $g^{(0)}$ is as well. 
    
    Assume the proposition holds for $n$, and consider $n+1$. Let $m$ be any positive number. $\forward{\nabla_x^{(m)} g^{(n+1)}(x)}=\forward{\nabla_x^{(m)}(\nabla_x \g^{(n)}(x) + \g^{(n)}(x) \nabla_x l(x))}$. Since $\g^{(n)}(x)$ is identical under evaluation by the inductive hypothesis, $\forward{\nabla_x^{(m)}\nabla_x \g^{(n)}(x)} = \forward{\nabla_x^{(m+1)}\g^{(n)}(x)}=\nabla_x^{(m+1)}\g^{(n)}(x)$. 
    
    Next, using the general Leibniz rule, we find that $\forward{\nabla_x^{(m)}\g^{(n)}(x) \nabla_x l(x)}=\sum_{j=0}^m{m \choose j\nabla_x^{(m-j)}\g^{(n)}(x)}\forward{\nabla_x^{(j+1)}l(x)}$. Since both $\g^{(n)}(x)$ and $l(x)$ are identical under evaluation, this is equal to $\sum_{j=0}^m{m \choose j} \nabla_x^{(m-j)}\g^{(n)}(x)\nabla_x^{(j+1)}l(x)) = \nabla_x^{(m)}\g^{(n)}(x) \nabla_x l(x)$. 
    
    Therefore, $\forward{\nabla_x^{(m)} g^{(n+1)}(x)} = \nabla_x^{(m)}(\nabla_x \g^{(n)}(x) + \g^{(n)}(x) \nabla_x l(x)) = \nabla_x^{(m)}\g^{(n+1)}(x)$, which shows that $\g^{(n+1)}(x)$ is identical under evaluation. 
\end{proof}

We next introduce a very useful proposition that we will use to prove unbiasedness of the \emph{Storchastic} framework. This result was first used without proof in \cite{farquharLoadedDiCETrading2019}.

\begin{proposition}
    \label{prop:DiCE_multiply}
    For any three functions $l_1(x)$, $l_2(x)$ and $f(x)$, $\magic(l_1(x)+l_2(x)) f(x) \equivforward \magic(l_1(x)) \magic(l_2(x)) f(x)$. That is, for all $n\in (0, 1, 2, ...)$.

    \begin{equation}
        \forward{\nabla_x^{(n)} \magic(l_1(x)+l_2(x)) f(x)} = \forward{\nabla_x^{(n)} \magic(l_1(x)) \magic(l_2(x)) f(x)}
    \end{equation}
\end{proposition}
\begin{proof}
    Start with the base case $n=0$. Then, $\forward{\magic(l_1(x)+l_2(x))f(x)} = \forward{f(x)} = 1\cdot 1\cdot \forward{f(x)} = \forward{\magic(l_1(x))\magic(l_2(x))f(x)}$.

    Next, assume the proposition holds for $n$. Then consider $n+1$:
    \begin{align}
        &\forward{\nabla_x^{(n+1)}\magic(l_1(x))\magic(l_2(x))f(x)} \\
        =& \forward{\nabla_x^{(n)}\nabla_x \magic(l_1(x)\magic(l_2(x))f(x))} \\
        =& \forward {\nabla_x^{(n)} \magic(l_1(x))\magic(l_2(x))(f(x)\nabla_x l_1(x) + f(x)\nabla_x l_2(x) + \nabla_x f(x) )} \\
        =& \forward {\nabla_x^{(n)} \magic(l_1(x))\magic(l_2(x))(f(x)\nabla_x (l_1(x)+l_2(x)) +  \nabla_x f(x) )}
    \end{align}

    Define function $h(x) = f(x)\nabla_x (l_1(x)+l_2(x)) +  \nabla_x f(x)$. Since the proposition works for any function, we can apply the inductive hypothesis replacing $f(x)$ by $h(x)$:
    \begin{align}
     \forward {\nabla_x^{(n)} \magic(l_1(x))\magic(l_2(x)) h(x)}  \forward {\nabla_x^{(n)} \magic(l_1(x)+l_2(x)) h(x) } 
    \end{align}
    Finally, we use Proposition \ref{prop:DICE} with $g^{(1)}(x)=h(x)$ and $l(x)=l_1(x)+l_2(x)$:
    \begin{align}
         &\forward {\nabla_x^{(n)} \magic(l_1(x)+l_2(x)) f(x)\nabla_x (l_1(x)+l_2(x)) + \nabla_x f(x) }  \\ 
         =&\forward{\nabla_x^{(n)}\nabla_x \magic(l_1(x)+l_2(x)) f(x) } =\forward{\nabla_x^{(n+1)}\magic(l_1(x)+l_2(x)) f(x) }
    \end{align}
\end{proof}

It should be noted that it cannot be proven that $\nabla_x^{(n)} \magic(\sum_{i=1}^k l_i(x)) f(x) = \nabla_x^{(n)} \magic(\sum_{i=1}^{k-1} l_i(x)) \magic(l_k(x)) f(x)$ because the base-case cannot be proven without the $\forward{}$ operator interpreting the $\magic$ operator.

Also note the parallels with the exponential function, where $e^{l_1(x)+l_2(x)}=e^{l_1(x)} e^{l_2(x)}$.

\section{The \emph{Storchastic} framework (formal)}\label{sec:dice-formulation}

In this section we formally introduce \emph{Storchastic} to provide the mathematical machinery needed to prove our results.
Let $\stochastic_1, \dots, \stochastic_k$ be a partition of $\stochastic_{\before\detnode}$. 
Assume the sets $\stochastic_1, \dots, \stochastic_k$ are topologically sorted, that is, there is no $i<j$ such that there exists a stochastic node $\stochnode\in\stochastic_j$ that is also in $\stochastic_{<i}=\bigcup_{j=1}^{j-1} \stochastic_j$.
We use assignment $\vals{i}$ to denote a set that gives a value to each of the random variables $\stochnode \in \stochastic_i$. That is, $\vals{i}\in \prod_{\stochnode \in \stochastic_i}\sspace_\stochnode$. We additionally use $\vals{< i}$ to denote a set that gives values to all random variables in $\stochastic_{< i}$. 
In the same vein, $\sample{i}$ denotes a set of sets of values $\vals{i}$, that is $\sample{i}=\{\vals{i, 1}, ..., \vals{i, |\sample{i}|}\}$.

\begin{definition}
For each partition $\stochastic_i$ there is a \textbf{gradient estimator} $\gradestim$ where $\fproposalcond{i}$ is a distribution over a set of values $\sample{i}$ conditioned on $\vals{<i}$, $w_i:\prod_{\stochnode \in \stochastic_i}\sspace_\stochnode \rightarrow \mathbb{R}^+ $ is the weighting function that weights different values $\vals{i}$, $l_i:\prod_{\stochnode \in \stochastic_i}\sspace_\stochnode \rightarrow \mathbb{R}$ is the \tmultipl{} that provides the gradient produced by each $\vals{i}$, and the \tadditive{} $a_i:\prod_{j=1}^{i}\prod_{\stochnode \in \stochastic_j}\sspace_\stochnode \rightarrow \mathbb{R}$ is a function of both $\vals{i}$ and $\vals{<i}$.
\end{definition}

$\fproposalcond{i}$ is factorized as follows: Order stochastic nodes $\stochnode_{i, 1}, \dots, \stochnode_{i, m} \in \stochastic_i$ topologically, then
$\fproposalcond{i}=\prod_{j=1}^m q(\sampleset_{i, j}|\sampleset_{i, <j}, \vals{<i})$.

%

In the rest of this appendix, we will define some shorthands to declutter the notation, as follows:
\begin{itemize}
    \item $\weights{i} = \fweight{i}$ and $\weights{i} = \prod_{j=1}^i \weights{i}$
    \item $\multipl{i} = \fmultipl{i}$ and $\Multipl{i} = \sum_{j=1}^i \multipl{i}$
    \item $\additive{i} = \fadditive{i}$
    \item $\proposal{1} = \fproposal{1}$ and $\proposalcond{i} = \fproposalcond{i}$ (for $i > 1$)
\end{itemize}
These interfere with the functions and distributions themselves, but it should be clear from context which of the two is meant.

\begin{proposition}
\label{prop:translate-surrogate-loss}
Given a topologically sorted partition $\stochastic_1, ..., \stochastic_k$ of $\stochastic_{\before\detnode}$ and corresponding gradient estimators $\langle q_i, w_i, l_i, a_i \rangle$ for each $1\leq i\leq k$, the evaluation of the $n$-th order derivative of the \emph{Storchastic} surrogate loss $\forward{\nabla_\node^{(n)} \SL}$ of Equation \ref{eq:surrogate-loss} is equal in expectation to
\begin{align}
    \label{eq:storchastic-expectation}
   \mathbb{E}_{\proposal{1}} \bigg[ \sum_{\itersample{1}} \forward{\nabla_\node^{(n)} \weights{1} \additive{1}} + 
        \dots \mathbb{E}_{\proposalcond{k}} \bigg[ \sum_{\itersample{k}}\forward{ \nabla_\node^{(n)} \weights{k}   
             \magic( \Multipl{k-1} )\additive{k} } +  \forward{\nabla_\node^{(n)} \weights{k}   
             \magic( \Multipl{k} )  \detnode  } \bigg] \dots \bigg]
\end{align}
where the $i$-th term in the dots is $\mathbb{E}_{\proposalcond{i}}[\sum_{\itersample{i}}\forward{\nabla_\node^{(n)} \weights{i} \magic{\Multipl{i-1}}\additive{i}} + (\dots)]$
\end{proposition}
\begin{proof}
    By moving the weights inwards and using the $\Multipl{i}$ notation,
    \begin{align}
        \forward{\nabla_\node^{(n)} \SL}=&\forward{\nabla_\node^{(n)} \sum_{\itersample{1}} \weights{1}\Big[\additive{1} + 
         \dots   
        + \sum_{\itersample{k}} \weights{k}  \Big[ \magic( \Multipl{k-1})\additive{k} + \magic(  \Multipl{k} )\cost  \Big] \dots \Big] \Big] } \\
        &=\forward{\nabla_\node^{(n)} \sum_{\itersample{1}}  \weights{1}\additive{1} +  
         \dots \sum_{\itersample{i}} \weights{i} \magic(\Multipl{i-1}) \additive{i} + \dots}\\
        &\forward{+ \sum_{\itersample{k}} \weights{k} \magic( \Multipl{k-1})\additive{k} + \weights{k}\magic(  \Multipl{k} )\cost    } \\
        &= \sum_{\itersample{1}} \forward{\nabla_\node^{(n)} \weights{1}\additive{1}} +  
         \dots \sum_{\itersample{i}} \forward{\nabla_\node^{(n)}\weights{i} \magic(\Multipl{i-1}) \additive{i}} + \dots\\
        &+ \sum_{\itersample{k}} \forward{\nabla_\node^{(n)}\weights{k} \magic( \Multipl{k-1})\additive{k}} + \forward{\nabla_\node^{(n)}\weights{k}\magic(  \Multipl{k} )\cost}   
    \end{align}
    This is all under sampling $\sampleset_1 \sim \fproposal{1}, \sampleset_2 \sim \fproposalcond{2}, ..., \sampleset_k \sim \fproposalcond{k}$. Taking expectations over these distributions before the respective summation over $\sampleset{i}$ gives the result.
\end{proof}



In the \emph{Storchastic} framework, we require that $\mathbb{E}[F]$ is identical under evaluation, that is,  $\forward{\nabla_\node^{(n)} \mathbb{E}[F]}=\nabla_\node^{(n)} \mathbb{E}[F]$. This in practice means that the probability distributions and functions in the stochastic computation graph contain no stop gradient operators ($\bot$).

Using Proposition \ref{prop:DICE}, we give a recursive expression for $\forward{\nabla^{(n)}_{\node}\weights{i} \big(\additive{i} + \magic( \multipl{i}) f(\vals{i}) \big) }$.

\begin{proposition}
    \label{prop:grad_recur}
    For any gradient estimator $\gradestim$ it holds that  
    \begin{equation}
        \forward{\nabla^{(n)}_{\node}\weights{i} \big(\magic(\Multipl{i-1})\additive{i} + \magic( \multipl{i}) f(\vals{i}) \big) } = \forward{ \nabla^{(n)}_\node \weights{i} \magic(\Multipl{i-1})\additive{i} + \g_i^{(n)}(\vals{i})}
    \end{equation}
     where $\g_i^{(n)}(\vals{i})=\nabla_\node \g_i^{(n-1)}(\vals{i})+\g_i^{(n-1)}\nabla_\node \multipl{i}$ for $n>0$, and $\g_i^{(0)}(\vals{i})=\weights{i}f(\vals{i})$.
\end{proposition}

\begin{proof}
    Using Proposition \ref{prop:DICE}, we find that 
    \begin{align}
        \forward { \nabla_\node^{(n)} \weights{i}\magic(\Multipl{i-1})\additive{i} + \g_i^{(n)}(\vals{i}) } &= \forward {\nabla_\node^{(n)} \weights{i} \magic(\Multipl{i-1}) \additive{i}} + \forward{\nabla_\node^{(n)} \weights{i} \magic(\multipl{i})f(\vals{i}) }  \\
        &= \forward {\nabla_\node^{(n)} \weights{i} (\magic(\Multipl{i-1})\additive{i} + \magic(\multipl{i})f(\vals{i}) )} 
    \end{align}
\end{proof}

Proposition \ref{prop:grad_recur} is useful because it gives a fairly simple recursion to proof unbiasedness of any-order estimators with, when the gradient estimator is implemented in \emph{Storchastic}. Note that it doesn't itself show that such gradient estimators are unbiased in any-order derivatives.


\subsection{Unbiasedness of the \emph{Storchastic} framework}
\label{seq:unbiasedness-proof}
In this section, we use the equivalent expectation from Proposition \ref{prop:translate-surrogate-loss}
\begin{theorem}
    \label{thrm:storchastic}
    Let $\gradestim$ for $i=1, ..., k$ be a sequence of gradient estimators. Let the stochastic computation graph $\mathbb{E}[F]$ be identical under evaluation\footnote{In other words, all deterministic functions, and all probability measures associated with the stochastic nodes are identical under evaluation.}. The evaluation of the $n$th-order derivative of the \emph{Storchastic} surrogate loss is an unbiased estimate of $\nabla_\node^{(n)}\mathbb{E}[F]$ , that is
    \begin{align}
        \nabla_\node^{(n)} \mathbb{E}[\detnode] =
   \mathbb{E}_{\proposal{1}} \bigg[ \sum_{\itersample{1}} \forward{\nabla_\node^{(n)} \weights{1} \additive{1}} + 
        \dots \mathbb{E}_{\proposalcond{k}} \bigg[ \sum_{\itersample{k}}\forward{ \nabla_\node^{(n)} \weights{k}   
             \magic( \Multipl{k-1} )\additive{k} } +  \forward{\nabla_\node^{(n)} \weights{k}   
             \magic( \Multipl{k} )  \detnode  } \bigg] \dots \bigg]
    \end{align}
    if the following conditions hold for all estimators $i=1, ..., k$ and all preceding orders of differentiation $n\geq m \geq 0$:
    \begin{enumerate}
        \item  $\mathbb{E}_{q_i}[\sum_{\itersample{i}} \forward{\nabla^{(m)}_\node \weights{i} \magic(\multipl{i})f(\vals{i})}]= \forward{\nabla_\node^{(m)} \mathbb{E}_{\stochastic_i}[f(\vals{i})] }$ for any deterministic function $f$;
        \item $\mathbb{E}_{q_i}[\sum_{\itersample{i}} \forward{\nabla^{(m)}_\node \weights{i} \additive{i}}]=0$;
        \item  for $n\geq m>0$, $\mathbb{E}_{q_i}[\sum_{\itersample{i}} \forward{\nabla_\node^{(m)} \weights{i}}]=0$;
        \item $\forward{\fproposalcond{i}} = \fproposalcond{i}$.
    \end{enumerate}
\end{theorem}


\begin{proof}
    In this proof, we make extensive use of the general Leibniz rule, which states that 
    \begin{equation}
        \nabla_x^{(n)} f(x) g(x) = \sum_{m=0}^n {n \choose m} \nabla_x^{(n-m)} f(x) \nabla_x^{(m)} g(x).
    \end{equation}
    We consider the terms  $\mathbb{E}_{\proposalcond{i}}\bigg[\sum_{\itersample{k}} \forward { \nabla_\node^{(n)} \weights{i} \magic(\Multipl{i-1})\additive{i} }\bigg]$ and the term $\mathbb{E}_{\proposalcond{k}}\bigg[\sum_{\itersample{k}} \forward { \nabla_\node^{(n)} \weights{k} \magic(\Multipl{k}) \detnode }\bigg]$ separately, starting with the first.

    \begin{lemma}
        \label{lemma:additive}
        For any positive number $1\leq j \leq k$,
        \begin{align}
            \mathbb{E}_{\proposal{1}}\bigg[\sum_{\itersample{1}}
                \dots \mathbb{E}_{\proposalcond{j}}\bigg[\sum_{\itersample{j}} \forward { \nabla_\node^{(n)} \weights{j} \magic(\Multipl{j}) \additive{j}  }\bigg]\dots \bigg] = 0 .
        \end{align}
    \end{lemma}
    \begin{proof}

    We will prove the lemma using induction. First, let $j=1$. Then, using condition 2,
    \begin{align}
        \mathbb{E}_{\proposal{1}}\bigg[\sum_{\itersample{1}} \forward { \nabla_\node^{(n)} \weights{1} \additive{1}  }\bigg] = 0
    \end{align}
    Next, assume the inductive hypothesis holds for $j$, and consider the inner expectation of $j+1$:
    \begin{align}
        =& \mathbb{E}_{\proposalcond{j+1}}\bigg[\sum_{\itersample{j+1}} \forward{ \nabla_\node^{(n)}\magic(\Multipl{j})\additive{j+1} \weights{j+1} }\bigg] = \mathbb{E}_{\proposalcond{j+1}}\bigg[\sum_{\itersample{j+1}} \forward{ \nabla_\node^{(n)}\weights{j+1}\additive{j+1} \magic(\Multipl{j})\weights{j} }\bigg]\\
        =& \mathbb{E}_{\proposalcond{j+1}}\bigg[\sum_{\itersample{j+1}} \forward{ \sum_{m=0}^n \binom{n}{m}\nabla_\node^{(m)}\weights{j+1}  \additive{j+1} \nabla_\node^{(n-m)} \magic(\Multipl{j})\weights{j}  }\bigg]
    \end{align}
    Next, note that $\weights{j}$ and $\Additive{j}$ are both independent of $\vals{j+1}$. 
    Therefore, they can be moved out of the expectation. To do this, we implicitly use condition 4 to move the $\forward{}$ operator through the expectation. 
    \begin{align}
        & \mathbb{E}_{\proposalcond{j+1}}\bigg[\sum_{\itersample{j+1}} \sum_{m=0}^n \binom{n}{m}  \forward { \nabla_\node^{(m)} \weights{j+1}  \additive{j+1} \nabla_\node^{(n-m)} \magic(\Multipl{j})\weights{j} } \bigg] \\
        =& \sum_{m=0}^n \binom{n}{m} \forward{\nabla_\node^{(n-m)} \magic(\Multipl{j})\weights{j}} \mathbb{E}_{\proposalcond{j+1}}\Big[\sum_{\itersample{j+1}}   \forward { \nabla_\node^{(m)} \weights{j+1} \additive{j+1}   }\Big]
    \end{align}
    By condition 2 of the theorem, $\mathbb{E}_{\proposalcond{j+1}}\Big[\sum_{\itersample{j+1}}   \forward { \nabla_\node^{(m)} \weights{j+1} \additive{j+1}   }\Big]=0$. Therefore, we can remove this term and conclude that 
    \begin{align}
        \mathbb{E}_{\proposal{1}}\bigg[\sum_{\itersample{1}}\dots \mathbb{E}_{\proposalcond{j+1}} \bigg[ \sum_{\itersample{j+1}}\forward{\nabla_\node^{(n)} \weights{j+1} \magic(\Multipl{j}) \additive{j+1}} \bigg] \dots \bigg]=0.
    \end{align}
\end{proof} 

    Next, we consider the term $\mathbb{E}_{\proposalcond{k}}\bigg[\sum_{\itersample{k}} \forward { \nabla_\node^{(n)} \weights{k} \magic(\Multipl{k}) \detnode }\bigg]$ and prove using induction that

    \begin{lemma}
        \label{lemma:multiplicative}
        For any $1\leq j\leq k$, it holds that
    \begin{align}
        \mathbb{E}_{\proposal{1}}\bigg[\sum_{\itersample{1}}
            \dots \mathbb{E}_{\proposalcond{j}}\bigg[\sum_{\itersample{j}} \forward { \nabla_\node^{(n)} \weights{j} \magic(\Multipl{j}) \detnode' }\bigg]\dots \bigg] = \nabla_\node^{(n)} \mathbb{E}[\detnode]
    \end{align}
    where $\detnode'=\mathbb{E}_{\stochastic_{j+1}, ..., \stochastic_{k}}[\detnode]$. Furthermore, for $1< j \leq k$, it holds that 
    \begin{align}
        & \mathbb{E}_{\proposal{1}}\bigg[\sum_{\itersample{1}}
            \dots \mathbb{E}_{\proposalcond{j}}\bigg[\sum_{\itersample{j}} \forward { \nabla_\node^{(n)} \weights{j} \magic(\Multipl{j}) \detnode' }\bigg]\dots \bigg] \\
        =& \mathbb{E}_{\proposal{1}}\bigg[\sum_{\itersample{1}}
        \dots \mathbb{E}_{\proposalcond{j-1}}\bigg[\sum_{\itersample{j-1}} \forward { \nabla_\node^{(n)} \weights{j-1} \magic(\Multipl{j-1}) \mathbb{E}_{\stochastic_j}[\detnode'] }\bigg]\dots \bigg]
    \end{align}
    \end{lemma}
    
    \begin{proof}
        The base case $j=1$ directly follows from condition 1: 
        \begin{align}
            \mathbb{E}_{\proposal{1}}\bigg[\sum_{\itersample{1}} \forward { \nabla_\node^{(n)} \weights{1} \magic(\multipl{i})\detnode'  }\bigg] = \forward{\nabla_\node^{(n)} \mathbb{E}_{\stochastic_1}[\detnode']} = \nabla_\node^{(n)} \mathbb{E}[\detnode],
        \end{align}
        since $\mathbb{E}[\detnode] = \mathbb{E}_{\stochastic_1, ..., \stochastic_k}[\detnode]$ and by the assumption that $\mathbb{E}[\detnode]$ is identical under evaluation. 
    
        Assume the lemma holds for $j<k$ and consider $j+1$. First, we use Proposition \ref{prop:DiCE_multiply} and reorder the terms:
        \begin{align}
            & \mathbb{E}_{\proposalcond{j+1}}\bigg[\sum_{\itersample{j+1}} \forward { \nabla_\node^{(n)} \weights{j+1}  \magic(\Multipl{j+1}) \detnode' }\bigg] \\
            =& \mathbb{E}_{\proposalcond{j+1}}\bigg[\sum_{\itersample{j}}  \forward {  \nabla_\node^{(n)} \weights{j+1} \magic(\Multipl{j+1}) \weights{j+1}  \magic( \multipl{j+1})  \detnode' }\bigg] 
        \end{align}
        Next, we again use the general Leibniz rule:
        \begin{align}
            & \mathbb{E}_{\proposalcond{j+1}}\bigg[\sum_{\itersample{j+1}}  \forward {  \nabla_\node^{(n)} \weights{j} \magic(\Multipl{j}) \weights{j+1}  \magic( \multipl{j+1})  \detnode' }\bigg] \\
            =&\mathbb{E}_{\proposalcond{j+1}}\bigg[\sum_{\itersample{j+1}}  \forward {  \sum_{m=0}^n \binom{n}{m} \nabla_\node^{(n-m)} \weights{j}   \magic(\Multipl{j}) \nabla_\node^{(m)} \weights{j+1}  \magic( \multipl{j+1})  \detnode' }\bigg] 
        \end{align}
        where we use for the general Leibniz rule $f=\weights{j}  \magic(\Multipl{j})$ and $g=\weights{j+1}  \magic( \multipl{j+1})  \detnode'$. 
        Note that $\nabla_\node^{(n-m)} \weights{j} \magic(\Multipl{j})$ does not depend on $\vals{j+1}$. Therefore,
        \begin{align}
            &\mathbb{E}_{\proposalcond{j+1}}\bigg[\sum_{\itersample{j+1}}  \forward {  \sum_{m=0}^n \binom{n}{m} \nabla_\node^{(n-m)} \weights{j}   \magic(\Multipl{j}) \nabla_\node^{(m)} \weights{j+1}  \magic( \multipl{j+1})  \detnode' }\bigg] \\
            =& \sum_{m=0}^n \binom{n}{m}  \forward {  \nabla_\node^{(n-m)} \weights{j}   \magic(\Multipl{j})} \mathbb{E}_{\proposalcond{j+1}}\bigg[\sum_{\itersample{j+1}} \forward{ \nabla_\node^{(m)} \weights{j+1}  \magic( \multipl{j+1})  \detnode' }\bigg] \\
            =&   \forward {\sum_{m=0}^n \binom{n}{m}  \nabla_\node^{(n-m)} \weights{j}   \magic(\Multipl{j}) \nabla_\node^{(m)} \mathbb{E}_{\stochastic_{j+1}}[ \detnode' ]} \\
            =& \forward {  \nabla_\node^{(n)} \weights{j}   \magic(\Multipl{j}) \mathbb{E}_{\stochastic_{j+1}}[ \detnode' ]}
        \end{align}

        From lines 2 to 3, we use condition 1 to reduce the expectation. In the last line, we use the general Leibniz rule in the other direction. We showed that 

        \begin{align}
            &\mathbb{E}_{\proposal{1}}\bigg[\sum_{\itersample{1}}
            \dots \mathbb{E}_{\proposalcond{j+1}}\bigg[\sum_{\itersample{j+1}} \forward { \nabla_\node^{(n)} \weights{j+1} \magic(\Multipl{j+1}) \detnode' }\bigg] \dots \bigg] \\
            =& \mathbb{E}_{\proposal{1}}\bigg[\sum_{\itersample{1}}
            \dots \mathbb{E}_{\proposalcond{j}}\bigg[\sum_{\itersample{j}} \forward {  \nabla_\node^{(n)} \weights{j} \magic(\Multipl{j}) \mathbb{E}_{\stochastic_{j+1}}[ \detnode' ]} \bigg] \dots \bigg] 
            = \nabla_\node^{(n)}\mathbb{E}[\detnode]
        \end{align}
        where we use the inductive hypothesis from step 2 to 3, using that $\mathbb{E}_{\stochastic_{j+1}}[\detnode'] = \mathbb{E}_{\stochastic_{j+1}, ..., \stochastic_k}[\detnode]$. 
    \end{proof}

    Using these two lemmas and condition 4, it is easy to show the theorem:
    \begin{align}
        &\mathbb{E}_{\proposal{1}}\bigg[\sum_{\itersample{1}}
            \dots \mathbb{E}_{\proposalcond{k}}\bigg[\sum_{\itersample{k}} \forward { \nabla_\node^{(n)} \weights{k} \Big( \Additive{k} +\magic(\Multipl{k}) \detnode \Big) }\bigg]\dots \bigg] \\
        =&\mathbb{E}_{\proposal{1}}\bigg[\sum_{\itersample{1}}
            \dots \mathbb{E}_{\proposalcond{k}}\bigg[\sum_{\itersample{k}} \forward { \nabla_\node^{(n)} \weights{k} \Additive{k} }\bigg]\dots \bigg] \\
        +&\mathbb{E}_{\proposal{1}}\bigg[\sum_{\itersample{1}}
        \dots \mathbb{E}_{\proposalcond{k}}\bigg[\sum_{\itersample{k}} \forward { \nabla_\node^{(n)} \weights{k} \magic(\Multipl{k}) \detnode  }\bigg]\dots \bigg] \\
        =& 0 + \nabla_\node^{(n)}\mathbb{E}[\detnode]= \nabla_\node^{(n)}\mathbb{E}[\detnode]
    \end{align}
\end{proof}

Note that we used in the proof that condition 1 implies that $\nabla_\node^{(n)}\mathbb{E}_\proposalcond{i}[\sum_{\itersample{i}} \forward{\weights{i}\magic(\multipl{i})}] = 0$, which can be seen by taking $f(\vals{i}) = 1$ and noting that $\nabla_\node^{(m)}\mathbb{E}_{\stochastic_i}[1] = 0$ for $n> 0$. 
\section{Any-order \tadditive{}}
\label{sec:baselines}
Many gradient estimators are combined with \tadditive{}s to reduce variance.
We consider \tadditive{}s for any-order derivative estimation. 
\cite{maoBaselineAnyOrder2019} introduces an any-order baseline in the context of score functions, but only provides proof that this is the baseline for the second-order gradient estimate. 
We use the \emph{Storchastic} framework to prove that it is also the correct baseline for any-order derivatives\footnote{We use a slight variant of the baseline introduced in \cite{maoBaselineAnyOrder2019} to solve an edge case. We will explain in the end of this section how they differ.}. 
Furthermore, we generalize the ideas behind this baseline to all \tadditive{}s, instead of just score-function baselines.

The \tadditive{} that implements any-order baselines is: 
\begin{equation}
    \label{eq:higher-order-baseline}
    \fadditive{i} = (1-\magic(\multipl{i})) \fbaseline{i}.
\end{equation}
First, we show that baselines satisfy condition 2 of Theorem \ref{thrm:storchastic}. 
We will assume here that we take only 1 sample with replacement, but the result generalizes to taking multiple samples in the same way as for the first-order baseline. 
For $n=0$, the any-order baseline evaluates to zero which can be seen by considering $\forward{1-\magic(\multipl{i})}$. 
If $n>0$, then noting that $\baseline{i}$ is independent of $x_i$,
\begin{align}
    \mathbb{E}_{q_i}[ \forward{ \nabla^{(n)}_\node(1-\magic(\multipl{i}))\baseline{i}}]=&   \mathbb{E}_{x_{i}}[\forward{-\baseline{i}\nabla^{(n)}_\node-\magic(\multipl{i})}] 
    = \forward{-\baseline{i}}  \nabla^{(n)}_\node\mathbb{E}_{x_{i}}[1] = 0
\end{align}
We next provide a proof for the validity of this baseline for variance reduction of any-order gradient estimation. To do this, we first prove a new general result on the $\magic$ operator:
\begin{proposition}
    \label{prop:baseline-generator}
    For any sequence of functions $\{\multipl{1}, ..., \multipl{k}\}$, $\magic(\Multipl{k})$ is equivalent under evaluation for orders of differentiation $n>0$ to $\sum_{i=1}^k(\magic(\multipl{i}) - 1)\magic(\Multipl{i-1})$. 
    That is, for all positive numbers $n>0$,
    \begin{equation}
        \forward{\nabla_\node^{(n)} \magic(\Multipl{k})} = \forward{\nabla_\node^{(n)} \sum_{i=1}^k \big(\magic(\multipl{i}) - 1\big) \magic(\Multipl{i-1})}
    \end{equation}
\end{proposition}

\begin{proof}
    We will prove this using induction on $k$, starting with the base case $k=1$. Since $n>0$,
    \begin{equation}
        \forward{\nabla_\node^{(n)}  (\magic(\multipl{1}) - 1) \magic(0)} = \forward{\magic(0)\nabla_\node^{(n)}  \magic(\multipl{1})} = \forward{\nabla_\node^{(n)}  \magic(\multipl{1})}
    \end{equation}
    Next, assume the proposition holds for $k$ and consider $k+1$. Then by splitting up the sum,
    \begin{align}
            \forward{\nabla_\node^{(n)} \sum_{i=1}^{k+1} (\magic(\multipl{i}) - 1) \magic(\Multipl{i-1}) } 
            =&\forward{\nabla_\node^{(n)} (\magic(\multipl{k+1}) - 1) \magic(\Multipl{k}) + \nabla_\node^{(n)} \sum_{i=1}^{k} (\magic(\multipl{i}) - 1) \magic(\Multipl{i-1}) }\\
            =&\forward{\nabla_\node^{(n)} (\magic(\multipl{k+1}) - 1) \magic(\Multipl{k}) + \nabla_\node^{(n)}  \magic(\Multipl{k}) }
    \end{align}
    where in the second step we use the inductive hypothesis.

    We will next consider the first term using the general Leibniz rule:
    \begin{align}
        \forward{\nabla_\node^{(n)} (\magic(\multipl{k+1}) - 1) \magic(\Multipl{k}) }
        = \forward{\sum_{m=0}^n \binom{n}{m} \nabla_\node^{(m)} (\magic(\multipl{k+1}) - 1) \nabla_\node^{(n-m)} \magic(\Multipl{k}) }
    \end{align}
    We note that the term corresponding to $m=0$ can be ignored, as $\forward{\magic(\multipl{k+1}) - 1}=(1-1)=0$. Furthermore, for $m>0$,  $\forward{\nabla_\node^{(m)}(\magic(\multipl{k+1}) - 1)}=\forward{\nabla_\node^{(m)}\magic(\multipl{k+1})}$. Therefore,
    \begin{align}
        \forward{\nabla_\node^{(n)} (\magic(\multipl{k+1}) - 1) \magic(\Multipl{k}) }= \forward{\sum_{m=1}^n \binom{n}{m} \nabla_\node^{(m)} \magic(\multipl{k+1}) \nabla_\node^{(n-m)} \magic(\Multipl{k}) }
    \end{align}
    Finally, we add the other term $\nabla_\node^{(n)}  \magic(\Multipl{k})$ again. Then using the general Leibniz rule in the other direction and Proposition \ref{prop:DiCE_multiply},
    \begin{align}
        =& \forward{\sum_{m=1}^n \binom{n}{m} \nabla_\node^{(m)} \magic(\multipl{k+1}) \nabla_\node^{(n-m)} \magic(\Multipl{k}) + \nabla_\node^{(n)} \magic(\Multipl{k}) } \\
        =& \forward{\sum_{m=1}^n \binom{n}{m} \nabla_\node^{(m)} \magic(\multipl{k+1}) \nabla_\node^{(n-m)} \magic(\Multipl{k}) + \magic(\multipl{k+1})\nabla_\node^{(n)} \magic(\Multipl{k}) } \\
        =& \forward{ \sum_{m=0}^n \binom{n}{m} \nabla_\node^{(m)} \magic(\multipl{k+1}) \nabla_\node^{(n-m)} \magic(\Multipl{k}) }
        = \forward{\nabla_\node^{(n)} \magic(\multipl{k+1})\magic(\Multipl{k})} = \forward{\nabla_\node^{(n)} \magic(\Multipl{k+1})} 
    \end{align}
\end{proof}
Next, we note that we can rewrite the expectation of the \emph{Storchastic} surrogate loss in Equation \eqref{eq:storchastic-expectation} to 
\begin{equation}
   \mathbb{E}_{\proposal{1}} \bigg[ \sum_{\itersample{1}}  + 
        \dots \mathbb{E}_{\proposalcond{k}} \bigg[ \sum_{\itersample{k}}\forward{ \nabla_\node^{(n)} \weights{k}\Big(\Additive{k} +    
             \magic( \Multipl{k} )  \detnode\Big)  } \bigg] \dots \bigg]
\end{equation}
where $\Additive{k}=\sum_{i=1}^{k} \magic\Big(\sum_{j=1}^{i-1} \multipl{j}\Big) \additive{i}$. This can be seen by using Condition 1 and 4 of Theorem 1 to iteratively move the $\magic\Big(\sum_{j=1}^{i-1} \multipl{j}\Big)\additive{i}$ terms into the expectations, which is allowed since they don't depend on $\stochastic_{>i}$.
\begin{theorem}
    Under the conditions of Theorem 1, 
    \begin{equation}
        \Additive{k} + \magic(\Multipl{k})\detnode \equivforward \sum_{i=1}^k \magic(\Multipl{i-1})(\additive{i} + (\magic(\multipl{i}) - 1)\detnode) + \detnode, 
    \end{equation}
    where $\Additive{k}=\sum_{i=1}^{k} \magic\Big(\sum_{j=1}^{i-1} \multipl{j}\Big) \additive{i}$. 
\end{theorem}
\begin{proof}%
    \begin{align}
        \forward{\nabla_\node^{(n)}(\Additive{k} + \magic(\Multipl{k})\detnode)}
        =& \forward{\nabla_\node^{(n)}\Additive{k} + \sum_{m=0}^n\binom{n}{m}\nabla_\node^{(m)} \magic(\Multipl{k}) \nabla_\node^{(n-m)} \detnode}\label{eq:thrm2-2}\\
        =& \forward{\nabla_\node^{(n)}\Additive{k} + \sum_{m=1}^n\binom{n}{m}\nabla_\node^{(m)} \magic(\Multipl{k}) \nabla_\node^{(n-m)} \detnode + \nabla_\node^{(n)} \detnode}\label{eq:thrm2-3}\\
        =& \forward{\nabla_\node^{(n)}\Additive{k} + \sum_{m=1}^n\binom{n}{m}\nabla_\node^{(m)} \sum_{i=1}^k (\magic(\multipl{i}) - 1)\magic(\Multipl{i-1}) \nabla_\node^{(n-m)} \detnode + \nabla_\node^{(n)} \detnode}\label{eq:thrm2-4} \\
        =& \forward{\nabla_\node^{(n)}\big(\sum_{i=1}^k \magic(\Multipl{i-1}) \additive{i} + \sum_{i=1}^k(\magic(\multipl{i}) - 1) \magic(\Multipl{i-1})\detnode +  \detnode\big)}\label{eq:thrm2-5} \\
        =& \forward{\nabla_\node^{(n)}\big(\sum_{i=1}^k \magic(\Multipl{i-1}) (\additive{i} + (\magic(\multipl{i}) - 1) \detnode) +  \detnode\big)} 
    \end{align}
    From \eqref{eq:thrm2-2} to \eqref{eq:thrm2-3}, we use that $m=0$ evaluates to $\nabla_\node^{(n)} \detnode$. From \eqref{eq:thrm2-3} to \eqref{eq:thrm2-4}, we use Proposition \ref{prop:baseline-generator}. From \eqref{eq:thrm2-4} to \eqref{eq:thrm2-5}, we do a reversed general Leibniz rule on the second term. To be able do that, we use that setting $m=0$ in the second term would evaluate to 0 as $\forward{\magic(\multipl{i}) - 1} = 0$.
\end{proof}

Next, consider the inner computation of the \emph{Storchastic} framework in which all $\additive{i}$ use a baseline of the form in Equation \ref{eq:higher-order-baseline}. Note that $\additive{i}=0$ is also in this form by setting $\baseline{i} = 0$. Assume $n>0$ and without loss of generality\footnote{This is assumed simply to make the notation clearer. If the weights are differentiable, the same thing can be shown using an application of the general Leibniz rule.} assume $\nabla_\node^{(m)}\weights{i}=0$ for all $m$ and $i$. Then using Proposition \ref{prop:baseline-generator},
\begin{align}
    &\prod_{i=1}^k \weights{i} \forward{ \nabla_\node^{(n)} \Big( \sum_{i=1}^k(1-\magic(\multipl{i}))\magic(\Multipl{i-1}) b_i +\magic(\Multipl{k}) \detnode \Big) }\\
    =&\prod_{i=1}^k \weights{i} \forward{ \nabla_\node^{(n)} \Big( -\sum_{i=1}^k(\magic(\multipl{i})-1)\magic(\Multipl{i-1}) b_i +\sum_{i=1}^k(\magic(\multipl{i})-1)\magic(\Multipl{i-1})  \detnode \Big) }\\
    =& \prod_{i=1}^k \weights{i} \forward{ \nabla_\node^{(n)} \sum_{i=1}^k(\magic(\multipl{i})-1)\magic(\Multipl{i-1}) (\detnode - \baseline{i}) }
\end{align}

The intuition behind the variance reduction of this any-order gradient estimate is that all terms of the gradient involving $\multipl{i}$, possibly multiplied with other $\multipl{j}$ such that $j<i$, use the $i$-th baseline $\baseline{i}$. This allows modelling baselines for each sampling step to effectively make use of background knowledge or known statistics of the corresponding set of random variables.

We note that our baseline is slightly different from \cite{maoBaselineAnyOrder2019}, which instead of $\magic(\Multipl{i-1})=\magic(\sum_{j=1}^{i-1}\multipl{j})$ used $\magic(\sum_{\stochastic_{j}\before \stochastic_i} \multipl{j})$.
Although this might initially seem more intuitive, we will show with a small counterexample why we should consider any stochastic nodes ordered topologically before $i$ instead of just those that directly influence $i$. 

Consider the stochastic computation graph with stochastic nodes $p(\stochnode_1|\node)$ and $p(\stochnode_2|\node)$ and cost function $f(x_1, x_2)$. For simplicity, assume we use single-sample score function estimators for each stochastic node. Consider the second-order gradient of the cost function using the recursion in Proposition \ref{prop:DICE}:
\begin{align}
    \nabla_\node^2 \mathbb{E}_{\stochnode_1, \stochnode_2}[f(x_1, x_2)] =& \mathbb{E}_{\stochnode_1, \stochnode_2}[\forward{\nabla_\node^2 \magic(\sum_{i=1}^2\log p(x_i|\node)) f(x_1, x_2)}] \\
    =& \mathbb{E}_{\stochnode_1, \stochnode_2}[f(x_1, x_2)\big(\sum_{i=1}^2 \nabla_\node^2\log p(x_i|\node) + (\nabla_\node \log p(x_i|\node))^2  \\
    &+ 2\nabla_\node\log p(x_1|\node)\nabla_\node\log p(x_2|\node) )) \big)] 
\end{align}
Despite the fact that $x_1$ does not directly influence $x_2$, higher-order derivatives will have terms that involve both the log-probabilities of $x_1$ and $x_2$, in this case $2\nabla_\node\log p(x_1|\node)\nabla_\node\log p(x_2|\node)$. Note that since $a$ does not directly influence $b$, the baseline generated for the second-order derivative using the method in \cite{maoBaselineAnyOrder2019} would be  
\begin{align}
    \forward{\nabla_\node^2 \sum_{i=1}^2\additive{i}} = \forward{\sum_{i=1}^2\nabla_\node^2 (1-\magic(\log p(x_i|\node)))\magic(0) \baseline{i} }
    = -\sum_{i=1}^2\nabla_\node^2 \log p(x_i|\node))) \baseline{i} 
\end{align}
This baseline does not have a term for $2\nabla_\node\log p(x_1|\node)\nabla_\node\log p(x_2|\node)$, meaning the variance of that term will not be reduced through a baseline. The baseline introduced in Equation \ref{eq:higher-order-baseline} will include it, since
\begin{align}
    \forward{\nabla_\node^2\sum_{i=1}^2\additive{i}} =& \forward{\nabla_\node^2 (1-\magic(\log p(x_1|\node))) \baseline{i} + (1-\magic(\log p(x_2|\node)))\magic(\log p(x_1|\node)) \baseline{i} }\\
    =& -\sum_{i=1}^2\nabla_\node^2 \log p(x_i|\node))) \baseline{i} - 2\nabla_\node\log p(x_1|\node)\nabla_\node\log p(x_2|\node)
\end{align}

Designing a good baseline function $\fbaseline{i, j}$ that will reduce variance significantly is highly application dependent. Simple options are a moving average and the leave-one-out baseline, which is given by $\fbaseline{i} = \frac{1}{m-1}\sum_{j'=1, j'\neq j}^m \bot( \circ )f(\vals{< i}, x_{i, j'}))$ \citep{koolBuyREINFORCESamples2019,mnihVariationalInferenceMonte2016}. More advanced baselines can take into account the previous stochastic nodes $\stochastic_1, ..., \stochastic_{i-1}$ \citep{weberCreditAssignmentTechniques2019}. Here, one should only consider the stochastic nodes that directly influence $\stochnode_i$, that is, $\stochastic_{\before i}$. Another popular choice is self-critical baselines \citep{rennieSelfCriticalSequenceTraining2017,koolAttentionLearnSolve2019} that use deterministic test-time decoding algorithms to find $\hat{\vals{i}}$ and then evaluate it, giving $\fbaseline{i}=f(\hat{\vals{i}})$.

\section{Examples of Gradient Estimators}
\label{sec:gradient-estimators}
In this section, we prove the validity of several gradient estimators within the \emph{Storchastic} framework, focusing primarily on discrete gradient estimation methods. 


\subsection{Expectation}
\label{sec:expectation}
Assume $p(x_i)$ is a discrete (ie, categorical) distribution with a finite amount of classes $1, ..., C_i$. While this is not an estimate but the true gradient, it fits in the \emph{Storchastic} framework as follows:

\begin{enumerate}
    \item $\weights{i}(x_i)=p(x_i|\vals{<i})$
    \item $\fproposalcond{i} = \delta_{\{1, ..., C_i\}}(\sampleset_i)$ (that is, a dirac delta distribution with full mass on sampling exactly the sequence $\{1, ..., C_i\}$)
    \item $\multipl{i}(x_i)=0$
    \item $\additive{i}(x_i)=0$
\end{enumerate}

Next, we prove the individual conditions to show that this method can be used within \emph{Storchastic}, starting with condition 1:
\begin{align}
    \mathbb{E}_{\proposalcond{i}}[\sum_{x_i\in \sampleset{i}} \forward{\nabla_\node^{(n)} \weights{i}\magic(\multipl{i})f(x_i)}] 
    =& \sum_{j=1}^{C_i} \forward{\nabla_\node^{(n)} p(x_i=j|\vals{<i})\magic(0)f(j)} \\
    =& \forward{\sum_{j=1}^{C_i} \sum_{m=0}^n  \nabla_\node^{(n-m)} p(x_i=j|\vals{<i})\nabla_\node^{(m)}\magic(0)f(j)}
\end{align}
Using the recursion in Proposition \ref{prop:DICE}, we see that $\forward{\nabla_\node^{(m)}\magic(0)f(j)}=\nabla_\node^{(m)} f(j)$, since $\nabla_\node \multipl{i}=\nabla_\node 0= 0$. So,
\begin{align}
    \forward{\sum_{j=1}^{C_i} \sum_{m=0}^n  \nabla_\node^{(n-m)} p(x_i=j|\vals{<i})\nabla_\node^{(m)} f(j)} = \forward{\sum_{j=1}^{C_i} \nabla_\node^{(n)} p(x_i=j|\vals{<i}) f(j)} = \nabla_\node^{(n)}\mathbb{E}_{x_i}[f(x_i)].
\end{align}


Condition 2 follows simply from $\additive{i}(x_i)=0$, and condition 3 follows from the fact that $\sum_{j=1}^{C_i}p(x_i=j|\vals{<i}) = 1$, that is, constant. Condition 4 follows from the SCG being identical under evaluation, ie $\forward{p(x_i=j|\vals{<i})}=p(x_i=j|\vals{<i})$. 

It should be noted that this proof is not completely trivial, as it shows how to implement the expectation so that it can be combined with other gradient estimators while making sure the pathwise derivative through $f$ also gets the correct gradient.  


\subsection{Score Function}
The score function is the best known general gradient estimator and is easy to fit in \emph{Storchastic}. 

\subsubsection{Score Function with Replacement}
\label{sec:sfwr}
We consider the case where we take $m$ samples with replacement from the distribution $p(x_i|\vals{<i})$, and we use a baseline $\fbaseline{i}$ for the first-order gradient estimate. 
\begin{enumerate}
    \item $\weights{i}(x_i)= \frac{1}{m}$
    \item $\proposalcond{i}=\prod_{j=1}^m p(x_{i, j}|\vals{<i})$. That is, $x_{i, 1}, ..., x_{i, m}\sim p(x_i|\vals{<i})$.
    \item $\multipl{i}(x_i)=\log p(x_i|\vals{<i})$
    \item $\fadditive{i} = (1-\magic(\multipl{i})) \fbaseline{i}$, where $\fbaseline{i}$ is not differentiable, that is, $\forward{\nabla_\node^{(n)}\fbaseline{i}}=0$ for $n>0$. 
\end{enumerate}
We start by showing that condition 1 holds. We assume $p(x_i|\vals{<i})$ is a continuous distribution and note that the proof for discrete distributions is analogous. 

We will show how to prove that sampling a set of $m$ samples with replacement can be reduced in expectation to sampling a single sample. Here, we use that $x_{i,1}, ..., x_{i,m}$ are all independently (line 1 to 2) and identically (line 2 to 3) distributed.
\begin{align}
    \mathbb{E}_{\proposalcond{i}}[\sum_{j=1}^m \forward{\nabla_\node^{(n)} \frac{1}{m}\magic(\multipl{i, j}) f(\vals{< i}, x_{i, j}) }] 
    =&\frac{1}{m} \sum_{j=1}^m \mathbb{E}_{x_{i,j}\sim p(x_i)}[ \forward{\nabla_\node^{(n)}  \magic(\multipl{i, j}) f(\vals{\leq i}, x_{i, j}) } ]\\
    =&\frac{1}{m} \sum_{j=1}^m \mathbb{E}_{x_{i}\sim p(x_i)}[ \forward{\nabla_\node^{(n)}  \magic(\multipl{i}) f(\vals{\leq i}) } ] =\mathbb{E}_{x_i}\Big[\forward{\nabla_\node^{(n)}  \magic(\multipl{i}) f(\vals{\leq i}) } \Big]
\end{align}
A proof that $\mathbb{E}_{x_i}\Big[\forward{\nabla_\node^{(n)}  \magic(\multipl{i}) f(\vals{\leq i}) } \Big] =  \nabla_\node^{(n)} \mathbb{E}_{x_i}[f(\vals{\leq i})]$ was first given in \cite{foersterDiCEInfinitelyDifferentiable2018}. For completeness, we give a similar proof here, using induction.

First, assume $n=0$. Then, $\mathbb{E}_{x_i}[\forward{\magic(\multipl{i})}f(\vals{\leq i})] = \mathbb{E}_{x_i}[\forward{f(\vals{\leq i})}] = \mathbb{E}_{x_i}[f(\vals{\leq i}]$.

Next, assume it holds for $n$, and consider $n+1$. Using Proposition \ref{prop:DICE}, we find that $g^{(n+1)}(\vals{\leq i}) = \nabla_\node g^{(n)}(\vals{\leq i}) + g^{(n)}(\vals{\leq i})\nabla_\node \log p(x_i|\vals{<i})$. Writing the expectation out, we find 
\begin{align}
    &\mathbb{E}_{x_i}[\forward{\nabla_\node g^{(n)}(\vals{\leq i}) + g^{(n)}(\vals{\leq i})\nabla_\node \log p(x_i|\vals{<i})}] \\
    =&\int \forward{p(x_i|\vals{<i})(\nabla_\node g^{(n)}(\vals{\leq i}) + g^{(n)}(\vals{\leq i}) \frac{\nabla_\node p(x_i|\vals{<i})}{p(x_i|\vals{<i})})} dx_i\\ 
    =& \int \forward{\nabla_\node p(x_i|\vals{<i})g^{(n)}(\vals{\leq i})} dx_i 
    =\forward{\nabla_\node\mathbb{E}_{x_i}[   g^{(n)}(\vals{\leq i})]} 
\end{align}
By Proposition~\ref{prop:equiv-eval}, $g^{(n)}(x_i)$ is identical under evaluation, since by the assumption of Theorem~\ref{thrm:storchastic} both $p(x_i|\vals{<i})$ and $f(\vals{\leq i})$ are identical under evaluation. As a result, $\forward{\nabla_\node \mathbb{E}_{x_i}[ g^{(n)}(\vals{\leq i}) ]} = \nabla_\node \mathbb{E}_{x_i}[  \forward {g^{(n)}(\vals{\leq i})}]$. Therefore, by the induction hypothesis,
\begin{align}
    \mathbb{E}_{\proposalcond{i}}[\forward{g^{(n+1)}(\vals{\leq i})}]=\nabla_\node \mathbb{E}_{x_i}[  \forward {g^{(n)}(\vals{\leq i})}] = \nabla_\node^{(n+1)} \mathbb{E}_{x_i}[ f(\vals{\leq i})]
\end{align}

Since the weights ($\frac{1}{m}$) are constant, condition 3 is satisfied.

\subsubsection{Importance Sampling}
\label{sec:importance-sampling}
A common use case for weighting samples is importance sampling \citep{rubinsteinSimulationMonteCarlo2016}. In the context of gradient estimation, it is often used in off-policy reinforcement-learning \citep{mahmoodWeightedImportanceSampling2014} to allow unbiased gradient estimates using samples from another policy. For simplicity, we consider importance samples within the context of score function estimators, single-sample estimates, and use no baselines. The last two can be introduced using the techniques in Section \ref{sec:sfwr} and \ref{sec:baselines}.
\begin{enumerate}
    \item $\weights{i} = \bot(\frac{p(x_i|\vals{<i})}{q(x_i|\vals{<i})})$,
    \item $q(x_i|\vals{<i})$ is the sampling distribution,
    \item $\fmultipl{i} = \log p(x_i|\vals{<i})$,
    \item $\fadditive{i} = 0$.
\end{enumerate}
Condition 3 follows from the fact that $\nabla_\node^{(n)}\weights{i}=0$ for $n>0$, since the importance weights are detached from the computation graph. Condition 1:
\begin{align}
    \mathbb{E}_{\proposalcond{i}}[\forward{\nabla_\node^{(n)} \bot\Big(\frac{p(x_i|\vals{<i})}{q(x_i|\vals{<i})}\Big) \magic(\multipl{i}) f(\vals{\leq i}) }] 
    =&\int_{\sspace_i}q(x_i|\vals{<i}) \frac{p(x_i|\vals{<i})}{q(x_i|\vals{<i})} \forward{\nabla_\node^{(n)}  \magic(\multipl{i}) f(\vals{\leq i}) } dx_i\\
    =&\mathbb{E}_{\stochastic_i}[\forward{\nabla_\node^{(n)} \magic(\multipl{i}) f(\vals{\leq i}) }] = \forward{\mathbb{E}_{\stochastic_i}[f(\vals{\leq i})]}
\end{align}
where in the last step we use the proven condition 1 of \ref{sec:sfwr}. Note that this holds both for $n=0$ and $n>0$.

\subsubsection{Discrete Sequence Estimators}
\label{sec:discrete-seqs}
Recent literature introduced several estimators for sequences of discrete random variables. These are quite similar in how they are implemented in \emph{Storchastic}, which is why we group them together.

The sum-and-sample estimator chooses a set of sequences $\hat{\sampleset_i}\subset \sspace_{i}$ and chooses $k - |\hat{\sampleset_i}|>0$ samples from $\sspace_i \setminus \hat{\sampleset_{i}}$. This set can be the most probable sequences \citep{liuRaoBlackwellizedStochasticGradients2019} or can be chosen randomly \citep{koolEstimatingGradientsDiscrete2020}. This is guaranteed not to increase variance through Rao-Blackwellization \citep{casellaRaoblackwellisationSamplingSchemes1996,liuRaoBlackwellizedStochasticGradients2019}. It is often used together with deterministic cost functions $f$, which allows memorizing the cost-function evaluations of the sequences in $\hat{\sampleset_i}$. In this context, the estimator is known as Memory-Augmented Policy Optimization \citep{liangMemoryAugmentedPolicy2018}. 
\begin{enumerate}
    \item $\fweight{i}=I[\vals{i}\in \hat{\sampleset_i}] p(\vals{i}|\vals{<i}) + I[\vals{i}\not\in \hat{\sampleset_i}] \frac{p(\vals{i}\not\in \hat{\sampleset_i})}{k - |\hat{\sampleset_i}|}$
    \item $q(\sampleset_{i}) = \delta_{\hat{\sampleset_i}}(\vals{i, 1}, ..., \vals{i, |\hat{\sampleset_i}|})\cdot \prod_{j=|\hat{\sampleset_i}| + 1}^{k} p(\vals{i, j}|\vals{i, j}\not\in \hat{\sampleset_{i}}, \vals{<i})$
\end{enumerate}
were $p(\vals{i}\not\in \hat{\sampleset_i})=1-\sum_{\vals{i}'\in \hat{\sampleset_i}  }p(\vals{i}'|\vals{<i})$. This essentially always `samples' the set $\hat{\sampleset_i}$ using the Dirac delta distribution, and then samples $k$ more samples out of the remaining sequences, with replacement. The  estimator resulting from this implementation is
\begin{align}
    \mathbb{E}_{\proposalcond{i}}[\sum_{j=1}^{|\hat{\sampleset_i}|} p(\vals{i, j}|\vals{<i}) f(\vals{<i}, \vals{i, j}) + \sum_{j=|\hat{\sampleset_i}|+1}^k \frac{p(\vals{i}\not\in \hat{\sampleset_i})}{k - |\hat{\sampleset_i}|} \magic(\multipl{i})f(\vals{<i}, \vals{i, j})]
\end{align}
Using the result from Section \ref{sec:expectation}, we see that 
\begin{align}
\forward{\nabla_\node^{(n)}  \mathbb{E}_{\proposalcond{i}}[\sum_{j=1}^{|\hat{\sampleset_i}|} p(\vals{i, j}|\vals{<i}) f(\vals{<i}, \vals{i, j})]} 
=&\forward{\nabla_\node^{(n)} p(\vals{i}\in \hat{\sampleset_i}) \mathbb{E}_{\proposalcond{i}}[\sum_{j=1}^{|\hat{\sampleset_i}|} p(\vals{i, j}|\vals{i, j}\in \hat{\sampleset_i}, \vals{<i}) f(\vals{<i}, \vals{i, j})]} \\
=& \nabla_\node^{(n)} p(\vals{i}\in \hat{\sampleset_i})\mathbb{E}_{p(\vals{i}|\vals{i}\in \hat{\sampleset_{i}}, \vals{<i})}[f(\vals{\leq i} )].
\end{align}
Similarly, from the result for sampling with replacement of score functions in Section \ref{sec:sfwr}, 
\begin{align}
    &\forward{\nabla_\node^{n}\mathbb{E}_{\proposalcond{i}}[\sum_{j=|\hat{\sampleset_i}|+1}^k \frac{p(\vals{i}\not\in \hat{\sampleset_i})}{k - |\hat{\sampleset_i}|} \magic(\multipl{i})f(\vals{<i}, \vals{i, j})]} 
    =\nabla_\node^{(n)} p(\vals{i}\not\in \hat{\sampleset_i})\mathbb{E}_{p(\vals{i}|\vals{i}\not\in \hat{\sampleset_{i}}, \vals{<i})}[f(\vals{\leq i} )]
\end{align}
Added together, these form $\nabla_\node^{(n)}\mathbb{E}_{\stochastic_i}[f(\vals{\leq i})]$, which shows that the sum-and-sample estimator with the score function is unbiased for any-order gradient estimation. The variance of this estimator can be further reduced using a baseline from Section \ref{sec:baselines}, such as the leave-one-out baseline.

The \emph{unordered set estimator} is a low-variance gradient estimation method for a sequence of discrete random variables $\stochastic_i$ \citep{koolEstimatingGradientsDiscrete2020}. It makes use of samples without replacement to ensure that each sequence in the sampled batch will be different. We show here how to implement this estimator within \emph{Storchastic}, leaving the proof for validity of the estimator for \cite{koolEstimatingGradientsDiscrete2020}. 

\begin{enumerate}
    \item $\fproposalcond{i}$ is an ordered sample without replacement from $p(\stochastic_i|\vals{<i})$. For sequences, samples can efficiently be taken in parallel using ancestral gumbel-top-k sampling \citep{koolAncestralGumbeltopkSampling2020,koolStochasticBeamsWhere2019}. An ordered sample without replacement means that we take a sequence of samples, where the $i$th sample cannot equal the $i-1$ samples before it.
    \item $\fweight{i} = \bot\Big( \frac{p(\vals{i}|\vals{<i}) p(U=\sampleset_i|o_1=\vals{i}, \vals{<i})}{p(U=\sampleset_i|\vals{<i})} \Big)$, where $p(U=\sampleset_i|\vals{<i})$ is the probability of the \emph{unorderd} sample without replacement, and $p(U=\sampleset_i|o_1=\vals{i}, \vals{<i})$ is the probability of the unordered sample without replacement, given that, if we were to order the sample, the first of those ordered samples is $\vals{i}$.
    \item $\fmultipl{i} = \log p(\vals{i}|\vals{<i})$
    \item $\fadditive{i} =  (1-\magic(\multipl{i})) \baseline{i}(\vals{<i}, \sampleset_i)$, where~$\baseline{i}(\vals{<i}, \sampleset_i)= \\ \sum_{\vals{i}'\in\sampleset_i}\bot\Big( \frac{p(\vals{i}'|\vals{<i})p(U=\sampleset_i|o_1=\vals{i}, o_2=\vals{i}', \vals{<i})}{p(U=\sampleset_i|o_1=\vals{i}, \vals{<i})} f(\vals{i}') \Big)$
\end{enumerate}

This estimator essentially reweights each sample without replacement to ensure it remains unbiased under this sampling strategy. This estimator can be used for any-order differentiation, since $\mathbb{E}_{\proposalcond{i}}[\sum_{\itersample{i}}\forward{\weights{i} f(\vals{i})}] = \mathbb{E}_{\stochastic_i}[\forward{f(\vals{i})}]$ (see \cite{koolEstimatingGradientsDiscrete2020} for the proof) and $\forward{\nabla_\node^{(n)} \weights{i}}=0$ for $n>0$. The baseline is 0 in expectation for the zeroth and first order evaluation \citep{koolEstimatingGradientsDiscrete2020}. We leave for future work whether it is also a mean-zero baseline for $n>1$.

\subsubsection{LAX, RELAX and REBAR}
\label{sec:relax}
REBAR \citep{tuckerREBARLowvarianceUnbiased2017} and LAX and RELAX \citep{grathwohlBackpropagationVoidOptimizing2018} are single-sample score-function based methods that learn a control variate to minimize variance. 
The control variate is implemented using reparameterization. 
We start with LAX as it is simplest, and then extend the argument to RELAX, since REBAR is a special case of RELAX.
We use $\baseline{i, \phi}$ to denote the learnable control variate. 
We have to assume there is no pathwise dependency of $\node$ with respect to $\baseline{i, \phi}$. Furthermore, we assume $\vals{i}$ is a reparameterized sample of $p(\vals{i}|\vals{\leq i})$.
The \tadditive{} component then is: 
\begin{equation}
    \fadditive{i} = \baseline{i, \phi}(\vals{\leq i}) - \magic(\multipl{i})\bot(\baseline{i, \phi}(\vals{\leq i}))
\end{equation}
Since LAX uses normal single-sample score-function, we only have to show condition 2, namely that this \tadditive{} component has 0 expectation for all orders of differentiation. 
\begin{align}
    \mathbb{E}_{\stochastic_i}\bigg[\forward{\nabla_\node^{(n)} \big(\baseline{i, \phi} - \magic(\multipl{i})\bot(\baseline{i, \phi}) \big) }\bigg] = 0
\end{align}
$\forward{\mathbb{E}_{\stochastic_i}[\nabla_\node^{(m)} \baseline{i, \phi}]}$ is the reparameterization estimate of $\forward{\nabla_\node^{(m)} \mathbb{E}_{\stochastic_i}[\baseline{i, \phi}]}$ and $\forward{\mathbb{E}_{\stochastic_i}[\nabla_\node^{(m)} \magic(\log p(\vals{i}|\vals{\leq i}))\bot(\baseline{i, \phi})]}$ is the score-function estimate under the assumption that $\baseline{i, \phi}$ has no pathwise dependency. 
As both are unbiased expectations of the $m$-th order derivative, their difference has to be 0 in expectation, proving condition 2. 
Furthermore, the 0th order evaluation is exactly 0.  
The parameters $\phi$ are trained to minimize the gradient estimate variance.

The \tadditive{} for RELAX \citep{grathwohlBackpropagationVoidOptimizing2018}, an extension of LAX to discrete random variables, is similar. 
It first samples a continuously relaxed input $q(z_i|\vals{<i})$, which is then transformed to a discrete sample $\vals{i}\sim p(\vals{i}|\vals{<i})$. 
See \cite{grathwohlBackpropagationVoidOptimizing2018,tuckerREBARLowvarianceUnbiased2017} for details on how this relaxed sampling works.
It also samples a relaxed input \emph{condition on the discrete sample}, ie $q(\tilde{z_i}|\vals{\leq i})$.
The corresponding \tadditive{} is
\begin{equation}
    \fadditive{i} = \baseline{i, \phi}(z_i) - \bot(\baseline{i, \phi}(z_i) ) - \baseline{i, \phi}(\tilde{z_i}) + (2 - \magic(\multipl{i}))\bot(\baseline{i, \phi}(\tilde{z_i}))
\end{equation}
Here, we subtract $\bot(\baseline{i, \phi}(z_i))$ to ensure the first two terms together sum to 0 during 0th order evaluation, and add $2 \bot(\baseline{i, \phi}(\tilde{z_i}))$ to ensure the last two terms sum to 0.  
Note that for $n>0$, $\forward{\nabla_\node^{(n)} \fadditive{i}}=\forward{\nabla_\node^{(n)} \big(\baseline{i, \phi}(z_i) - \baseline{i, \phi}(\tilde{z_i}) - \magic(\multipl{i}) \bot(\baseline{i, \phi}(\tilde{z_i}))}\big)$. 
We refer the reader to \cite{grathwohlBackpropagationVoidOptimizing2018, tuckerREBARLowvarianceUnbiased2017} for details on why this \tadditive{} is zero in expectation for 1st order differentiation.
We note that the results extend to higher-order differentiation since the $n$-th order derivative of $\magic(\multipl{i})$ gives $n$th-order score functions which are unbiased expectations of the $n$-th order derivative.

\subsubsection{ARM}
ARM is a score-function based estimator for multivariate Bernouilli random variables. 
For our implementation, we use the baseline formulation mentioned in \cite{yinARMAugmentREINFORCEMergeGradient2019}, and we follow the derivation in terms of the Logistic random variables from \cite{dongDisARMAntitheticGradient2020}. 
ARM assumes a real-valued parameter vector $\alpha$, which can be the output of a neural network. 
The probabilities of the Bernoulli random variable are then assumed to be $\sigma(\alpha)$ where $\sigma$ is the sigmoid function. 

\begin{enumerate}
    \item $\fproposalcond{i}$ is a reparameterized sample from the multivariate Bernouilli distribution. 
    First, it samples $\beps\sim \operatorname{Logistic}(\bzero, \boldsymbol{1})$. 
    Define $\bz_i=\alpha + \beps$ and $\tilde{\bz_i}=\alpha - \beps$.
    We find $\vals{i} = I[\bz_i > \bzero]$. Then, with this procedure, $\vals{i}\sim \operatorname{Bernouilli}(\sigma(\alpha))$. 
    \item $\fweight{i} = 1$
    \item $\fmultipl{i} = \log q_\alpha(\bz_i)$, where $q_\alpha$ is the density function of $\operatorname{Logistic}(\alpha, 1)$. 
    \item $\fadditive{i} = \magic(1-\fmultipl{i}) \frac{1}{2} (f(\vals{<i}, \bz_i > \bzero) + f(\vals{<i}, \tilde{\bz_i} > \bzero))$
\end{enumerate}

Since $\mathbb{E}_{\vals{i}\sim\operatorname{Bernoulli}(\sigma(\alpha_\btheta))}[f(\vals{i})] = \mathbb{E}_{\beps\sim \operatorname{Logistic}(0, 1)}[f(\alpha_\btheta + \beps > \bzero)] = \mathbb{E}_{\bz_i\sim \operatorname{Logistic}(\alpha_\btheta, 1)}[f(\bz_i > \bzero)]$, any unbiased estimate of the logistic reparameterization must also be an unbiased estimate of the original Bernouilli formulation. 
This equality follows because the CDF of the logistic distribution is the logistic function (that is, the sigmoid function). 
$\fmultipl{i}$ is the (unbiased) score function of the logistic reparameterization, which we proved to be an unbiased estimate.

The \tadditive{} has expectation 0 for zeroth and first order differentiation. 
This is because it relies on the score function being an odd function \citep{buesingStochasticGradientEstimation2016}, that is, $\nabla_\node \log q_{\alpha}(\bz_i) = - \nabla_\node \log q_{\alpha}(\tilde{\bz_i})$.
Therefore, $\mathbb{E}_{\beps}[(f(\vals{<i}, \bz_i > 0) + f(\vals{<i}, \tilde{\bz_i} > 0)) \nabla_\node \log q_\alpha(\bz_i)] = \mathbb{E}_{\beps}[f(\vals{<i}, \bz_i > 0) \nabla_\node \log q_\alpha(\bz_i) - f(\vals{<i}, \tilde{\bz_i} > 0) \nabla_\node \log q_\alpha(\tilde{\bz_i})]$. 
Note that, by symmetry of the logistic distribution, $\mathbb{E}_{\beps}[f(\vals{<i}, \bz_i > 0) \nabla_\node \log q_\alpha(\bz_i)] = - \mathbb{E}_{\beps}[f(\vals{<i}, \tilde{\bz_i} > 0) \nabla_\node \log q_\alpha(\tilde{\bz_i})]$, meaning the baseline is zero in expectation.
However, this derivation only holds for odd functions! 
Unfortunately, the second-order score function $\frac{\nabla_\node^{(2)} q_\alpha(\bz_i)}{q_\alpha(\bz_i)}$ is an even function since the derivative of an odd function is always an even function. 
Therefore, the ARM estimator will only be unbiased for first-order gradient estimation.



\subsubsection{GO Gradient}
\label{sec:gogradient}
The GO gradient estimator \citep{congGOGradientExpectationbased2019} is a method that uses the CDF of the distribution to derive the gradient. 
For continuous distributions, it reduces to implicit reparameterization gradients which can be implemented through transforming the computation graph, like other reparameterization methods.
For $m$ independent discrete distributions of $d$ categories, the first-order gradient is given as:
\begin{equation}
    \mathbb{E}_{p(\vals{i}|\vals{\leq i})}\Big[\sum_{j=1}^m (f(\vals{\leq i}) - f(\vals{\leq i\setminus \vals{i, j}}, \vals{i, j} + 1)) \frac{\nabla_\node\sum_{k=1}^{\vals{i, j}}  p_j(k|\vals{<i})}{p_j(\vals{i, j}|\vals{<i})} \Big]
\end{equation}
Note that if $\vals{i, j}=d$, then the estimator evaluates to zero since $\nabla_\node \sum_{k=1}^{d} p_j(k|x_{<i})=0$.

We derive the \emph{Storchastic} implementation by treating the GO estimator as a \tadditive{} of the single-sample score function.
To find this \tadditive{}, we subtract the score function from this estimator,
that is, we subtract $f(\vals{\leq i})\nabla_\node\log p(\vals{i}|\vals{< i}) = f(\vals{\leq i}) \sum_{j=1}^m \nabla_\node \log p(\vals{i, j}|\vals{<i})= f(\vals{\leq i}) \sum_{j=1}^m \frac{\nabla_\node p(\vals{i, j}|\vals{<i})}{p(\vals{i, j}|\vals{<i})}$ where we use that each discrete distribution is independent.
By unbiasedness of the GO gradient, the rest of the estimator is 0 in expectation, as we will show. 

Define $f_{j, k}=f(\vals{\leq i \setminus \vals{i, j}}, \vals{i, j}=k)$, $p_{j, k} = p_j(k|\vals{<i})$ and $P_{j, k} = \sum_{k'=1}^k p_j(k'|\vals{<i})$.
Then the GO \tadditive{} is:
\begin{align}
    \additive{i}(\vals{\leq i}) = \sum_{j=1}^m I[\vals{i, j}<d]\Big( \bot\big(\frac{f_{j, \vals{i, j}}  - f_{j, \vals{i, j} + 1}}{p_{j, \vals{i, j}}}\big) (\magic(P_{j, \vals{i, j}}) - 1) \Big) 
    - \bot(f_{j, d})(\magic(\log p_{j, d}) - 1)
\end{align}
The first line will evaluate to the GO gradient estimator when differentiated, and the second to the single-sample score function gradient estimator. 

Note that this gives a general formula for implementing any unbiased estimator into \emph{Storchastic}: Use it as a control variate with the  score function subtracted to ensure interoperability with other estimators in the stochastic computation graph. 

\subsection{SPSA}
Simultaneous perturbation stochastic approximation (SPSA) \citep{spallMultivariateStochasticApproximation1992} is a gradient estimation method based on finite difference estimation. It stochastically perturbs parameters and uses two functional evaluations to estimate the (possibly stochastic) gradient. 
Let $\btheta$ be the $d$-dimensional parameters of the distribution $p_\btheta(\vals{i}|\vals{<i})$. 
SPSA samples $d$ times from the Rademacher distribution (a Bernoulli distribution with 0.5 probability for 1 and 0.5 probability for -1) to get a noise vector $\beps$. 
We then get two new distributions: $\vals{i, 1} \sim p_{\btheta + c\beps}$ and $\vals{i, 2} \sim p_{\btheta - c\beps}$ where $c>0$ is the perturbation size. 
The difference $\frac{f(\vals{i, 1}) - f(\vals{i, 2})}{2c \beps}$ is then an estimate of the first-order gradient. 
Higher-order derivative estimation is also possible, but left for future work.

An easy way to implement SPSA in \emph{Storchastic} is by using importance sampling (Appendix \ref{sec:importance-sampling}). 
Assuming $p_{\btheta + c\beps}$ and $p_{\btheta - c\beps}$ have the same support as $p$,  we can set the weighting function to $\bot\left(\frac{p(\vals{i, 1}|\vals{<i})}{p_{\btheta + c\beps}(\vals{i, 1}|\vals{<i})}\right)$ for the first sample, and $\bot\left(\frac{p(\vals{i, 2}|\vals{<i})}{p_{\btheta - c\beps}(\vals{i, 2}|\vals{<i})}\right)$ for the second sample.

To ensure the gradients distribute over the parameters, we define the \tmultipl{} as $\btheta\bot\left(\frac{p_{\btheta + c\beps}(\vals{i, 1}|\vals{<i})}{2c\beps p(\vals{i, 1}|\vals{<i})}\right)$ for the first sample and $-\btheta\bot\left(\frac{p_{\btheta + c\beps}(\vals{i, 2}|\vals{<i})}{2c\beps p(\vals{i, 2}|\vals{<i})}\right)$ for the second sample. 
This cancels out the weighting function, resulting in the SPSA estimator.

\subsection{Measure Valued Derivatives}
\emph{Storchastic} allows for implementing Measure Valued Derivatives (MVD) \citep{heidergottMeasurevaluedDifferentiationMarkov2008}, however, it is only unbiased for first-order differentiation and cannot easily be extended to higher-order differentiation. 
The implementation is similar to SPSA, but with some nuances.
We will give a simple overview for how to implement this method in \emph{Storchastic}, and leave multivariate distributions and higher-order differentiation to future work.

First, define the weak derivative for parameter $\theta$ of $p$ as the triple $(c_\theta, p^+, p^-)$ by decomposing $p(\vals{i}|\vals{<i})$ into the positive and negative parts $p^+(\vals{i}^+)$ and $p^- (\vals{i}^-)$, and let $c_\theta$ be a constant. 
For examples on how to perform this decomposition, see for example \cite{mohamedMonteCarloGradient2020}. 
To implement MVDs in \emph{Storchastic}, we use the samples from $p^+$ and $p^-$, and, similar to SPSA, treat them as importance samples (Appendix \ref{sec:importance-sampling}) for the zeroth order evaluation. 

That is, the \tproposal{} is defined over tuples $\sampleset_{i}=(\vals{i}^+, \vals{i}^-)$ such that $\fproposalcond{i} = p^+(\vals{i}^+)p^-(\vals{i}^-)$. The weighting function can be derived depending on the support of the positive and negative parts of the weak derivative. For weak derivatives for which the positive and negative part both cover an equal proportion of the distribution $p(\vals{i}|\vals{<i})$, the weighting function can be found using importance sampling by $\bot\left(\frac{p(\vals{i}^+|\vals{<i})}{2p^+(\vals{i}^+)}\right)$ for samples from the positive part, and $\bot\left(\frac{p(\vals{i}^-|\vals{<i})}{2p^-(\vals{i}^-)}\right)$ for samples from the negative part. This gives unbiased zeroth order estimation by using importance sampling. 

We then set $\fadditive{i}=0$ and use the following \tmultipl{}: $\fmultipl{i} = \theta \cdot \bot(c_\theta \frac{2p^+(\vals{i}^+)}{p(\vals{i}^+|\vals{<i})})$ for positive samples and $\fmultipl{i} = -\theta \cdot \bot(c_\theta \frac{2p^-(\vals{i}^-)}{p(\vals{i}^-|\vals{<i})})$ for negative samples. This will compensate for the weighting function by ensuring the importance weights are not applied over the gradient estimates. For the first-order gradient, this results in the MVD $\nabla_\node \theta \bot(c_\theta) (f(\vals{i}^+) - f(\vals{i}^-))$.

For other distributions for which $p^+$ and $p^-$ do not cover an equal proportion of $p$, more specific implementations have to be derived. 
For example, for the Poisson distribution one can implement its MVD by noting that $p^+$ has the same support as $p$. 
Then, we can use one sample from $p^+$ using the importance sampling estimator using score function (Appendix \ref{sec:importance-sampling}), and use a trick similar to the GO gradient by defining a \tadditive{} that subtracts the score function and adds the MVD, which is allowed since the MVD and score function are both unbiased estimators.

\chapter[A-NeSI]{A-NeSI: A Scalable Approximate Method for Probabilistic Neurosymbolic Inference}


\section{Constrained structured output prediction}
\label{appendix:background-knowledge}
We consider how to implement the constrained structured output prediction task considered in (for example) \citep{ahmedNeuroSymbolicEntropyRegularization2022, ahmedSemanticProbabilisticLayers2022,xuSemanticLossFunction2018} in the \textsc{A-NeSI} framework. Here, the goal is to learn a mapping of some $X$ to a structured \emph{output} space $W$, where we have some constraint $\symfun(\bw)$ that returns 1 if the background knowledge holds, and 0 otherwise. We can model the constraints using $Y=\{0, 1\}$; that is, the `output' in our problem setup is whether $\bw$ satisfies the background knowledge $\symfun$ or not. We give an example of this setting in Figure \ref{fig:flow_aorb}.

Then, we design the inference model as follows. 1) $q_\paramP(y|\bP)$ is tasked with predicting the probability that randomly sampled outputs $\bw\sim p(\bw|\bP)$ will satisfy the background knowledge. 2) $q_\paramE(\bw|y=1, \bP)$ is an approximate posterior over structured outputs $\bw$ that satisfy the background knowledge $c$. 

This setting changes the interpretation of the set $W$ from \emph{unobserved} worlds to \emph{observed} outputs. We will train our perception module using a ``strongly'' supervised learning loss where $\bx, \bw \sim \mathcal{D_L}$: 
\begin{equation}
	\mathcal{L}_{Perc}(\btheta)=-\log q_\paramE(\bw|y=1, \bP=f_\btheta(\bx)).
\end{equation}
If we also have unlabeled data $\mathcal{D}_U$, we can use the prediction model to ensure the perception model gives high probabilities for worlds that satisfy the background knowledge. This approximates Semantic Loss \citep{xuSemanticLossFunction2018}: Given $\bx \sim \mathcal{D}_U$,
\begin{equation}
	\mathcal{L}_{SL}(\btheta)=-\log q_\paramP(y=1| \bP=f_\btheta(\bx)).
\end{equation}
That is, we have some input $\bx$ for which we have no labelled output. Then, we increase the probability that the belief $\bP$ the perception module $f_\theta$ predicts for $\bx$ would sample structured outputs $\bw$ that satisfy the background knowledge.

Training the inference model in this setting can be challenging if the problem is very constrained. Then, random samples $\bP, \bw\sim p(\bP, \bw)$ will usually not satisfy the background knowledge. Since we are only in the case that $y=1$, we can choose to sample from the inference model $q_\bphi$ and exploit the symbolic pruner to obtain samples that are guaranteed to satisfy the background knowledge. Therefore, we modify equation \ref{eq:joint-match} to the \emph{on-policy joint matching loss}
\begin{align}
	\begin{split}
		\label{eq:joint-match-on-policy}
		\mathcal{L}_{Expl}(\bP, \bphi)=\mathbb{E}_{q_\bphi(\bw|y=1,\bP)}\left[\left(\log \frac{q_\bphi(\bw,y=1|\bP)}{p(\bw|\bP)}\right)^2\right]
	\end{split}
\end{align}
Here, we incur some sampling bias by not sampling structured outputs from the true posterior, but this bias will reduce as $q_\bphi$ becomes more accurate. We can also choose to combine the on- and off-policy losses. Another option to make learning easier is using the suggestions of Section \ref{sec:output-factorization}: factorize $y$ to make it more fine-grained. 


\section{\textsc{A-NeSI} as a Gradient Estimation method}
\label{appendix:gradient-estimation}
In this appendix, we discuss using the method \textsc{A-NeSI} introduced for general gradient estimation \citep{mohamedMonteCarloGradient2020}. We first define the gradient estimation problem. Consider some neural network $f_\btheta$ that predicts the parameters $\bP$ of a distribution over unobserved variable $\bz\in Z$: $p(\bz|\bP=f_\btheta(\bx))$. This distribution corresponds to the distribution over worlds $p(\bw|\bP)$ in \textsc{A-NeSI}. Additionally, assume we have some deterministic function $g(\bz)$ that we want to maximize in expectation. This maximization requires estimating the gradient
\begin{equation}
	\nabla_\btheta \mathbb{E}_{p(\bz|\bP=f_\btheta(\bx))}[g(\bz)].
\end{equation}
Common methods for estimating this gradient are reparameterization \citep{kingmaAutoencodingVariationalBayes2014}, which only applies to continuous random variables and differentiable $r$, and the score function \citep{mohamedMonteCarloGradient2020,schulmanGradientEstimationUsing2015} which has notoriously high variance. 

Instead, our gradient estimation method learns an \emph{inference model} $q_\bphi(r|\bP)$ to approximate the distribution of outcomes $r = g(\bz)$ for a given $\bP$. In \textsc{A-NeSI}, this is the prediction network $q_\paramP(\by|\bP)$ that estimates the WMC problem of Equation \ref{eq:wmc}. Approximating a \emph{distribution} over outcomes is similar to the idea of distributional reinforcement learning \citep{bellemareDistributionalPerspectiveReinforcement2017}. Our approach is general: Unlike reparameterization, we can use inference models in settings with discrete random variables $\bz$ and non-differentiable downstream functions $g$. 

We derive the training loss for our inference model similar to that in Section \ref{sec:pred-only}. First, we define the joint on latents $\bz$ and outcomes $r$ like the joint process in \ref{eq:forward-gen} as $p(r, \bz|\bP)=p(\bz|\bP) \cdot \delta_{g(z)}(r)$, where $\delta_{g(z)}(r)$ is the dirac-delta distribution that checks if the output of $g$ on $\bz$ is equal to $r$. Then we introduce a prior over distribution parameters $p(\bP)$, much like the prior over beliefs in \textsc{A-NeSI}. An obvious choice is to use a prior conjugate to $p(\bz|\bP)$. We minimize the expected KL-divergence between $p(r|\bP)$ and $q_\bphi(r|\bP)$:
\begin{align}
	&\mathbb{E}_{p(\bP)}[D_{KL}(p||q_\bphi)] \\
	= &\mathbb{E}_{p(\bP)}\left[\mathbb{E}_{p(r|\bP)}[\log q_\bphi(r|\bP)]\right] + C \\
	= &\mathbb{E}_{p(\bP)}\left[\int_{\mathbb{R}} p(r|\bP) \log q_\bphi(r|\bP)] dr \right] + C
\end{align}
Next, we marginalize over $\bz$, dropping the constant:
\begin{align}
	&\mathbb{E}_{p(\bP)}\left[\int_Z\int_{\mathbb{R}} p(r, \bz|\bP)\log q_\bphi(r|\bP)] dr d\bz \right]  \\
	= &\mathbb{E}_{p(\bP)}\left[\int_Z p(\bz|\bP) \int_{\mathbb{R}} \delta_{g(\bz)}(r) \log q_\bphi(r|\bP)] dr d\bz \right]  \\
	= &\mathbb{E}_{p(\bP)}\left[\int_Z p(\bz|\bP) \log q_\bphi(g(\bz)|\bP)] d\bz \right]  \\
	=&\mathbb{E}_{p(\bP, \bz)}\left[\log q_\bphi(g(\bz)|\bP)\right] 
\end{align}
This gives us a negative-log-likelihood loss function similar to Equation \ref{eq:pred-only-loss}.
\begin{equation}
	\mathcal{L}_{Inf}(\bphi)= -\log q_\bphi(g(\bz)|\bP), \quad \bP, \bz\sim p(\bz,\bP)
\end{equation}
where we sample from the joint $p(\bz, \bP)=p(\bP) p(\bz|\bP)$.

We use a trained inference model to get gradient estimates: 
\begin{equation}
	\label{eq:A-NeSI-grad-estim}
	\nabla_\bP \mathbb{E}_{p(\bz|\bP)}[g(\bz)]\approx \nabla_\bP \mathbb{E}_{q_\bphi(r|\bP)}[r]
\end{equation}
We use the chain rule to update the parameters $\btheta$. This requires a choice of distribution $q_\bphi(r|\bP)$ for which computing the mean $\mathbb{E}_{q_\bphi(r|\bP)}[r]$ is easy. The simplest option is to parameterize $q_\bphi$ with a univariate normal distribution. We predict the mean and variance using a neural network with parameters $\bphi$. For example, a neural network $m_\bphi$ would compute $\mu=m_\bphi(\bP)$. Then, the mean parameter is the expectation on the right-hand side of Equation \ref{eq:A-NeSI-grad-estim}. The loss function for $f_\btheta$ with this parameterization is:
\begin{equation}
	\mathcal{L}_{NN}(\btheta)= - m_\bphi(f_\btheta(\bx)), \quad \bx\sim \mathcal{D}
\end{equation}
Interestingly, like \textsc{A-NeSI}, this gives zero-variance gradient estimates! Of course, bias comes from the error in the approximation of $q_\bphi$. 

Like \textsc{A-NeSI}, we expect the success of this method to rely on the ease of finding a suitable prior over $\bP$ to allow proper training of the inference model. See the discussion in Section \ref{sec:prior}. We also expect that, like in \textsc{A-NeSI}, it will be easier to train the inference model if the output $r=g(\bz)$ is structured instead of a single scalar. We refer to Section \ref{sec:output-factorization} for this idea. Other challenges might be stochastic and noisy output measurements of $r$ and non-stationarity of $g$, for instance, when training a VAE \citep{kingmaAutoencodingVariationalBayes2014}.

\section{\textsc{A-NeSI} and GFlowNets}
\label{appendix:gflownets}
\begin{figure*}
	\includegraphics[width=\textwidth]{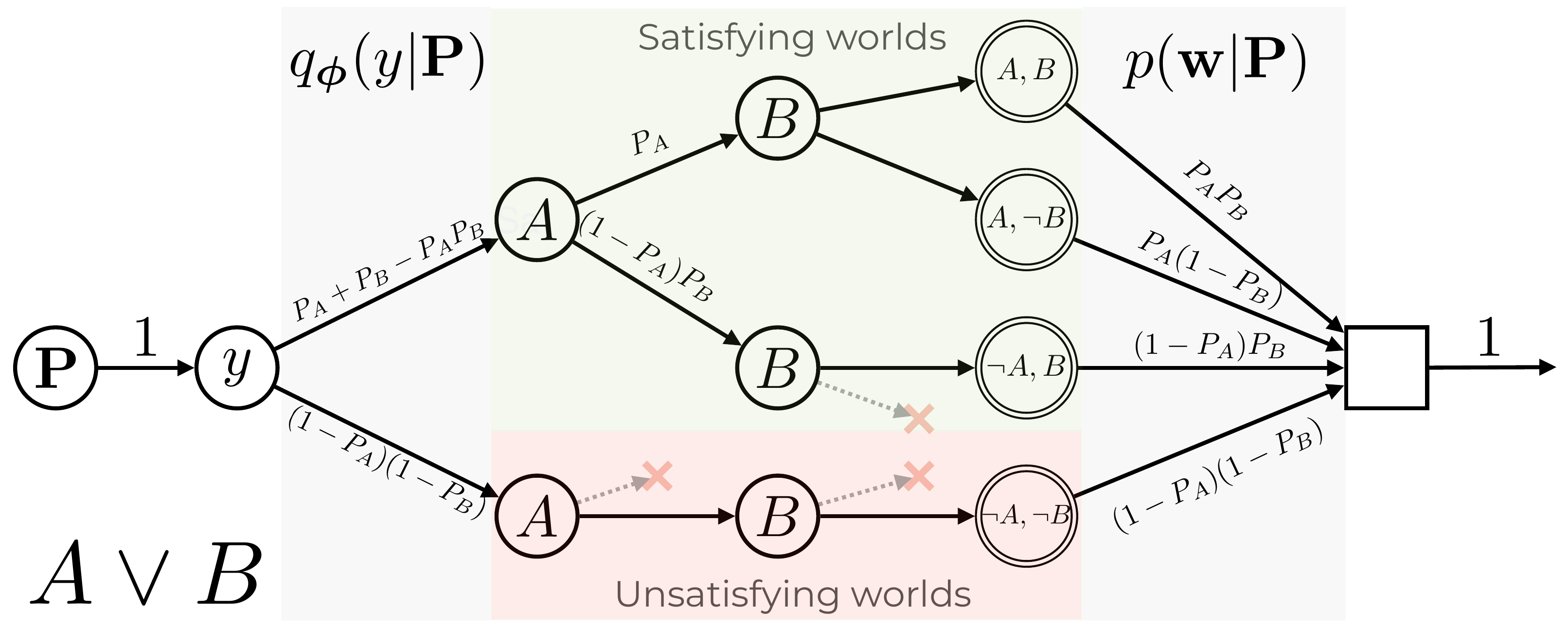}
	\caption[The tree flow network corresponding to weighted model counting on the formula $A\vee B$]{The tree flow network corresponding to weighted model counting on the formula $A\vee B$. Following edges upwards means setting the corresponding binary variable to true (and to false by following edges downwards). We first choose probabilities for the propositions $A$ and $B$, then choose whether we want to sample a world that satisfies the formula $A\vee B$. $y=1$ is the WMC of $A\vee B$, and equals its outgoing flow $P_A+P_B-P_AP_B$. Terminal states (with two circles) represent choices of the binary variables $A$ and $B$. These are connected to a final sink node, corresponding to the prior over worlds $p(\bw|\bP)$. The total ingoing and outgoing flow to this network is 1, as we deal with normalized probability distributions $p$ and $q_\bphi$.}
	\label{fig:flow_aorb}
\end{figure*}
A-NeSI is heavily inspired by the theory of GFlowNets \citep{bengioFlowNetworkBased2021b,bengioGFlowNetFoundations2022}, and we use this theory to derive our loss function. In the current section, we discuss these connections and the potential for future research by taking inspiration from the GFlowNet literature. In this section, we will assume the reader is familiar with the notation introduced in \cite{bengioGFlowNetFoundations2022} and refer to this paper for the relevant background. 

\subsection{Tree GFlowNet representation}
\label{appendix:gflownet-tree}
The main intuition is that we can treat the inference model $q_\bphi$ in Equation \ref{eq:nim} as a `trivial' GFlowNet. We refer to Figure \ref{fig:flow_aorb} for an intuitive example. It shows what a flow network would look like for the formula $A\vee B$. We take the reward function $R(\bw, \by)=p(\bw, \by)$. We represent states $s$ by $s=(\bP, \by_{1:i}, \bw_{1:j})$, that is, the belief $\bP$, a list of some dimensions of the output instantiated with a value and a list of some dimensions of the world assigned to some value. Actions $a$ set some value to the next output or world variable, i.e., $A(s)=Y_{i+1}$ or $A(s)=W_{j+1}$. 

Note that this corresponds to a flow network that is a tree everywhere but at the sink since the state representation conditions on the whole trajectory observed so far. We demonstrate this in Figure \ref{fig:flow_aorb}. We assume there is some fixed ordering on the different variables in the world, which we generate the value of one by one. Given this setup, Figure \ref{fig:flow_aorb} shows that the branch going up from the node $y$ corresponds to the regular weighted model count (WMC) introduced in Equation \ref{eq:wmc}.  

The GFlowNet forward distribution $P_F$ is $q_\bphi$ as defined in Equation \ref{eq:nim}. The backward distribution $P_B$ is $p(\bw, \by|\bP)$ at the sink node, which chooses a terminal node. Then, since we have a tree, this determines the complete trajectory from the terminal node to the source node. Thus, at all other states, the backward distribution is deterministic.

\subsection{Lattice GFlowNet representation}
\begin{figure*}
	\includegraphics[width=\textwidth]{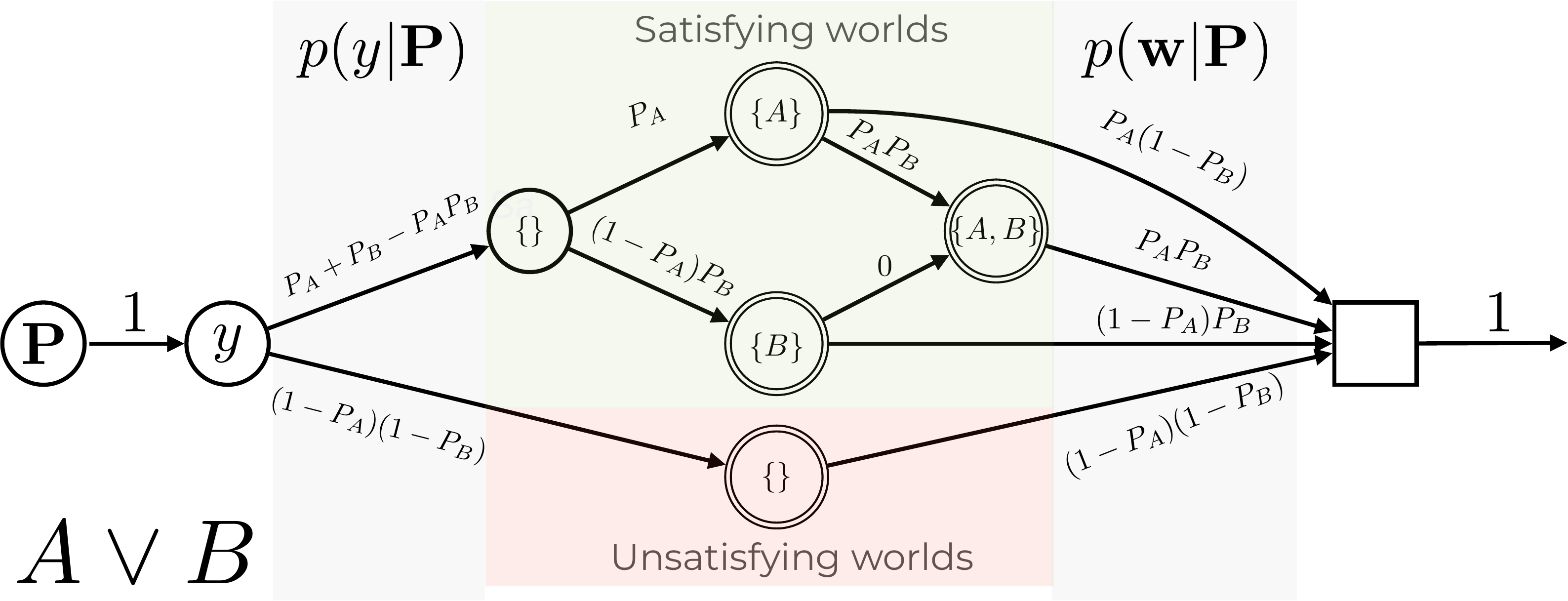}
	\caption[The lattice flow network corresponding to weighted model counting on the formula $A\vee B$]{The lattice flow network corresponding to weighted model counting on the formula $A\vee B$. In this representation, nodes represent sets of included propositions. Terminal states represent sets of random variables such that $A\vee B$ is true given $y=1$, or false otherwise.}
	\label{fig:lattice_flow}
\end{figure*}

Our setup of the generative process assumes we are generating each variable in the world in some order. This is fine for some problems like MNISTAdd, where we can see the generative process as `reading left to right'. For other problems, such as Sudoku, the order in which we would like to generate the symbols is less obvious. Would we generate block by block? Row by row? Column by column? Or is the assumption that it needs to be generated in some fixed order flawed by itself?

In this section, we consider a second GFlowNet representation for the inference model that represents states using sets instead of lists. We again refer to Figure \ref{fig:lattice_flow} for the resulting flow network of this representation for $A\vee B$. We represent states using $s=(\bP, \{y_i\}_{i\in I_Y}, \{w_i\}_{i\in I_W})$, where $I_Y\subseteq \{1, ..., \Ydim\}$ and $I_W\subseteq \{1, ..., \Wdim\}$ denote the set of variables for which a value is chosen. The possible actions from some state correspond to $A(s)=\bigcup_{i\not\in I_W} W_i$ (and analogous for when $\by$ is not yet completely generated). For each variable in $W$ for which we do not have the value yet, we add its possible values to the action space. 

With this representation, the resulting flow network is no longer a tree but a DAG, as the order in which we generate the different variables is now different for every trajectory. What do we gain from this? When we are primarily dealing with categorical variables, the two gains are 1) we no longer need to impose an ordering on the generative process, and 2) it might be easier to implement parameter sharing in the neural network that predicts the forward distributions, as we only need a single set encoder that can be reused throughout the generative process. 

However, the main gain of the set-based approach is when worlds are all (or mostly) binary random variables. We illustrate this in Figure \ref{fig:lattice_flow}. Assume $W=\{0, 1\}^\Wdim$. Then we can have the following state and action representations: $s=(\bP, \by, I_W)$, where $I_W\subseteq \{1, ..., \Wdim\}$ and $A(s) = \{1, ..., \Wdim\}\setminus I_W$. The intuition is that $I_W$ contains the set of all binary random variables that are set to 1 (i.e., true), and $\{1, ..., \Wdim\}\setminus I_W$ is the set of variables set to 0 (i.e., false). The resulting flow network represents a \emph{partial order} over the set of all subsets of $\{1, ..., \Wdim\}$, which is a \emph{lattice}, hence the name of this representation.

With this representation, we can significantly reduce the size and computation of the flow network required to express the WMC problem. As an example, compare Figures \ref{fig:flow_aorb} and \ref{fig:lattice_flow}, which both represent the WMC of the formula $A\vee B$. We no longer need two nodes in the branch $y=0$ to represent that we generate $A$ and $B$ to be false, as the initial empty set $\{\}$ already implies they are. This will save us two nodes. Similarly, we can immediately stop generating at $\{A\}$ and $\{B\}$, and no longer need to generate the other variable as false, which also saves a computation step.

While this is theoretically appealing, the three main downsides are 1) $P_B$ is no longer trivial to compute; 2) we have to handle the fact that we no longer have a tree, meaning there is no longer a unique optimal $P_F$ and $P_B$; and 3) parallelization becomes much trickier. We leave exploring this direction in practice for future work. 

\section{Analyzing the Joint Matching Loss}
\label{appendix:loss}
This section discusses the loss function we use to train the joint variant in Equation \ref{eq:joint-match}. We recommend interested readers first read Appendix \ref{appendix:gflownet-tree}. Throughout this section, we will refer to $p:=p(\bw, \by|\bP)$ (Equation \ref{eq:forward}) and $q:=q_\bphi(\bw, \by|\bP)$ (Equation \ref{eq:reverse}). We again refer to \cite{bengioGFlowNetFoundations2022} for notational background.

\subsection{Trajectory Balance}
We derive our loss function from the recently introduced Trajectory Balance loss for GFlowNets, which is proven to approximate the true Markovian flow when minimized. This means sampling from the GFlowNet allows sampling in proportion to reward $R(s_n)=p$. The Trajectory Balance loss is given by
\begin{equation}
	\mathcal{L}(\tau) = \left( \log \frac{F(s_0)\prod_{t=1}^n P_F(s_t|s_{t-1})}{R(s_n)\prod_{t=1}^n P_B(s_{t-1}|s_t)} \right)^2,
\end{equation}
where $s_0$ is the source state, in our case $\bP$, and $s_n$ is some terminal state that represents a full generation of $\by$ and $\bw$. In the tree representation of GFlowNets for inference models (see Appendix \ref{appendix:gflownet-tree}), this computation becomes quite simple:
\begin{enumerate}
	\item $F(s_0)=1$, as $R(s_n)=p$ is normalized;
	\item $\prod_{t=1}^n P_F(s_t|s_{t-1})=q$: The forward distribution corresponds to the inference model $q_\bphi(\bw, \by|\bP)$;
	\item $R(s_n)=p$, as we define the reward to be the true joint probability distribution $p(\bw, \by|\bP)$;
	\item $\prod_{t=1}^n P_B(s_{t-1}|s_t)=1$, since the backward distribution is deterministic in a tree.
\end{enumerate}
Therefore, the trajectory balance loss for (tree) inference models is 
\begin{equation}
	\mathcal{L}(\bP, \by, \bw) = \left(\log \frac{q}{p}\right)^2=\left(\log q - \log p \right)^2,
\end{equation}
i.e., the term inside the expectation of the joint matching loss in Equation \ref{eq:joint-match}. This loss function is stable because we can sum the individual probabilities in log-space. 

A second question might then be how we obtain `trajectories' $\tau=(\bP, \by, \bw)$ to minimize this loss over. The paper on trajectory balance \citep{malkinTrajectoryBalanceImproved2022} picks $\tau$ \emph{on-policy}, that is, it samples $\tau$ from the forward distribution (in our case, the inference model $q_\bphi$). We discussed when this might be favorable in our setting in Appendix \ref{appendix:background-knowledge} (Equation \ref{eq:joint-match-on-policy}). However, the joint matching loss as defined in Equation \ref{eq:joint-match} is \emph{off-policy}, as we sample from $p$ and not from $q_\bphi$.

\subsection{Relation to common divergences}
These questions open quite some design space, as was recently noted when comparing the trajectory balance loss to divergences commonly used in variational inference \citep{malkinGFlowNetsVariationalInference2023}. Redefining $P_F=q$ and $P_B=p$, the authors compare the trajectory balance loss with the KL-divergence and the reverse KL-divergence and prove that 
\begin{equation}
	\nabla_\bphi D_{KL}(q||p)=\frac{1}{2}\mathbb{E}_{\tau \sim q}[\nabla_\bphi \mathcal{L}(\tau)].
\end{equation}
That is, the \emph{on-}policy objective minimizes the \emph{reverse} KL-divergence between $p$ and $q$. We do not quite find such a result for the \emph{off-}policy version we use for the joint matching loss in Equation \ref{eq:joint-match}:
\begin{align}
	\nabla_\bphi D_{KL}(p||q)&=-\mathbb{E}_{\tau \sim p}[\nabla_\bphi \log q] \\
	\mathbb{E}_{\tau \sim p}[\nabla_\bphi \mathcal{L}(\tau)] &= -2\mathbb{E}_{\tau \sim p}[(\log p - \log q)\nabla_\bphi \log q]
\end{align}
So why do we choose to minimize the joint matching loss rather than the (forward) KL divergence directly? This is because, as is clear from the above equations, it takes into account how far the `predicted' log-probability $\log q$ currently is from $\log p$. That is, given a sample $\tau$, if $\log p < \log q$, the joint matching loss will actually \emph{decrease} $\log q$. Instead, the forward KL will increase the probability for every sample it sees, and whether this particular sample will be too likely under $q$ can only be derived through sampling many trajectories. 

Furthermore, we note that the joint matching loss is a `pseudo' f-divergence with $f(t) = t \log^2 t$ \citep{malkinGFlowNetsVariationalInference2023}. It is not a true f-divergence since $t\log^2 t$ is not convex. A related well-known f-divergence is the Hellinger distance given by
\begin{equation}
	H^2(p||q) = \frac{1}{2}\mathbb{E}_{\tau \sim p}[\left(\sqrt{p}-\sqrt{q}\right)^2].
\end{equation}
This divergence similarly takes into account the distance between $p$ and $q$ in its derivatives through squaring. However, it is much less stable than the joint matching loss since both $p$ and $q$ are computed by taking the product over many small numbers. Computing the square root over this will be much less numerically stable than taking the logarithm of each individual probability and summing.

Finally, we note that we minimize the on-policy joint matching $\mathbb{E}_{q_\bphi}[(\log p - \log q)^2]$ by taking derivatives $\mathbb{E}_{q_\bphi}[\nabla_\bphi(\log p - \log q)^2]$. This is technically not minimizing the joint matching, since it ignores the gradient coming from sampling from $q_\bphi$.

\section{Dirichlet prior}
\label{appendix:dirichlet-prior}
This section describes how we fit the Dirichlet prior $p(\bP)$ used to train the inference model. During training, we keep a dataset of the last 2500 observations of $\bP=f_\btheta(\bx)$. We have to drop observations frequently because $\btheta$ changes during training, meaning that the empirical distribution over $\bP$s changes as well. 

We perform an MLE fit on $\Wdim$ independent Dirichlet priors to get parameters $\balpha$ for each. The log-likelihood of the Dirichlet distribution cannot be found in closed form \citep{minkaEstimatingDirichletDistribution2000}. However, since its log-likelihood is convex, we run ADAM \citep{kingmaAdamMethodStochastic2017} for 50 iterations with a learning rate of 0.01 to minimize the negative log-likelihood. We refer to \cite{minkaEstimatingDirichletDistribution2000} for details on computing the log-likelihood and alternative options. Since the Dirichlet distribution accepts positive parameters, we apply the softplus function on an unconstrained parameter during training. We initialize all parameters at 0.1. 

We added L2 regularization on the parameters. This is needed because at the beginning of training, all observations $\bP = f_\btheta(\bx)$ represent uniform beliefs over digits, which will all be nearly equal. Therefore, fitting the Dirichlet on the data will give increasingly higher parameter values, as high parameter values represent low-entropy Dirichlet distributions that produce uniform beliefs. When the Dirichlet is low-entropy, the inference models learn to ignore the input belief $\bP$, as it never changes. The L2 regularization encourages low parameter values, which correspond to high-entropy Dirichlet distributions.


\section{Designing symbolic pruners}
\label{appendix:symbolic_pruner}
We next discuss four high-level approaches for designing the optional symbolic pruner:
\begin{enumerate}
	\item \textbf{Mathematically derive efficient solvers.} For simple problems, we could mathematically derive an exact solver. One example of an efficient symbolic pruner, along with a proof for exactness, is given for Multi-digit MNISTAdd in Appendix \ref{appendix:MNISTAdd-pruner}. This pruner is linear-time. However, for most problems we expect the pruner to be much more computationally expensive. 
	\item \textbf{Use SAT-solvers.} Add the sampled symbols $\by$ and $\bw_{1:i}$ to a CNF-formula, and ask an SAT-solver if there is an extension $\bw_{i+1:\Wdim}$ that satisfies the CNF-formula. SAT-solvers are a general approach that will work with every function $\symfun$, but using them comes at a cost. 
	
	The first is that we would require grounding the logical representation of the problem. Furthermore, to do SAT-solving, we have to solve a linear amount of NP-hard problems. However, competitive SAT solvers can deal with substantial problems due to years of advances in their design \citep{balyoProceedingsSATCompetition2022}, and a linear amount of NP-hard calls is a lower complexity class than \#P hard. Using SAT-solvers will be particularly attractive in problem settings where safety and verifiability are critical.
	\item \textbf{Prune with local constraints.} In many structured prediction tasks, we can use local constraints of the symbolic problem to prune paths that are guaranteed to lead to branches that can never create possible worlds.  However, local constraints do not guarantee that each non-pruned path contains a possible world, but this does not bias the inference model, as the neural network will (eventually) learn when an expansion would lead to an unsatisfiable state. 
	
	One example is the shortest path problem, where we can filter out directions that would lead outside the $N\times N$ grid, or that would create cycles. However, this just ensures we find \emph{a} path, but not that it is the shortest one! 
	\item \textbf{Learn the pruner.} Finally, we can learn the pruner, that is, we can train a neural network to learn satisfiability checking. One possible approach is to reuse the inference model trained on the belief $\bP$ that uniformly distributes mass over all worlds.  
	
	Learned pruners will be as quick as regular inference models, but are less accurate than symbolic pruners and will not guarantee that constraints are always satisfied during test-time. We leave experimenting with learning the pruner for future work.
\end{enumerate}

\section{MNISTAdd Symbolic Pruner}
\label{appendix:MNISTAdd-pruner}
In this section, we describe a symbolic pruner for the Multi-digit MNISTAdd problem, which we compute in time linear to $N$. Note that $\bw_{1:N}$ represents the first number and $\bw_{N+1:2N}$ the second. We define $n_1 = \sum_{i=1}^{N} w_{i} \cdot 10^{N-i - 1}$ and $n_2 = \sum_{i=1}^{N} w_{N+i} \cdot 10^{N-i - 1}$ for the integer representations of these numbers, and $y=\sum_{i=1}^{N+1}y_i \cdot 10^{N-i}$ for the sum label encoded by $\by$. We say that partial generation $\bw_{1:k}$ has a \emph{completion} if there is a $\bw_{k+1:2N}\in \{0, \dots, 9\}^{2N-k}$ such that $n_1+n_2=y$. 

\begin{proposition}
	For all $N\in \mathbb{N}$, $\by\in \{0, 1\} \times \{0, \dots, 9\}^{N}$ and partial generation $\bw_{1:k-1} \in \{0, \dots, 9\}^k$ with $k\in \{1, \dots, 2N\}$, the following algorithm rejects all $w_k$ for which $\bw_{1:k}$ has no completions, and accepts all $w_k$ for which there are:
	\begin{itemize}
		\item $k\leq N$: Let $l_k = \sum_{i=1}^{k+1} y_k \cdot 10^{k+1-i}$ and $p_k = \sum_{i=1}^{k} w_k\cdot 10^{k-i}$. Let $S=1$ if $k=N$ or if the $(k+1)$th to $(N+1)$th digit of $y$ are all 9, and $S=0$ otherwise. We compute two boolean conditions for all $w_k\in\{0, \dots, 9\}$: 
		\begin{equation}
			\label{eq:mnist-constraint}
		0 \leq l_k -p_k \leq 10^{k} - S
		\end{equation}
		We reject all $w_k$ for which either condition does not hold.
		\item $k>N$: Let $n_2=y - n_1$. We reject all $w_k \in \{0, \dots, 9\}$ different from $w_k= \lfloor \frac{n_2} {10^{N-k-1}}\rfloor \bmod 10$, and reject all $w_k$ if $n_2<0$ or $n_2 \geq 10^N$.
	\end{itemize}
\end{proposition}
\begin{proof}
	For $k>N$, we note that $n_2$ is fixed given $y$ and $n_1$ through linearity of summation, and we only consider $k\leq N$. We define $a_k = \sum_{i=k+2}^{N+1}y_i\cdot 10^{N+1-i}$ as the sum of the remaining digits of $y$. We note that $y = l_k\cdot 10^{N-k} + a_k$. 

%

	\textbf{Algorithm rejects $w_k$ without completions} We first show that our algorithm only rejects $w_k$ for which no completion exists. We start with the constraint $0\leq l_k - p_k$, and show that whenever this constraint is violated (i.e., $p_k > l_k$), $\bw_{1:k}$ has no completion. Consider the smallest possible completion of $\bw_{k+1:N}$: setting each to 0. Then $n_1=p_k\cdot 10^{N-k}$. First, note that 
	\begin{equation*}
		10^{N-k} > 10^{N-k}-1\geq a_k
	\end{equation*}
	Next, add $l_k\cdot 10^{N-k}$ to both sides
	\begin{equation*}
		(l_k + 1)\cdot 10^{N-k} > l_k\cdot 10^{N-k} + a_k =y
	\end{equation*}
	By assumption, $p_k$ is an integer upper bound of $l_k$ and so $p_k \geq l_k+1$. Therefore,
	\begin{equation*}
		n_1=p_k\cdot 10^{N-k} > y
	\end{equation*}

	Since $n_1$ is to be larger than $y$, $n_2$ has to be negative, which is impossible.

	Next, we show the necessity of the second constraint. Assume the constraint is unnecessary, that is, $l_k > p_k + 10^k - S$. Consider the largest possible completion $\bw_{k+1: N}$ by setting each to 9. Then 
	\begin{align*} 
		n_1 &= p_k \cdot 10^{N-k} + 10^{N-k}-1\\
		&=(p_k+1)\cdot 10^{N-k} - 1
	\end{align*}
	We take $n_2$ to be the maximum value, that is, $n_2=10^N-1$. Therefore, 
	\begin{equation*}
		n_1+n_2 = 10^N - (p_k + 1)\cdot 10^{N-k} - 2
	\end{equation*}
	We show that $n_1 + n_2 < y$. Since we again have an integer upper bound, we know $l_k \geq p_k + 10^k - S + 1$. Therefore, 
	\begin{align*}
		y &\geq (p_k + 1 +10^k - S) 10^{N-k} + a_k\\
		&\geq n_1 + n_2 + 2 - S \cdot 10^{N-k} + a_k
	\end{align*}
	There are two cases. 
	\begin{itemize}
		\item $S=0$. Then $a_k < 10^{N-k} - 1$, and so 
		\begin{align*}
			y \geq n_1 + n_2 +2 + a_k > n_1 + n_2.
		\end{align*}
		\item $S=1$. Then $a_k = 10^{N-k} - 1$, and so 
		\begin{align*}
			y \geq n_1 + n_2 + 1 > n_1 + n_2.
		\end{align*}
	\end{itemize}

	\textbf{Algorithm accepts $w_k$ with completions} Next, we show that our algorithm only accepts $w_k$ with completions. Assume Equation \ref{eq:mnist-constraint} holds, that is, $0\leq l_k - p_k \leq 10^k - S$. We first consider all possible completions of $\bw_{1:k}$. Note that $p_k\cdot 10^{N-k}\leq n_1 \leq p_k\cdot 10^{N-k}+10^{N-k}-1$ and $0\leq n_2 \leq 10^N-1$, and so
	\begin{equation*}
		p_k\cdot 10^{N-k}\leq n_1 + n_2 \leq (p_k + 1)\cdot 10^{N-k} + 10^N -2.
	\end{equation*}
	Similarly,
	\begin{equation*}
		l_k\cdot 10^{N-k} \leq y \leq (l_k + 1)\cdot 10^{N-k} - 1.
	\end{equation*}
	By assumption, $p_k \leq l_k$, so $p_k\cdot 10^{N-k}\leq l_k\cot 10^{N-k}$. For the upper bound, we again consider two cases. We use the second condition $l_k \leq 10^k +p_k - S$:
	\begin{itemize}
		\item $S=0$. Then (since there are no trailing 9s), 
		\begin{align*}
		y &\leq (l_k + 1) \cdot 10^{N-k} - 2\\
		&\leq (10^k+p_k + 1) \cdot 10^{N-k}-1\\
		&=(p_k + 1) \cdot 10^{N-k} + 10^N - 2.
		\end{align*} 
		\item $S=1$. Then with trailing 9s, 
		\begin{align*}
			y&=(l_k + 1)\cdot 10^{N-k} - 1\\
			&\leq (10^k + p_k)\cdot 10^{N-k} - 1\\ 
			&= p_k\cdot 10^{N-k} + 10^N - 1 \\
			&\leq (p_k + 1)\cdot 10^{N-k} + 10^N - 2, 
		\end{align*}
		since $10^{N-k}\geq 1$. 
	\end{itemize}
	Therefore, 
	\begin{equation*}
		p_k\cdot 10^{N-k}\leq y \leq (p_k + 1)\cdot 10^{N-k} + 10^N -2
	\end{equation*}
	and so there is a valid completion.
	
\end{proof}

\section{Details of the experiments}

\subsection{Multi-digit MNISTAdd}
\label{appendix:mnist_add}
\subsubsection{Hyperparameters}
We performed hyperparameter tuning on a held-out validation set by splitting the training data into 50.000 and 10.000 digits, and forming the training and validation sets from these digits. We progressively increased $N$ from $N=1$, $N=3$, $N=4$ to $N=8$ during tuning to get improved insights into what hyperparameters are important. The most important parameter, next to learning rate, is the weight of the L2 regularization on the Dirichlet prior's parameters which should be very high. We used ADAM \citep{kingmaAdamMethodStochastic2017} throughout. We ran each experiment 10 times to estimate average accuracy, where each run computes 100 epochs over the training dataset. We used Nvidia RTX A4000s GPUs and 24-core AMD EPYC-2 (Rome) 7402P CPUs.  

We give the final hyperparameters in Table \ref{table:hyperparameters}. We use this same set of hyperparameters for all $N$. \# of samples refers to the number of samples we used to train the inference model in Algorithm \ref{alg:train-NIM}. For simplicity, it is also the beam size for the beam search at test time. The hidden layers and width refer to MLP that computes each factor of the inference model. There is no parameter sharing. The perception model is fixed in this task to ensure performance gains are due to neurosymbolic reasoning (see \cite{manhaeveDeepProbLogNeuralProbabilistic2018}).

\begin{table}
	\centering
	\begin{tabular}{l | l || l | l}
		Parameter name & Value & Parameter name & Value\\
		\hline
		Learning rate & 0.001 & Prior learning rate & 0.01 \\
		Epochs & 100 & Amount beliefs prior & 2500 \\
		Batch size & 16 & Prior initialization & 0.1 \\
		\# of samples & 600 & Prior iterations & 50 \\
		Hidden layers & 3 & L2 on prior & 900.000 \\
		Hidden width & 800 & & \\
	\end{tabular}
	\caption{Final hyperparameters for the multi-digit MNISTAdd task.}
	\label{table:hyperparameters}
\end{table}

\subsubsection{Other methods}
\label{appendix:other_methods}
We compare with multiple neurosymbolic frameworks that previously tackled the MNISTAdd task. Several of those are probabilistic neurosymbolic methods: DeepProbLog \citep{manhaeveDeepProbLogNeuralProbabilistic2018}, DPLA* \citep{manhaeveApproximateInferenceNeural2021}, NeurASP \citep{yangNeurASPEmbracingNeural2020} and NeuPSL \citep{pryorNeuPSLNeuralProbabilistic2022}. We also compare with the fuzzy logic-based method LTN \citep{badreddineLogicTensorNetworks2022} and with Embed2Sym \citep{aspisEmbed2SymScalableNeuroSymbolic2022} and DeepStochLog \citep{wintersDeepStochLogNeuralStochastic2022}. We take results from the corresponding papers, except for DeepProbLog and NeurASP, which are from \cite{manhaeveApproximateInferenceNeural2021}, and LTN from \cite{pryorNeuPSLNeuralProbabilistic2022}\footnote[1]{We take the results of LTN from \cite{pryorNeuPSLNeuralProbabilistic2022} because \cite{badreddineLogicTensorNetworks2022} averages over the 10 best outcomes of 15 runs and overestimates its average accuracy.}. We reran Embed2Sym, averaging over 10 runs since its paper did not report standard deviations. We do not compare DPLA* with pre-training because it tackles an easier problem where part of the digits is labeled.

Embed2Sym \citep{aspisEmbed2SymScalableNeuroSymbolic2022} uses three steps to solve Multi-digit MNISTAdd: First, it trains a neural network to embed each digit and to predict the sum from these embeddings. It then clusters the embeddings and uses symbolic reasoning to assign clusters to labels. \textsc{A-NeSI} has a similar neural network architecture, but we train the prediction network on an objective that does not require data. Furthermore, we train \textsc{A-NeSI} end-to-end, unlike Embed2Sym. For Embed2Sym, we use \textbf{symbolic prediction} to refer to Embed2Sym-NS, and \textbf{neural prediction} to refer to Embed2Sym-FN, which also uses a prediction network but is only trained on the training data given and does not use the prior to sample additional data

\subsection{Visual Sudoku Puzzle Classification}
\label{appendix:visual_sudoku}
\begin{table}
	\centering
	\begin{tabular}{l | l || l | l}
		Parameter name & Value & Parameter name & Value\\
		\hline
		Perception Learning rate & 0.00055 & Prior learning rate & 0.0029 \\
		Inference learning rate & 0.003 & Amount beliefs prior & 2500 \\
		Batch size & 20 & Prior initialization & 0.02 \\
		\# of samples & 500 & Prior iterations & 18 \\
		Hidden layers & 2 & L2 on prior & 2.500.000 \\
		Hidden width & 100 & &  \\
		Epochs & 5000 & Pretraining epochs & 50
	\end{tabular}
	\caption{Final hyperparameters for the visual Sudoku puzzle classification task.}
	\label{table:hyperparameters_visudo}
\end{table}

\subsubsection{Hyperparameters and other methods}
We used the Visual Sudoku Puzzle Classification dataset from \cite{augustineVisualSudokuPuzzle2022}. This dataset offers many options: We used the simple generator strategy with 200 training puzzles (100 correct, 100 incorrect). We took a corrupt chance of 0.50, and used the dataset with 0 overlap (this means each MNIST digit can only be used once in the 200 puzzles). There are 11 splits of this dataset, independently generated. We did hyperparameter tuning on the 11th split of the $9\times 9$ dataset. We used the other 10 splits to evaluate the results, averaging results over runs of each of those. 

The final hyperparameters are reported in Table \ref{table:hyperparameters_visudo}. The 5000 epochs took on average 20 minutes for the $4\times 4$ puzzles, and 38 minutes for the $9\times 9$ puzzles on a machine with a single NVIDIA RTX A4000. The first 50 epochs we only trained the prediction model to ensure it provides reasonably accurate gradients. 

While \cite{augustineVisualSudokuPuzzle2022} used NeuPSL, we had to rerun it to get accuracy results and results on $9 \times 9$ grids. 

We implemented the exact inference methods using what can best be described as Semantic Loss \citep{xuSemanticLossFunction2018}. We encoded the rules of Sudoku described at the beginning of this section as a CNF, and used PySDD (\url{https://github.com/wannesm/PySDD}) to compile this to an SDD \citep{kisaProbabilisticSententialDecision2014}. This was almost instant for the $4\times 4$ CNF, but we were not able to compile the $9\times 9$ CNF within 4 hours, hence why we report a timeout for exact inference. To implement Semantic Loss, we modified a PyTorch implementation available at \url{https://github.com/lucadiliello/semantic-loss-pytorch} to compute in log-space for numerically stable behavior. This modified version is included in our own repository. We ran this method for 300 epochs with a learning rate of 0.001. We ran this method for fewer epochs because it is much slower than A-NeSI even on $4\times 4$ puzzles (1 hour and 16 minutes for those 300 epochs, so about 63 times as slow).

For both \textsc{A-NeSI} and exact inference, we train the perception model on \emph{correct} puzzles by maximizing the probability that $p(y=1|\bP)$. \textsc{A-NeSI} does this by maximizing $\log q_\paramP(\by=\mathbf{1}|\bP)$, while Semantic Loss uses PSDDs to exactly compute $\log p(y=1|\bP)$. For \emph{incorrect} puzzles, there is not much to be gained since we cannot assume anything about $\by$. Still, for both methods we added the loss $\-log (1-p(y=1|\bP))$ for the incorrect puzzles.

\subsection{Warcraft Visual Path Planning}
\label{appendix:path_planning}
\subsubsection{A-NeSI definition}

For the perception model, we use a single small CNN $f_\btheta$ for each of the $N\times N$ grid cells. That is, for each grid cell, we compute $\bP_{i, j}=f_\btheta(\bx_{i, j})$. The CNN has a single convolutional layer with 6 output dimensions, a $2\times 2$ maxpooling layer, a hidden layer of $24\times 84$ and a softmax output layer of $24\times 5$, with ReLU activations. The five possible output costs are $[0.8, 1.2, 5.3, 7.7, 9.2]$, and correspond to the five possible costs in the Warcraft game. 

The pretraining of the prediction model used 185.000 iterations (200 samples each) for $12\times 12$, and 370.000 iterations (20 samples) for $30\times 30$. We used fewer examples per iteration for the larger grid because Dijkstra's algorithm became a computational bottleneck.Pretraining took 23 hours for $12 \times 12$ and 44 hours for $30 \times 30$. Both used a learning rate of $2.5\cdot 10^{-4}$ and an independent fixed Dirichlet prior with $\alpha=0.005$. Standard deviations over 10 runs are reported over multiple perception model training runs on the same frozen pretrained prediction model. We trained the perception model for only 1 epoch using a learning rate of 0.0084 and a batch size of 70. 

\subsubsection{Other methods}
\label{appendix:path_planning_other_methods}
We compare to SPL \citep{ahmedSemanticProbabilisticLayers2022} and I-MLE \citep{niepertImplicitMLEBackpropagating2021}. SPL is also a probabilistic neurosymbolic method, and uses exact inference. Its setup is quite different from ours, however. Instead of using Dijkstra's algorithm, it trains a ResNet18 to predict the shortest path end-to-end, and uses symbolic constraints to ensure the output of the ResNet18 is a valid path. Furthermore, it only considers the 4 cardinal directions instead of all 8 directions. SPL only reports a single training rule in their paper.

I-MLE is more similar to our setup and also uses Dijkstra's algorithm. It uses the first five layers of a ResNet18 to predict the cell costs given the input image. One big difference to our setup is that I-MLE uses continuous costs instead of a choice out of five discrete costs. This may be easier to optimize, as it gives the model more freedom to move costs around. I-MLE is reported using the numbers from the paper, and averages over 5 runs. 

To be able to compare to another scalable baseline with the same setup, we added REINFORCE using the leave-one-out baseline (RLOO, \citep{koolBuyREINFORCESamples2019}), implemented using the PyTorch library Storchastic \citep{vankriekenStorchasticFrameworkGeneral2021}. It uses the same small CNN to predict a distribution over discrete cell costs, then takes 10 samples, and feeds those through Dijkstra's to get the shortest path. Here, we represent the shortest path as an $N\times N$ grid of zeros and ones. The reward function for RLOO is the Hamming loss between the predicted path and the ground truth path. We use a learning rate of $5\cdot 10^{-4}$ and a batch size of 70. We train for 10 epoch and report the standard deviation over 10 runs. We note that RLOO gets quite expensive for $30\times 30$ grids, as it needs 10 Dijkstra calls per training sample. 

Finally, we experimented with running \textsc{A-NeSI} and RLOO simultaneously. We ran this for 10 epochs with a learning rate of $5\cdot 10^{-4}$ and a batch size of 70. We report the standard deviation over 10 runs.



\RemoveLabels
\AddBibLabels

\bibliographystyle{abbrvnat}
\bibliography{references.bib}

\RemoveLabels
\AddSumLabels

\clearpage
\begin{samenvatting}
	In the last few years, Artificial Intelligence (AI) has reached the public consciousness through high-profile applications such as chatbots (like ChatGPT), image generators, speech synthesis and transcription. These are all due to the success of \emph{deep learning}: A class of machine learning algorithms that learn tasks from massive amounts of data. The neural network models used in deep learning involve many parameters, often in the order of billions. These vast models often fail on tasks that computers are traditionally very good at, like calculating complex arithmetic expressions, reasoning about many different pieces of information, planning and scheduling complex systems, and retrieving information from a  database. These tasks are traditionally solved using symbolic methods in AI, which are based on logic and formal reasoning. 

\emph{Neurosymbolic AI} instead aims to integrate deep learning with symbolic AI. This integration has many promises, such as decreasing the amount of data required to train the neural networks, improving the explainability and interpretability of answers given by models and verifying the correctness of trained systems. We mainly study \emph{neurosymbolic learning}, where we do not just have data but also background knowledge expressed using symbolic languages. How do we connect the symbolic and neural components to communicate this background knowledge to the neural networks? 

There are roughly two answers to this: Fuzzy and probabilistic reasoning. Fuzzy reasoning studies degrees of truth. We might say a person is \emph{very tall}, or \emph{somewhat tall}, or maybe \emph{not that tall, but also not \textbf{not} tall}. Clearly, tallness is not a binary concept. In contrast, probabilistic reasoning studies the probability that something is true or will happen. A fair coin has a 0.5 probability of landing on heads. We would, however, never say it landed on \say{somewhat heads}. What happens when we use either fuzzy (part I of this dissertation) or probabilistic (part II) approaches to neurosymbolic learning? And do these approaches really use the background knowledge how we expect them to? 

This dissertation explores these questions. Our first research question studies how the many different forms of fuzzy reasoning combine with neural network learning. We find surprising results like a connection to the Raven paradox from the philosophy of evidence, which roughly states that we confirm that \say{ravens are black} when we observe a green apple. In this study, we gave our neural network a training objective created from the background knowledge. However, we do not use the background knowledge when we deploy our neural networks after training. In our second research question, we studied how to use background knowledge in deployed neural networks. To this end, we developed a new neural network layer based on fuzzy reasoning.

The remaining two research questions study probabilistic approaches to neurosymbolic learning. Probabilistic reasoning is a more natural fit for neural networks, which we usually train to be probabilistic themselves. However, the probabilistic approaches come at a cost: They are much more expensive to compute and do not scale well to large tasks. In our third research question, we study how to connect probabilistic reasoning with neural networks by sampling. When sampling, we pick several representative examples and estimate what happens on average. By sampling, we circumvent the need to compute reasoning outcomes for all combinations of inputs. In the fourth and final research question, we specifically study how to scale probabilistic neurosymbolic learning with background knowledge to much larger problems than possible before. Our insight is to train a neural network to predict the result of probabilistic reasoning. We perform this training process with just the background knowledge: We do not need to collect data. 

The title on the cover of this dissertation is \say{\emph{Optimisation} of neurosymbolic learning systems}. How is this all related to \say{optimisation}? Optimisation in AI refers to mathematical and algorithmic methods for finding the best solution to a problem, such as finding the best neural network parameters. All our research questions are related to problems in optimisation: Within neurosymbolic learning, optimisation with popular methods like gradient descent undertake a form of reasoning. There is still ample opportunity to study how this optimisation perspective can improve our neurosymbolic learning methods. We hope this dissertation provides some of the answers needed to make practical neurosymbolic learning a reality: Where practitioners provide both data and knowledge that the neurosymbolic learning methods use as efficiently as possible to train the next generation of neural networks.
\end{samenvatting}

\RemoveLabels
\listoffigures
\listoftables
\clearpage


\end{document}